\documentclass{ucetd}

\newcommand{\thesistitle}{Principled Approximation Methods for Efficient and Scalable Deep Learning}
\newcommand{\thesisauthor}{Pedro Savarese}
\department{TOYOTA TECHNOLOGICAL INSTITUTE AT CHICAGO}
\degree{Doctor of Philosophy in Computer Science}
\date{August, 2025}

\title{\thesistitle}
\author{\thesisauthor}

\dedication{Dedication Text}
\epigraph{Epigraph Text}

\usepackage{doi}
\usepackage{xurl}
\hypersetup{bookmarksnumbered,
            unicode,
            linktoc=all,
            pdftitle={\thesistitle},
            pdfauthor={\thesisauthor},
            pdfsubject={},                                % Add subject/description
            pdfborder={0 0 0}}
\makeatletter
\let\ORG@hyper@linkstart\hyper@linkstart
\protected\def\hyper@linkstart#1#2{%
  \lowercase{\ORG@hyper@linkstart{#1}{#2}}}
\makeatother

\usepackage{needspace}
\usepackage{amsfonts}
\usepackage{amsmath}
\usepackage{amssymb}
\usepackage{mathtools}
\usepackage{amsthm}
\usepackage{cancel}
\usepackage{bbm}
\usepackage{graphicx} % Required for inserting images
\usepackage[dvipsnames, table]{xcolor}
\usepackage{algorithm}
\usepackage{algpseudocode}% http://ctan.org/pkg/algorithmicx
\usepackage{accents}
\usepackage{mleftright}
\usepackage{tensor}
\usepackage{easybmat}
\usepackage{bm}
\usepackage{nicematrix,tikz}
\usepackage{multirow,bigdelim}
\usepackage{algorithmicx}
\usepackage[authoryear]{natbib}
\usepackage{multirow}
\usepackage{nicefrac}       % compact symbols for 1/2, etc.
\usepackage{microtype}      % microtypography
\usepackage{url}            % simple URL typesetting
\usepackage{booktabs}       % professional-quality tables
\usepackage{multirow}
\usepackage{arydshln}
\usepackage[font={small}]{caption}
\usepackage{subcaption}
\usepackage{url}
\usepackage{mathtools}
\usepackage{graphicx}
\usepackage{wrapfig}

\usepackage[toc,page]{appendix}
\usepackage[section]{placeins}

\usepackage{tikz}
\usetikzlibrary{spy,calc}

\def\Vhrulefill{\leavevmode\leaders\hrule height 0.7ex depth \dimexpr0.4pt-0.7ex\hfill\kern0pt}
\newcommand{\pb}[1]{\cellcolor{gray!15}{\textcolor{gray}{\textbf{#1}}}}

\newcommand{\bO}{\mathcal O}

\newcommand{\norm}[1]{\left\lVert{#1}\right\rVert}

\newcommand{\R}{\mathbb{R}}

\newcommand{\N}{\mathbb{N}}

\DeclareMathOperator\erf{erf}

\newcommand{\ie}{\textit{i.e.,} }
\newcommand{\eg}{\textit{e.g.,} }

\newtheorem{theorem}{Theorem}[chapter]
\newtheorem*{theorem*}{Theorem}
\newtheorem{lemma}{Lemma}[chapter]
\newtheorem{corollary}{Corollary}[chapter]
\newtheorem*{corollary*}{Corollary}
\newtheorem{proposition}{Proposition}[chapter]
\newtheorem{definition}{Definition}[chapter]

\newcommand{\operator}[1]{\mathop{\mathchoice
{\vcenter{\hbox{\large $#1$}}}
{\vcenter{\hbox{$#1$}}}{#1}{#1}}\displaylimits}

\newcommand{\E}{\operator{\mathbb{E}}}

\newcommand{\expec}[2]{~\smashoperator{\E_{#1}}~\mleft[ {#2} \mright]}

\renewcommand*{\i}[1]{\csname iindex\romannumeral#1\endcsname}

\renewcommand*{\j}[1]{\csname jindex\romannumeral#1\endcsname}

\newcommand{\thmref}[1]{Theorem~\ref{#1}}

\renewcommand{\eqref}[1]{(\ref{#1})}

\newcommand{\currw}{w_t}
\newcommand{\nextw}{w_{t+1}}

\newcommand{\w}{w}

\newcommand{\fs}{f_s}

\newcommand{\prevg}{g_{t-1}}
\newcommand{\currg}{g_t}
\newcommand{\currgi}{g_{t,i}}

\newcommand{\prevm}{m_{t-1}}
\newcommand{\currm}{m_t}

\newcommand{\prevv}{v_{t-1}}
\newcommand{\currv}{v_t}

\newcommand{\currvi}{v_{t,i}}

\newcommand{\preve}{\eta_{t-1}}
\newcommand{\curre}{\eta_t}
\newcommand{\currei}{\eta_{t,i}}
\newcommand{\prevei}{\eta_{t-1,i}}

\newcommand{\curra}{\alpha_t}

\newcommand{\normed}[1]{\left\lVert {#1} \right\rVert}

\newcommand{\btwo}{\beta_2}
\newcommand{\bone}{\beta_1}
\newcommand{\btwot}{\beta_{2,t}}
\newcommand{\bonet}{\beta_{1,t}}

\newcommand{\smooth}{L}
\newcommand{\lowinf}{L_\infty}
\newcommand{\lowinft}{L_{t, \infty}}
\newcommand{\highinf}{H_\infty}
\newcommand{\highinft}{H_{t,\infty}}
\newcommand{\gradb}{G_2}

\newcommand{\supp}{\mathcal S}

\newcommand{\algrand}{\substack{s \sim \dist^T \\ t \sim \mathcal P(t|s)}}

\newcommand{\dist}{\mathcal D}

\renewcommand{\eqref}[1]{(\ref{#1})}
\newcommand{\fst}{f_{s_t}}
\newcommand{\gradbtwo}{G_2}

\begin{document}

\maketitle
\abstract
{
Recent progress in deep learning has been driven by increasingly larger models. However, their computational and energy demands have grown proportionally, creating significant barriers to their deployment and to a wider adoption of deep learning technologies. This thesis investigates principled approximation methods for improving the efficiency of deep learning systems, with a particular focus on settings that involve discrete constraints and non-differentiability.

We study three main approaches toward improved efficiency: architecture design, model compression, and optimization. For model compression, we propose novel approximations for pruning and quantization that frame the underlying discrete problem as continuous and differentiable, enabling gradient-based training of compression schemes alongside the model's parameters. These approximations allow for fine-grained sparsity and precision configurations, leading to highly compact models without significant fine-tuning. In the context of architecture design, we design an algorithm for neural architecture search that leverages parameter sharing across layers to efficiently explore implicitly recurrent architectures. Finally, we study adaptive optimization, revisiting theoretical properties of widely used methods and proposing an adaptive optimizer that allows for quick hyperparameter tuning.

Our contributions center on tackling computationally hard problems via scalable and principled approximations. Experimental results on image classification, language modeling, and generative modeling tasks show that the proposed methods provide significant improvements in terms of training and inference efficiency while maintaining, or even improving, the model's performance.
}
\clearpage
\begin{center}
    \thispagestyle{empty}
    \vspace*{\fill}
    To my parents.
    \vspace*{\fill}
\end{center}
\clearpage
\acknowledgments
{
I am extremely grateful to my advisors, Michael Maire and David McAllester, whose mentorship was invaluable throughout this journey of learning how to do research. Michael taught me how to aim for precise and controlled research, carefully accounting for the many potential confounding factors that arise when studying deep learning, and to favor simple, general, and principled ideas. I am also thankful for the intellectual freedom I was given, which enabled me to pursue a variety of problems and to explore both theoretical and empirical research directions. I'm also grateful for all the discussions with David, from which I learned enormously through his unique perspectives on machine learning.

I would also like to thank Greg Shakhnarovich for his advice and thoughtful conversations, which influenced how I approach research, not only from a technical perspective but also how I view research as a central activity in a scholar's life. I am also thankful to Yanjing Li, Suriya Gunasekar, Nati Srebro, Jason Lee, and Daniel Soudry for their mentorship on various research projects in which I was fortunate to be involved.

I'm grateful to my colleagues at TTIC and UChicago: Rachit Nimavat, Qingming Tang, Sudarshan Babu, Omar Montasser, Ankita Pasad, Xin Yuan, David Yunis, Mrinalkanti Ghosh, Shubham Toshniwal, for their camaraderie.

I would like to thank my friends: Daniel Specht, Mauricio Mota, Hugo Sant'Anna, Min Jeong Kang, and Bernardo Soares, as well as my cats Zelda and Totoro, for their patience and companionship throughout this process.

Finally, I am deeply grateful to my family in Brazil, including my aunts, uncles, and cousins. Most importantly, I owe everything to my parents, whose love, patience, and unshakeable support made this PhD possible.
}

%\makededication
%\makeepigraph

\tableofcontents
\listoffigures
\listoftables

\mainmatter

%\input{sec0_abstract}
%\newpage
\chapter{Introduction}

\section{Overview}

Deep neural networks have become widely adopted in fields such as computer vision (CV) and natural language processing (NLP) due to their unprecedented performance. In CV, increasing a network's depth from tens to hundreds of layers led to major accuracy improvements in problems such as image classification and object detection, turning applications like autonomous driving and automated imaging diagnostics more reliable and closer to industry-level adoption. In NLP, the design of recurrent neural networks (RNNs) and transformers -- models specialized for sequence-to-sequence problems -- led to significant progress in language modeling and conversational agents, which are now being used to automate tasks such as customer service and content creation.

%\pnote{my plan is to discuss how some of the main `breakthroughs' in deep learning, like LSTMs and ResNets, were architectural changes that made optimization easier in some sense by diminishing the effects of the whole vanishing/exploding gradients problem. I removed this discussion from the Introduction since I think it requires either 1) prior mention of what the vanishing/exploding gradients problem is, along with arguments (possibly informal) on how things like residual connections help alleviate them; or 2) assuming that these objects are already familiar to the reader, and they do not require supporting arguments to be given at all. Alternatively, I can look for papers that discuss that in detail (how these architectural changes made optimization easier) and just cite them instead of providing arguments myself.}

% obstacles: training, non-convex
Nonetheless, deep learning faces significant obstacles in both theory and practice. Training deep networks is notoriously difficult due to the non-convex nature of the underlying optimization problem and remains poorly-understood in theory. While sophisticated optimization methods with desirable theoretical properties have been proposed and studied, empirical studies have shown that their theoretical qualities rarely translate into practice: first-order, gradient-based methods like SGD and Adam remain the backbone of neural network training.
%\pnote{might be a bit vague and too informal: the methods I'm referring to are 2nd (or above) order, or that employ variance reduction, or that use theoretically better momentum, or that estimate some 2nd order information (like Hessian eigenvalues using the power method). I definitively don't want to get that technical this early in the document, though.}

% obstacles: overparameterization, resource costs
Moreover, state-of-the-art models are becoming increasingly larger and more resource-intensive, posing additional challenges due to higher computational costs. Training such models typically requires numerous GPUs/TPUs, extensive time, and significant energy consumption, raising concerns about the technology's applicability and accessibility. Addressing these challenges is fundamental to democratize the field and broaden its use in real-world applications. Possible solutions include advances in hardware design and production, use of more efficient network architectures, and faster training methods.

% model compression: an approach to reduce resource requirements
One approach to improve the efficiency of deep networks aims to reduce the size of overparameterized models in order to decrease their resource requirements when deployed. Recently, compression techniques such as pruning and quantization have proved successful in producing competitive models that can be deployed on edge devices, allowing for a broader range of applications to benefit from advances in the field.
%\pnote{here I'm assuming that the reader knows what pruning is.}

% discrete problems arise in model compression, and are bad for GD
Most model compression techniques \eg pruning and quantization, introduce a discrete optimization sub-problem into the network's training pipeline. These emerge from the need to make binary decisions, such as whether to remove a parameter, that can significantly impact the model's performance and efficiency. Addressing such discrete sub-problems correctly can substantially reduce a model's size and computational requirements while preserving, or even even enhancing, its performance on the underlying task. However, optimization problems of discrete nature are generally hard to solve optimally and cannot be tackled with gradient-based methods: discrete decisions, such as whether to prune or quantize a parameter, are not only non-differentiable but discontinuous.

% discrete problems are hard, especially for large models
Moreover, since the number of possible configurations increases exponentially with the dimensionality of the discrete variables, the computational complexity of the optimization problem generally follows the same intractable rate. This is particularly prohibitive in fine-grained compression of large models: for example, the complexity of parameter-wise pruning is exponential in the total number of parameters, which can be in the order of billions for large overparameterized models.
%\pnote{here I'm trying to emphasise that, the larger a model is, the harder compressing it will be (in theory, of course). I'll try to rewrite this later since I don't think using terms like `fine-grained compression' and `parameter-wise pruning' is a good idea this early in the document}

Among other challenges that arise when introducing discrete variables to a neural network's training, one that is particularly aggravating arises from the dynamic nature of the optimization problem. More specifically, the optimal configuration for the discrete variables depends on the model's weights, which are constantly changing during training.  This poses a major obstacle to using model compression techniques to reduce training costs: compressing a partially-trained network yields configurations that are typically suboptimal for the weights' final values, hence the configuration of the discrete variables needs to be revisited after further training.
%\pnote{probably shorten this at some point. I think explaining this without math is tricky and ended up sounding confusing.}

Designing efficient methods to approximately solve discrete optimization problems has been an active research topic across various fields, resulting in an extensive literature of methods that have been analyzed and evaluated thoroughly. Many classical approaches, like greedy and linear approximations, are general and flexible enough that they can be readily applied to network compression. However, these general methods have been largely surpassed by newer ones tailored specifically for deep learning, which can frequently achieve high compression ratios while preserving the model's performance.

Although state-of-the-art neural network compression methods are often still costly, they can be \emph{unreasonably} effective at producing small models via aggressive pruning and quantization, in some cases removing up to 99\% of a trained network's weights while preserving its accuracy. Coupled with SGD's inability to train small models from scratch, these observations raise key questions about the role of overparameterization in deep learning. Is the main role of overparameterization facilitating training with SGD? Are there optimization methods that do not require such aid and can successfully train small networks from scratch?
%\pnote{my main goal here is show how compression methods being \emph{too good} in deep learning can be suggestive that SGD might not be as good as we believe and the way we train networks could actually be very inefficient. I do have to check if there were any recent, significant advances in the literature on that, though. I might also add a sentence or two on the LTH since that's where these questions started to arise in the community.}

%In the discussion above, we have discussed how compression, overparameterization, and optimization in deep learning are related in more ways than one: from impacting deep learning's accessibility and sustainability to raising fundamental questions about neural network training. The size of severely overparameterized state-of-the-art models limits which organizations are able to train and deploy them. Moreover, compression methods aim to tackle this issue, but introduce additional optimization challenges due to their discrete nature. Finally, recent observations from the model compression literature -- that aggressive compression of large, fully-trained models can produce unreasonably small networks that SGD cannot train from scratch -- suggest that alternative optimization methods might alleviate the need for overparameterization. %\pnote{possibly repetitive and too provoking}
\section{Thesis Goal}

This thesis aims to explore methods to improve neural network training and inference efficiency through model compression and novel network training approaches. In particular, we study and design methods that have access to the network's structure \ie its computational graph, differing from most approaches in the literature. For model compression, we discuss sparsification/pruning, quantization, and parameter re-use, focusing on designing approximations that allow the compression scheme to be trained jointly with the model's parameters via gradient-based methods. For network training, we study alternatives to SGD and widely-used adaptive methods like Adam, specifically optimizers that leverage the model's architecture to generate better parameter updates.

Formally, we consider the problem of optimizing network parameters $\theta \in \Theta$ to minimize a loss $\mathcal L : \Theta \to \R$. In particular, we assume that the model can be naturally expressed as a composition of intermediate functions %\pnote{`naturally' is very vague here, probably should change that or at least describe what I mean}
\begin{equation}
    h^{(l)} = f_l(h^{(1)}, \dots, h^{(l-1)}, \theta) \,,
    \label{eq-one}
\end{equation}
which we call \emph{layers} or \emph{nodes}. In this case, $h^{(L)}$ denotes the model's output, where $L$ is the model's number of layers/nodes, or its depth, and $f_l$ is the function computed by node $l$ \eg convolution, batch normalization.

Distinct optimization obstacles can arise in \eqref{eq-one}: for example, $f_l$ might be non-differentiable w.r.t. $\theta$, or discontinuous w.r.t. $h^{(l-1)}$, and so on. Each setting explored in this thesis contains an optimization obstacle whose root lies in \eqref{eq-one} and which is treated in a unique fashion. Therefore, the content is split into sections, each discussing a different setting and a particular technique used to address its optimization challenges.
\section{Organization and Contributions}

The thesis is organized as follows:

\begin{itemize}
    \item \textbf{Chapter 2 (Preliminaries)} provides a more detailed discussion on the goal of the thesis, delving deeper into the connection between overparameterization, model compression, and optimization. It also contains a guideline on how readers can modify or combine different approaches to tackle new problems. 

    Additionally, this chapter introduces key definitions and the notation adopted throughout the thesis. Finally, it reviews prior works that are related to at least two of the settings covered in each section. For example, it covers works on discrete optimization and on neural network training, which are common topics across different sections and settings.

    \item \textbf{Chapter 3 (Neural Architecture Search via Parameter Sharing)}
    % P1: task description, what its goal is, how it's framed, why naive solutions are expensive
    considers the task of searching for efficient and well-performing neural network architectures in an automated fashion, in contrast to manual design performed by researchers which includes significant trial and error.

    Neural Architecture Search (NAS) approaches architecture design by searching over a pool of candidates (search space) and selecting those that maximize some objective, such as a trade-off between accuracy and the model's size. A naive approach would require fully training at least one model for each possible architecture, leading to prohibitive computational costs as the search space grows.
    
    % P2: what kind of approximations prior works employ
    Designing cheap approximations for this problem is an active research area, and methods often rely on two key components to reduce the computational cost of NAS. First, reusing parameter configurations across different architectures, which enables different models to be trained jointly instead of independently and from scratch. Second, using partially-trained networks to perform evaluation and selection, resulting in sigificant computational cost savings.
    
    % P3: how better approximations have been designed in prior works
    Early methods adopted
    %\pnote{or 'rely'?}
    reinforcement learning and genetic algorithms to explore the search space and select a final architecture, as they can naturally tackle discrete decisions. However, these are computationally expensive since they do not employ continuous approximations for the discrete selection problem. Recent works offer significantly faster NAS by performing search with gradient-based methods -- an approach that is made possible by adopting a continuous relaxation to the discrete problem.
    
    %\pnote{need to add discussion about \eqref{eq-one}, in particular how, in this setting, the discrete architectural variables result in $f_l$ being discontinuous w.r.t. $\theta$ with a derivative of $0$ a.e.. need to think of best way to discuss this since the discrete problem can be framed in many different but equivalent ways: discrete architectural variables, or continuous architectural variables that are fed to a step function, and so on.}
    
    % P4: limitations of prior works and how they deal with them
    The majority of NAS papers considers a search space with a rigid structure, where architectures consist of a pre-defined number of operations, each with a small number of possible configurations. For example, networks consisting of convolutional layers that can be either be depth-wise separable or not, and whose kernel size can be either 5 or 3. In contrast to prior works, we present a novel search space by incorporating recurrent connections to NAS, allowing networks to become better-suited for sequential tasks.
    
    % P5: what we do different
    In particular, this chapter focuses on implicit recurrence in CNNs, allowing for CNN-RNN hybrids to emerge from the architecture search procedure and hence going beyond the standard static and feedforward search spaces. Instead of using parameter sharing for efficiency as in prior works, our method, originally published as \cite{implicitrecurrent}, frames the problem of searching for recurrent connections as one of learning parameter sharing schemes.
    
    % P6: our approximation
    We approximate the discrete selection problem by employing a linear relaxation -- a classical and flexible approach to approximate discrete optimization problems -- where a variable's discrete domain is extended by taking linear combinations of its possible values. In this setting, different architectures can be linearly combined which yields well-defined gradients w.r.t.~the architecture itself. However, our method contains significant improvements over standard linear relaxation by leveraging information of the architectures to be searched.
    
    %\pnote{need to describe the task in more detail and briefly introduce NAS and even what `parameter sharing' means here. probably want to add a discussion on the number of discrete variables their domain (few integer variables) to contrast with sparsification (many binary variables) below.}
    
  %  \pnote{rewriting and polishing the discussion on linear lax / fractional solutions / rounding, and how the method covered here differs from standard linear lax}
    
    % P7: results
    Experiments on image classification show that our method can achieve significant parameter reduction while maintaining or even improving the original model's performance on the CIFAR and ImageNet datasets. Furthermore, it can create hybrids between RNNs and CNNs with new and potentially beneficial architectural biases. A synthetic algorithmic task is designed and used to study these hybrid models, showing that recurrent connections emerge naturally during training and lead to faster adaptation and improved performance.

    \item \textbf{Chapter 4 (Sparsification)} %\pnote{task description, following tone and terminology used above. removed discussion on $\ell_0$ regularization -- I think describing the task without going over it should be possible} % P1: task description, what its goal is, how it's framed, why naive solutions are expensive
    discusses the process of removing unnecessary weights or neurons from a neural network, with the goal of creating a more efficient, sparse subnetwork whose performance is comparable to that of the original model.
    
    The discrete nature of the problem comes from the binary decision of whether a parameter or neuron should be removed from the network. The number of possible configurations grows exponentially with the network size, making exact solutions intractable to compute.
    
    % P2: what kind of approximations prior works employ
    Methods to approximate the search for sparse subnetworks often rely on heuristics to select which parameters to remove, for example removing weights with the smallest magnitude -- a technique named \emph{magnitude pruning}. Another family of techniques employ stochastic approximations, where the binary decisions are sampled from a distribution whose parameters are trained with gradient descent.
    
    %\pnote{discuss \eqref{eq-one} here as well. in this setting the discrete sparsification variables also result in $f_l$ being discontinuous w.r.t. $\theta$ with a derivative of $0$ a.e., similarly to the NAS case above. same as before: need to think of best way to discuss this since the discrete problem can be framed in many different but equivalent ways. probably should also discuss important differences from the NAS case: here we have many binary variables, where in the NAS settings we have a few integer variables (roughly 6 to 10 elements in domain)}

    % P3: how better approximations have been designed in prior works
    % to write?
    
    % P4: limitations of prior works and how they deal with them
    While heuristic methods like magnitude pruning rely on strong assumptions \eg that a weight's magnitude is a reliable proxy for its importance, they have proved successful in finding extremely sparse and efficient subnetworks. On the other hand, stochastic approaches aim to solve the original problem, but can suffer from training instabilities due to high variance caused by the additional stochasticity.
    
    % P5: what we do different
    This chapter presents a simple and novel approach for sparsification that relies on a continuous and deterministic approximation, rather than using heuristics or stochastic proxies. Inspired by continuation methods, our approach, first published as \cite{ContinuousSparsification}, creates a homotopy between the original, intractable problem, and a smooth loss that can be optimized with gradient-based methods.
    
    % P6: our approximation
    This homotopy can be seen as a path connecting a smooth loss functional -- similar to standard network training objectives -- to the computationally hard sparsification problem. It includes a temperature which is annealed during training, in such way that the objective is initially smooth and well-suited for SGD, and gets closer to the original problem as training progresses, effectively decreasing the approximation error while making optimization increasingly more difficult.
    
    The discrete variables, which represent the removal of a parameter or neuron, are relaxed and re-parameterized into real-valued variables which can be trained jointly with the network's weights. Since we adopt a fully deterministic approximation, no additional variance is introduced to the training dynamics, resulting in more stable and hence faster optimization.
    
    The proposed method, Continuous Sparsification (CS), continuously sparsifies the network throughout training, and the removal of low importance weights is performed seamlessly by gradient descent. This is in contrast to prior approaches, where weights are either removed at pre-defined intervals (magnitude pruning) or only once the network has been fully trained (stochastic approximations), allowing for potential training cost reduction given specialized hardware.
    
    % P7: results
    Empirical studies on CIFAR and ImageNet show that CS produces networks that have been aggressively sparsified without noticeable impact to their performance. Moreover, its computational cost is significantly lower than competing methods, especially in the task of finding sparse subnetworks that can be successfully trained from scratch (ticket search).

    \item \textbf{Chapter 5 (Quantization)}
    % P1: task description, what its goal is, how it's framed, why naive solutions are expensive
    studies the technique of limiting the number of bits used to represent a neural network's parameters and intermediate activations, leading to immediate reductions to memory footprints and, given specialized hardware, faster training and inference times.
    
    In neural network quantization, weights are typically stored in fixed-point representation with 8 bits or less, in contrast to standard, full-precision models with 32-bit floating point weights. The problem involves finding low precision values to represent the network, and training its quantized weights to maximize the performance under the precision constraints.
    
    The discrete nature of the problem is two-fold: the number of bits used to represent the network is a positive integer, and the quantized weights can only assume a limited number of possible values -- more specifically, $2^p$ possible configurations for a weight represented with $p$ many bits. Once again, the number of possible configurations for the discrete variables grows exponentially with the model's size, hence finding optimal solutions is generally intractable.
    
    % P2: what kind of approximations prior works employ
    Optimizing the precision values poses additional challenges compared to training the quantized weights, which the chapter discusses in detail. Because of this, most approaches assume a pre-defined and fixed precision for all parameters in a neural network, and focus on designing strategies to circumvent optimization obstacles caused by quantization constraints.
    
    % P3: how better approximations have been designed in prior works
    A simple yet effective approach to train quantized weights is to optimize real-valued weights using gradients w.r.t.~their quantized forms \ie using the straight-through estimator, a strategy that is commonly adopted in the quantization literature. Numerous works aim to further decrease the total precision required to represent a network by introducing architectural changes, such as carefully-designed full-precision pathways, that make the model more robust to quantization.
    
    % P4: limitations of prior works and how they deal with them
    However, works that optimize the precision values used to represent a network's weights remain scarce, especially in settings where parameter groups can be represented with different precision levels. Allowing for fine-grained precision assignments offers additional flexibility to possible compression schemes and can result in significant energy savings on specialized hardware. Moreover, the few works that fall in this category either rely on heuristics to assign precisions or employ costly strategies such as reinforcement learning, and are ill-suited to allocate a different precision to each weight -- the finest level of granularity which offers the highest compression rates.
    
    % P5: what we do different
    In contrast to prior works, the technique presented in this chapter is able to optimize precision values for a network's weights and intermediate activations at any level of granularity -- from all weights and activations being represented with the same number of bits, to the most fine-grained setting where distinct precisions are used to represent each weight and activation throughout the network. Additionally, the precisions are optimized jointly with the network's weights via gradient descent, offering a simple yet modular approach to quantization.
    
    % P6: our approximation
    The designed algorithm, Searching for Mixed-Precisions by Optimizing Limits for Perturbations (SMOL), was originally published as \cite{smol} and relies on a fundamental connection between quantization error and tolerance to random perturbations. In particular, the more a weight can be randomly perturbed without harming the network's performance, the more aggressively it can be quantized by representing it with a lower precision. This correspondence is used to derive an approximation for the intractable precision assignment problem which is differentiable w.r.t.~the precision variables and hence can be tackled with gradient descent.
    
    % P7: results
    Experimental evaluations on image classification, image generation, and machine translation show that SMOL can produce smaller, high-performing networks \ie that require lower total precision to be represented while offering comparable or superior accuracy. Ablation experiments highlight potential benefits of fine-grained precision assignment and show that weights from the same layer can vastly differ in terms of quantization tolerance.

    \item \textbf{Chapter 6 (Parameter Structure)} is the last chapter of this thesis and focuses on more efficient optimizers for neural networks instead of compression methods. Designing better training algorithms for deep learning is an orthogonal and equally crucial approach to address the high resource costs of state-of-the-art models. This approach focuses directly on improving the efficiency of training, in contrast to compression methods which are better suited to reduce inference costs.

    % P1: task description, what its goal is, how it's framed, why naive solutions are expensive
    The chapter focuses on adaptive methods, a family of first-order optimization methods that is widely used in deep learning, often being necessary to successfully train models like RNNs and transformers. Unlike SGD, these methods designate dynamic learning rates for each parameter based on zeroth and first-order statistics, which can result in significantly fewer iterations required for training.
    
    Although model compression techniques and adaptive methods are highly unrelated topics, they share similarities which are relevant in the context of this thesis. In particular, approximations are also a key component of adaptive optimization, however this is mostly restricted to a conceptual level: adaptive methods can be seen as approximations to higher-order optimizers.
    
    Second-order methods offer superior convergence rates compared to gradient descent and are widely adopted when solving smaller optimization problems in real-world applications. However, the use of second-order information \eg computing and inverting the Hessian matrix, is computationally unfeasible in deep learning, where models rarely have less than a million parameters. The dynamic learning rates computed by optimizers such as Adam involve the use of the gradient's second moment, and can be seen as approximating the curvature adjustments provided by second-order methods while avoiding the computational costs of computing and inverting the Hessian.

    % P2: what kind of approximations prior works employ
    Despite their wide use, adaptive methods offer worse generalization compared to SGD in some settings, and designing optimizers that are able to train complex networks while achieving comparable or better generalization has received significant attention. Moreover, adaptive optimizers typically do not take parameter structure into account when computing dynamic learning rates: their update equations do not consider each layer's number of parameters, hence ignoring inherent differences between layers in a network.

    This chapter discusses novel methods and results originally published as \cite{avagrad}, and revisits theoretical properties of adaptive methods such as Adam, identifying limitations in terms of convergence guarantees and proposing modifications to overcome them. From a refined analysis of Adam's convergence rate, we design AvaGrad, a novel adaptive method with better rates and which facilitates hyperparameter tuning. 

    An extensive empirical comparison between different adaptive methods, including tasks such as image classification with CNNs, language modeling with LSTMs, and image generation with GANs, shows that AvaGrad is capable of training models as fast as Adam while inducing generalization performance comparable to SGD's. Moreover, we assess how easily each adaptive method can be tuned via hyperparameter optimization procedures and observe that AvaGrad can significantly reduce hyperparameter tuning costs compared to Adam.

    %\item \textbf{Chapter 8 (Conclusion)}

\end{itemize}
\newpage
\chapter{Preliminaries}

\section{Notation}

\textbf{Notation.} For vectors $a = [a_1, a_2, \dots], b = [b_1, b_2, \dots] \in \R^d$ we use $\frac1{a} = [\frac1{a_1}, \frac1{a_2}, \dots]$ for element-wise division, $\sqrt a = [\sqrt{a_1}, \sqrt{a_2}, \dots]$ for element-wise square root, and $a \odot b = [a_1 b_1, a_2 b_2, \dots]$ for element-wise multiplication.
$\normed{a}$ denotes the $\ell_2$-norm, while other norms are specified whenever used.

The subscript $t$ is used to denote a vector
related to the $t$-th iteration of an
algorithm, while $i, j, k$ are used for coordinate indexing. Superscripts are always written in parentheses, and typically denote the layer index of a model's activations or weights, \eg $W^{(l)}$ represents the parameters associated with the model's $l$-th layer. When used together, $t$ precedes coordinate-indexing subscripts: $W^{(l)}_{t,i} \in \R$
denotes the $i$-th coordinate of $W^{(l)}_t \in \R^d$.

For any $k \in \N$, we let $[k] = \{1, 2, 3, \dots, k\}$. We also denote the indicator function by
\begin{equation}
    \mathbbm 1 \{ P \} \coloneqq
        \begin{cases}
            1, \text{if $P$ is true} \\
            0, \text{otherwise} \,.
        \end{cases}
\end{equation}
Additionally, we refer to a  $\R^d \to \R^d$ function $f$ as element-wise when for all $\theta \in \R^d$ and $i \in [d]$, $f(\theta)_i = f(\theta_i)$. Finally, for any $\theta \in \Theta$, we let $\mathcal D_\theta$ be any distribution such that $\mathrm{supp}(\mathcal D_\theta) \subseteq \Theta$, and $\mathcal U(S)$ denotes the uniform distribution over the set $S$.

\section{Background}

This section provides a broader context for the techniques explored throughout the thesis by discussing the core themes common to all sections in this thesis. While the specific chapters focus on distinct challenges such as neural architecture search, pruning, quantization, and adaptive optimization, they are all motivated by the central goal of improving the efficiency of deep learning models without harming their performance. Here, we elaborate on three recurring themes: efficiency, model compression, and parameter optimization.

\subsection{Efficiency}

The parameter counts of state-of-the-art models now commonly exceed hundreds of millions or even billions. While such overparameterized models often achieve superior performance, their computational and energetic costs have become major obstacles. As a result, designing techniques that can reduce these costs while preserving model performance has emerged as a key area of research.

In this context, efficiency refers to reducing resources such as memory, computation time, or power consumption required for training and inference. Approaches to improve efficiency vary significantly: from reducing the number of parameters or operations via sparsification or quantization to improving convergence rates during training through better optimizers. While these strategies differ in implementation, they share a common goal of decoupling model performance from brute-force scaling.

Throughout this thesis, we approach efficiency as a principled approximation problem. Rather than viewing performance and efficiency as opposites, we demonstrate that carefully designed approximations can yield models that are both compact and effective.

\subsection{Model Compression}

Model compression techniques aim to reduce the size and computational costs of deep networks by producing smaller, more efficient variants. The main forms of compression are sparsification, where unnecessary parameters are removed from the network, and quantization, where parameters are represented with few bits. Both of these strategies introduce discrete decisions into the overall optimization problem, such as whether to remove a weight or how many bits to allocate to a parameter.

The key challenge is that these discrete decisions are inherently non-differentiable and induce a combinatorial search space: evaluating  all possible discrete configurations is unfeasible for large-scale models. Moreover, the optimal configuration is typically strongly coupled to the network's parameters configurations which keep changing during training. Therefore, traditional approaches usually perform compression as a post-processing step, applied to a fully pre-trained network. This limits the extent to which compression schemes can contribute in terms of efficiency.

In contrast, this thesis investigates continuous approximations to these discrete problems, allowing the compression scheme itself to be trained jointly with the network parameters via gradient-based optimization. These methods are not only computationally cheaper but also enable dynamic compression throughout training.

\subsection{Parameter Optimization}

Another important approach for improving efficiency in deep learning is the development of better optimization methods. Training large neural networks is often constrained by the substantial number of iterations required by gradient-based methods, making the optimizer's convergence speed and stability critical factors in reducing training costs.

Stochastic gradient descent (SGD) and its adaptive variants such as Adam have served as the backbone for training deep networks. However, these methods can be sensitive to hyperparameter configurations and often require extensive tuning across tasks. Moreover, recent works have shown that adaptive methods might fail to convergence due to suboptimal hyperparameter choices.

In this thesis, we study strategies for improving the convergence rates of adaptive methods while reducing their hyperparameter sensitivity. The goal is to reduce the computational costs associated with training and hyperparameter tuning, thereby making optimization a more efficient component of deep learning.
\newpage
\chapter{Neural Architecture Search via Parameter Sharing}

\section{Introduction}
\label{sec-nas-intro}

\subsection{Motivation \& Strategy}
\label{sec-nas-intro-motivation}

This chapter introduces the first of several strategies in this thesis for enhancing the efficiency of deep learning systems. In particular, we study the design of more efficient neural architectures with the goal of maintaining performance while minimizing computational cost. Here, efficiency primarily refers to reductions in inference time and number of parameters -- key factors in the deployment of large models.

Traditionally, architecture design has been a manual process based on trial and error \citep{xception, howard2017mobilenets, densenet}, requiring technical expertise and often being time-consuming. More recently, Neural Architecture Search (NAS)~\citep{nas} has been proposed as a strategy to automate this process by framing it as an optimization problem: out of a large pre-defined \emph{search space} of candidate architectures, choose one that meets a target performance level while minimizing resource requirements. However, the NAS objective is inherently discrete and combinatorial: choosing from all possible architectures is computationally infeasible, as each candidate requires a complete training cycle before it can be evaluated. Therefore, most existing NAS approaches focus on designing methods to reduce the costs of architecture search -- these are discussed in Section \ref{sec-nas-intro-related}.

We first formalize this computationally hard optimization problem in Section \ref{sec-nas-method-formal}, discussing how the combinatorial nature of NAS impacts efficiency. Section \ref{sec-nas-method-reframing} re-frames the original problem in an equivalent form that is more amenable to approximations. Then, in Section \ref{sec-nas-method-approx}, we present an efficient approximation that forms the basis of our NAS method and further elaborate on how we can produce architectures with backward connections. Subsequent sections outline our experimental setup, report empirical results, and offer a detailed analysis of our method, including a discussion of its contributions and impacts.

\subsection{Related Work}
\label{sec-nas-intro-related}

\paragraph{Architecture Search.}

Early efforts to design efficient neural architectures typically relied on manual design and experimentation~\citep{squeezenet, sparsenet}. For example, \cite{xception} shows that separable convolutions can reduce computational costs significantly while matching the performance of standard convolutions. Subsequently, architectures like MobileNet \citep{howard2017mobilenets, mobilenetv2} and ShuffleNet \citep{zhang2018shufflenet} employ factorizations or groupings over convolutions to improve the inference time of CNNs.

Newer approaches aim to automate architecture design, a task referred to as Neural Architecture Search (NAS). \cite{nas} and \cite{nasnet} use reinforcement learning to search over large spaces, discovering architectures that outperform hand-designed ones. Although the produced networks offer superior efficiency, the search phase is computationally expensive and requires thousands of GPU hours. More recent works adopt continuous relaxations that yield more efficient search algorithms, such as DARTS~\citep{darts}, ENAS~\citep{enas}, and SNAS~\citep{snas}.

Both manual design approaches and automated NAS strategies share the common goal of optimizing a trade-off between accuracy and efficiency. The method proposed in this chapter advances this research topic by providing a novel approximation to search over a new and complex search space.

\paragraph{CNN-RNN Hybrids.}

Combining convolutional neural networks (CNNs) and recurrent neural networks (RNNs) has been explored at a broad structural level, where a hybrid but fixed architecture is hand-designed to operate on a sequence of spatial features. Such hybrids have proven successful in multiple settings, including scene labeling~\citep{pinheiro}, image captioning with attention~\citep{showattendtell}, and video understanding~\citep{ltrcnn}. Generic hybrid models, such as convolutional LSTMs and feedback networks, have also been proposed as versatile architectures that can be applied in a plug-and-play fashion on diverse domains. Even for non-sequential tasks like image classification, enforcing layer reuse throughout a CNN allows for parameter reduction, although at the cost of some performance degradation~\citep{cnnlayerreuse}.

Studies on residual networks~\citep{resnet1} suggest that introducing residual connections to CNNs results in intriguing parallels to RNNs. In particular, \cite{unrolled} show that residual blocks may refine hidden representations through iterative processes, similar to unrolled loops that RNNs represent. \cite{sharesnet} propose residual architectures where different layers explicitly reuse parameters, effectively inducing a rigid recurrence mechanism. Similarly, \cite{resnet_iter} explore layer reuse and iterative processes within residual networks, further establishing connections to recurrent architectures.

\paragraph{Hypernetworks and Algorithmic Architectures.}

Some lines of research explore the idea of hypernetworks~\citep{hypernetworks}, where one network (the `hypernetwork') generates or modifies the weights of another model (the `trunk network'), leading to more flexible or dynamic architectures. This decoupling of parameters from layers in the trunk network facilitates parameter sharing across modules or tasks, and opens up new possibilities for meta-learning and memory footprint reduction. Our proposed mechanism can be seen as a minimal hypernetwork in which weights are generated via a simple linear combination.

Another relevant body of work attempts to design more algorithmic or `program-like' architectures. Neural Turing Machines (NTMs)~\citep{ntm} and Differentiable Neural Computers (DNCs)~\citep{dnc} introduce ways to write and read information from an external memory, enabling models to present algorithmic behavior such as copying and sorting data. A key, shared component among these architectures is the repeated call of modules in a loop, which resemble subroutine calls in classical programming.

In contrast to previous NAS works, our method adopts a search space where architectures are allowed to have backward connections. This results in networks whose layers can be `executed' multiple times, much like a computer program that invokes the same function repeatedly. Although our approach adopts simple weight sharing across layers, it yields models with a more structured and algorithmic behavior, in a similar flavor to NTMs and DNCs.

% ResNets \citep{resnet1}, DenseNets \citep{densenet}, LSTMs \citep{lstm}, RNN variants \citep{
% stackedrnn,clockwork,densenetclique}, 

% \citep{
% vgg,
% googlenet,
% resnet1,
% densenet,
% sparsenet}
% ~\citep{
% nas,
% nas_progressive}
% \citep{resnet1,wide}
% \citep{DeepLab,PSPNet}
% \citep{cifar}
% \citep{imagenet}
% \citep{ntm}

% \citep{convLSTM}
% \citep{pinheiro}
% \citep{showattendtell}
% \citep{ltrcnn}
% \citep{resnet1}
% \citep{resnet_cortex}
% \citep{hypernetworks}
% \citep{squeezenet}
% \citep{deepexpander,condensenet,sparsenet}
% \citep{han2015deep_compress}
% \citep{npi,cai2017}
% \citep{ntm,dnc,zaremba2016,neuralalu}
% \citep{resnext}
% \citep{nas,snas,darts,enas,amoeba}
% \citep{cutout}
% \citep{darts}
% \citep{bn}
% \citep{relu}
% \citep{adam}
% \citep{ortho}
\section{Method}
\label{sec-nas-method}

\subsection{Formalizing the Architecture Search Problem}
\label{sec-nas-method-formal}

The search problem we address involves selecting an architecture from a pool of candidates that minimizes a given metric while meeting a set of constraints. We refer to this pool as the \emph{search space}. By design, it does not assume any structure or impose constraints on the architectures beyond matching the input and output shapes prescribed by the dataset.

Throughout this chapter, we assume a fixed dataset $D$ that is chosen a-priori, which may contain both samples and labels in a supervised setting or just samples in an unsupervised scenario. In either case, our formulation of the architecture search problem and our proposed method's details do not depend on the nature of the underlying task.

We let $f : \theta_f \to f(\theta_f)$ denote a neural architecture and $\Theta_f$ be its weight space, such that for any $\theta_f \in \Theta_f$, $f(\theta_f)$ represents a network (the architecture $f$ equipped with weight values $\theta_f$ that align with $f$'s structure). For any sample in $D$, $f(\theta_f)$ produces an output that aligns with the task at hand (\eg a class label or a reconstruction). We further define $\mathcal L: f(\theta_f) \to \mathcal L(f(\theta_f)) \in \R$ as the loss function that evaluates the network $f(\theta_f)$ on the dataset $D$ (omitting $D$ in the notation for brevity). The search space $\mathcal F$ is a discrete set of architectures consistent with the task, in terms of input and output dimensions.

We can then formalize the architecture search problem as follows:
\begin{definition}[Architecture Search Problem]
    \label{def-nas-formal}
    Let $\mathcal F$ be a discrete set of architectures, and for each $f \in \mathcal F$, let $\Theta_f$ be its weight space. For any $\theta_f \in \Theta_f$, $\mathcal L(f(\theta_f)) \in \R$ denotes the loss incurred by $f(\theta_f)$ on the fixed dataset $D$. Then, the \emph{architecture search problem} is:
    \begin{equation}
        \min_{f \in \mathcal F}
        ~
        \min_{\theta_f \in \Theta_f}
        \quad
        \mathcal L(f(\theta_f))
        \,.
    \end{equation}
\end{definition}

Two constrained variants of the problem above are (1) minimizing $\mathcal L(f(\theta_f))$ subject to a constraint on the network cost $C(f(\theta_f)) \leq \tau$, and (2) minimizing $C(f(\theta_f))$ (to find the most efficient network) subject to $\mathcal L(f(\theta_f)) \leq \tau$. Although these formulations are natural ways to include cost constraints in architecture search, most works focus on the unconstrained problem, typically assuming that all architectures in $\mathcal F$ are sufficiently efficient so that explicit cost constraints become redundant.

\paragraph{Search Spaces.}

The general architecture search problem, as framed above, leaves limited scope for practical approximations or efficiency improvements unless we specify additional structure for the search space. Leveraging similarities or redundancies among potential architectures allows us to reduce the combinatorial cost that is inherent in the original problem.

In most NAS methods, $\mathcal F$ follows a rigid, feedforward structure: the architectures have a fixed number of blocks $L \in \N$, and each block is chosen from a set $\mathcal H = (h^{(1)}, \dots, h^{(k)})$ of candidate functionals (or `block types').

Here, each block type $h \in \mathcal H$ is typically a short sequence of common layers in deep learning (\eg convolution + batch normalization + non-linear activation). In this setting, the instantiation of distinct architectures comes from specifying which block type from $\mathcal H$ occupies each position in the network.

\begin{definition}[Functional Search Space]
    \label{def-nas-space1}
    Given $L,k \in \N$, let $\mathcal H = (h^{(j)})_{j=1}^k$ be a set of functionals associated with a weight space $\mathcal W$. Suppose any $h,h' \in \mathcal H$ can be composed when equipped with weights $W,W' \in \mathcal W$. We define the \emph{functional search space} $\mathcal F_F$ as
    \begin{equation}
        \mathcal F_F = \left\{ {(W^{(l)})}_{l=1}^L \mapsto h^{(\alpha_L)}(W^{(L)}) \circ \dots \circ h^{(\alpha_1)}(W^{(1)}) ~\Big|~ \alpha \in [k]^L \right\} \,,
    \end{equation}
    where the weight space of each architecture $f \in \mathcal F_F$ is $\Theta_f = \mathcal W^L$.
\end{definition}

As a concrete example, consider:
\begin{equation}
    \mathcal H = 
    \Big\{
        W \mapsto \text{Conv}_{3 \times 3}(W), 
        W \mapsto \text{Conv}_{5 \times 5}(W), 
        W \mapsto \text{Id}
     \Big\} \,,
\end{equation}
where Id is the identity mapping. For a specified depth $L \in \N$, there are $3^L$ architectures in the induced search space $\mathcal F_F$. These include, for instance, the identity mapping at all $L$ blocks, an all $3 \times 3$ convolution stack, or an alternating sequence of $3 \times 3$ and $5 \times 5$ convolutions.

We propose a different search space that supports backward connections and layer reuse. Instead of defining $L$ unique blocks, we allow architectures to compose a smaller set of $k \in \N$ layer configurations in various sequences, resulting in $k$ total weight tensors $W^{(1)}, \dots, W^{(k)}$. While the network depth may still be $L$, any of the $k$ layer configurations can appear multiple times in the network. In a sense, our search space allows for the architecture's topology to be learned.

\begin{definition}[Topological Search Space]
    \label{def-nas-space2}
    Given $L,k \in \N$, let $(h^{(l)})_{l=1}^L$ be a sequence of functionals associated with a weight space $\mathcal W$. Suppose any $h^{(l)}, h^{(l-1)}$ can be composed when equipped with weights $W,W' \in \mathcal W$. We define the \emph{topological search space} $\mathcal F_T$ as
    \begin{equation}
        \mathcal F_T = \left\{ {(W^{(l)})}_{l=1}^k \mapsto h^{(L)}(W^{(\alpha_L)}) \circ \dots \circ h^{(1)}(W^{(\alpha_1)}) ~|~ \alpha \in [k]^L \right\} \,,
    \end{equation}
    where the weight space of each architecture $f \in \mathcal F_T$ is $\Theta_f = \mathcal W^k$.
\end{definition}

This `topological' search space incorporates one aspect of programmatic modularity, allowing the architecture to reuse a layer configuration at arbitrary depths. Performing architecture search over our search space amounts to effectively optimizing `loop lengths' (\ie how many times a configuration is repeated), and the content of these loop-like constructs (layer configurations). This enables architectures with a more `program-like' structure to be produced, which already brings notable gains compared to feedforward baselines.

Though recurrent neural networks (RNNs) do incorporate a loop-like structure by design, their loop length is \emph{fixed} rather than learned, leading to potential mismatches with the underlying task. In contrast, our approach allows coexistence of loops and feed-forward layers in the same architecture. For example, a search over 50-layer architectures might produce a two-layer loop that repeats five times between layers 10 and 20, a three-layer loop that repeats four times from layers 30 to 42, and otherwise assign independent weights to the remaining layers. By integrating feed-forward and recurrent architectures under the same space, our method forms a natural \emph{hybrid}, allowing for richer topologies that can more flexibly adapt to the underlying task.

\subsection{Re-framing the Architecture Search Problem}
\label{sec-nas-method-reframing}

The functional and topological search spaces presented in Definitions \ref{def-nas-space1} and \ref{def-nas-space2} have a useful property: each architecture is uniquely determined by a configuration $\alpha \in [k]^L$. In the functional space $\mathcal F_F$, $\alpha$ induces an architecture by specifying which block type $h \in \mathcal H$ appears in each of the $L$ positions. In the topological space $\mathcal F_T$, $\alpha$ instead dictates which of the $k$ weight tensors is used at each step in the forward pass. Formally, we denote by $f_\alpha$ the architecture corresponding to a configuration $\alpha \in [k]^L$.

Another important observation is that, once $L$ and $k$ are fixed, all architectures in these search spaces share the same weight space $\Theta$. In other words, for any pair of architectures $f, f' \in \mathcal F$, we have $\Theta_f = \Theta_{f'} = \Theta$. This follows from the fact that each block type (or layer module) is defined on a uniform parameter domain, even for functions like the identity mapping.

By leveraging these two observations, we can recast the original architecture search problem (Definition \ref{def-nas-formal}) in an equivalent, but more structured, form. Instead of optimizing over $f \in \mathcal F$ and $\theta_f \in \Theta_f$, we optimize over integer architectural identifiers $\alpha \in [k]^L$ and a single parameter configuration $\theta \in \Theta$:
\begin{definition}[Re-framed Architecture Search Problem]
    \label{def-nas-re1}
    Let $\mathcal F$ be either the functional or topological search space from Definitions \ref{def-nas-space1} and \ref{def-nas-space2}. Given fixed $L,k \in \N$, define for each $\alpha \in [k]^L$ the unique architecture $f_\alpha = \mathcal F(\alpha) \in \mathcal F$. For any $\theta \in \Theta$, $\mathcal L(f(\theta)) \in \R$ denotes the loss incurred by $f(\theta)$ on the fixed dataset $D$. Then, the architecture search problem can be rewritten as
    \begin{equation}
        \min_{\alpha \in [k]^L}
        ~
        \min_{\theta \in \Theta}
        \quad
        \mathcal L( \mathcal F(\alpha)(\theta))
        \,,
    \end{equation}
    which is equivalent to the problem in Definition \ref{def-nas-formal}.
\end{definition}

Although restating the search space in terms of `architectural variables' $\alpha \in [k]^L$ and a shared weight tensor $\theta \in \Theta$ makes the structure more explicit, it remains a discrete (integer) optimization problem and is as computationally hard as the original. However, this formulation is naturally suited to approximation techniques for integer programs. In the next section, we leverage these ideas to relax the discrete domain and obtain an efficient approximation.

\subsection{Approximating the Architecture Search Problem}
\label{sec-nas-method-approx}

In the previous section, we re-framed the architecture search problem by introducing an architectural configuration $\alpha \in [k]^L$ that induces the network topology. While this recasting makes the task more structured, it remains an integer optimization problem and hence intractable. Now, we approximate the problem by relaxing the discrete domain of $\alpha$, thus rendering the overall optimization problem continuous and amenable to gradient-based methods.

Recall from Definition \ref{def-nas-space2} that in the topological search space $\mathcal F_T$, each block in the network implements a function $h(W^{(\alpha_l)})$, where $\alpha_l \in [k]$ dictates which of the $k$ weight tensors to use. We now define a relaxed version of $\mathcal F_T$, denoted $\overline{\mathcal F}_T$, in where we replace the discrete weight selection with a linear combination of all weights.

\begin{definition}[Relaxed Topological Search Space]
    \label{def-nas-space2lax}
    Given $L,k \in \N$, let $(h^{(l)})_{l=1}^L$ be a sequence of functionals associated with a weight space $\mathcal W$. Suppose any $h^{(l)}, h^{(l-1)}$ can be composed when equipped with weights $W,W' \in \mathcal W$. We define the \emph{relaxed topological search space} $\overline{\mathcal F}_T$ as
    \begin{equation}
        \overline{\mathcal F}_T = \left\{ {(T^{(i)})}_{i=1}^k \mapsto h^{(L)} \left(\sum_{i=1}^k \alpha^{(L)}_i T^{(i)} \right) \circ \dots \circ h^{(1)} \left(\sum_{i=1}^k \alpha^{(1)}_i T^{(i)} \right) ~\Big|~ \alpha \in (\R^k)^L \right\} \,,
    \end{equation}
    where $\alpha = (\alpha^{(l)})_{l=1}^L$ and the weight space of each architecture $f \in \overline{\mathcal F}_T$ is $\Theta_f = \mathcal W^k$.
\end{definition}

Here, each $\alpha^{(l)}$ is a real-valued vector that combines the $k$ weight templates $(T^{(i)})_{i=1}^k$ in a linear fashion. As before, each configuration $\alpha \in (\R^k)^L$ uniquely identifies an architecture $f_\alpha = \overline{\mathcal F}_T(\alpha) \in \overline{\mathcal F}_T$. In this setting, we denote the shared weight templates as $T$ rather than $W$ to emphasize their distinct role: they are not directly used as \emph{effective weights} by the layers, but rather serve as components that are linearly combined -- via the mixing coefficients $\alpha$ -- to generate the effective weights $W$.

Note that the original topological search space $\mathcal F_T$ is recovered as a special case if each $\alpha^{(l)}$ is constrained to belong to the standard $k$-dimensional basis $(e^{(j)})_{j=1}^k \subset \R^k$, where $e^{(j)}$ is the all-zeros vector except for a 1 in position $j$. In this setting, the linear combination acts as a hard selection since $\sum_{i=1}^k e^{(j)}_i T^{(i)} = T^{(j)}$, and hence $\overline{\mathcal F}_T(\alpha) \in \mathcal F_T$.

More formally, denoting the standard $k$-dimensional basis by
\begin{equation}
    \mathbb B^k = (e^{(j)})_{j=1}^k \subset \R^k \,,
\end{equation}
we have that
\begin{equation}
    \left\{ \overline{\mathcal F}_T(\alpha) ~\big|~ \alpha \in (\mathbb B^k)^L \subset (\R^k)^L \right\} = \mathcal F_T \,,
\end{equation}
and thus $\mathcal F_T \subset \overline{\mathcal F}_T$.

Therefore, we can use $\overline{\mathcal F}_T$ to express the original search problem over $\mathcal F_T$ in an equivalent way -- but still with \emph{discrete} variables $\alpha$:
\begin{equation}
    \min_{ \substack{\alpha \in (\mathbb B^k)^L \\ \theta \in \Theta}}
    \quad
    \mathcal L( \overline{\mathcal F}_T(\alpha)(\theta))
    \,.
    \label{eq-nas-re2}
\end{equation}

This version simply restates Definition \ref{def-nas-re1} using the relaxed topological search space $\overline{\mathcal F}_T$, which behaves exactly like $\mathcal F$ under the above domain constraints on the variables $\alpha$.

Finally, we employ our approximation by relaxing the domain of each $\alpha^{(l)}$ to $\text{span}\big( \mathbb B^k \big) = \R^k$ (or some other continuous subset, such as $\text{hull}\big( \mathbb B^k \big)$ or the $k$-simplex). Formally,

\begin{definition}[Approximate Architecture Search Problem]
    \label{def-nas-approx}
    Let $\overline{\mathcal F}_T$ be the relaxed topological search space from Definition \ref{def-nas-space2lax}. Given fixed $L,k \in \N$, define for each $\alpha \in (\R^k)^L$ the unique architecture $f_\alpha = \overline{\mathcal F}_T(\alpha) \in \overline{\mathcal F}_T$. For any $\theta \in \Theta$, $\mathcal L(f(\theta)) \in \R$ denotes the loss incurred by $f(\theta)$ on the fixed dataset $D$. Then, the \emph{approximate architecture search problem} is
    \begin{equation}
        \min_{ \substack{\alpha \in (\R^k)^L \\ \theta \in \Theta}}
        \quad
        \mathcal L( \overline{\mathcal F}_T(\alpha)(\theta))
        \,.
    \end{equation}
\end{definition}

Here, each block's effective weights become a learnable linear combination of the weight templates $\theta = (T^{(i)})_{i=1}^k$. Because $\alpha$ and $\theta$ are both continuous and real-valued, the entire system can be trained end-to-end with gradient-based methods such as SGD. This yields a continuous approximation to the original architecture search problem.

\begin{figure}[t]
\centering
\includegraphics[width=.99\linewidth]{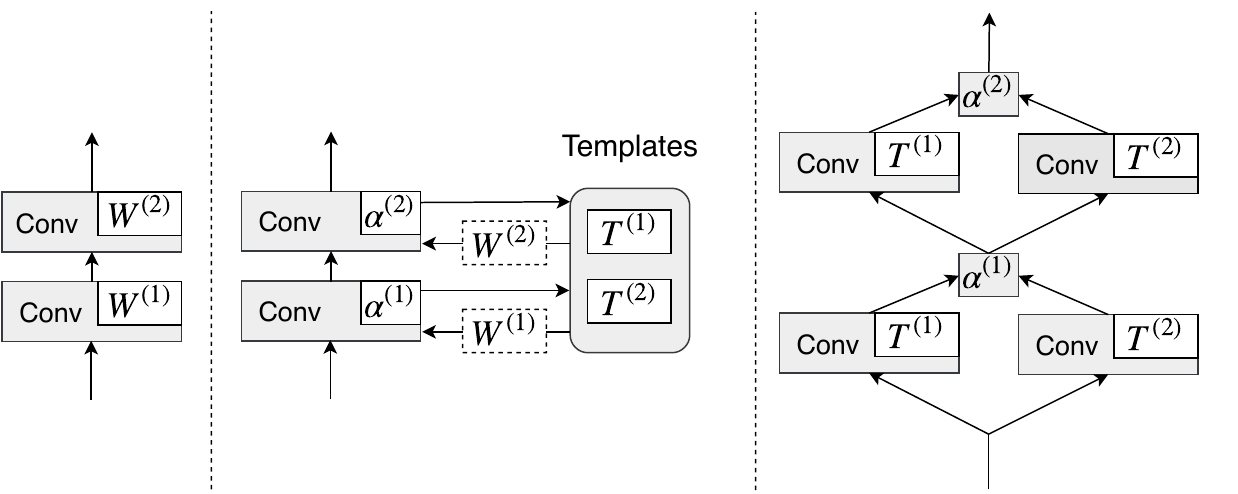}
    \caption{
          Illustration of our soft parameter sharing scheme.
          \emph{\textbf{Left:}}
             A two-layer feedforward CNN.
          \emph{\textbf{Middle:}}
             Our parameter sharing scheme applied to the same two-layer CNN..
          \emph{\textbf{Right:}}
             Equivalent view of the middle diagram, where template layers are used and combined instead of template weights.
        }
        \label{fig-nas-sharing}
\end{figure}

\subsection{The Proposed Architecture Search Method}
\label{sec-nas-method-method}

In practice, we treat $\alpha = (\alpha^{(l)})_{l=1}^L$ and $\theta = (T^{(i)})_{i=1}^k$ as joint parameters of a `super architecture' $\overline{f}: (\alpha, \theta) \mapsto \overline{\mathcal F}_T(\alpha)(\theta)$. We then train them simultaneously using standard gradient-based optimization methods (\eg Adam, SGD). By doing so, we effectively learn:
\begin{itemize}
    \item Which `mix' of weight templates each layer uses ($\alpha^{(l)}$), and
    \item The weight templates themselves ($T^{(i)}$)
\end{itemize}

Although the relaxed formulation no longer forces a discrete template selection at each layer, architectures from $\mathcal F_T$ can still emerge in practice if each $\alpha^{(l)}$ converges to a one-hot vector. Concretely, one may view each layer $l$ as still having its own weights $W^{(l)}$, but rather than being free parameters, they are given by a linear combination of weight templates:
\begin{equation}
    W^{(l)} = \sum_{i=1}^k \alpha^{(l)}_i T^{(i)} \,.
\end{equation}

Hence, our approximation can also be seen as a soft-parameter sharing scheme in which each layer has access to all weight templates $(T^{(i)})_{i=1}^k$, and the sharing mechanism is governed by the soft selection given by $(\alpha^{(l)})_{l=1}^L$.

Figure~\ref{fig-nas-sharing} (left) depicts a two-layer feedforward CNN where each convolutional layer $l$ has its own weight tensor $W^{(l)}$. In contrast, Figure~\ref{fig-nas-sharing} (middle) illustrates our soft parameter sharing scheme for the same CNN. Here, the weight templates $T^{(1)}, T^{(2)}$ are shared globally, and each layer $l$ has only a two-dimensional parameter vector ${\alpha}^{(l)}$. The effective weights $W^{(l)}$ (shown with dotted boxes to indicate they are no longer free parameters) are then generated by combining the templates $T^{(1)}, T^{(2)}$ with ${\alpha}^{(l)}$.

Additionally, this approach decouples the number of weight templates $k$ from the network's depth $L$. Specifically, an $L$-layer CNN trained with soft-sharing has $O(kC^2 K^2)$ total parameters, compared to $O(L C^2 K^2)$ without sharing. As we will see, one application of this method is to reduce the parameter count of deep models.

\paragraph{Interpretation.}

Because convolution and matrix multiplication are both linear operations, one can also view our scheme as learning template layers that are shared across a network in a soft manner. Formally, if $u^{(l)}$ denotes the output of the $l$-th layer, then
\begin{equation}
    \begin{split}
    u^{(l)}(u^{(l-1)}) = W^{(l)} u^{(l-1)} = \left( \sum_{i=1}^k \alpha^{(l)}_i T^{(i)} \right) u^{(l-1)} &= \sum_{i=1}^k \alpha^{(l)}_i \left( T^{(i)} u^{(l-1)} \right) \\
    & = \sum_{i=1}^k \alpha^{(l)}_i \tilde u^{(i)}(u^{(l-1)}) \,,
    \label{eq-nas-templatelayer}
    \end{split}
\end{equation}
where $T^{(i)} u^{(l-1)}$ can be interpreted as the output of a template layer $\tilde u^{(i)}$ with individual parameters $T^{(i)}$ and input $u^{(l-1)}$. These template layers $\tilde u^{(1)}, \dots, \tilde u^{(k)}$ act as global feature extractors, and the coefficients $\alpha^{(l)}$ specify which features to use at the $l$-th computation. Figure~\ref{fig-nas-sharing} (right) depicts this new view for a two-layer CNN: two template convolutions extract features that are combined at each step via $\alpha^{(l)}$.

This perspective reveals a direct connection between $\alpha$ and the network's topology. For instance, if $\alpha^{(l)} = \alpha^{(l')}$ for some $l<l'$, then layers $l$ and $l'$ are functionally equivalent (they implement the same mapping). In such case, we can rewire the network by sending the output of layer $l'-1$ to layer $l$, and then passing the output of layer $l$ (on its second execution) to layer $l'+1$. This results in an equivalent network and allows for the removal of layer $l'$. Unlike standard models, this `folded' network has backward connections, and each layer's output may feed into different layers based on its execution step.

\paragraph{Network Folding.}

The folding process can also be applied approximately, when two layers are only similar rather than identical, however at risk of performance degradation. Since functional similarity between layers depends on their effective weights $W^{(l)}$, we can directly compare the low-dimensional vectors $\alpha^{(l)}$ instead of induced mappings. Specifically, we can build an $L \times L$ layer similarity matrix (LSM) $S$, where $S_{l,l'}$ is the similarity between the coefficients $\alpha^{(l)}$ and $\alpha^{(l')}$.

When batch normalization~\citep{bn} (or other normalization layer) is used, a layer's output becomes invariant to rescaling of its weights. Thus, if $\alpha^{(l)} = c \cdot \alpha^{(l')}$ for $c \in \R_{>0}$, layers $l$ and $l'$ will turn out to be functionally equivalent. Even if $c \in \R_{<0}$, we can still merge the layers by negating its input at the appropriate step.

A natural metric that is invariant to rescaling is the absolute cosine similarity:
\begin{equation}
    S_{l,l'} = \frac{| \langle \alpha^{(l)}, \alpha^{(l')} \rangle |} {\lVert \alpha^{(l)} \rVert \lVert \alpha^{(l')} \rVert} \,.
\end{equation}

\begin{figure}[t]
\centering
\includegraphics[width=.99\linewidth]{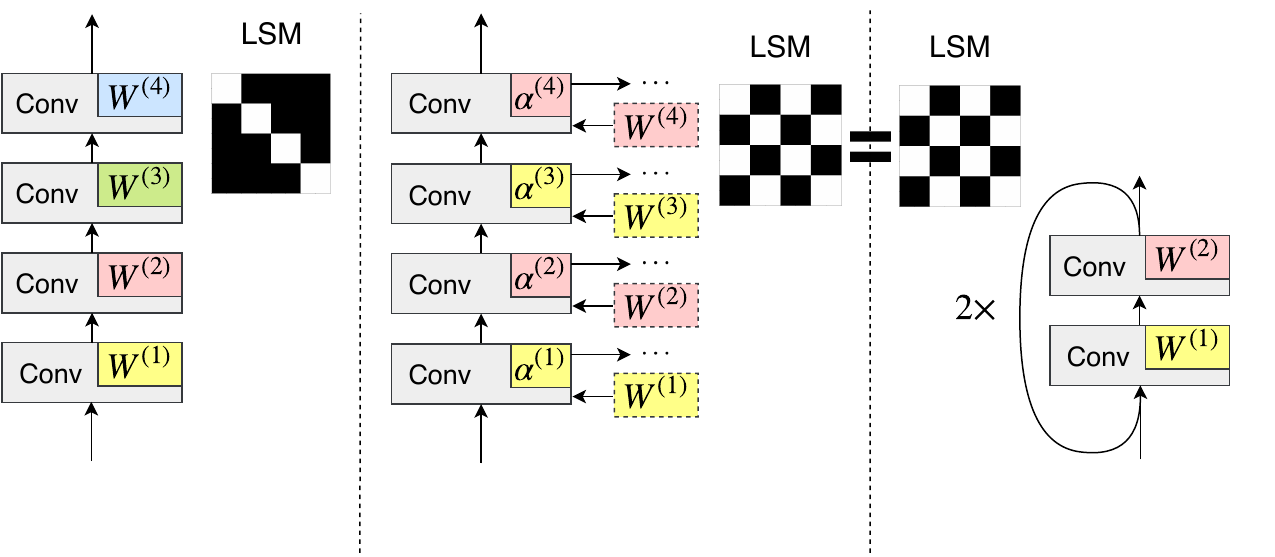}
   \caption{
      LSM when training with and without parameter sharing, and its connection to the network's topology. White and black entries correspond to maximum and minimum similarities.
      \emph{\textbf{Left:}}
         LSM of networks trained without parameter sharing.
      \emph{\textbf{Middle:}}
         LSM of networks trained with our soft parameter sharing mechanism.
      \emph{\textbf{Right:}}
         Using the LSM, we can fold the network to produce a smaller model with backward connections.
      }
      \label{fig-nas-implicitrecurrence}
\end{figure}

Figure~\ref{fig-nas-implicitrecurrence} shows examples of the LSM:
\begin{itemize}
    \item (Left) A network without parameter sharing shows little to no similarity across layer weights.
    \item (Middle) With soft parameter sharing, layers often end up functionally similar, indicated by the shared color of their coefficients and effective weights
    \item (Right) By folding these layers, we produce a more compact architecture featuring backward connections and self-loops.
\end{itemize}

\paragraph{Reparameterization.}

Although we have relaxed our problem such that each $\alpha^{(l)} \in \R^k$, there might be cases where we want produce an architecture from the original topological search space $\mathcal F_T$. Recall that, if each $\alpha^{(l)} \in \mathbb B^k$ (\ie is one of the standard basis vectors of $\R^k$), then $\overline{\mathcal F}_T(\alpha) \in \mathcal F_T$, so we can produce $f \in \mathcal F$ by simply `rounding' each $\alpha^{(l)} \in \R^k$ to a standard basis vector.

If done naively, this rounding is likely to cause significant performance degradation since learned vectors $\alpha^{(l)}$ are typically far from one-hot vectors. However, unlike most approximated integer optimization problems, our setup allows for a reparameterization that can decrease the rounding error significantly.

Specifically, given a set of trained variables $(\alpha^{(l)})_{l=1}^L, (T^{(i)})_{i=1}^k$, we can \emph{reparameterize} them while preserving each layer's effective weights and hence the behavior of the network:
\begin{proposition}
    Let $(\alpha^{(l)})_{l=1}^L \in (\R^k)^L$ and $(T^{(i)})_{i=1}^k \in \mathcal W^k$ be mixing coefficients and weight templates of a model trained with soft parameter sharing. Let $\alpha \in \R^{k \times L}$ and $T \in \R^{k \times d}$ be the matrices containing all variables. Given a $k \times k$ invertible matrix $B$, let
    \begin{equation}
        \begin{split}
            T' &= B T \\
            \alpha' &= (B^T)^{-1} \alpha \,,
        \end{split}
    \end{equation}
    then the network with coefficients $\alpha'$ and templates $T'$ is equivalent to the original one, as in the effective weights of each layer $l$ are the same.
\end{proposition}

One strategy to leverage the reparameterization above is as follows:
\begin{enumerate}
    \item After training, compute the LSM matrix and group together layers that are sufficiently similar. Let $g_l \in \N$ denote the group of the $l$-th layer and $n$ the total number of groups.
    \item Construct new coefficients $\alpha'$, setting ${\alpha'}^{(l)} = e^{(g_l)}$ for each $l \in [L]$.
    \item Compute the transposed reparameterization matrix $B^T = \alpha {\alpha'}^T (\alpha' {\alpha'}^T)^{-1}$.
    \item Create new weight templates $T' = B T$.
\end{enumerate}

Note that if $n < k$, then we only keep the first $n$ rows of $\alpha'$, making it a $n \times L$ matrix. This results in $B$ being $n \times k$, and $T'$ containing only $n<k$ templates, as expected, since only $n$ templates are needed to fully represent the network.

As a concrete example, consider, for $L=5$ and $k=4$,
\begin{equation}
    \alpha =
    \begin{bmatrix}
    0.8  & 0.1  & -1.2 & 0.2  & -0.4  \\
    -0.2 & 0.6  & 0.3  & 1.2  & -0.3 \\
    0.6  & -0.2 & -0.9 & -0.4 &  0.7 \\
    0.2  & 1.2  & -0.3 & 2.4  & -0.1 \\
    \end{bmatrix} \,,
\end{equation}
which yields the LSM (up to 2 decimals)
\begin{equation}
    LSM =
    \begin{bmatrix}
    1    & 0.05 & 1    & 0.05 & 0.15  \\
    0.05 & 1    & 0.05 & 1    & 0.40  \\
    1    & 0.05 & 1    & 0.05 & 0.15  \\
    0.05 & 1    & 0.05 & 1    & 0.40  \\
    0.15 & 0.40 & 0.15 & 0.40 & 1     \\
    \end{bmatrix} \,.
\end{equation}

Then, we can create group 1 containing layers 1 and 3, group 2 containing layers 2 and 4, and group 3 containing layer 5. Since we ended up with $3 < 4 = k$ groups, the new coefficient matrix will be $3 \times 4$:
\begin{equation}
    \alpha' =
    \begin{bmatrix}
    1 & 0 & 1 & 0 & 0 \\
    0 & 1 & 0 & 1 & 0 \\
    0 & 0 & 0 & 0 & 1
    \end{bmatrix} \,.
\end{equation}

Next, we compute the reparameterization matrix:
\begin{equation}
    B = (\alpha {\alpha'}^T (\alpha' {\alpha'}^T)^{-1})^T =
    \begin{bmatrix}
    -0.2 & 0.05 & -0.15 & -0.05 \\
    0.15 & 0.9 & -0.3 & 1.8 \\
    -0.4 & 0.3 & 0.7 & -0.1 \\
    \end{bmatrix} \,,
\end{equation}
and, finally, create the new templates $T' = BT$. This procedure preserves the network functionality and transforms its coefficients $\alpha$ such that it can be directly mapped to an architecture from the original topological search space $\mathcal F_T$.

\paragraph{Recurrence Regularizer.}

While layer similarities might emerge naturally, it may be desirable to encourage explicit recurrencies. One way to achieve this is by introducing a regularizer to the loss, pushing parameters to configurations that result in more loops. For example, we can penalize the sum of the LSM's elements:
\begin{equation}
    \mathcal L_{R}(\alpha, \theta) = \mathcal L(\alpha, \theta) - \lambda \sum_{l,l'} S_{l,l'}(\alpha) \,,
\end{equation}
where $\lambda$ controls the regularizer's strength. As $\lambda$ grows, the entries of $S$ are driven closer to 1. In the extreme, every pair of layers becomes functionally equivalent, effectively collapsing the network into a single layer and a self-loop.

\section{Experimental Setup}
\label{sec-nas-expsetup}

In this section, we describe the datasets, models, training protocols, and evaluation metrics used in our experiments. Our primary goal is to compare the accuracy and parameter efficiency of models trained with soft parameter sharing against standard architectures and established architecture search methods.

\subsection{Datasets}
\label{sec-nas-expsetup-data}

\paragraph{CIFAR-10 and CIFAR-100.}

The CIFAR-10 and CIFAR-100 datasets~\citep{cifar} consist of $32 \times 32$ color images. CIFAR-10 includes 50,000 training images and 10,000 test images, split across 10 classes, while CIFAR-100 has 100 classes but retains the same training/test splits. We apply standard channel-wise normalization using statistics computed from the training set, followed by the data augmentation procedure of~\cite{resnet1}, which involves random crops and horizontal flips.

\paragraph{ImageNet.}

For ImageNet~\citep{imagenet}, we use the ILSVRC 2012 subset containing approximately 1.28 million training images and 50,000 validation images across 1,000 different object classes. We use single $224 \times 224$ center-crop images during both training and validation.
Following~\cite{gross}, our data augmentation includes random scaling (extracting crops of varying aspect ratios/sizes and upsampling with bicubic interpolation), photometric distortions (random brightness, contrast, and saturation changes), lighting noise, and horizontal flips. We also apply channel-wise normalization using statistics from a sampled subset of the training set.

\paragraph{Synthetic Shortest-Paths Task.}

We additionally design a synthetic shortest-paths task to evaluate the extrapolation capacity of implicitly recurrent models. Each sample is a $32 \times 32$ grid containing two randomly placed query points and several obstacles, encoded in two binary channels: the first channel indicates query points, and the second indicates obstacles. Except for the two query points, each pixel has $10\%$ probability of containing an obstacle. The objective is to label all pixels that belong to at least one shortest path between the two query points. We generate five sets of 5,000 samples each, forming a curriculum that increases the maximum distance between the two query points by increments of 2, and starting from an initial distance of 2 (including diagonal movements).

\subsection{Models}
\label{sec-nas-expsetup-models}

\paragraph{Baseline CNN Architectures.}

We use Wide ResNets (WRNs)~\citep{wide} as our main baselines for CIFAR-10, CIFAR-100, and ImageNet. WRNs extend pre-activation ResNets~\citep{resnet2} by scaling the number of channels by a “widen” factor. We denote these models as WRN-L-w, where \(L\) is the number of layers and \(w\) is the widen factor. In addition, we include ResNeXts~\citep{resnext} and DenseNets~\citep{densenet} as stronger baselines.

\paragraph{Soft-Shared Wide ResNet (CIFAR).}

Directly sharing parameters across all convolutions in a WRN is not possible because layers have varying input/output channel dimensions. Instead, we group together convolutional layers that share the same input/output channel dimensions, enabling them to share a common set of weight templates. This ensures the shapes of the weight templates remain compatible with the convolution's input/output dimensions.

CIFAR WRNs begin with an initial $3\times3$ convolution followed by three stages. Each stage has an initial depth-increasing residual block (which contains two sequential $3\times3$ convolutions and a convolutional residual connection), and several depth-preserving residual blocks (each with two $3\times3$ convolutions).  

Since the first block in each stage differs in input/output depth, we exclude it from parameter sharing. The remaining convolutions in each stage have matching input/output depth, hence we group them together to share parameters. Formally, if there are $\frac{L-1}{3}$ convolutions per stage, excluding the first block leaves $\frac{L-1}{3}-3 = \frac{L-10}{3}$ convolutions in each of the three groups. We let SWRN-L-w-k denote a WRN with $L$ layers, widen factor $w$, and $k$ weight templates shared per group. If $k = \frac{L-10}{3}$, then we have as many shared weights as convolutions per group, so there is no parameter reduction relative to a standard WRN. In this case, we omit the $k$ value and simply write SWRN-L-w. We use WRN-28 as the base architecture for most CIFAR experiments, which has six convolutions per group.

\paragraph{Soft-Shared Wide ResNet (ImageNet).}

Wide ResNets on ImageNet feature four stages, each composed of bottleneck residual blocks with three convolutions: a depth-decreasing, a depth-preserving, and a depth-increasing layer. To share parameters across this structure, we divide each stage into three groups, one for each type of convolution (depth-decreasing, depth-preserving, depth-increasing), while again excluding the first block of each stage from sharing --  this yields a total of 12 groups across the four stages.

For our ImageNet experiments, we use WRN-50, which has 3, 4, 6, and 3 bottleneck blocks in each of its four stages, respectively. Our variant with parameter sharing allocates as many weight templates for each group as its number of convolutional layers, hence there is no parameter reduction. Any accuracy improvements would primarily arise from faster/better training dynamics or architectural bias.

\paragraph{NAS Baselines.}

We also compare our approach against established NAS methods that adopt a standard feed-forward search space, including NASNet, AmoebaNet, DARTS, SNAS, and ENAS. This highlights how our implicit recurrence search method fares in relation to established NAS strategies.

\paragraph{CNN for Synthetic Task.}

For the synthetic shortest-paths task, we employ a 22-layer CNN: one $1 \times 1$ convolution mapping 2 input channels to 32, followed by 20 depth-preserving $3 \times 3$ convolutions (each followed by batch normalization~\citep{bn} and ReLU~\citep{relu}), and a final $1 \times 1$ convolution mapping 32 channels to 1. All $3 \times 3$ convolutions have a skip connection. For the soft-shared variant, we assign 20 weight templates across the 20 $3 \times 3$ convolutional layers.

\subsection{Training}
\label{sec-nas-expsetup-training}

\paragraph{CIFAR-10 and CIFAR-100.}

For our CIFAR experiments, we train each model for 200 epochs using SGD with a Nesterov momentum of 0.9. The initial learning rate is set to 0.1 and is reduced by a factor of 5 at epochs 60, 120, and 160. We use a batch size of 128 on a single GPU, applying a weight decay of 0.0005 to all parameters except the soft-sharing coefficients $\alpha$. Convolutional filters and weight matrices are initialized using Kaiming initialization.

\paragraph{ImageNet.}

Following \citet{gross} and \citet{wide}, we use SGD with a Nesterov momentum of 0.9 and train for 100 epochs. The learning rate starts at 0.1 and is decayed by a factor of 10 at epochs 30, 60, and 90. We use a total batch size of 256, distributed evenly across 4 GPUs. As with CIFAR, Kaiming initialization is applied to all convolutional filters, and no weight decay is applied to the coefficients $\alpha$. For the remaining parameters, we use a weight decay of 0.0001.

\paragraph{Synthetic Shortest-Paths Task.}

We use Adam~\citep{adam} with a learning rate of 0.01, training for 50 epochs at each of the 5 curriculum phases. In each phase, the maximum distance between the two query points increases, making the task more challenging. Convolutional filters are again initialized via Kaiming's scheme, and we do not apply weight decay in these experiments.

\subsection{Evaluation}
\label{sec-nas-expsetup-eval}

\paragraph{Test Accuracy and Top-1/Top-5 Accuracy.}

For the CIFAR datasets, we measure and report the accuracy on the 10,000 test samples. On ImageNet, we use the top-1 and top-5 accuracy on the 50,000 validation images. Top-5 accuracy is the fraction of samples for which the true label appears among the top five predicted classes.

\paragraph{Number of Parameters.}

We track the total number of parameters in each model to gauge memory efficiency, highlight potential parameter savings, and compare storage costs among different approaches.

\paragraph{Architecture Search Time.}

When comparing against NAS methods, we track the total training time to perform architecture search and output a final, trained model.

\paragraph{F1 Score.}

For the synthetic shortest-paths task, we adopt the F1 score as metric instead of accuracy due to the significant class imbalance (most cells in the grid do not belong to at least one shortest path). It is defined as
\begin{equation}
    F_1(y, \hat y) = \frac{\sum_i \mathbbm 1 \{y_i = \hat y_i = 1\}}{\sum_i \mathbbm 1\{y_i=1\} + \mathbbm 1\{\hat y_i=1\}} \,,
\end{equation}
where $y, \hat y$ are the vectors containing the true and predicted label for a set of samples and $\mathbbm 1 \{\cdot\}$ is the indicator function.

\section{Results}
\label{sec-nas-results}

\subsection{Improving Network Efficiency}
\label{sec-nas-results-efficiency}

\begin{figure}[t]
   \centering
   \includegraphics[width=0.75\textwidth]{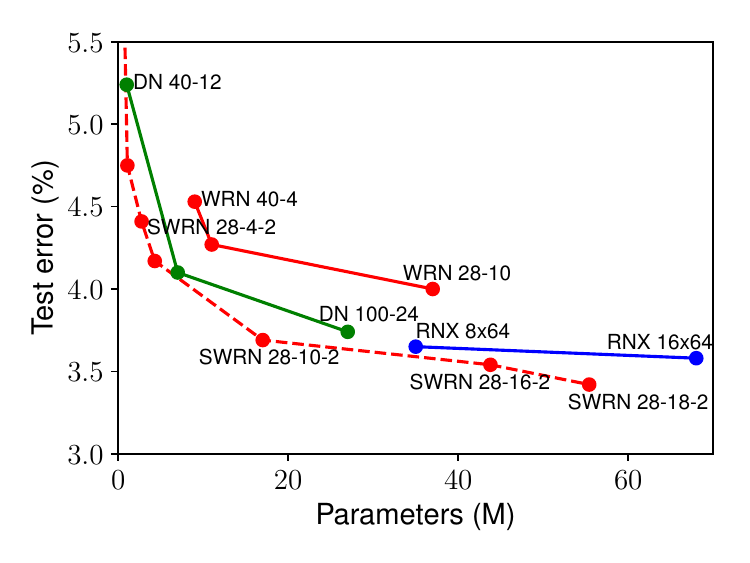}
   \caption{
      Parameter efficiency for different models on CIFAR-10.
   }
   \label{fig-nas-c10}
\end{figure}

\begin{table}[h]
\centering
   \begin{tabular}{@{}l|c|c|c|c|c@{}}
      \textbf{CIFAR-10} & Templates / & \multirow{2}{*}{Dropout} &  \multirow{2}{*}{Params (M)} & Compression       & Test          \\
      Architecture      & Group size  &                          &                              &  Ratio ($\times$) & Error ($\%$)  \\ \hline
      WRN 28-10         &             &   0.0                    &  36                          &                   & 4.00          \\
      WRN 28-10         &             &   0.3                    &  36                          &                   & 3.89          \\ \hline
      ResNeXt-29 16x64  &             &   0.0                    &  68                          &                   & 3.58          \\
      DenseNet 100-24   &             &   0.0                    &  27                          &                   & 3.74          \\
      DenseNet 190-40   &             &   0.3                    &  26                          &                   & 3.46          \\ \hline
      SWRN 28-10-1      & 1 / 6       &   0.0                    &  12                          & 3.0               & 4.01          \\
      SWRN 28-10-2      & 2 / 6       &   0.3                    &  17                          & 2.1               & 3.75          \\ 
      SWRN 28-10        & 6 / 6       &   0.0                    &  36                          & 1.0               & 3.74          \\
      SWRN 28-10        & 6 / 6       &   0.3                    &  36                          & 1.0               & 3.88          \\ \hdashline
      SWRN 28-14-2      & 2 / 6       &   0.3                    &  33                          & 2.1               & 3.69          \\ 
      SWRN 28-14        & 6 / 6       &   0.3                    &  71                          & 1.0               & 3.67          \\ \hdashline
      SWRN 28-18-2      & 2 / 6       &   0.3                    &  55                          & 2.1               & \underline{\textbf{3.43}} \\
      SWRN 28-18        & 6 / 6       &   0.3                    &  118                         & 1.0               & 3.48          \\
   \end{tabular}
      \caption{
      Performance and parameter count of different models on CIFAR-10. * indicates models trained with dropout $p=0.3$ \citep{dropout}. Best WRN/SWRN result is in bold, and best overall performance is underlined. Results are average of 5 runs.
    }
    \label{tab-nas-c10-wrn}
\end{table}

\begin{figure}[h]
   \centering
   \includegraphics[width=0.75\textwidth]{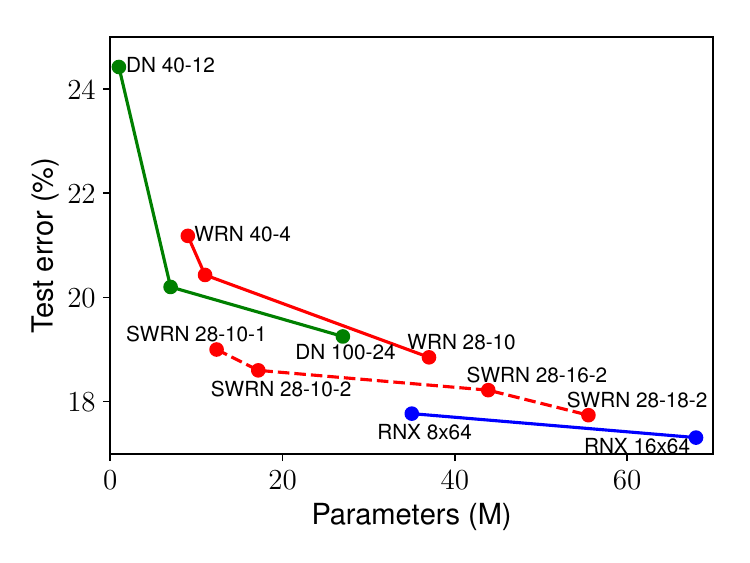}
   \caption{
      Parameter efficiency for different models on CIFAR-100.
   }
   \label{fig-nas-c100}
\end{figure}

\begin{table}[h]
\centering
   \begin{tabular}{@{}l|c|c|c|c|c@{}}
      \textbf{CIFAR-100}& Templates / & \multirow{2}{*}{Dropout} &  \multirow{2}{*}{Params (M)} & Compression       & Test          \\
      Architecture      & Group size  &                          &                              &  Ratio ($\times$) & Error ($\%$)  \\ \hline
      WRN 28-10         &             &   0.0                    &  36                          &                   & 19.25         \\
      WRN 28-10         &             &   0.3                    &  36                          &                   & 18.85         \\ \hline
      ResNeXt-29 16x64  &             &   0.0                    &  68                          &                   & 17.31         \\
      DenseNet 100-24   &             &   0.0                    &  27                          &                   & 19.25          \\
      DenseNet 190-40   &             &   0.3                    &  26                          &                   & \underline{17.18}          \\ \hline
      SWRN 28-10-1      & 1 / 6       &   0.0                    &  12                          & 3.0               & 19.73          \\
      SWRN 28-10-2      & 2 / 6       &   0.3                    &  17                          & 2.1               & 18.66          \\ 
      SWRN 28-10        & 6 / 6       &   0.0                    &  36                          & 1.0               & 18.78          \\
      SWRN 28-10        & 6 / 6       &   0.3                    &  36                          & 1.0               & 18.43          \\ \hdashline
      SWRN 28-14-2      & 2 / 6       &   0.3                    &  33                          & 2.1               & 18.37          \\ 
      SWRN 28-14        & 6 / 6       &   0.3                    &  71                          & 1.0               & 18.25          \\ \hdashline
      SWRN 28-18-2      & 2 / 6       &   0.3                    &  55                          & 2.1               & 17.75          \\
      SWRN 28-18        & 6 / 6       &   0.3                    &  118                         & 1.0               & \textbf{17.43} \\
   \end{tabular}
      \caption{
      Performance and parameter count of different models on CIFAR-100. * indicates models trained with dropout $p=0.3$ \citep{dropout}. Best WRN/SWRN result is in bold, and best overall performance is underlined. Results are average of 5 runs.
    }
    \label{tab-nas-c100-wrn}
\end{table}

\begin{table}[h]
\centering
   \begin{tabular}{@{}l|c|c|c@{}}
      \textbf{CIFAR-10} &
      Params (M) & Training Time (GPU days) & Test error (\%) \\ \hline
      NASNet-A     &      3.3    &  1800        &     2.65 \\
      NASNet-A     &     27.6    &  1800        &     2.40  \\
      AmoebaNet-B  &      2.8    &  3150        &     2.55 \\
      AmoebaNet-B  &     13.7    &  3150        &     2.31 \\
      AmoebaNet-B  &     26.7    &  3150        &     2.21 \\
      AmoebaNet-B  &     34.9    &  3150        &     2.13 \\
      DARTS        &      3.4    &     4        &     2.83 \\
      SNAS         &      2.8    &     1.5      &     2.85 \\
      ENAS         &      4.6    &     0.45      &     2.89 \\
      \hline

      \vspace{-4pt}WRN 28-10              & \multirow{2}{*}{36.4} & \multirow{2}{*}{0.4} & \multirow{2}{*}{3.08} \\
      \scriptsize{(baseline with cutout)} &                       &                      &                       \\ \hline
      SWRN 28-4-2   &    2.7    &      0.12      &     3.45  \\
      SWRN 28-6-2   &    6.1    &      0.25      &     3.00  \\
      \hdashline
      SWRN 28-10   &     36.4    &     0.4      &     2.70  \\
      SWRN 28-10-2 &     17.1    &     0.4      &     2.69 \\
      \hdashline
      SWRN 28-14   &     71.4    &     0.7      &     2.55 \\
      SWRN 28-14-2 &     33.5    &     0.7      &     2.53 \\
   \end{tabular}
   \caption{
      Performance of models found via neural architecture search (NAS) on CIFAR-10 (all trained with cutout). 
   }
   \label{tab-nas-c10-nas}
\end{table}

\paragraph{CIFAR:}

Tables~\ref{tab-nas-c10-wrn} and~\ref{tab-nas-c100-wrn} present results on CIFAR. Networks trained with our method yield superior performance in the setting with no parameter reduction: SWRN 28-10 presents $6.5\%$ and $2.5\%$ lower relative test errors compared to the base WRN 28-10 model on CIFAR-10 and CIFAR-100, suggesting that our method aids optimization since both models have the same capacity.

With fewer templates than layers, SWRN 28-10-1 (all 6 layers of each group perform the same operation), performs virtually the same as the base WRN 28-10 network, while having $\frac{1}{3}$ of its parameters and less capacity. On CIFAR-10, parameter reduction ($k=2$) is beneficial to test performance: the best performance is achieved by SWRN 28-18-2 ({3.43\% test error}), outperforming the ResNeXt-29 16x64 model~\citep{resnext}, while having fewer parameters (55M against 68M) and no bottleneck layers.

Figures~\ref{fig-nas-c10} and \ref{fig-nas-c100} show that our parameter sharing scheme uniformly improves accuracy-parameter efficiency on CIFAR-10 and CIFAR-100. Comparing the WRN model family (solid red) to our SWRN models (dotted red), we see that SWRNs are significantly more efficient than WRNs. DN and RNX denotes DenseNet and ResNeXt, respectively, and are plotted for illustration: both models employ orthogonal efficiency techniques, such as bottleneck layers.

Table~\ref{tab-nas-c10-nas} presents a comparison between our method and neural architecture search (NAS) techniques \citep{nas,snas,darts,enas,amoeba} on CIFAR-10 -- results differ from Table~\ref{tab-nas-c10-wrn} solely due to cutout~\citep{cutout}, which is commonly used in NAS literature; NAS results are quoted from their respective papers.  Our method outperforms architectures discovered by recent NAS algorithms, such as DARTS~\citep{darts}, SNAS~\citep{snas} and ENAS~\citep{enas}, while having similarly low training cost.  We achieve $2.69\%$ test error after training less than $10$ hours on a single NVIDIA GTX 1080 Ti.  This accuracy is only bested by NAS techniques which are several orders of magnitude more expensive to train. Being based on Wide ResNets, our models do, admittedly, have more parameters.

Comparison to recent NAS algorithms, such as DARTS and SNAS, is particularly interesting as our method, though motivated differently, bears some notable similarities.  Specifically, all three methods are gradient-based and use an extra set of parameters (architecture parameters in DARTS and SNAS) to perform some kind of soft selection (over operations/paths in DARTS/SNAS; over templates in our method). As Section \ref{sec-nas-results-search} will show, our learned template coefficients $\alpha$ can often be used to fold our networks into an explicitly recurrent form - a discovered CNN-RNN hybrid.

To the extent that our method can be interpreted as a form of architecture search, it might be complementary to standard NAS methods.  While NAS methods typically search over operations (\emph{e.g.} activation functions; $3\times3$ or $5\times5$ convolutions; non-separable, separable, or grouped filters; dilation; pooling), our soft parameter sharing can be seen as a search over recurrent patterns (which layer processes the output at each step).  These seem like orthogonal aspects of neural architectures, both of which may be worth examining in an expanded search space.  When using SGD to drive architecture search, these aspects take on distinct forms at the implementation level: soft parameter sharing across layers (our method) vs hard parameter sharing across networks (recent NAS methods).

\paragraph{ImageNet:}
Without any change in hyperparameters, the network trained with our method outperforms the base model and also deeper models such as DenseNets (though using more parameters), and performs close to ResNet-200, a model with four times the number of layers and a similar parameter count. See Table~\ref{tab-nas-imagenet}.

\begin{table}[t]
\centering
   \begin{tabular}{@{}l|c|c|c@{}}
   \textbf{ImageNet}  &  Params (M)   & Top-1 error ($\%$)   & Top-5 error ($\%$) \\ \hline
         WRN 50-2     &  69      & 22.00    & 6.05 \\
         DenseNet-264 &  33      & 22.15    & 6.12 \\
         ResNet-200   &  65      & 21.66    & 5.79 \\ \hline
         SWRN 50-2    &  69      & 21.74    & 5.95 \\
   \end{tabular}
   \caption{
      Performance of different models on ImageNet.
   }
   \label{tab-nas-imagenet}
\end{table}

\subsection{Architecture Search}
\label{sec-nas-results-search}

\begin{figure}[t]
    \centering
    \includegraphics[trim={0 0.3cm 0cm 0cm},clip,width=0.85\textwidth]{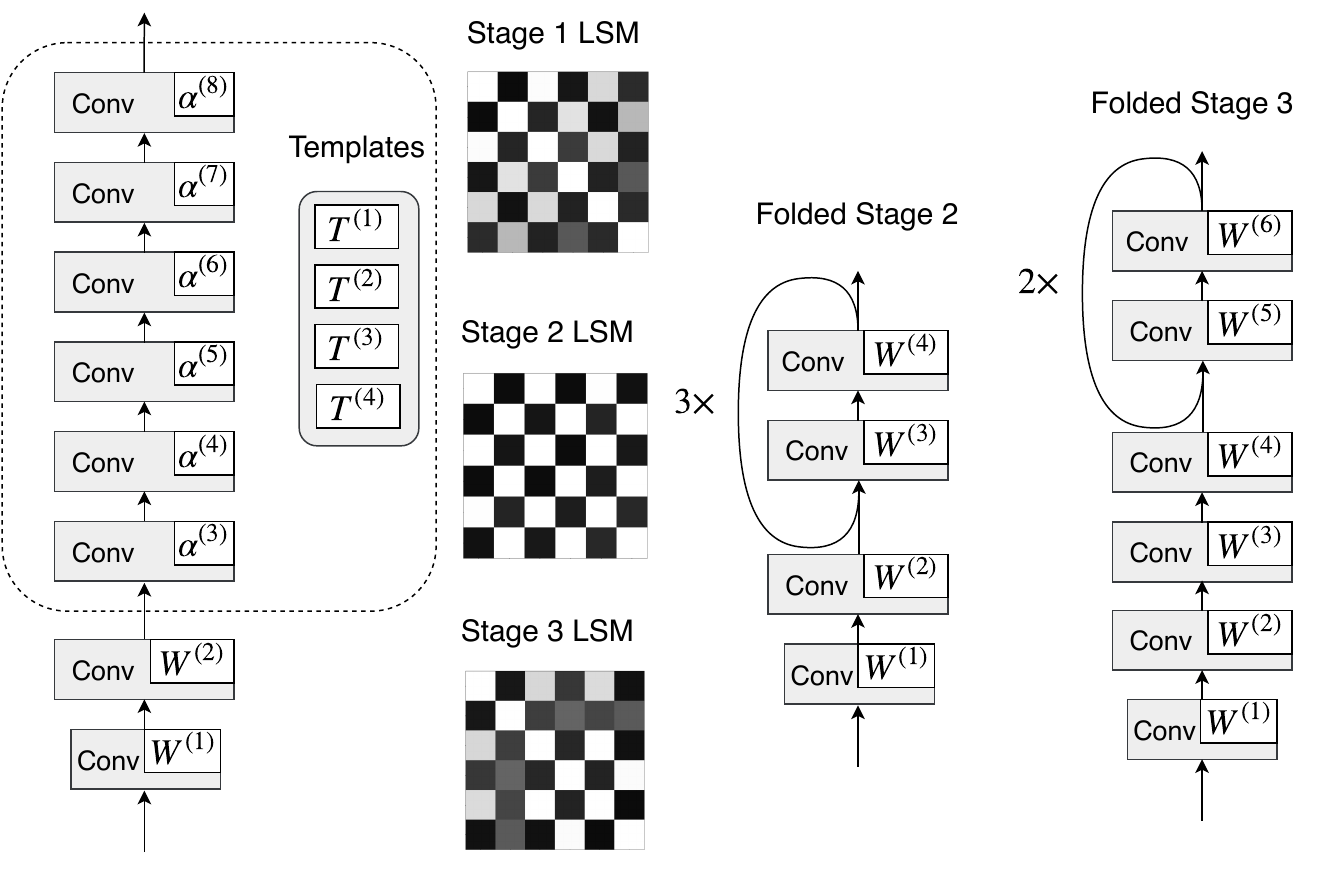}
    \caption{
      Folding a SWRN 28-10-4 trained on CIFAR-10.
      \emph{\textbf{Left:}}
         Illustration of its stages.
      \emph{\textbf{Middle:}}
         LSM for each stage after training.
      \emph{\textbf{Right:}}
         Folding stages 2 and 3.
      }
    \label{fig-nas-folding1}
\end{figure}

\begin{figure}[t!]
   \centering
   \includegraphics[width=0.75\textwidth]{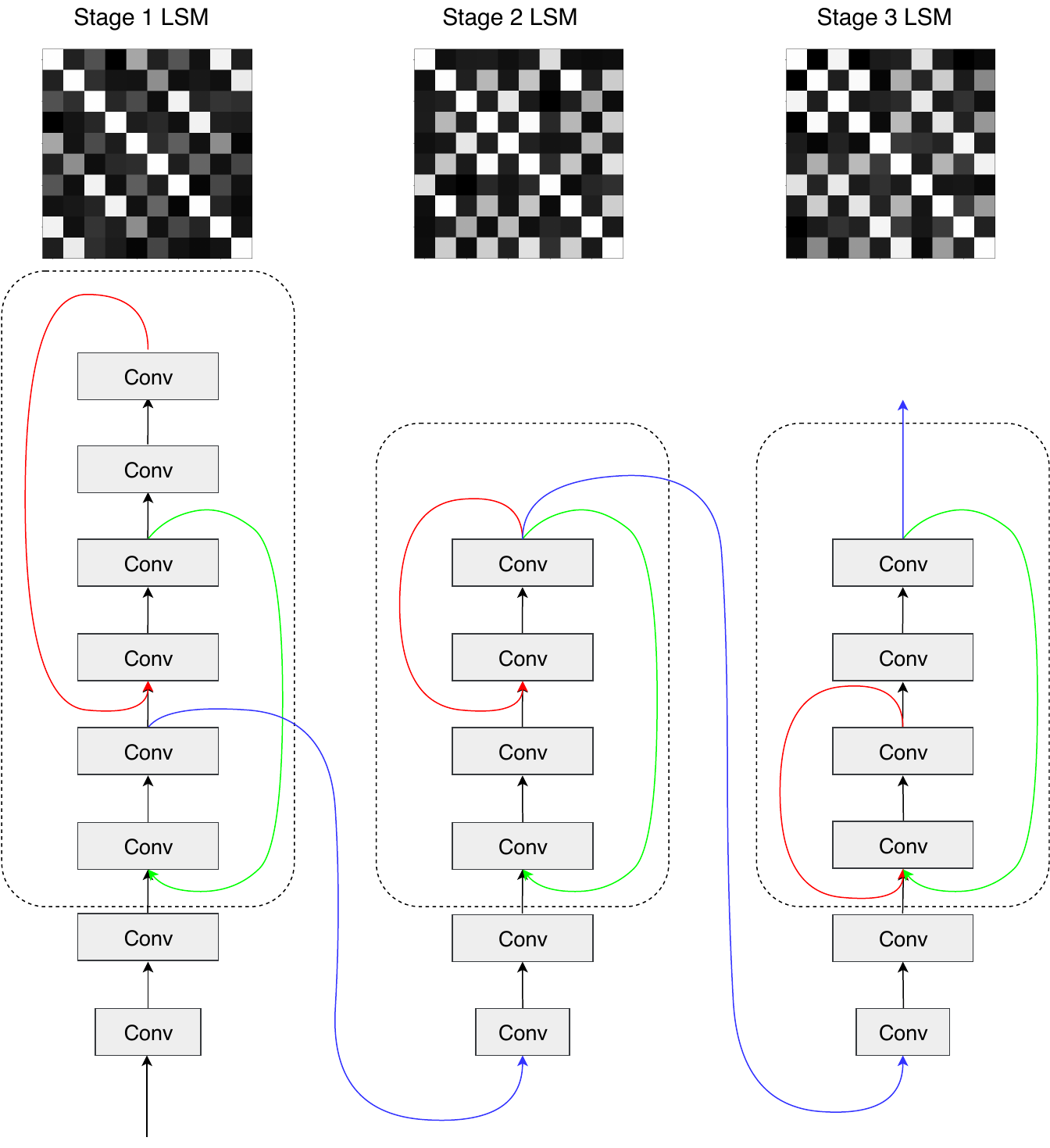}
   \caption{1
      Folding a SWRN 40-8-8 trained on CIFAR-10. Red paths are taken before green, which are taken before blue.
   }
   \label{fig-nas-folding2}
\end{figure}

\begin{figure}[t]
\begin{center}
      \begin{minipage}[t]{0.85\linewidth}
         \begin{center}
            \includegraphics[width=1.0\linewidth]{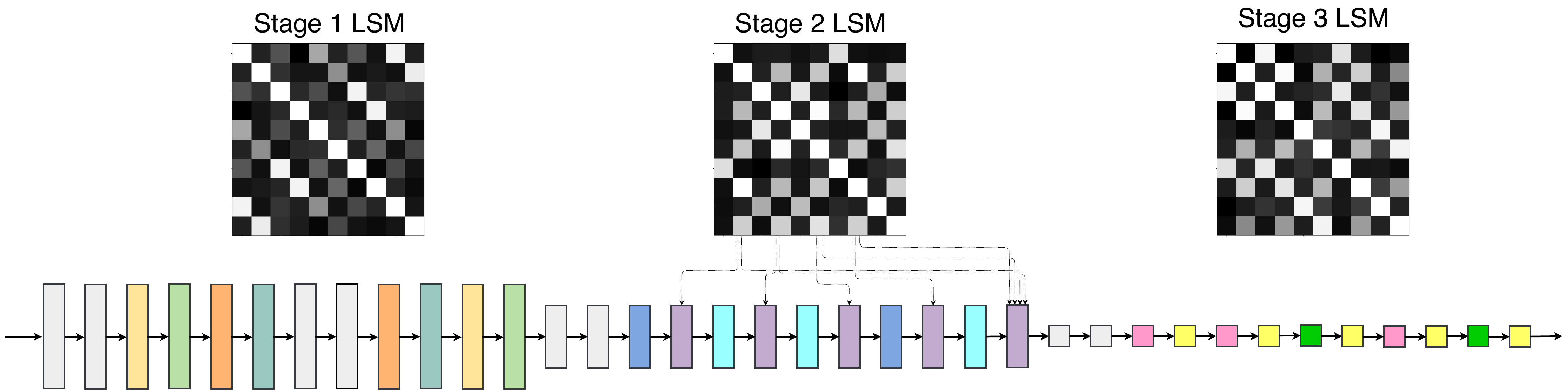}
         \end{center}
      \end{minipage}\\
      \begin{minipage}[t]{0.375\linewidth}
         \vspace{0pt}
         \begin{center}
            \includegraphics[width=1.0\linewidth]{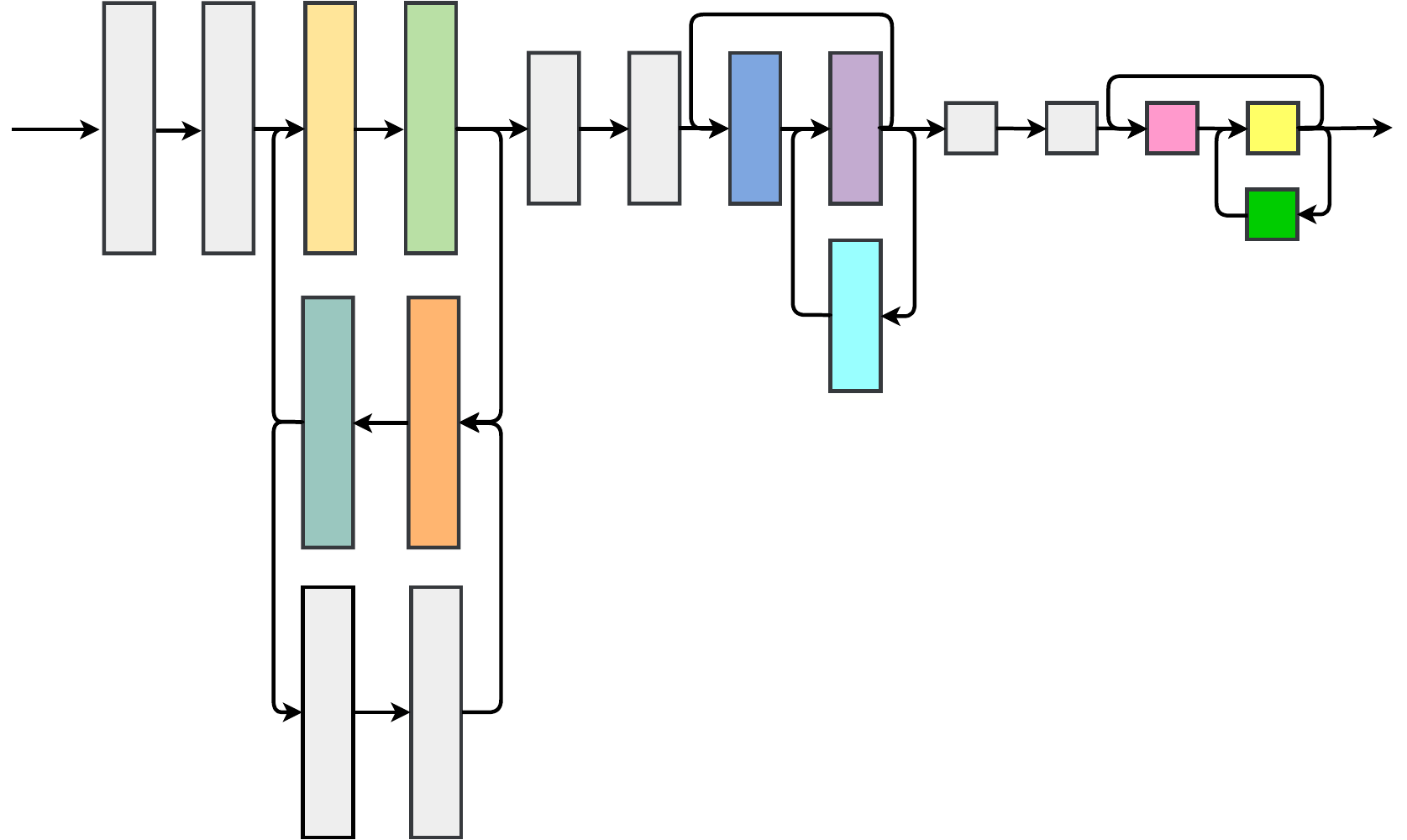}
         \end{center}
      \end{minipage}
\end{center}
    \caption{
    Folding a SWRN 40-8-8 trained on CIFAR-10.
    }
    \label{fig-nas-folding3}
\end{figure}

Results on CIFAR suggest that training networks with few parameter templates $k$ in our soft sharing scheme results in performance comparable to the base models, which have significantly more parameters.  The lower $k$ is, the larger we should expect the layer similarities to be: in the extreme case where $k=1$, all layers in a sharing scheme have similarity $1$, and can be folded into a single layer with a self-loop.

For the case $k > 1$, there is no trivial way to fold the network, as layer similarities depend on the learned coefficients.  We can inspect the model's layer similarity matrix (LSM) and see if it presents implicit recurrences: a
form of recurrence in the rows/columns of the LSM.  Surprisingly, we observe that rich structures emerge naturally in networks trained with soft parameter sharing, \emph{even without the recurrence regularizer}.

Figure~\ref{fig-nas-folding1} shows the per-stage LSM for CIFAR-trained SWRN 28-10-4. Here, the six layers of its stage-2 block can be folded into a loop of two layers, leading to an error increase of only $0.02\%$:
\begin{itemize}
    \item (Left) Illustration of the stages of a SWRN-28-10-4 (residual connections omitted for clarity).  The first two layers contain individual parameter sets, while the other six share four templates. All 3 stages of the network follow this structure.
    \item (Middle) LSM for each stage after training on CIFAR-10, with many elements close to $1$.  Hard sharing schemes can be created for pairs with large similarity by tying their coefficients (or, equivalently, their effective weights).
    \item (Right) Folding stages 2 and 3 leads to self-loops and a CNN with backward connections -- LSM for stage 2 is a repetition of 2 rows/columns, and folding decreases the number of parameters.
\end{itemize}

Figure~\ref{fig-nas-folding2} shows a SWRN 40-8-8 (8 parameter templates shared among groups of $\frac{40-10}{3} = 10$ layers) trained with soft parameter sharing on CIFAR-10.  Each stage (originally with 12 layers -- the first two do not participate in parameter sharing) can be folded to yield blocks with complex recurrences.  For clarity, we use colors to indicate the computational flow: red takes precedence over green, which in turn has precedence over blue.  Colored paths are only taken once per stage.

Although not trivial to see, recurrences in each stage's folded form are determined by row/column repetitions in the respective Layer Similarity Matrix.  For example, for stage 2 we have $S_{5,3} \approx S_{6,4} \approx 1$, meaning that layers 3, 4, 5 and 6 can be folded into layers 3 and 4 with a loop (captured by the red edge).  The same holds for $S_{7,1}$, $S_{8,2}$, $S_{9,3}$ and $S_{10,4}$, hence after the loop with layers 3 and 4, the flow returns to layer 1 and goes all the way to layer 4, which generates the stage's output.  Even though there is an approximation when folding the network (in this example, we are tying layers with similarity close to $0.8$), the impact on the test error is less than $0.3\%$.  Also note that the folded model has a total of 24 layers (20 in the stage diagrams, plus 4 which are not shown, corresponding to the first layer of the network and three $1 \times 1$ convolutions in skip-connections), instead of the original 40.

Finally, Figure~\ref{fig-nas-folding3} shows the same SWRN 40-8-8 more aggressively folded, where a lower threshold for the required coefficient similarity results in two fewer layers. Residual connections are omitted for simplicity, and layers are colored the same if their coefficients are sufficiently similar for them to be approximated as a single layer. Grey is used to depict layers with independent weights, which either didn't participate in parameter sharing (the first two layers of each stage) or that didn't present sufficient similarity to any other layer in order to be folded. Excluding the first convolution and the shortcut connections, the final folded network has a total of 18 layers, compared to 36 of the original network.

\subsection{Algorithmic Task}
\label{sec-nas-results-recurrenttask}

\begin{figure}[t]
\centering
   \includegraphics[width=0.4\linewidth]{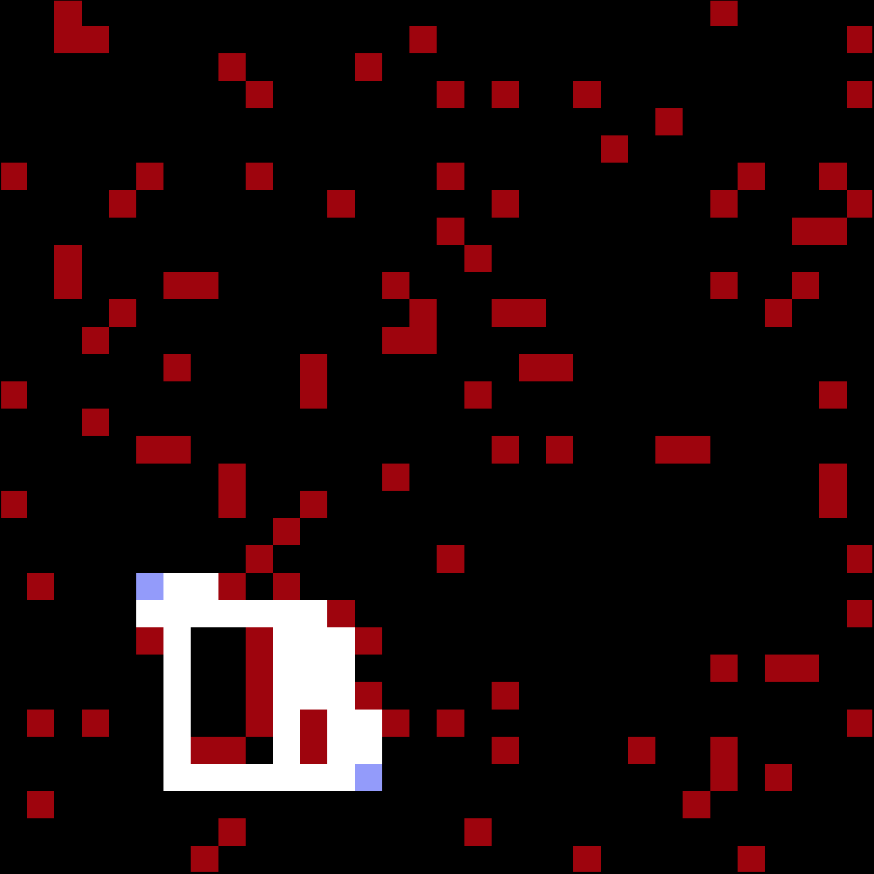}
   \caption{
      Example of the synthetic shortest paths task.
   }
  \label{fig-nas-synthdata}
\end{figure}

\begin{figure}[t]
\centering
   \includegraphics[width=0.65\linewidth]{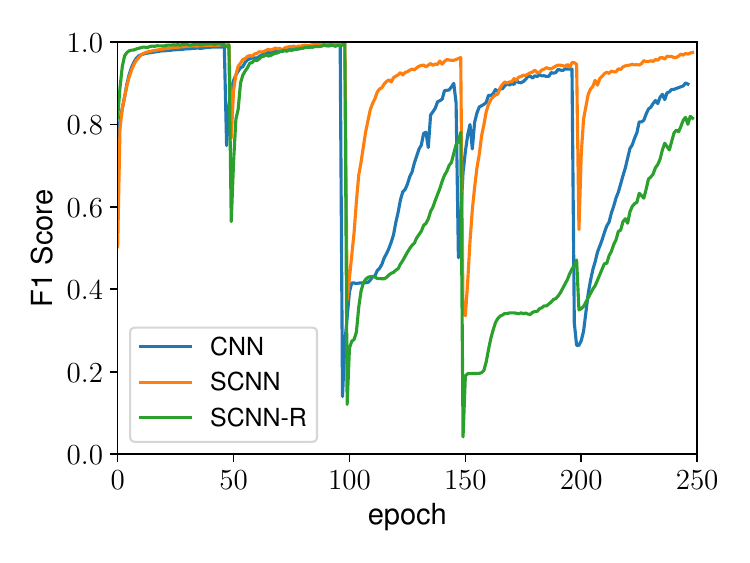}
   \caption{
      Training curves for the shortest paths task.
   }
  \label{fig-nas-synthcurves}
\end{figure}

While the propensity of our parameter sharing scheme to encourage learning of recurrent networks is a useful parameter reduction tool, we would also like to leverage it for qualitative advantages over standard CNNs.  On tasks for which a natural recurrent algorithm exists, does training CNNs with soft parameter sharing lead to better extrapolation?

We train a CNN, a CNN with soft parameter sharing and one template per layer (SCNN), and an SCNN with recurrence regularizer $\lambda_R = 0.01$.  Each model trains for 50 epochs per phase with Adam \citep{adam} and a fixed learning rate of $0.01$.  As classes are heavily unbalanced and the balance itself changes during phases, we compare $F_1$ scores instead of classification error.

Each model starts with a $1 \times 1$ convolution, mapping the 2 input channels to 32 output channels.  Next, there are 20 channel-preserving $3 \times 3$ convolutions, followed by a final $1 \times 1$ convolution that maps 32 channels to 1.  Each of the 20 $3 \times 3$ convolutions is followed by batch normalization~\citep{bn}, a ReLU non-linearity~\citep{relu}, and has a 1-skip connection.

Figure~\ref{fig-nas-synthdata} shows one example from our generated dataset: blue pixels indicate the query points; red pixels represent obstacles, and white pixels are points in a shortest path (in terms of Manhattan distance) between query pixels. The task consists of predicting the white pixels (shortest paths) from the blue and red ones (queries and obstacles).

Figure~\ref{fig-nas-synthcurves} shows training curves for the 3 trained models: the SCNN not only outperforms the CNN, but adapts better to harder examples at new curriculum phases.  The SCNN is also advantaged over a more RNN-like model: with the recurrence regularizer $\lambda_R = 0.01$, all entries in the LSM quickly converge $1$, as in a RNN.  This leads to faster learning during the first phase, but presents issues in adapting to difficulty changes in latter phases.

\section{Experimental Analysis}
\label{sec-nas-analysis}

\begin{figure}[h]
   \centering
   \includegraphics[width=0.75\textwidth]{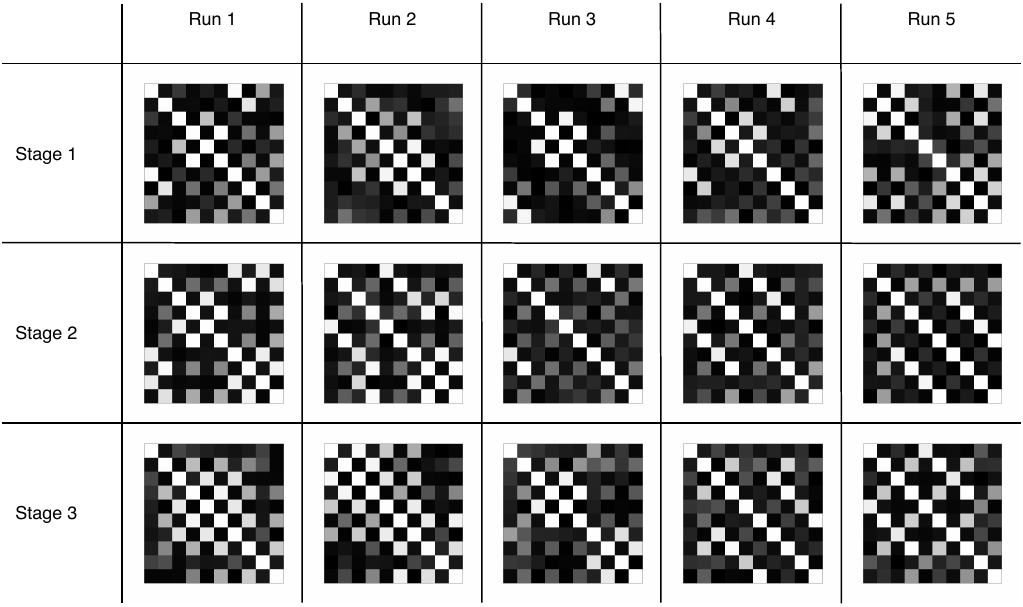}
   \caption{
      LSMs of a SWRN 40-8-8 over different runs.
   }
   \label{fig:appendix_several}
\end{figure}

\begin{figure}[h!]
   \centering
   \includegraphics[width=0.75\textwidth]{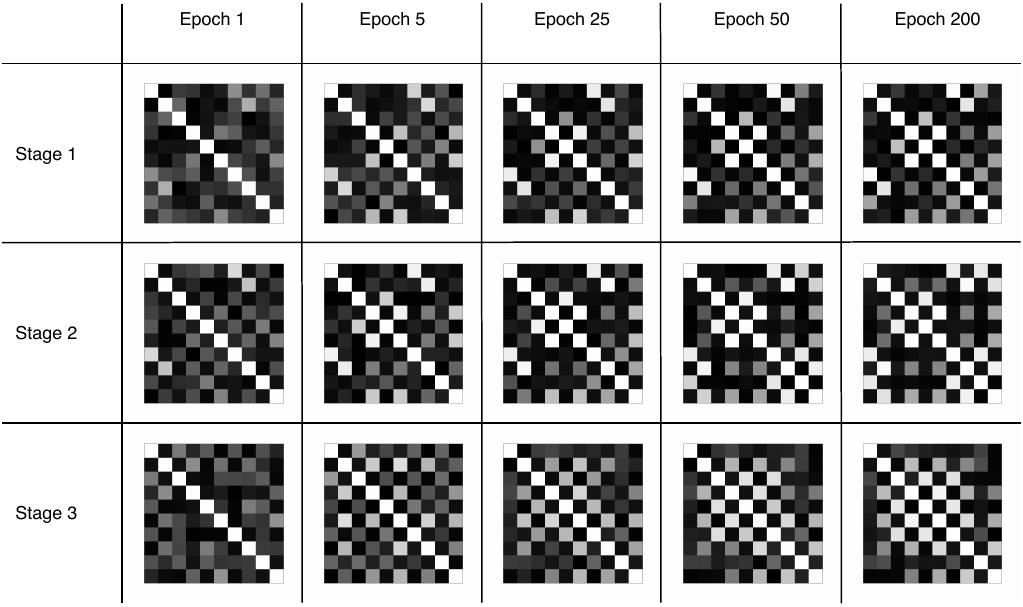}
   \caption{
      LSMs of a SWRN 40-8-8 over different epochs of the same run.
   }
   \label{fig:appendix_evolution}
\end{figure}

\begin{figure}[h]
    \centering
     \includegraphics[trim={0 0 0 0},clip,width=0.27\linewidth]{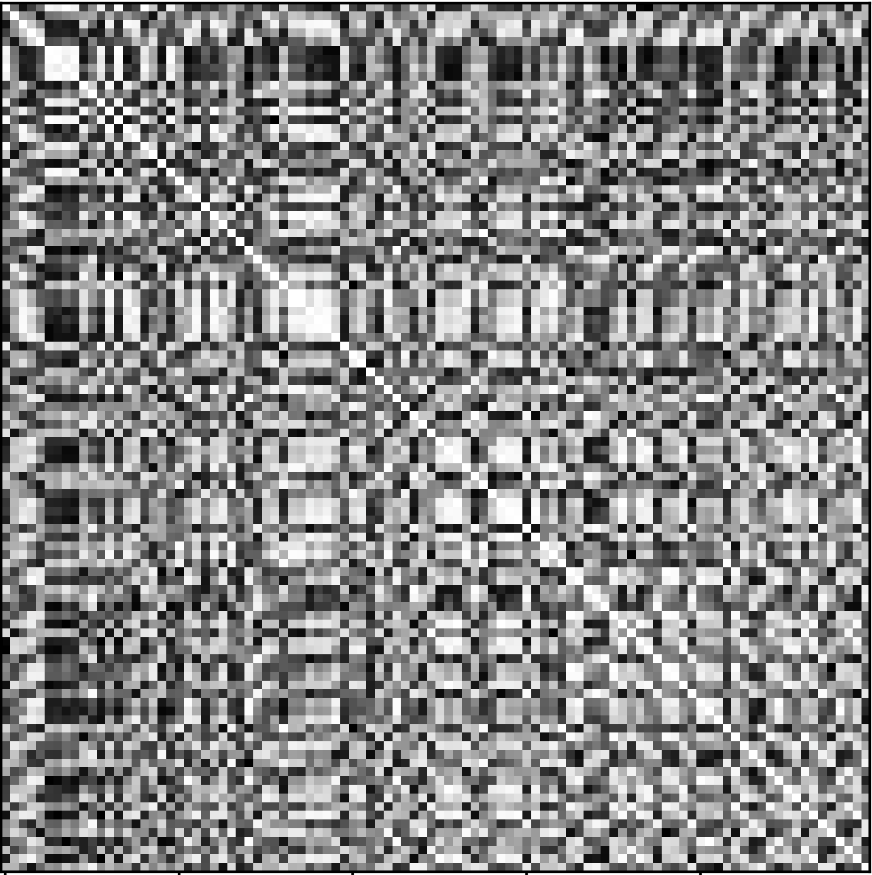}
     \includegraphics[trim={0 0 0 0},clip,width=0.27\linewidth]{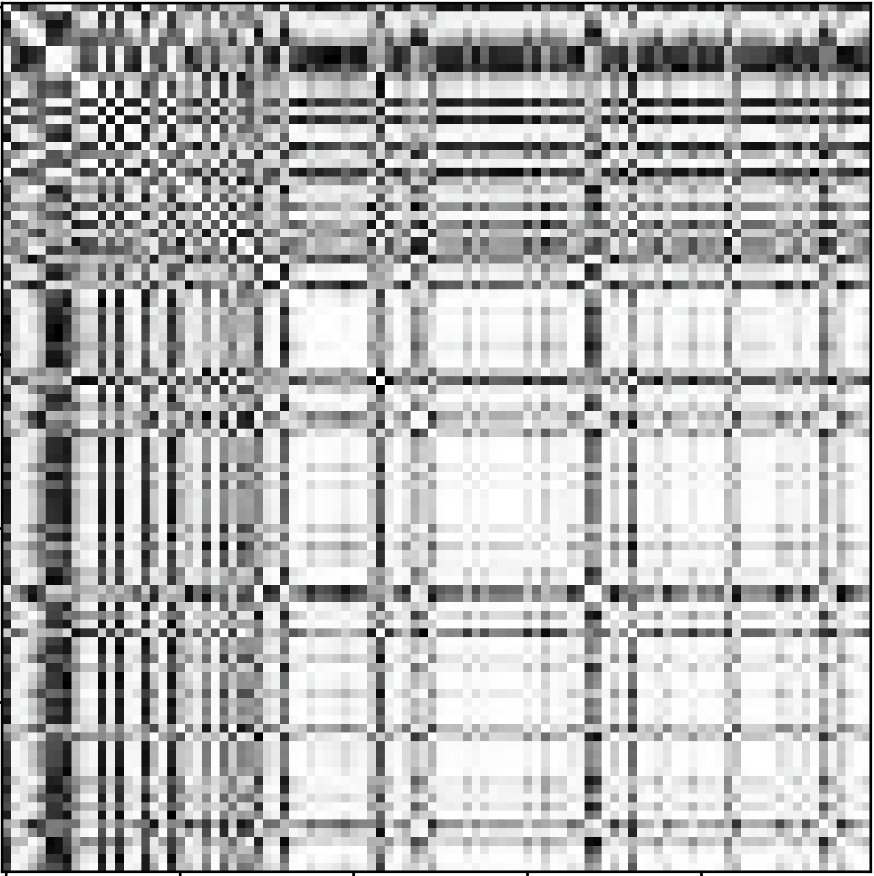}
     \includegraphics[trim={0 0 0 0},clip,width=0.27\linewidth]{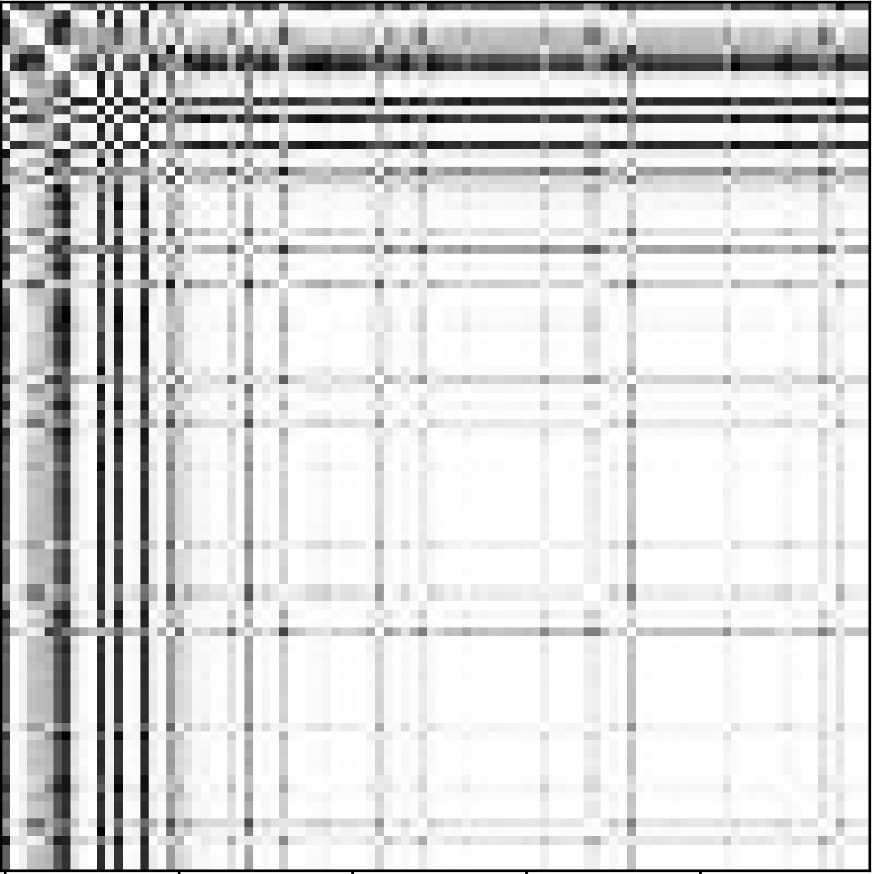}
     \caption{
      LSM evolution for 100-layer CNN trained on synthetic task (0, 20, 40 epochs).
     }
     \label{fig:appendix_evolution2}
\end{figure}

Figure \ref{fig:appendix_several} shows LSMs of a SWRN 40-8-8 (composed of 3 stages, each with 10 layers sharing 8 templates) trained on CIFAR-10 for 5 runs with different random seeds. Although the LSMs differ across different runs, hard parameter sharing can be observed in all cases (off-diagonal elements close to 1, depicted by white), characterizing implicit recurrences which would enable network folding.  Moreover, the underlying structure is similar across runs, with hard sharing typically happening among layers $i$ and $i+2$, leading to a ``chessboard'' pattern.

Figure \ref{fig:appendix_evolution} shows LSMs of a SWRN 40-8-8 (composed of 3 stages, each with 10 layers sharing 8 templates) at different epochs during training on CIFAR-10.  The transition from an identity matrix to the final LSM happens mostly in the beginning of training: at epoch 50, the LSM is almost indistinguishable from the final LSM at epoch 200, and most of the final patterns are observable already at epoch 25.

Figure \ref{fig:appendix_evolution2} shows the LSM evolution of a 100-layer CNN trained on our synthetic task. Note that, after training, a large group of layers have become functionally similar (white rectangular blocks in the LMS), while most of the remaining ones form a second group. Therefore, the network can be folded to compute long loops involving a single layer, where a secondary layer is employed between these loop calls.

\section{Discussion}
\label{sec-nas-discussion}

In this chapter, we presented a novel approach to neural architecture search based on soft parameter sharing, transforming a combinatorial NAS problem into one that is more tractable for gradient-based optimization. This work contributes to the broader goal of designing efficient deep learning systems by reducing both inference and training costs without sacrificing performance.

\subsection{Contributions}
\label{sec-nas-discussion-contribs}

This chapter makes several key contributions:
\begin{itemize}
    \item We introduce a new search space for NAS which enables architectures to have backward connections, or, equivalently, reuse layer configurations across the network. This effectively decouples the network's depth and its parameter count, makes recurrence patterns learnable, and allows for feedforward-recurrent hybrids to emerge.
    \item We propose a continuous approximation for NAS under our search space, enabling the use of gradient-based methods to optimize the network's topology and weights jointly. Additionally, we show that the architectural coefficients can be transformed along with the weight templates to reduce rounding errors when discretizing the learned soft parameter sharing scheme.
    \item Our strategy is capable of drastically reducing the parameter count of residual networks trained on CIFAR while preserving or even improving performance. Our soft parameter sharing scheme also improves classification accuracy on ImageNet, and boosts both generalization and learning speed on algorithmic tasks.
\end{itemize}

\subsection{Research Directions}
\label{sec-nas-discussion-directions}

Several promising research directions emerge from our work. While we focus on CNNs trained on computer vision tasks, our methodology might be extended to other model families and domains, such as RNNs and Transformers for natural language processing, or policy networks for reinforcement learning. Additionally, adopting more complex mappings from templates to effective weights, akin to employing deeper hypernetworks, could lead to further parameter reductions by reducing the dimensionality of the weight templates (something that our adopted mapping does not allow). Finally, there is significant potential in designing dynamic search spaces that evolve during training, allowing architectures to more easily learn nested loops and multi-level recurrences.

\subsection{Impact}
\label{sec-nas-discussion-impact}

Since its introduction, our soft parameter sharing approach has inspired several follow-up studies. Researchers have extended the framework to cover a broader range of architectural components, achieving improvements in both training speed and inference efficiency \citep{deeplyshared,latephase,superweights,gsharp,mixturegrowth,sharelstm, acdc}. Others have generalized architecture search under our search space~\citep{npas}, and applied our methodology to other problems such as meta-learning~\citep{metalearnhypernets}. Collectively, these efforts underscore the practical relevance of our approach and its potential to influence the design of more efficient deep learning systems.

\newpage

\chapter{Sparsification}

\section{Introduction}
\label{sec-cs-intro}

\subsection{Motivation \& Strategy}
\label{sec-cs-intro-motivation}

In the previous chapter, we addressed efficiency by automating the design of neural architectures via NAS. In this chapter, we tackle efficiency from a different perspective: instead of searching for efficient architectures, we improve a model's efficiency by removing unnecessary or redundant modules.

This strategy -- commonly referred to as \emph{pruning} or \emph{sparsification} -- aims to remove parameters from a pre-trained model without compromising its performance. By reducing the number of parameters, sparsification directly decreases both inference time and memory usage. Although sparsification is typically performed after training, it can also be applied gradually during training to reduce training time costs as well. Similar to NAS, sparsification involves a combinatorial optimization problem, as one must select which parameters to eliminate -- a process that is computationally hard. We review previous efforts to approximate network sparsification in Section \ref{sec-cs-intro-related}.

The remainder of the chapter is organized as follows. In Section \ref{sec-cs-method-formal}, we formally state the sparsification problem formally and discuss its computational challenges and impact on model efficiency. Next, in Section \ref{sec-cs-method-reframing}, we re-frame the original problem multiple times to arrive at a form that can be easily approximated. Section \ref{sec-cs-method-approx} presents our approximation of the underlying sparsification problem, which is the core component of our novel sparsification method. Section \ref{sec-cs-method-method} introduces our final algorithm and describes how it can be used to perform ticket search. The subsequent sections outline our experimental setup, report empirical result, and provide a detailed analysis of our approach, including a discussion of its contributions and impacts.

\subsection{Related Work}
\label{sec-cs-intro-related}

\paragraph{Sparsification.}

Numerous approaches have been proposed to reduce the number of parameters in deep networks without sacrificing performance. Early methods rely on heuristics to identify and remove unimportant weights from a fully-trained model~\citep{braindamage}. These methods use a statistic, such as the magnitude of each weight, as a proxy for its importance, and then prune the weights deemed to have the least impact on the model performance~\citep{magnitudepruning}. This two-phase process, where a model is first trained then pruned, can be repeated iteratively to further enhance sparsity and compression~\citep{stateofsparsity, gmp}.

An alternative to heuristic-based pruning is to directly incorporate $\ell_0$ regularization into the training process. Several works propose stochastic approximations that learn a parametric distribution over binary masks, optimized jointly with the network's weights~\citep{l0bernoulli, sparsityl0}. Gradients for the distribution's parameters are typically estimated via the straight-through estimator~\citep{straightthrough}, although this introduces additional variance.

While heuristic methods depend critically on having a well-trained model prior to pruning, approaches that approximate $\ell_0$ regularization enable dynamic sparsification even during early training epochs. However, these methods can be challenging when, at aggressive sparsity levels, weights are updated infrequently, requiring longer training times.

\paragraph{Ticket Search.}

Pruning a fully trained network often yields a sparse model that, when retrained from scratch, fails to recover the performance of the original dense network. However, \cite{lth} revealed that, if reinitialized with the same random seed as the original network, these sparse subnetworks can be trained to match or even exceed the dense model's performance.

To find these subnetworks -- called \emph{winning tickets} -- \cite{lth} propose Iterative Magnitude Pruning (IMP), which alternates between training, magnitude pruning, and rewinding the remaining weights to their initial values (or an early state). This procedure, known as \emph{ticket search}, has gathered significant attention not only because it reveals the presence of highly efficient subnetworks, but also because winning tickets can often be transferred across tasks and training pipelines~\citep{transfertickets, transtickets2, generaltickets}. A variant of this task -- non-retroactive ticket search, where the weights remain fixed at initialization while the pruning mask is optimized -- has also been explored by \cite{deconstructing}.

Despite its potential, the standard ticket search method, IMP, is computationally expensive due to the need for multiple rounds of training and pruning, motivating the search for more efficient sparsification strategies.

\section{Method}
\label{sec-cs-method}

\subsection{Formalizing the Sparsification Problem}
\label{sec-cs-method-formal}

The sparsification problem we address consists of removing redundant or unimportant parameters from a neural network in order to reduce resource requirements -- such as inference time and memory cost -- while preserving performance. As in the previous chapter, we assume the dataset $D$ is chosen a-priori and remains fixed, which contains samples (and, when applicable, labels), although our approach is independent of the nature of task.

Let $f: \theta \mapsto f(\theta)$ denote a neural architecture with an associated weight space $\Theta \subseteq \R^d$, where any $\theta \in \Theta$ equips $f$ to produce a network $f(\theta)$: a function mapping samples from $D$ to outputs consistent with the underlying task. Furthermore, define the loss $\mathcal L : f(\theta) \mapsto \mathcal L(f(\theta)) \in \R$ which evaluates the performance of a network $f(\theta)$ on the fixed dataset $D$.

Given a network architecture $f$, our goal in sparsification is to find a weight configuration $\theta \in \Theta$ that has as few nonzero components as possible, while ensuring that the induced network $f(\theta)$ performs comparably to the original dense model.

Unlike many pruning approaches that use heuristics to select weights based on proxies of importance~\citep{magnitudepruning,gmp,dnw}, our method relies on approximating the formal sparsification objective. This approach establishes a clear trade-off between sparsity and performance, thus eliminating the need for ad-hoc criteria to determine when and which weights to remove.

\begin{definition}[Sparsification Problem]
    \label{def-cs-formal}
    Let $f$ be a neural architecture with weight space $\Theta \subseteq \R^d$. For any $\theta \in \Theta$, $\mathcal L(f(\theta)) \in \R$ denotes the loss incurred by $f(\theta)$ on the fixed dataset $D$. Then, the \emph{sparsification problem} is:
    \begin{equation}
         \min_{ \substack{\theta \in \Theta}}
        \quad
        \mathcal L(f(\theta)) + \lambda \| \theta \|_0
        \,,
    \end{equation}
    where $\lambda \geq 0$ controls the trade-off between model performance and sparsity, and the $\ell_0$ pseudonorm is given by
    \begin{equation}
        \|\theta\|_{0} = \sum_{i=1}^d \mathbbm 1 \{\theta_i \neq 0 \}
    \end{equation}
\end{definition}

The $\ell_0$ pseudonorm counts the number of nonzero parameters, effectively turning the problem into a discrete selection task over which weights to zero-out. This leads to $2^d$ possible subsets, rendering the problem inherently combinatorial and intractable for most networks.

Although the $\ell_0$ regularizer is differentiable almost everywhere, its derivative is zero at any nonzero weight. In practice, unless a weight is exactly zero (which is unlikely if weights are randomly initialized), the gradient contribution from the $\ell_0$ term is zero. Consequently, during gradient-based training, the regularizer does not contribute to the update, effectively causing the optimization method to ignore the regularizer completely.

\paragraph{Structured Sparsification.} In practice, removing individual weights, known as \emph{unstructured sparsification},  may reduce memory usage, but most hardware (\eg GPUs) cannot effectively leverage irregular sparsity patterns for speed gains. In contrast, \emph{structured sparsification} enforces sparsity under specific structural constraints -- for example, by removing entire neurons from fully-connected layers or whole output channels from convolutional layers. This structured approach effectively reduces layer width and enables practical speedups on conventional hardware.

The structured sparsification problem can be formulated as in Definition~\ref{def-cs-formal}, but with a modified $\ell_0$ pseudonorm that reflects the desired structure. For example, consider an $L$-layer fully-connected network with parameters $\theta = (W^{(l)})_{l=1}^L$, where $W^{(l)} \in \R^{d_l \times d_{l-1}}$. To remove the $i$-th neuron of layer $l$, we set all its incoming weights (\ie the $i$-th row of $W^{(l)}$, $W^{(l)}_i$) to zero. In this setting, defining the $\ell_0$ pseudonorm as
\begin{equation}
    \| (W^{(l)})_{l=1}^L \|_0 = \sum_{l=1}^L \sum_{i=1}^{d_l} \mathbbm 1 \{ W^{(l)}_i \neq \vec 0 \} \,,
\end{equation}
enforces the removal of entire neurons (or channels) rather than individual weights. Regardless of the imposed structure, the problem remains combinatorial and is typically intractable for large models.

\subsection{Re-framing the Sparsification Problem}
\label{sec-cs-method-reframing}

In this Section, we reformulate the sparsification problem in three progressive steps to arrive at a form that is amenable to the approximation we will employ.

First, we decouple the selection of weights from their values by introducing binary mask $m \in \{0,1\}^d$. Instead of characterizing the removal of the $i$-th component of $\theta$ by setting $\theta_i = 0$, we now indicate weight removal by $m_i = 0$, regardless of $\theta_i$. The effective weights are given by the element-wise product
\begin{equation}
    (m \odot \theta)_i = m_i \theta_i =
    \begin{cases}
        \theta_i, ~\text{if } m_i = 1 \\
        0, ~\text{if } m_i = 0 \,,
    \end{cases}
\end{equation}
so that the number of retained weights is $\sum_{i=1}^d m_i$. Since $m$ is binary, we can express this sum as $\|m\|_p^p = \sum_{i=1}^d m_i^p$ for any $p > 0$. This reformulation leads to the following equivalent problem:
\begin{definition}[Re-framed Sparsification Problem (Version 1)]
    \label{def-cs-re1}
    Let $f$ be a neural architecture with weight space $\Theta \subseteq \R^d$. For any $\theta \in \Theta$, $\mathcal L(f(\theta)) \in \R$ denotes the loss incurred by $f(\theta)$ on the fixed dataset $D$. Then, for any $p>0$, the sparsification problem can be rewritten as
    \begin{equation}
        \min_{ \substack{m \in \{0,1\}^d \\ \theta \in \Theta}}
        \quad
        \mathcal L(f(m \odot \theta)) + \lambda \| m \|_p^p
        \,,
    \end{equation}
    which is equivalent to the problem in Definition \ref{def-cs-formal}.
\end{definition}

This version highlights the combinatorial nature of the problem, now condensed in the discrete domain of $m$.

% Stated in this form, its combinatorial nature becomes evident due to the binary domain of the mask variables $m$. The regularizer is now amenable to subgradient descent and, for any $p \geq 1$, will also be convex. Therefore, we have `transferred' the hardness of the problem from the $\ell_0$ regularizer to the discrete domain of $m$.

% Note that, at this point, we could relax the domain of $m$ to $[0,1]^d$ as a continuous approximation and tackle with directly with gradient-based methods. However, we will rely on a different type of approximation, which requires framing the problem in a different form.

Next, we reparameterize the binary mask $m$ by expressing it as the output of a surjective binary function $b : \R \to \{0,1\}$. Specifically, we define $m = b(s)$, where $b$ is applied element-wise and $s \in \R^d$ becomes the new optimization variable. This changes the problem to one of optimizing over continuous variables $s$ instead of binary $m$:
\begin{definition}[Re-framed Sparsification Problem (Version 2)]
    \label{def-cs-re2}
    Let $f$ be a neural architecture with weight space $\Theta \subseteq \R^d$ and let $b : \R \to \{0,1\}$ be a surjective binary function applied element-wise. For any $\theta \in \Theta$, $\mathcal L(f(\theta)) \in \R$ denotes the loss incurred by $f(\theta)$ on the fixed dataset $D$. Then, for any $p>0$, the sparsification problem can be rewritten as
    \begin{equation}
        \min_{ \substack{s \in \R^d \\ \theta \in \Theta}}
        \quad
        \mathcal L \big( f(b(s) \odot \theta) \big) + \lambda \| b(s) \|_p^p
        \,,
    \end{equation}
    which is equivalent to the problem in Definition \ref{def-cs-formal}.
\end{definition}

Although the optimization variables are now continuous, the function $b$ (\eg step function) remains discontinuous with zero derivative almost everywhere, posing the same challenges as the original $\ell_0$ regularizer.

To address this, we replace $b$ with a family of smooth functions that converge point-wise to $b$. Let $\hat b : \R \times \R \to \R$ be such that, for all $s \in \R$,
\begin{equation}
    \lim_{\beta \to \infty} \hat b(\beta, s) = b(s) \,,
\end{equation}

hence for any continuous function $g$,
\begin{equation}
    \lim_{\beta \to \infty} g( \hat b(\beta, z) ) = g(b(z)) \,.
\end{equation}

% One strategy to approximate the sparsification problem as currently framed is to adopt a stochastic mapping $b$ instead (\eg using a Bernoulli distribution), and use gradient estimators to train $s$ via gradient descent~\citep{sparsityl0,l0bernoulli} -- these are the most popular sparsification approaches that rely on approximating the $\ell_0$ regularizer. Unlike previous works, we will design a completely deterministic approximation.

% For the last recasting we require an additional assumption that $\mathcal L \circ f$ is continuous, which is satisfied for the vast majority of deep learning applications. In this final step, we write the binary function $b$ as the point-wise limit of a family of functions. Specifically, let $\hat b: \R \times \R \to \R$ be a function such that

Then, the final re-formulation is:
\begin{definition}[Re-framed Sparsification Problem (Version 3)]
    \label{def-cs-re3}
    Let $f$ be a neural architecture with weight space $\Theta \subseteq \R^d$. For any $\theta \in \Theta$, $\mathcal L(f(\theta)) \in \R$ denotes the loss incurred by $f(\theta)$ on the fixed dataset $D$. Moreover, let $\hat b : \R \times \R \to \R$ satisfy
    \begin{equation}
        \forall s \in \R, \quad \lim_{\beta \to \infty} \hat b(\beta, s) = b(s),
    \end{equation}
    where $b : \R \to \{0,1\}$ is surjective.
    Then, if $\mathcal L \circ f$ is continuous, for any $p>0$ the sparsification problem can be rewritten as
    \begin{equation}
        \min_{ \substack{s \in \R^d \\ \theta \in \Theta}}
        \quad
        \lim_{\beta \to \infty} ~ \mathcal L \big(f(\hat b(\beta, s) \odot \theta) \big) + \lambda \| \hat b(\beta, s) \|_p^p
        \,,
    \end{equation}
    which is equivalent to the problem in Definition \ref{def-cs-formal}.
\end{definition}

At first glance, this formulation may not seem advantageous. However, note that $\hat b(\beta, \cdot)$ is allowed to be smooth for any finite $\beta$, thus being amenable to gradient-based optimization. The inherent difficulty of the original problem is then encapsulated in the limit $\beta \to \infty$.

\subsection{Approximating the Sparsification Problem}
\label{sec-cs-method-approx}

% We approximate $H$ by constructing a sequence of functions indexed by $\beta \in [1, \infty)$ given by $s \mapsto \sigma(\beta s)$ where $\sigma$ is the sigmoid function $\sigma(s) = \frac{1}{1 + e^{-s}}$ applied element-wise. This sequence can be seen as a \emph{path} parameterized by $\beta$, and given any fixed $s \in \R_{\neq 0}$, we have at one of its endpoints $\lim_{\beta \to \infty} \sigma(\beta s) = H(s)$. Conversely, for $\beta=1$ we have $\sigma(\beta s) = \sigma(s)$, the standard sigmoid activation function that is smooth and widely used in neural network models.

In Version 3 of our re-framed sparsification problem (Definition~\ref{def-cs-re3}), the objective involves taking the limit $\beta \to \infty$ to recover the binary mask $b$ from its smooth approximation $\hat b(\beta, \cdot)$. Our approximation consists of swapping the order of the minimization and the limit:
\begin{definition}[Approximate Sparsification Problem]
    \label{def-cs-approx}
    Let $f$ be a neural architecture with weight space $\Theta \subseteq \R^d$. For any $\theta \in \Theta$, $\mathcal L(f(\theta)) \in \R$ denotes the loss incurred by $f(\theta)$ on the fixed dataset $D$. Moreover, let $\hat b : \R \times \R \to \R$ be a smooth function that satisfies
    \begin{equation}
        \forall s \in \R, \quad \lim_{\beta \to \infty} \hat b(\beta, s) = b(s),
    \end{equation}
    where $b : \R \to \{0,1\}$ is surjective.
    Then, if $\mathcal L \circ f$ is continuous, for any $p>0$, the \emph{approximate sparsification problem} is
    \begin{equation}
        \lim_{\beta \to \infty}~ \min_{ \substack{s \in \R^d \\ \theta \in \Theta}}
        \quad
        \mathcal L \big(f(\hat b(\beta, s) \odot \theta) \big) + \lambda \| \hat b(\beta, s) \|_p^p
        \,.
    \end{equation}
\end{definition}

In this setting, the objective of the minimization problem,
\begin{equation}
    \mathcal L_\beta(\theta, s) \coloneqq \mathcal L \big(f(\hat b(\beta, s) \odot \theta) \big) + \lambda \| \hat b(\beta, s) \|_p^p \,,
    \label{eq-cs-softloss}
\end{equation}
is continuous and amenable to gradient-based optimization since $\hat b(\beta, \cdot)$ is smooth for any finite $\beta$. In practice, we iteratively update $\theta$ and $s$ to minimize $\mathcal L_\beta(\theta, s)$ while gradually annealing $\beta$ during training. Note that increasing $\beta$ makes the problem progressively harder, as $\hat b(\beta, \cdot)$ becomes sharper and $\mathcal L_\beta$ less smooth in the optimization sense.

There are two sources of error in our approximation. First, by swapping the minimization and the limit, we do not generally obtain an equivalent problem. In fact, for a sequence of real-valued functions $(g_n)_{n \in \N}$ defined on a set $X$ that converges point-wise to $g$, we have
\begin{equation}
    \min_{x \in X} \lim_{n \to \infty} g_n(x) \geq \lim_{n \to \infty} \min_{x \in X} g(x) \,.
    \label{eq-cs-minlim}
\end{equation}
Thus, our approximation effectively optimizes a lower bound of the original problem. In the framework of continuation methods, we typically consider additional assumptions -- such as $\Gamma$-convergence and equi-coercivity of the family $(\mathcal L_\beta)_{\beta}$ -- to ensure that there is no gap in the limit.

Second, as $\beta$ increases during training, the soft mask $\hat b(\beta, \cdot)$ necessarily becomes increasingly sharp. Consequently, $\mathcal L_\beta$ becomes less smooth and requires more training iterations to be effectively minimized. In practice, however, we do not increase the number of training iterations as $\beta$ grows in order to keep the method fast and efficient.

\begin{algorithm}[t]
    \textbf{Input:} Architecture $f$, $s_{init} \in \R$, $\lambda \geq 0$, $T \in \N$, $\beta_{final}>0$
    \caption{Continuous Sparsification}
    \begin{algorithmic}[1]
    \State Initialize $\theta \sim \dist_\theta$, $\beta \gets 1$, $s \gets \vec s_{init}$
    \For{$t = 1$ to $T$}
        \State Compute $\mathcal L_\beta(\theta, s) = \mathcal L(f(\sigma(\beta s) \odot \theta)) + \lambda \|\sigma(\beta s)\|_1$ and $\nabla_{\theta, s} \mathcal L_\beta(\theta,s)$
        \State Update $\theta$ and $s$ using $\nabla_\theta \mathcal L_\beta(\theta,s)$ and $\nabla_s \mathcal L_\beta(\theta,s)$
        \State Update $\beta \gets \beta \cdot \beta_{final}^{1/T}$ 
    \EndFor
    \State Compute $m = H(s) = \mathbbm 1 \{ s \geq 0 \}$ applied element-wise
    \State (Optional) Fine-tune $\theta$ to decrease $\mathcal L(f(m \odot \theta))$
    \State Output $f(m \odot \theta)$
    \end{algorithmic}
\label{alg-cs-csprune}
\end{algorithm}

\subsection{The Proposed Sparsification Method}
\label{sec-cs-method-method}

For our method, we instantiate the binary function $b$ as the Heaviside step function $b(s) = H(s) = \mathbbm 1\{s \geq 0\}$ and approximate it using the sigmoid function. Specifically, we define
\begin{equation}
    \hat b(\beta, s) = \sigma(\beta s) = \frac{1}{1 + e^{-\beta s}}\,.
\end{equation}

We also set the norm parameter $p=1$ so that the sparsity regularizer is simply $\| \hat b(\beta, s) \|_1 = \|\sigma(\beta s) \|_1$. In terms of sparsification, every negative component of $s$ will drive the corresponding element of the effective weight $\theta \odot \sigma(\beta s)$ towards zero as $\beta$ increases. Although analytically $\sigma(\beta s)$ is never exactly zero for any finite $\beta$, limited numerical precision ensures that, beyond a certain threshold, the weights are effectively pruned during training.

Our overall approach consists of jointly optimizing the network weights $\theta$ and the mask parameters $s$ to minimize $\mathcal L_\beta$
using gradient descent. We anneal $\beta$ during training following an exponential schedule defined by $\beta$ starting as $1$ and ending with a pre-determined value $\beta_{final}$ after $T$ time steps, which amounts to scaling $\beta$ by $\beta_{final}^{1/T}$ at the end of each iteration. This schedule allows the mask $\sigma(\beta s)$ to transform from smooth values to an approximation of a binary mask as training progresses. At the end of training, when $\beta$ is sufficiently large, we can either use the learned mask $\sigma(\beta s)$ directly (which, due to numerical precision, may be mostly binary), or explicitly round it to obtain a binary mask $m = H(s)$.

Our final method, Continuous Sparsification (CS), is given in Algorithm~\ref{alg-cs-csprune}. Although we included an optional fine-tuning step for the weights, our empirical results indicate that this step is typically unnecessary when $\beta_{final}$ is large (\eg around 500 or above).

Note that many alternative choices for the smooth function $\hat b(\beta, \cdot)$ are valid. In particular, any sigmoidal function may be used in place of $\sigma$. Although we have not explored these alternatives in our experiments, investigating different forms for $\hat b$ could potentially lead to performance improvements. For example:
\begin{itemize}
    \item $\hat b(\beta, s) = \frac12 (1 + \erf(\beta s))$
    \item $\hat b(\beta, s) = \frac12 \left( 1 + \frac{s}{ \left( \beta^{-1} + |s|^k \right)^{\frac{1}{k}}} \right)$ for any $k>0$
    \item $\hat b(\beta, s) = \frac12 \left( 1 + \frac{\pi}{2} \arctan(\beta s) \right)$
\end{itemize}

\paragraph{Ticket Search.}

\begin{algorithm}[t]
    \textbf{Input:} Architecture $f$, $s_{init} \in \R$, $\lambda \geq 0$, $T \in \N$, $R \in \N$, $k \in \N$, $\beta_{final}>0$
    \caption{Ticket Search with Continuous Sparsification}
    \begin{algorithmic}[1]
    \State Initialize $\theta \sim \dist_\theta$, $s \gets \vec s_{init}$
    \For{$r = 1$ to $R$}
        \State If $r>1$: set $s \gets \min(\beta \cdot s, s_{init}), \beta \gets 1$
        \For{$t = 1$ to $T$}
            \State Compute $\mathcal L_\beta(\theta, s) = \mathcal L(f(\sigma(\beta s) \odot \theta)) + \lambda \|\sigma(\beta s)\|_1$ and $\nabla_{\theta, s} \mathcal L_\beta(\theta,s)$
            \State Update $\theta$ and $s$ using $\nabla_\theta \mathcal L_\beta(\theta,s)$ and $\nabla_s \mathcal L_\beta(\theta,s)$
            \State If $r=1$ and $t=k$: store $\hat \theta \gets \theta$
            \State Update $\beta \gets \beta \cdot \beta_{final}^{1/T}$ 
        \EndFor
    \EndFor
    \State Compute $m = H(s) = \mathbbm 1 \{ s \geq 0 \}$ applied element-wise
    \State Output $f(m \odot \hat \theta)$
    \end{algorithmic}
\label{alg-cs-csticket}
\end{algorithm}

\begin{algorithm}[t]
    \textbf{Input:} Architecture $f$, $\tau \in (0,1)$ $T \in \N$, $R \in \N$, $k \in \N$
    \caption{Ticket Search with Iterative Magnitude Pruning \citep{lth2}}
    \begin{algorithmic}[1]
    \State Initialize $\theta \sim \dist$, $m \gets \vec 1^d$
    \For{$r = 1$ to $R$}
        \State If $r>1$: set $\theta \gets \theta_k$
        \For{$t = 1$ to $T$}
            \State Compute $\mathcal L(\theta,m) = \mathcal L(f(m \odot \theta))$ and $\nabla_\theta \mathcal L(\theta, m)$
            \State Update $\theta$ using $\nabla_\theta \mathcal L(\theta, m)$
            \State If $r=1$ and $t=k$: store $\hat \theta \gets \theta$
        \EndFor 
        \State For each $\theta_i$ that is among the $\tau$ fraction with smallest magnitude, set \(m_i \gets 0\)
    \EndFor
    \State Output $f(m \odot \hat \theta)$
    \end{algorithmic}
\label{alg-cs-impticket}
\end{algorithm}

Ticket search is the task of identifying sparse subnetworks (winning tickets) that can be trained from scratch to achieve performance comparable to that of the full, overparameterized model. Discovering such subnetworks not only challenges the conventional belief that overparameterization is essential for state-of-the-art performance, but also offers the potential to drastically reduce computational and memory requirements during both training and inference. This reduction in resource costs can enable more efficient deployment on hardware-constrained devices and lower the overall energy footprint of deep learning systems.

To design a ticket search strategy from our sparsification method, we follow Iterative Magnitude Pruning (IMP)~\citep{lth} and operate in rounds. At the start of each round, we reset $\beta$ to 1 to allow further sparsification, and then gradually increase $\beta$ again as the round progresses. Additionally, we reset the soft mask parameters $s$ for weights that have not been suppressed, ensuring that the pruning process remains dynamic. Specifically, we set $s \gets \min(\beta \cdot s, s_{init})$ when starting a new round. Algorithm~\ref{alg-cs-csticket} presents our method for ticket search, which, in contrast to IMP, does \textbf{not} rewind weights between rounds.

In our experiments, we compare Continuous Sparsification method with IMP, the algorithm used in~\cite{lth} to find winning tickets. At each round, IMP removes a fixed percentage of weights with the smallest magnitudes by setting their corresponding mask entries to zero, followed by rewinding the remaining weights to an earlier state (Algorithm~\ref{alg-cs-impticket}). While IMP has been shown successful when performing ticket search, its reliance on heuristic, magnitude-based pruning and the need for multiple rounds of training make it computationally expensive and potentially suboptimal. In contrast, our method integrates sparsification directly into the training process through a deterministic approximation of $\ell_0$ regularization, leading to a potentially more efficient approach for finding winning tickets.

\section{Experimental Setup}
\label{sec-cs-expsetup}

~

\subsection{Datasets}
\label{sec-cs-expsetup-data}

\paragraph{CIFAR-10.}

The CIFAR-10 dataset~\citep{cifar} consists of $32 \times 32$ color images split into 50,000 and 10,000 training and test samples, each belonging to one out of 10 classes. We pre-process the data by applying channel-wise normalization to all images using statistics computed from the training set. We use the standard data augmentation pipeline from~\cite{resnet1}, which includes random crops and horizontal flips.

\paragraph{ImageNet.}

We employ the ILSVRC 2012 subset of ImageNet~\citep{imagenet}, which contains roughly 1.28 million training images and 50,000 validation images belonging to 1,000 different object classes. We use single $224 \times 224$ center-crop images for both training and validation.
We follow~\cite{gross} for pre-processing and data augmentation, which consists of scale augmentation (random crops of different sizes and aspect ratios are rescaled back to the original size with bicubic interpolation), photometric distortions (random changes to brightness, contrast, and saturation), lighting noise, and horizontal flips. Channel-wise normalization is employed using statistics from a random subset of the training data.

\subsection{Models}
\label{sec-cs-expsetup-models}

\paragraph{VGG-16.} For CIFAR-10, we adopt a variant of VGG-16, a deep network characterized by a uniform CNN architecture built from multiple convolutional and pooling layers and followed by fully connected layers. In our implementation, the convolutions are followed by ReLU and batch normalization, which helps stabilize training and improve convergence. The relatively straightforward and homogeneous design of VGG-16 makes it an excellent candidate for analyzing the effects of sparsification, as any removal of parameters can be directly related to changes in performance.

\paragraph{ResNet-20.} Also for CIFAR-10, we employ ResNet-20—a residual network that incorporates skip connections to facilitate the training of deeper architectures. ResNet-20 is constructed from a series of residual blocks, each comprising two convolutions with batch normalization and ReLU activations. This architecture serves as a representative residual model to evaluate how sparsification impacts both parameter reduction and performance in more modern networks.

\paragraph{ResNet-50.} For experiments on ImageNet, we use ResNet-50, a widely adopted benchmark architecture that achieves a strong balance between depth, computational efficiency, and performance. ResNet-50 is built from bottleneck residual blocks -- each block consists of a $1 \times 1$ depth-reducing convolution, a $3 \times 3$ dept-preserving convolution, and a $1 \times 1$ depth-increasing convolution, each followed by batch normalization and ReLU activations. Our ResNet-50 experiments are designed to test the scalability of sparsification methods to large-scale networks and complex datasets.

\paragraph{6-layer CNN for Supermask Search.} In addition to the above models, we include a small 6-layer CNN for supermask search experiments on CIFAR-10. This network is constructed with three blocks, each containing two resolution-preserving $3 \times 3$ convolutional layers followed by $2 \times 2$ max-pooling. The number of channels increases across the blocks (64, 128, and 256 channels), and the convolutional layers are immediately followed by ReLU activations. The network concludes with a series of fully connected layers (256, 256, and 10 neurons). The simplicity of this architecture allows us to isolate the effects of learning binary masks on randomly initialized weights, providing valuable insights into the supermask formulation without the added complexity of deeper networks.

\subsection{Training}
\label{sec-cs-expsetup-training}

For CS, we set hyperparameters $\lambda=10^{-8}$ and $\beta_{final} = 200$, based on analysis in Section~\ref{sec-cs-analysis}, which studies how $\lambda$, $s_{init}$, and $\beta_{final}$ affect the sparsity of produced subnetworks. We observe that $s_{init}$ has a major impact on sparsity levels, while $\lambda$ and $\beta_{final}$ require little to no tuning.

\paragraph{Sparse Networks on CIFAR.} We train VGG-16 and ResNet-20 on CIFAR-10 for 200 epochs, with a initial learning rate of $0.1$ which is decayed by a factor of $10$ at epochs 80 and 120. The subnetwork is produced at epoch 160 and is then fine-tuned for 40 extra epochs with a learning rate of $0.001$. More specifically, at epoch 160 the subnetwork structured is fixed: AMC, MP, GMP and Slim zero-out elements in the binary matrix $m$ for the last time, while CS fixes $m = H(s)$ and $s$ is no longer trained.

\paragraph{Ticket Search on CIFAR.} We follow the setup from \citet{lth}: in each round, we train with SGD, a learning rate of $0.1$, and a momentum of $0.9$, for a total of $85$ epochs, using a batch size of $64$ for VGG and $128$ for ResNet. We decay the learning rate by a factor of $10$ at epochs $56$ and $71$, and utilize a weight decay of $0.0001$.

For CS, we do not apply weight decay to the mask parameters $s$, since they are already suffer $\ell_1$ regularization. Sparsification is performed on all convolutional layers, excluding the two skip-connections of ResNet-20 that have $1 \times 1$ kernels: for IMP, their parameters are not pruned, while for CS their weights do not have an associated learnable mask.

IMP performs global pruning at a per-round rate, removing $20\%$ of the remaining parameters with smallest magnitude. We run IMP for $30$ iterations, yielding $30$ tickets with varying sparsity levels ($80\%, 64\%, \dots$). To produce tickets of differing sparsity with CS, we vary $s_{init}$ across $11$ values from $-0.3$ to $0.3$, performing a run of 5 rounds for each setting. We repeat experiments 3 times, with different random seeds.

\paragraph{ImageNet.} Following \citet{lth2}, we train the network with SGD for 90 epochs, with an initial learning rate of $0.1$ that is decayed by a factor of 10 at epochs $30$ and $60$. We use a batch size of 256 distributed across 4 GPUs and a weight decay of $0.0001$. We run CS for a single round due to the high computational cost of training ResNet-50 on ImageNet, with $s_{init} \in \{0.0, -0.01, -0.02, -0.03, -0.05\}$ yielding 5 subnetworks with varying sparsity levels.

\paragraph{Supermask Search and Ticket Search on 6-layer CNN.} All parameters are trained using Adam and a learning rate of $3 \times 10^{-4}$, excluding the mask parameters $s$ for Stochastic Sparsification, for which we adopted SGD with a learning rate of $100$ following \cite{deconstructing}, and for Iterative Stochastic Sparsification, where $s$ is trained with SGD and a learning rate of $20$.

\subsection{Evaluation}
\label{sec-cs-expsetup-eval}

\paragraph{Test Accuracy and Top-1/Top-5 Accuracy.}

For the CIFAR datasets, we measure and report the accuracy on the 10,000 test samples. On ImageNet, we use the top-1 and top-5 accuracy on the 50,000 validation images. Top-5 accuracy is the fraction of samples for which the true label appears among the top five predicted classes.

\paragraph{Sparsity.}

We compare methods on the tasks of pruning and finding matching subnetworks. We quantify the performance of ticket search by focusing on two specific subnetworks produced by each method:
\begin{itemize}
    \item \textbf{Sparsest matching subnetwork}: the sparsest subnetwork that, when trained in isolation from an early iterate, yields performance no worse than that achieved by the trained dense counterpart.
    \item \textbf{Best performing subnetwork}: the subnetwork that achieves the best performance when trained in isolation from an early iterate, regardless of its sparsity.
\end{itemize}

We also measure the efficiency of each method in terms of total number of epochs to produce subnetworks, \emph{given enough parallel computing resources}. As we will see, CS is particularly suited for parallel execution since it requires relatively few rounds to produce subnetworks regardless of sparsity. On the other hand, CS offers no explicit mechanism to control the sparsity of the found subnetworks, hence producing a subnetwork with a pre-defined sparsity level can require multiple runs with different hyperparameter settings. For this use case, IMP is more efficient by design, since a single run suffices to produce subnetworks with varying, pre-defined sparsity levels.
\section{Results}
\label{sec-cs-results}

\begin{table}[t]
\centering
\begin{tabular}{@{}lrrrrrc@{}}
\toprule
 & \citet{sparsityl0} & AMC & MP & GMP & NetSlim & CS \\ \cmidrule{2-7}
VGG-16 & 18.2\% & 86.0\% & 97.5\% & 98.0\% & 99.0\% & \textbf{99.6\%}  \\
ResNet-20 & 13.6\% & 50.0\% & 80.0\% & 86.0\% & 85.0\% & \textbf{94.4\%} \\ \bottomrule
\end{tabular}
\caption{Sparsity (\%) of the sparsest subnetwork within $2\%$ test accuracy of the original dense model, for different pruning methods on CIFAR.}
\label{tab-cs-sparsecifar}
\end{table}

\begin{figure}[t]
      \centering
      \includegraphics[width=0.8\linewidth]{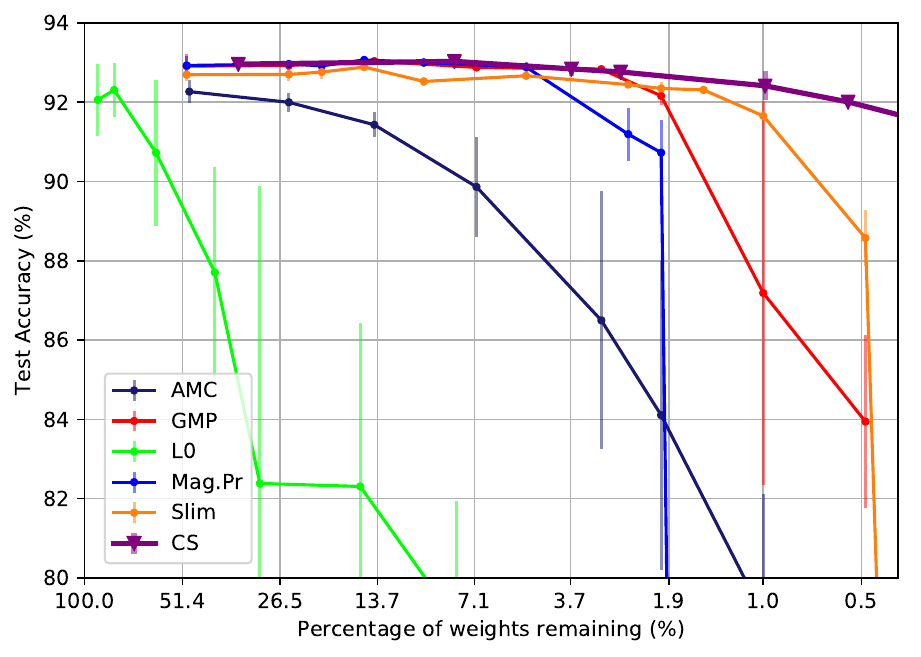}
    \caption{Performance of different methods when performing one-shot pruning on VGG-16, measured in terms of test accuracy and sparsity of produced subnetworks after fine-tuning.}
    \label{fig-cs-sparsevgg}
\end{figure}

\begin{figure}[t]
      \centering
      \includegraphics[width=0.8\linewidth]{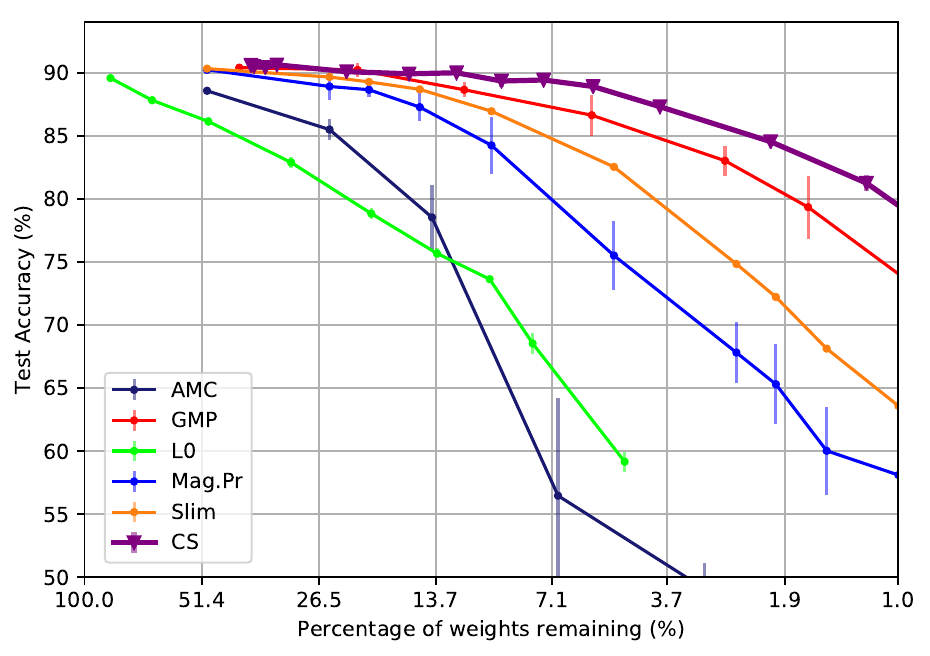}
    \caption{Performance of different methods when performing one-shot pruning on ResNet-20, measured in terms of test accuracy and sparsity of produced subnetworks after fine-tuning.}
    \label{fig-cs-sparseresnet}
\end{figure}

\subsection{Sparsification on CIFAR}
\label{sec-cs-results-sparse}

First, we evaluate CS as a general-purpose sparsification method. We compare it against the prominent pruning methods AMC~\cite{amc}, magnitude pruning (MP)~\cite{magnitudepruning}, GMP~\cite{gmp}, and Network Slimming (Slim)~\cite{slim}, along with the $\ell_0$-based method of \cite{sparsityl0} (referred to as ``$\ell_0$''), which, in contrast to ours, adopts a stochastic approximation for $\ell_0$ regularization.

Adopting the inference behavior suggested in \cite{sparsityl0} for $\ell_0$, \ie using the expected value of the uniform distribution to generate hard concrete samples, leads to poor results, including accuracy akin to random guessing at sparsity above $90\%$; this is also reported by \cite{stateofsparsity}. Instead, at epoch 160, we sample different masks and commit to the one that performs the best -- this strategy results in drastic improvements at high sparsity levels. This suggests that the gap between training and inference behavior introduced by stochastic approaches can be an obstacle. Although our modification improves results for $\ell_0$, the method still performs poorly compared to alternatives.

Moreover, some methods required modifications as they were originally designed to perform structured pruning. For AMC, Slim, and $\ell_0$ we replace a filter-wise mask by one that acts over weights. Since Network Slimming relies on the filter-wise scaling factors of batch norm, we introduce weight-wise scaling factors which are trained jointly with the weights. We observe that applying both $\ell_1$ and $\ell_2$ regularization to the scaling parameters, as done by \cite{slim}, yields inferior performance, which we attribute to over-regularization. A grid search over the penalty of each norm regularizer shows that only applying $\ell_1$ regularization with a strength of $\lambda_1 = 10^{-5}$ for ResNet-20 and $\lambda_1 = 10^{-6}$ for VGG-16 improves results.

Figures~\ref{fig-cs-sparsevgg} and~\ref{fig-cs-sparseresnet} display one-shot pruning results. On VGG, only CS and Slim successfully prune over $98\%$ of the weights without severely degrading the performance of the model, while on ResNet the best results are achieved by CS and GMP.

Table~\ref{tab-cs-sparsecifar} shows the percentage of weights that each method can remove while maintaining a performance within $2\%$ of the original, dense model. CS is capable of removing significantly more parameters than all competing methods on both networks: on ResNet-20, the pruned network found by CS contains $60\%$ less parameters than the one found by GMP, when counting prunable parameters only. CS not only offers significantly superior performance compared to the prior $\ell_0$-based method of \cite{sparsityl0}, but also comfortably outperforms all other methods, providing a new state-of-the-art for network pruning.

\begin{figure}[t]
    \centering
    \includegraphics[width=0.8\linewidth]{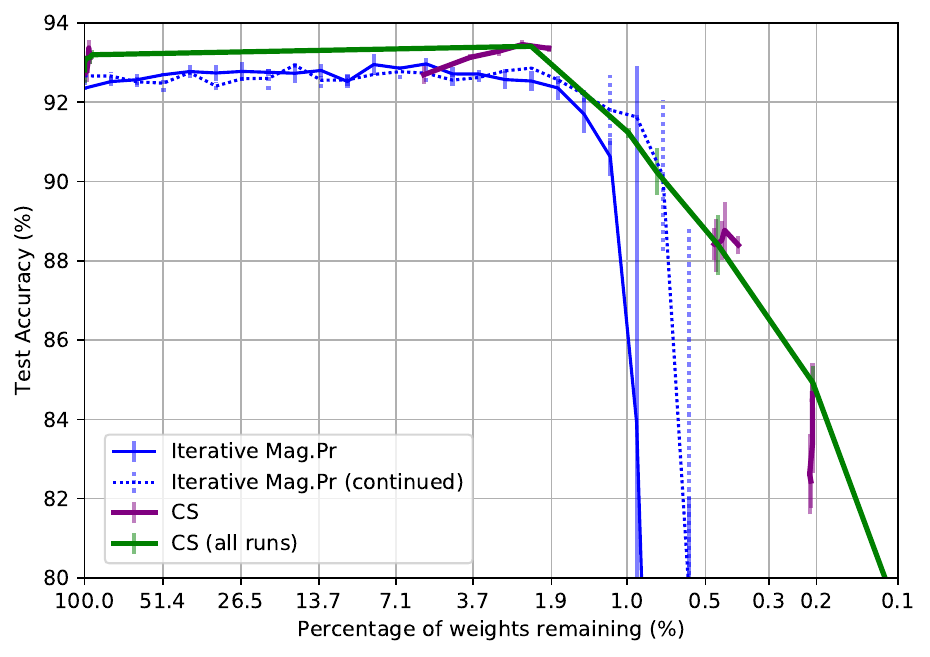}
    \caption{VGG-16 on CIFAR-10. Test accuracy and sparsity of subnetworks produced by IMP and CS after re-training from weights of epoch 2. Purple curves show individual runs of CS, while the green curve connects tickets produced after 5 rounds of CS with varying $s^{(0)}$. Iterative Magnitude Pruning (continued) refers to IMP without rewinding between rounds. Error bars depict variance across 3 runs.}
    \label{fig-cs-ticketvgg}
\end{figure}

\begin{figure}[t]
    \centering
    \includegraphics[width=0.8\linewidth]{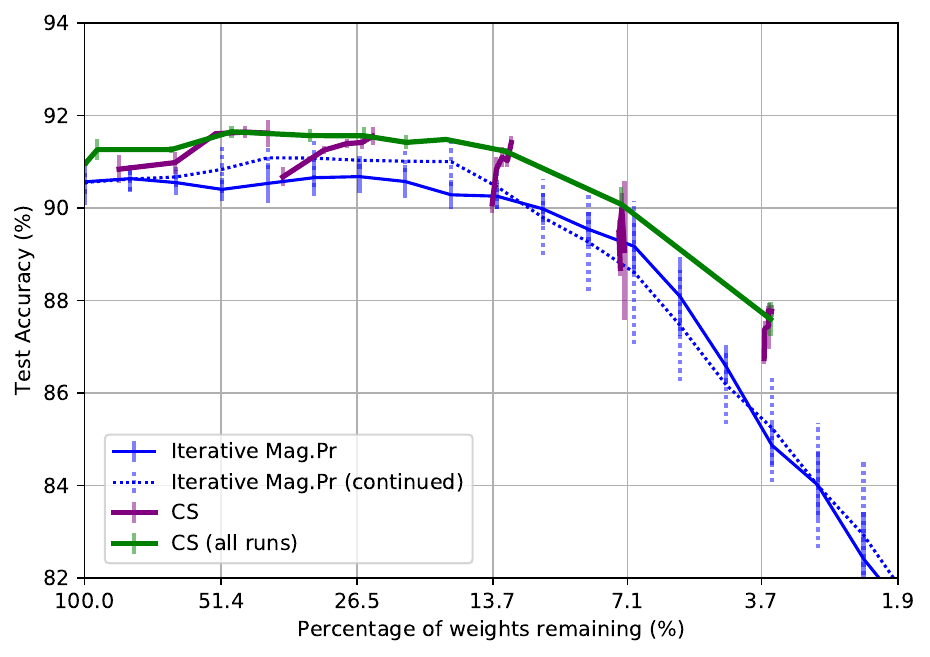}
    \caption{ResNet-20 on CIFAR-10. Test accuracy and sparsity of subnetworks produced by IMP and CS after re-training from weights of epoch 2. Purple curves show individual runs of CS, while the green curve connects tickets produced after 5 rounds of CS with varying $s^{(0)}$. Iterative Magnitude Pruning (continued) refers to IMP without rewinding between rounds. Error bars depict variance across 3 runs.}
    \label{fig-cs-ticketresnet}
\end{figure}

\begin{table}[t]
\centering
\begin{tabular}{@{}clcrrrcrrr@{}}
\toprule
\multicolumn{2}{c}{} && \multicolumn{3}{c}{VGG-16} && \multicolumn{3}{c}{ResNet-20}\\[-2pt]
\cmidrule{4-6} \cmidrule{8-10}
\multicolumn{2}{c}{\multirow{2}{*}{Method}} && \multirow{2}{*}{Round} & \multicolumn{1}{c}{Test}     & Weights   && \multirow{2}{*}{Round} & \multicolumn{1}{c}{Test}     & Weights\\[-1pt]
\multicolumn{2}{c}{}                        &&                        & \multicolumn{1}{c}{Accuracy} & Remaining &&                        & \multicolumn{1}{c}{Accuracy} & Remaining\\
\midrule
\multicolumn{2}{c}{Dense Network} &&  1 & 92.35\%          & 100.0\%         && 1 & 90.55\%          & 100.0\%\\
\\[-9pt]\hdashline\\[-8pt]
\multirow{1}{*}{Sparsest~}   & ~IMP   && 18 & 92.36\%          & 1.8\%           && 7 & 90.57\%          & 20.9\% \\
\multirow{1}{*}{Matching~}   & ~IMP-C && 18 & 92.56\%          & 1.8\%           && 8 & 91.00\%          & 16.7\% \\
\multirow{1}{*}{Subnetwork~}     & ~CS    &&  5 & 93.35\%          & \textbf{1.7\%}  && 5 & 91.43\%          & \textbf{12.3\%}\\
\\[-9pt]\hdashline\\[-8pt]
\multirow{1}{*}{Best~}       & ~IMP   && 13 & 92.97\%          & 5.5\%           && 6 & 90.67\%          & 26.2\% \\
\multirow{1}{*}{Performing~} & ~IMP-C && 12 & 92.77\%          & 6.9\%           && 4 & 91.08\%          & 40.9\% \\
\multirow{1}{*}{Subnetwork~}     & ~CS    &&  4 & \textbf{93.45\%} & 2.4\%           && 5 & \textbf{91.54\%} & 16.9\% \\[-2pt]
\bottomrule
\end{tabular}
\caption{Test accuracy and sparsity of the sparsest matching and best performing subnetworks produced by CS, IMP, and IMP-C (IMP without rewinding) for VGG-16 and ResNet-20 trained on CIFAR-10.}
\label{tab-cs-ticket}
\end{table}

\subsection{Finding Winning Tickets}
\label{sec-cs-results-ticket}

% \subsection{Ticket Search on Residual Networks and VGG}
% \label{sec:resnet}

Next, we evaluate how IMP and CS perform on the task of ticket search for VGG-16 \citep{vgg} and ResNet-20\footnote{We used the same network as \citet{lth} and \citet{lth2}, who refer to it as ResNet-18.} \citep{resnet1} trained on the CIFAR-10 dataset, a setting where IMP can take over $10$ rounds (850 epochs given 85 epochs per round~\citep{lth2}) to find sparse subnetworks.  We evaluate produced subnetworks by initializing their weights with the iterates from the end of epoch 2, similarly to \citet{lth2}, followed by re-training.

Figures~\ref{fig-cs-ticketvgg} and~\ref{fig-cs-ticketresnet} show the performance and sparsity of tickets produced by CS and IMP, including IMP without rewinding (continued). Purple curves show individual runs of CS for different values of $s_{init}$, each consisting of 5 rounds, and the green curve shows the performance of subnetworks produced with different hyperparameters. Plots of individual runs are available in Section~\ref{sec-cs-analysis}, but have been omitted here for the sake of clarity. Given a search budget of 5 rounds (\ie $5 \times 85 = 425$ epochs), CS successfully finds subnetworks with diverse sparsity levels. Notably, IMP produces tickets with superior performance when weight rewinding is not employed between rounds.

Table~\ref{tab-cs-ticket} summarizes the performance of each method when evaluated in terms of the sparsest matching and best performing subnetworks. IMP-C denotes IMP without rewinding, \ie IMP (continued) from Figures~\ref{fig-cs-ticketvgg} and~\ref{fig-cs-ticketresnet}. Sparsest matching subnetworks produced by CS are sparser than the ones found by IMP and IMP-C, while also delivering higher accuracy. CS also outperforms IMP and IMP-C when evaluating the best performing produced subnetworks. In particular, CS yields highly sparse subnetworks that outperform the original model by approximately $1\%$ on both VGG-16 and ResNet-20.

If all runs are executed in parallel, producing all tickets presented in Figure~\ref{fig-cs-sparseresnet}  takes CS a total of $5 \times 85=425$ training epochs, while IMP requires $30 \times 85 = 2550$ epochs instead. Note that our re-parameterization results in approximately $10\%$ longer training times on a GPU due to the mask parameters $s$, therefore wall-clock time for CS is $10\%$ higher per epoch. Sequential search takes $5 \times 11 \times 85 = 4675$ epochs for CS to produce all tickets, while IMP requires $30 \times 85 = 2550$ epochs, hence CS is faster given sufficient parallelism, but slower if run sequentially. Section~\ref{sec-cs-analysis} shows preliminary results of a variant of CS designed for sequential search.

% \subsection{Pruning}
% \label{seq-pruning}

% % To evaluate each method when finding masks with different sparsity levels, we run IMP with global pruning rates ranging from $50\%$ to $99.75\%$, and ISS and \methodacro~with initial mask values varying from $-0.5$ to $0$. Results are shown in Figure \ref{fig:pruning}. On VGG, both IMP and ISS fail at removing over $98\%$ of the weights without severely degrading the performance of the model, while \method~achieves $99.7\%$ sparsity in the convolutional layers while still yielding over $90\%$ test accuracy. When taken to the extreme, our method is capable of removing $99.85\%$ of the weights of VGG while still yielding over $83\%$ accuracy. On ResNet-20, at the $90\%$ sparsity level, CS achieves over $7.5\%$ higher accuracy compared to IMP, underperforming the dense network by only $1.2\%$.

% % The dramatic performance difference between ISS -- which, in this one-shot pruning setting, is similar to the methods in \citet{sparsityl0} and \citet{l0bernoulli} -- and CS shows that our proposed deterministic re-parameterization is key to achieve superior results in both network pruning and ticket search. The fact that it outperforms magnitude pruning, a standard technique in the pruning literature, suggests that further exploration of $\ell_0$-based methods could yield significant advances in pruning techniques.

\subsection{Residual Networks on ImageNet}
% \label{sec:imagenet}

We perform pruning and ticket search for ResNet-50 trained on ImageNet~\cite{imagenet}. Once the round is complete, we evaluate the performance of the produced subnetwork when fine-tuned (pruning) or re-trained from an early iterate (ticket search).

% \newlength{\oldcolumnsep}
% \newlength{\oldintextsep}
% \setlength{\oldcolumnsep}{\columnsep}
% \setlength{\oldintextsep}{\intextsep}
% \setlength\columnsep{15pt}
% \setlength\intextsep{-2pt}

\Needspace{0.45\textheight} % adjust if needed
\noindent
\begin{minipage}[t]{0.53\textwidth}
Table~\ref{tab-cs-imagenet} summarizes the results achieved by CS, IMP, and current state-of-the-art pruning methods GMP~\cite{gmp}, STR~\cite{softweight}, and DNW~\cite{dnw}. A $^\dagger$ superscript denotes results of a re-trained, rather than fine-tuned, subnetwork. Differences in each technique's methodology -- for example, the adopted learning rate schedule and number of epochs -- complicate the comparison.

CS produces subnetworks that, when re-trained, outperform the ones found by IMP by a comfortable margin (compare CS$^\dagger$ and IMP$^\dagger$). Moreover, when evaluated as a pruning method, CS outperforms all competing approaches, especially in the high-sparsity regime. Therefore, our method provides state-of-the-art results whether the network is fine-tuned (pruning) or re-trained (ticket search).
\end{minipage}\hfill
\begin{minipage}[t]{0.4\textwidth}
\renewcommand{\arraystretch}{0.75}
% \begin{table}[t]
\centering
\captionsetup{type=table}
\captionof{table}{Performance of sparsification on ResNet-50 trained on ImageNet. $^\dagger$ denotes performance of ticket search.}
\label{tab-cs-imagenet}
\begin{tabular}{@{}lrr@{}}
\toprule
 Method & \begin{tabular}[x]{@{}r@{}}Top-1\\Acc.\end{tabular} & Sparsity \\ \midrule
 GMP & 73.9\% & 90.0\%  \\ 
 DNW & 74.0\% & 90.0\%  \\ 
 STR & 74.3\% & 90.2\%  \\ \hdashline
 
 IMP$^\dagger$ & 73.6\% & 90.0\%  \\
 CS$^\dagger$ & \textbf{75.5\%} &  91.8\% \\ \midrule
 
 GMP & 70.6\% & 95.0\%  \\ 
 DNW & 68.3\% & 95.0\%  \\ 
 STR & 70.4\% & 95.0\%  \\ 
 CS & \textbf{72.4\%} &  95.3\% \\\hdashline
 
 IMP$^\dagger$ & 69.2\% & 95.0\%  \\ 
 CS$^\dagger$ & \textbf{71.1\%} &  95.3\% \\ \midrule
 
 STR & 67.2\% & 96.5\%  \\
 CS & \textbf{71.4\%} &  97.1\% \\ \hdashline
 
 CS$^\dagger$ & 69.6\% &  97.1\% \\ \midrule
 
 GMP & 57.9\% & 98.0\%  \\ 
 DNW & 58.2\% & 98.0\%  \\ 
 STR & 61.5\% & 98.5\%  \\
 CS & \textbf{70.0\%} &  98.0\% \\ \hdashline
 
 CS$^\dagger$ & 67.9\% &  98.0\% \\ \midrule
 
 GMP & 44.8\% & 99.0\%  \\ 
 STR & 54.8\% & 98.8\%  \\
 CS & \textbf{66.8\%} &  98.9\% \\ \hdashline
 
 CS$^\dagger$ & 64.9\% &  98.9\% \\ \bottomrule
\end{tabular}
\renewcommand{\arraystretch}{1.0}
\end{minipage}
\section{Experimental Analysis}
\label{sec-cs-analysis}

\subsection{Hyperparameter Analysis}

\paragraph{Continuous Sparsification.}

In this section, we study how the hyperparameters of CS affect its behavior in terms of sparsity and performance of the produced tickets. More specifically, we consider the following hyperparameters:

\begin{itemize}
    \item Final temperature $\beta_{final}$: the final value for $\beta$, which controls how close to the original $\ell_0$-regularized problem the proxy objective $L_\beta(w,s)$ is.
    \item $\ell_1$ penalty $\lambda$: the strength of the $\ell_1$ regularization applied to the soft mask $\sigma(\beta s)$, which promotes sparsity.
    \item Mask initial value $s_{init}$: the value used to initialize all components of the soft mask $m = \sigma(\beta s)$, where smaller values promote sparsity.
\end{itemize}

Our setup is as follows. To analyze how each of the 3 hyperparameters impact the performance of Continuous Sparsification, we train ResNet-20 on CIFAR-10, varying one hyperparameter while keeping the other two fixed. To capture how hyperparameters interact with each other, we repeat the described experiment with different settings for the fixed hyperparameters.

Since different hyperparameter settings naturally yield vastly distinct sparsity and performance for the found tickets, we report relative changes in accuracy and in sparsity.

In Figure~\ref{fig:hyper1}, we vary $\lambda$ between $0$ and $10^{-8}$ for three different $(s_{init}, \beta_{final})$ settings: $(s_{init}=-0.2, \beta_{final}=100)$, $(s_{init}=0.05, \beta_{final}=200)$, and $(s_{init}=-0.3,\beta_{final}=100)$. As we can see, there is little impact on either the performance or the sparsity of the found ticket, except for the case where $s_{init}=0.05$ and $\beta_{final}=200$, for which $\lambda=10^{-8}$ yields slightly increased sparsity.

\begin{figure}[t]
    \centering
     \includegraphics[width=\linewidth]{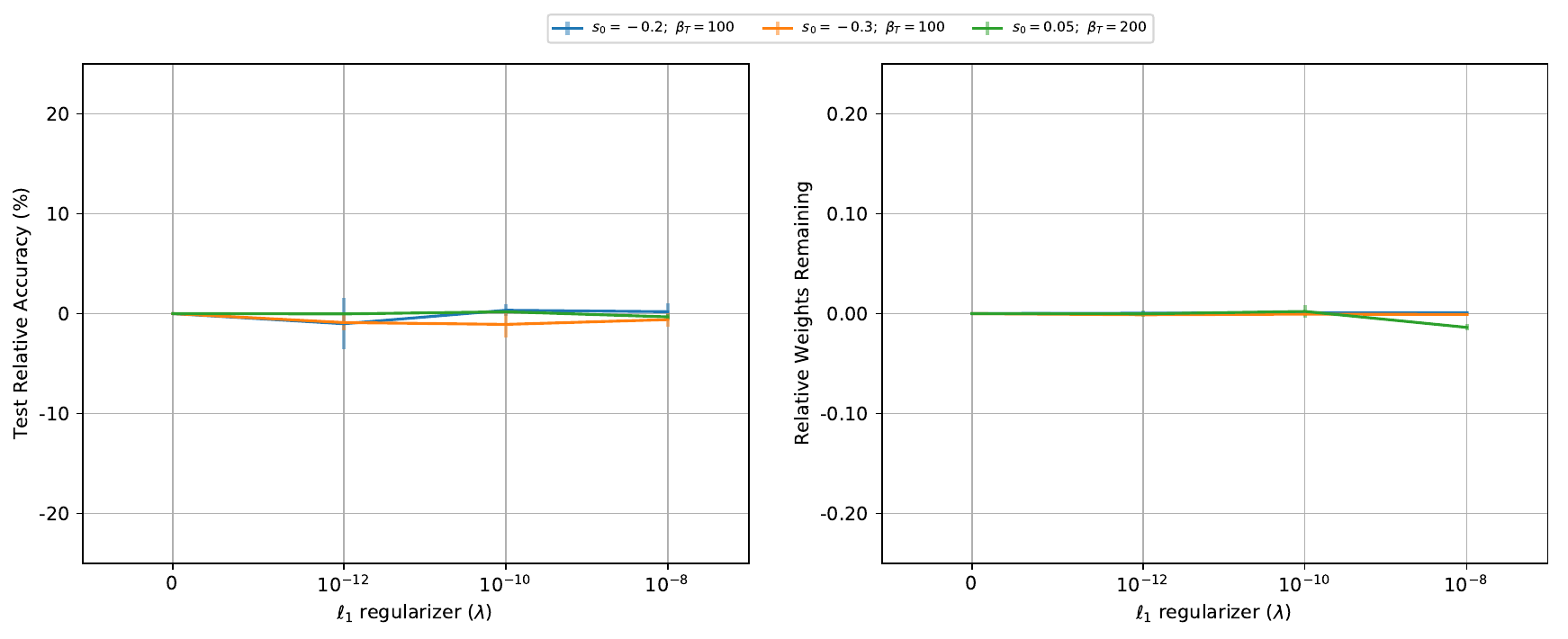}
    \caption{Impact on relative test accuracy and sparsity of tickets found by CS in a ResNet-20 trained on CIFAR-10, for different values of $\lambda$ and fixed settings for $\beta_{final}$ and $s_{init}$.}
    \label{fig:hyper1}
\end{figure}

\begin{figure}[t]
    \centering
     \includegraphics[width=\linewidth]{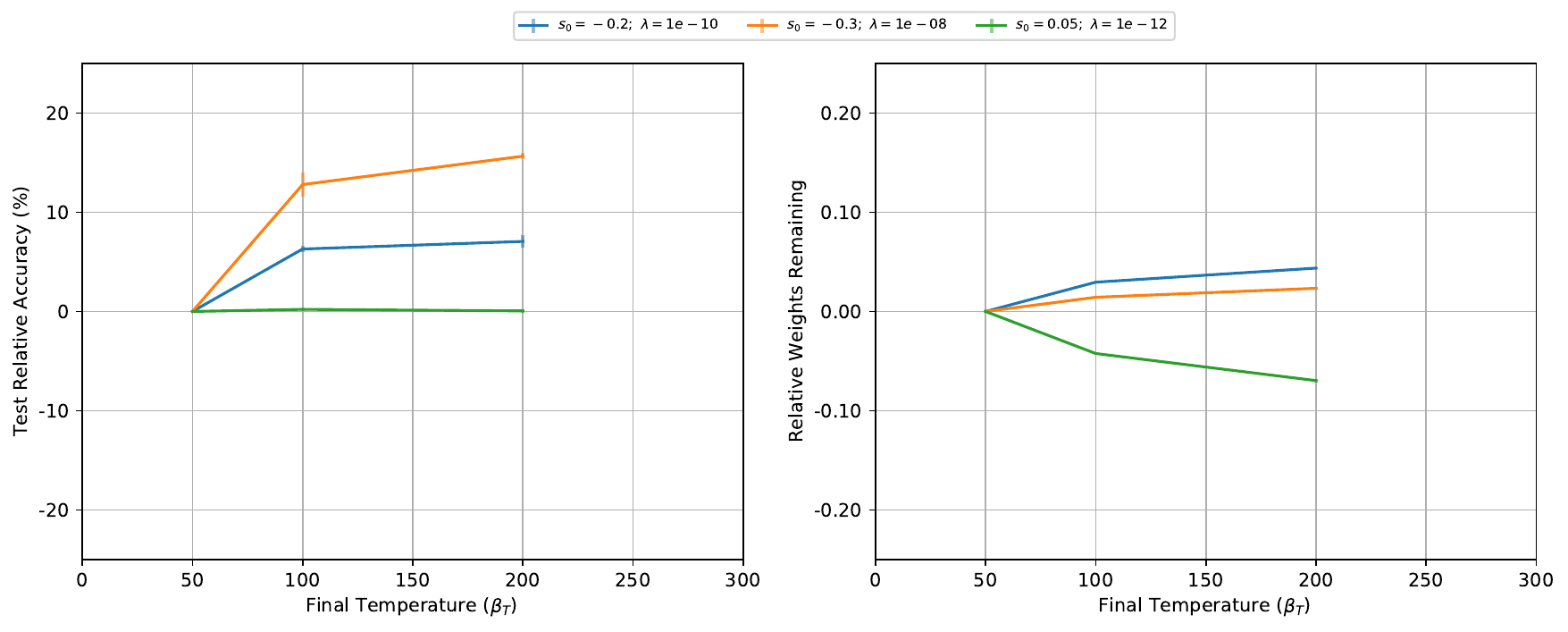}
    \caption{Impact on relative test accuracy and sparsity of tickets found by CS in a ResNet-20 trained on CIFAR-10, for different values of $\beta_{final}$ and fixed settings for $\lambda$ and $s_{init}$.}
    \label{fig:hyper2}
\end{figure}

\begin{figure}[t]
    \centering
     \includegraphics[width=\linewidth]{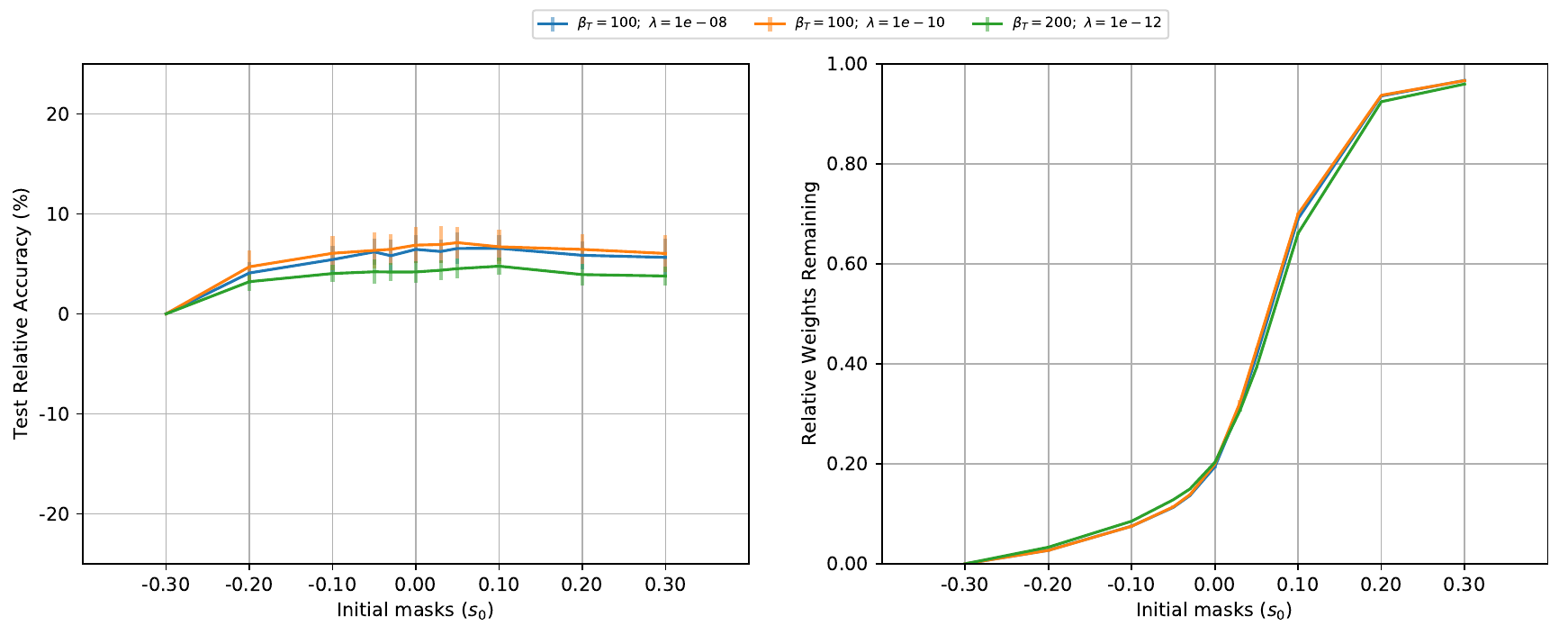}
    \caption{Impact on relative test accuracy and sparsity of tickets found by CS in a ResNet-20 trained on CIFAR-10, for different values of $s_{init}$ and fixed settings for $\beta_{final}$ and $\lambda$.}
    \label{fig:hyper3}
\end{figure}

Next, we consider the fixed settings $(s_{init}=-0.2, \lambda=10^{-10})$, $(s_{init}=0.05, \lambda=10^{-12})$, $(s_{init}=-0.3,\lambda=10^{-8})$, and proceed to vary the final inverse temperature $\beta_{final}$ between 50 and 200. Figure~\ref{fig:hyper2} shows the results: in all cases, a larger $\beta$ of $200$ yields better accuracy. However, it decreases sparsity compared to smaller temperature values for the settings $(s_{init}=-0.2, \lambda=10^{-10})$ and $(s_{init}=-0.3,\lambda=10^{-8})$, while at the same time increasing sparsity for $(s_{init}=0.05, \lambda=10^{-12})$. While larger $\beta$ appear beneficial and might suggest that even higher values should be used, note that, the larger $\beta_{final}$ is, the earlier in training the gradients of $s$ will vanish, at which point training of the mask will stop. Since the performance for temperatures between 100 and 200 does not change significantly, we recommend values around 150 or 200 when either pruning or performing ticket search.

Lastly, we vary the initial mask value $s_{init}$ between $-0.3$ and $+0.3$, with hyperpameter settings $(\beta_{final}=100, \lambda=10^{-10})$, $(\beta_{final}=200, \lambda=10^{-12})$, and $(\beta_{final}=100, \lambda=10^{-8})$. Results are given in Figure~\ref{fig:hyper3}: unlike the exploration on $\lambda$ and $\beta_{final}$, we can see that $s_{init}$ has a strong and consistent effect on the sparsity of the found tickets. For this reason, we suggest proper tuning of $s_{init}$ when the goal is to achieve a specific sparsity value. Since the percentage of remaining weights is monotonically increasing with $s_{init}$, we can employ search strategies over values for $s_{init}$ to achieve pre-defined desired sparsity levels (\eg binary search). In terms of performance, lower values for $s_{init}$ naturally lead to performance degradation, since sparsity quickly increases as $s_{init}$ becomes more negative.

\begin{figure}[t]
    \centering
    \includegraphics[width=0.8\linewidth]{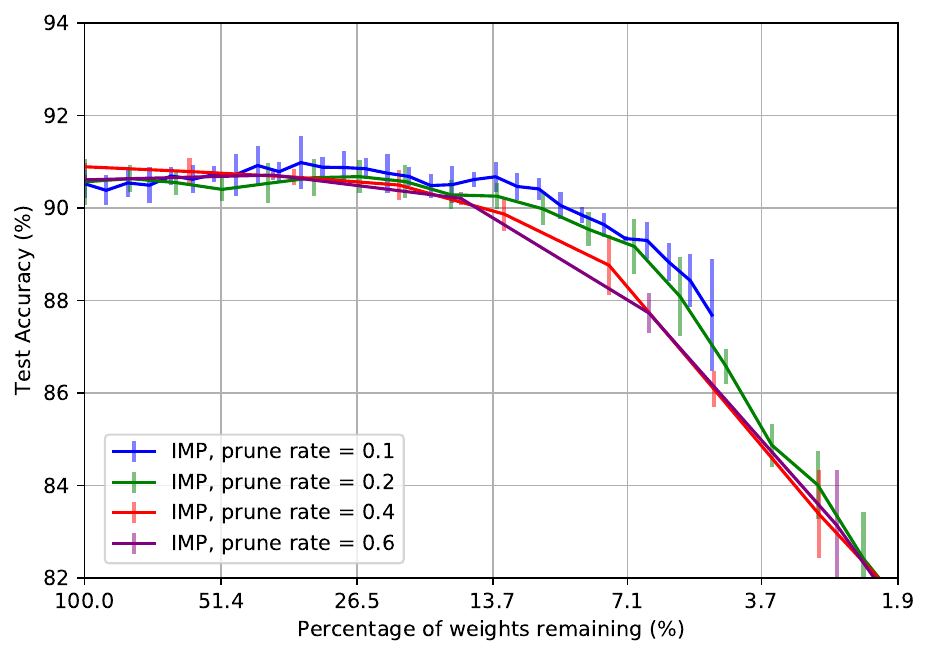}
    \caption{Performance of tickets found by Iterative Magnitude Pruning in a ResNet-20 trained on CIFAR-10, for different pruning rates.}
    \label{fig:hyper4}
\end{figure}

\paragraph{Iterative Magnitude Pruning.}

Here, we assess whether the running time of Iterative Magnitude Pruning can be improved by increasing the amount of parameters pruned at each iteration. The goal of this experiment is to evaluate if better tickets (both in terms of performance and sparsity) can be produced by more aggressive pruning strategies.

Following the same setup as the previous section, we train ResNet-20 on CIFAR-10. We run IMP for 30 iterations, performing global pruning with different pruning rates at the end of each iteration. Figure~\ref{fig:hyper4} shows that the performance of tickets found by IMP decays when the pruning rate is increased to $40\%$. In particular, the final performance of found tickets is mostly monotonically decreasing with the number of remaining parameters, suggesting that, in order to find tickets which outperform the original network, IMP is not compatible with more aggressive pruning rates.

\subsection{Iterative Stochastic Sparsification}

\begin{algorithm}[t]
    \caption{Iterative Stochastic Sparsification (inspired by \citet{deconstructing})}
    \textbf{Input:} Architecture $f$, $s_{init} \in \R$, $\lambda \geq 0$,  $T \in \N$, $R \in \N$, $k \in \N$
    \label{alg-cs-isp}
    \begin{algorithmic}[1]
    \State Initialize $\theta \sim \dist_\theta$, $s \gets \vec s_{init}$
    \For{$r = 1$ to $R$}
        \State If $r>1$: set $s \gets -\infty$ for components where $s < s_{init}$
        \For{$t = 1$ to $T$}
            \State Estimate $\mathcal L(\theta, s) = \expec{m \sim \text{Ber}(\sigma(s))}{\mathcal L(f(m \odot \theta))} + \lambda \|\sigma(\beta s)\|_1$ and $\nabla_{\theta, s} \mathcal L(\theta,s)$
            \State Update $\theta$ and $s$ using $\nabla_\theta \mathcal L(\theta,s)$ and $\nabla_s \mathcal L(\theta,s)$
            \State If $r=1$ and $t=k$: store $\hat \theta \gets \theta$
        \EndFor
    \EndFor
    \State Sample $m \sim \text{Ber}(\sigma(s))$
    \State Output $f(m \odot \hat \theta)$
    \end{algorithmic}
\end{algorithm}

Besides comparing our proposed method to Iterative Magnitude Pruning (Algorithm~\ref{alg-cs-impticket}), we also design a baseline method, Iterative Stochastic Sparsification (ISS, Algorithm~\ref{alg-cs-isp}), motivated by the procedure in~\cite{deconstructing} to find a binary mask $m$ with gradient descent in an end-to-end fashion. More specifically, ISS uses a stochastic re-parameterization $m \sim \text{Bernoulli}(\sigma(s))$ with $s \in \R^d$, and trains $w$ and $s$ jointly with gradient descent and the straight-through estimator~\citep{straightthrough}. Note that the method is also similar to the one proposed by~\cite{l0bernoulli} to prune networks. The goal of this baseline and comparisons is to evaluate whether the deterministic nature of CS's re-parameterization is advantageous when performing sparsification through optimization methods.

When run for multiple iterations, all components of the mask parameters $s$ which have decreased in value from initialization are set to $-\infty$, such that the corresponding weight is permanently removed from the network. While this might look arbitrary, we observe empirically that ISS was unable to remove weights quickly without this step unless $\lambda$ is chosen to be large -- in which case the model's performance degrades due to increased sparsity.

We also observe that the mask parameters $s$ require different settings in terms of optimization to be successfully trained. In particular, \cite{deconstructing} use SGD with a learning rate of 100 when training $s$, which is orders of magnitude larger than the one used when training CNNs. Our observations are similar, in that typical learning rates on the order of 0.1 cause $s$ to be barely updated during training, which is likely a side-effect of using gradient estimators to obtain update directions for $s$. The following sections present experiments that compare IMP, CS and ISS on ticket search tasks.

\subsection{Searching for Supermasks}

\begin{figure}[t]
    \centering
      \includegraphics[width=0.8\linewidth]{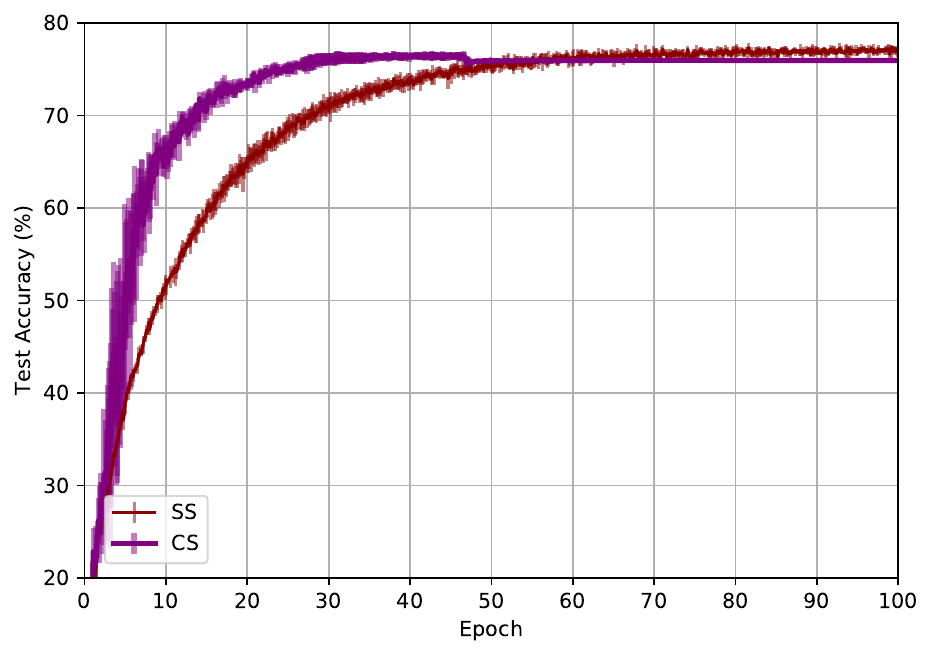}
    \caption{Learning a binary mask with weights frozen at initialization with Stochastic Sparsification (SS, Algorithm \ref{alg-cs-isp} with one iteration) and Continuous Sparsification (CS), on a 6-layer CNN on CIFAR-10. Training curves with hyperparameters for which masks learned by SS and CS were both approximately $50\%$ sparse. CS learns the mask significantly faster while attaining similar early-stop performance.}
    \label{fig:supermaskconv6}
\end{figure}

\begin{figure}[t]
    \centering
      \includegraphics[width=0.8\linewidth]{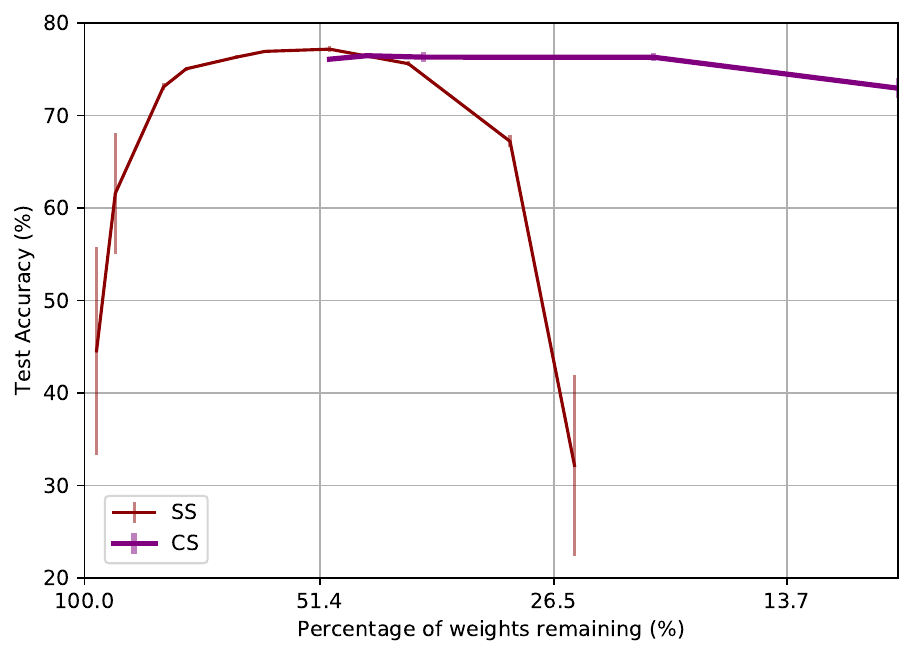}
    \caption{Learning a binary mask with weights frozen at initialization with Stochastic Sparsification (SS, Algorithm \ref{alg-cs-isp} with one iteration) and Continuous Sparsification (CS), on a 6-layer CNN on CIFAR-10. Sparsity and test accuracy of masks learned with different settings for SS and {CS}: our method learns sparser masks while maintaining test performance, while SS is unable to successfully learn masks with over $50\%$ sparsity.}
    \label{fig:supermaskconv6_2}
\end{figure}

We train a neural network with 6 convolutional layers on the CIFAR-10 dataset~\citep{cifar}, following~\cite{lth}. As a first baseline, we consider the task of learning a ``supermask'' \citep{deconstructing}: a binary mask $m$ that aims to maximize the performance of a network with randomly initialized weights once the mask is applied. This task is equivalent to pruning a randomly-initialized network since weights are neither updated during the search for the supermask, nor for the comparison between different methods.

We only compare ISS and CS for this specific experiment: the reason not to consider IMP is that, since the network weights are kept at their initialization values, IMP amounts to removing the weights whose initialization were the smallest. 

Hence, we compare ISS and CS, where each method is run for a single round composed of $100$ epochs. In this case, where it is run for a single round, ISS is equivalent to the algorithm proposed in \cite{deconstructing} to learn a supermask, referred here as simply Stochastic Sparsification (SS). We control the sparsity of the learned masks by varying $s_{init}$ and $\lambda$.

Figure~\ref{fig:supermaskconv6} presents results: CS is capable of finding high performing sparse supermasks (\ie $25\%$ or less remaining weights while yielding $75\%$ test accuracy), while SS fails at finding competitive supermasks for sparsity levels above $50\%$. Moreover, CS makes faster progress in training, suggesting that not relying on gradient estimators indeed results in better optimization and faster progress when measured in epochs or parameter updates.

\begin{figure}[t]
    \centering
    \includegraphics[width=0.8\linewidth]{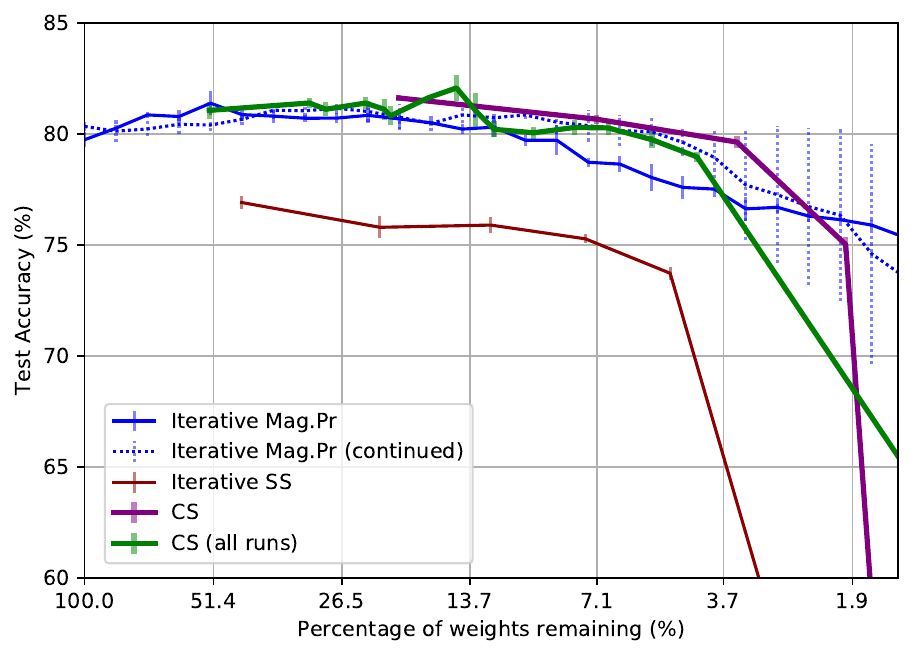}
    \caption{Accuracy and sparsity of tickets produced by IMP, ISS and CS after re-training, starting from initialization. Tickets are extracted from a Conv-6 network trained on CIFAR-10. Purple curves show individual runs of CS, while green curve connects tickets produced after 4 rounds of CS with varying $s_{init}$. Blue and red curves show performance and sparsity of tickets produced by IMP and ISS, respectively. Error bars depict variance across 3 runs.}
    \label{fig:conv6}
\end{figure}

\subsection{Additional Ticket Search Experiments}
\label{sec-cs-analysis-addticket}

\paragraph{6-layer CNN.}

In what follows we compare IMP, ISS and CS in the task of finding winning tickets on the Conv-6 architecture used in the supermask experiments in the previous section. The goal of these experiments is to assess how our deterministic re-parameterization compares to the common stochastic approximations to $\ell_0$-regularization~\citep{l0bernoulli, sparsityl0, deconstructing}. Therefore, we run CS \textbf{with weight rewinding} between rounds, so that we remove any advantages that might be caused by not performing weight rewinding -- in this case, we better isolate the effects caused by our re-parameterization. Following~\cite{lth}, we re-train the produced tickets from their values at initialization (\ie $k=0$ on each algorithm).

We run IMP and ISS for a total of 30 rounds, each consisting of 40 epochs. For IMP, we use pruning rates of $15\%/20\%$ for convolutional/dense layers. We initialize the Bernoulli parameters of ISS with $s_{init} = \vec 1$ and adopt $\ell_1$ regularization of $\lambda = 10^{-8}$. Each run of CS is limited to $4$ rounds, and we perform a total of 16 runs, each with a different value for the mask initialization $s_{init}$, from $-0.2$ up to $0.1$. Runs are repeated with 3 different random seeds so that error bars can be computed.

Figure~\ref{fig:conv6} presents tickets produced by each method, measured by their sparsity and test accuracy when trained from scratch. Even when performing weight rewinding, CS produces tickets that are significantly superior than the ones found by ISS, both in terms of sparsity and test accuracy, showing that our deterministic re-parameterization is fundamental to finding winning tickets.

\paragraph{Learned Sparsity Structure}

\begin{figure}[t]
    \centering
     \includegraphics[width=0.8\linewidth]{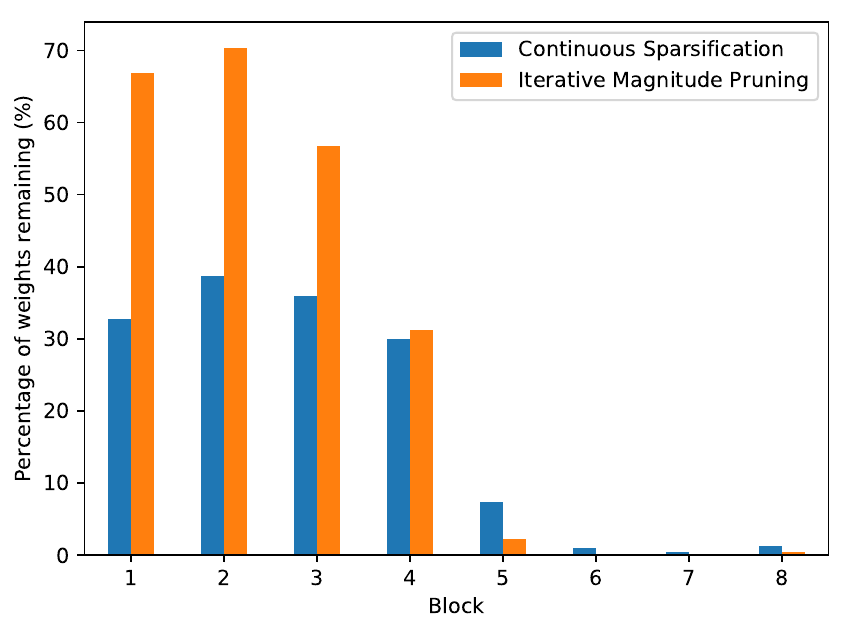}
    \caption{Sparsity patterns learned by CS and IMP for VGG-16 trained on CIFAR-10 -- each block consists of 2 non-overlapping consecutive layers of VGG.}
    \label{fig:learnedsparsities}
\end{figure}

To see how CS differs from magnitude pruning in terms of which layers are more heavily pruned by each method, we force the two to prune VGG to the same sparsity level in a single round. We first run CS with $s_{init}=0$, yielding $94.19\%$ sparsity, and then run IMP with global pruning rate of $94.19\%$, producing a sub-network with the same number of parameters.

Figure~\ref{fig:learnedsparsities} shows the final sparsity of blocks consisting of two consecutive convolutional layers (8 blocks total since VGG has 16 convolutional layers). CS applies a pruning rate that is roughly twice as aggressive as IMP to the first blocks. Both methods heavily sparsify the widest layers of VGG (blocks 5 to 8), while still achieving over $91\%$ test accuracy. More heavily pruning earlier layers in CNNs can offer inference speed benefits: due to the increased spatial size of earlier layers' inputs, each weight is used more times and has a larger contribution in terms of FLOPs.

\subsection{Sequential Search with Continuous Sparsification}

\begin{figure}[t]
    \centering
     \includegraphics[width=0.8\linewidth]{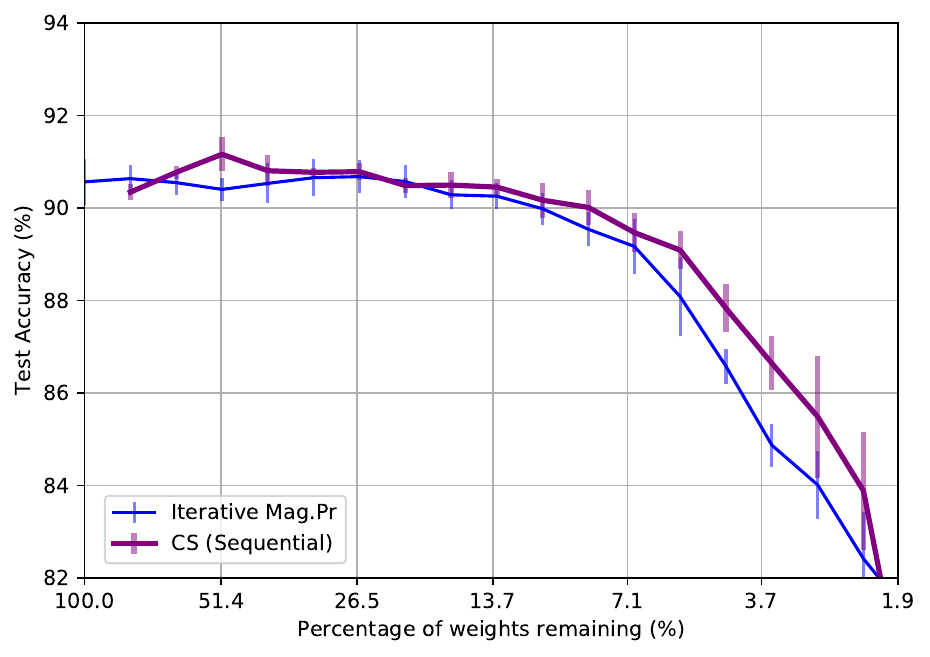}
    \caption{Accuracy and sparsity of tickets produced by IMP and Sequential CS after re-training, starting from weights of epoch 2. Tickets are extracted from a ResNet-20 trained on CIFAR-10.}
    \label{fig:sequential}
\end{figure}

There might be cases where the goal is either to find a ticket with a specific sparsity value or to produce a set of tickets with varying sparsity levels in a single run -- tasks that can be naturally performed with a single run of Iterative Magnitude Pruning. However, CS has no explicit mechanism to control the sparsity of the produced tickets, and, as shown in Sections~\ref{sec-cs-results-ticket} and ~\ref{sec-cs-analysis-addticket}, CS quickly sparsifies the network in the first few rounds and then roughly maintains the number of parameters during the following rounds until the end of the run. In this scenario, IMP has a clear advantage, as a single run suffices to produce tickets with varying, pre-defined sparsity levels.

Here, we present a sequential variant of CS, named Sequential Continuous Sparsification, that removes a fixed fraction of the weights at each round, hence being better suited for the task described above. Unlike IMP, this sequential form of CS removes the weights with lowest mask values $s$ -- note the difference from CS, which, given a large enough temperature $\beta$, removes \emph{all} weights whose corresponding mask parameters are negative.

Following the same experimental protocol from Section~\ref{sec-cs-results-ticket}, we again perform ticket search on ResNet-20 trained on CIFAR-10.  We run Sequential CS and Iterative Magnitude Pruning for a total of 30 rounds each, and with a pruning rate of $20\%$ per round. Note that unlike the experiments with CS (the non-sequential form), we perform a single run with $s_{init} = 0$, \ie no hyperparameters are used to control the sparsity of the produced tickets.

Figure~\ref{fig:sequential} shows the performance of tickets produced by Sequential CS and IMP, indicating that CS might be a competitive method in the sequential search setting. Note that the performance of the tickets produced by Sequential CS is considerably inferior to those found by CS (refer to Section \ref{sec-cs-results-ticket}, Figures \ref{fig-cs-ticketvgg} and \ref{fig-cs-ticketresnet}). Although these results are promising, additional experiments would be required to more thoroughly evaluate the potential of Sequential CS and its comparison to Iterative Magnitude Pruning.
\section{Discussion}
\label{sec-cs-discussion}

In this chapter, we have introduced a novel sparsification method -- Continuous Sparsification -- that approximates the intractable $\ell_0$ regularization objective in a fully deterministic manner. Unlike heuristic-based pruning approaches and stochastic $\ell_0$ approximations, our method leverages a smooth continuation to gradually remove weights from a network. This approach not only integrates sparsification directly into the training process but also offers a formal trade-off between model performance and compression.

\subsection{Contributions}
\label{sec-cs-discussion-contribs}

This chapter makes several key contributions:
\begin{itemize}
    \item We propose a new approximation for the intractable network sparsification problem. Unlike previous works, we design a fully deterministic approximation for $\ell_0$ regularization, which allows for gradient-based optimization without gradient estimators or additional variance.
    \item We present empirical evidence that our method is capable of achieving high sparsity levels while maintaining, and in some cases even improving, the predictive performance of dense models trained on CIFAR-10 and ImageNet. 
    \item Extending our sparsification method to perform ticket search, we successfully discover winning tickets on ResNets trained on CIFAR-10 and on ImageNet at a fraction of the time taken by Iterative Magnitude Pruning. 
\end{itemize}

\subsection{Research Directions}
\label{sec-cs-discussion-directions}

One potential research direction is integrating Continuous Sparsification with architecture search methods. Combining our sparsification framework with NAS or adaptive network design could enable simultaneous optimization of both network topology and sparsity, leading to even more efficient models tailored for specific tasks or hardware constraints. Finally, further theoretical analysis of our deterministic approximation and its convergence properties may provide deeper insights into the trade-offs between sparsity and performance. Such analysis could help direct modifications to Continuous Sparsification to guarantee near-optimal sparsity configurations.

\subsection{Impact}
\label{sec-cs-discussion-impact}

After being proposed, Continuous Sparsification has influenced subsequent works in neural network efficiency. In particular, \cite{xin1} designs an algorithm that dynamically grows and shrinks networks during training, where shrinkage is achieved extending CS to structured sparsity regimes. By starting from a smaller network and continuously growing it during training, the proposed method offers training time reductions at small or no final performance degradation. Furthermore, \cite{csq} propose Continuous Sparsification Quantization, which integrates CS into a bi-level optimization framework alongside mixed-precision quantization. CSQ enhances computational efficiency on resource-constrained settings where aggressive quantization is typically required alongside sparsification. These works punctuate the adaptability of Continuous Sparsification, showing that it serves as a foundational method for network compression that can be applied not only to network sparsification, but also to growing and quantization.

\newpage
\chapter{Quantization}

\section{Introduction}
\label{sec-qt-intro}

\subsection{Motivation \& Strategy}
\label{sec-qt-intro-motivation}

In Chapter 4, we improved model efficiency by gradually removing redundant or unimportant parameters during training: a process known as pruning or sparsification. In this chapter, we take a different approach: rather than eliminating parameters, we optimize the numerical precision at which each parameter is represented. Our goal is to minimize the total number of bits required to represent the network's parameters while preserving its performance.

Traditional quantization-aware optimization fixes the precision beforehand. In contrast, we address a more general problem in which the precision is not predetermined but is instead treated as an optimization variable alongside each parameter's quantized value. By doing so, we aim to achieve a heterogeneous precision assignment that can further reduce inference time and memory costs. Like NAS and network sparsification, this quantization problem is inherently combinatorial and intractable in its exact form. Previous approaches to approximate network quantization are reviewed in Section \ref{sec-qt-intro-related}.

Section \ref{sec-qt-method-formal} formally states the mixed precision quantization problem and discusses its core computational challenges. Then, Section \ref{sec-qt-method-reframing} re-frames the problem as one with infinite constraints, which can be readily approximated via stochastic methods. Section \ref{sec-qt-method-approx} presents our approximation strategy -- the core component of our final mixed precision method. The remaining sections outline our experimental setup, report empirical results assessing our proposed method, and offer a deeper experimental analysis of our approach, including a discussion of its contributions and impacts.

\subsection{Related Work}
\label{sec-qt-intro-related}

Most approaches in the quantization literature assume that a predefined precision is given prior to training \citep{DoReFa, LQNets, DSQ, SLB, QuantizedNets, LogQuant, PACT, ABCNets, DeepCompress, LSQ}, designating the number of bits used to represent each parameter of the network. The focus of these works is typically to circumvent the obstacle posed by the non-differentiability of the quantization function, as this makes first-order methods inapt to train quantized models due to the lack of meaningful gradients.

DoReFa \citep{DoReFa} updates real-valued weights using gradients w.r.t. their quantized forms (\ie a straight-through estimator \citep{STE}), enabling training over quantized parameter values with gradient descent. LQ-Nets \citep{LQNets}, PACT \citep{PACT}, and LSQ \citep{LSQ} introduce further flexibility to quantization by also optimizing the quantization step size, which we refer by \textit{quantization scale} throughout our work, \ie the real value that each bit corresponds to is also learned.

DSQ \citep{DSQ} proposes a smooth approximation to the quantization mapping, which allows weights to be directly trained with SGD without straight-through estimators. SLB \citep{SLB} introduces per-parameter auxiliary variables that induce distributions over values that quantized weights can take -- these variables are trained with gradient descent and a final quantized model is created by approximating each weight's distribution by a point-mass.

More recently, some works on quantization have explored ways to assign different precisions to parameter groups in a neural network \citep{DNAS,HAQ,HAWQ,BSQ} -- such methods, however, remain a minority in the quantization literature. DNAS \citep{DNAS} uses neural architecture search to map candidate precisions to different parameter groups. HAQ \citep{HAQ} uses reinforcement learning to learn a quantization policy. HAWQ \citep{HAWQ} uses second-order information to allocate precisions to different parameter groups. 

Closest to our work is BSQ \citep{BSQ}, which introduces new variables that are trained with gradient-based methods, which are then used to allocate precisions throughout the network. It operates by first mapping each real-valued weight to a bitstring variable whose length is chosen a-priori.
These variables are then iteratively trained with gradient descent and pruned based on their magnitudes, resulting in a decrease in precision whenever a component (a bit) of the bitstring is removed. Finally, the bitstrings are mapped back to weights, and the quantized model undergoes fine-tuning by adopting straight-through estimators for the quantized weights.

Lastly, \citet{QuantNoise} improve the performance of quantized networks by randomly selecting a subset of weights to be quantized during training while keeping the remaining parameters in their original format. While our approach also uses randomness to improve quantization in some fashion, the form of noise samples, along with how and why they are used, are vastly different.

% We start by defining a function family $\mathcal F$ parameterized by $\Theta \subseteq \mathbb R^d$: for any $\theta \in \Theta$, we have a function $f_
% \theta \in \mathcal F$, $f_\theta: \mathbb R^{d_{in}} \to \mathbb R^{d_{out}}$ that is induced by $\theta$. In our practical setting, the variables $\theta$ denote a configuration for the parameters of neural network while $f_\theta$ represents the function that the network implements under such parameter values.

% For the sake of convenience, we directly define a loss functional $L : \Theta \to \mathbb R$ that measures the (average) instant loss of a function $f_\theta \in \mathcal F$ induced by some $\theta \in \Theta$. For example, $L$ can be explicitly given by $L(\theta) = \frac{1}{N} \sum_{i=1}^N \ell(f_\theta(x_i), y_i)$, where $((x_i, y_i))_{i=1}^N$ is a fixed dataset and $\ell$ is an instant loss (\ie squared loss or cross-entropy).

\section{Method}
\label{sec-qt-method}

\subsection{Formalizing the Quantization Problem}
\label{sec-qt-method-formal}

The quantization problem we address involves finding an optimal low-precision representation for the parameters of a neural network. In our formulation, each weight is assigned a precision (number of bits) so that the overall model uses as few bits as possible while preserving its predictive performance. Unlike traditional quantization methods that fix the precision in advance, our approach treats precision as an optimization variable that can vary across individual parameters.

% We consider the problem of allocating a precision $p$ (a positive integer which represents the number of bits) to each weight $w$ of a network under a total precision budget (the total number of bits that can be used to represent the model). If we let $q(w,p)$ denote the quantization of $w$ with $p$ many bits, then our goal is to minimize $L(q(w,p))$ subject to $p \leq C$, where $C$ is our total budget measured in bits and $L$ captures the loss incurred by the weights -- typically the cross-entropy loss over a training set.

As in the previous chapters, we assume that the dataset $D$ is chosen a-priori and remains fixed, containing samples (and, when applicable, labels). Consequently, our formulation of the quantization problem -- and the design of our proposed heterogeneous quantization method -- does not depend on the nature of the underlying task.

We start by letting $f: \theta \mapsto f(\theta)$ be a neural architecture parameterized by $\Theta \subseteq \mathbb R^d$: for any $\theta \in \Theta$, $f(\theta)$ is a network that maps inputs from a pre-defined dataset $D$ to outputs aligned with the underlying task. We directly define a loss function $\mathcal L : f(\theta) \mapsto \mathcal L(f(\theta)) \in \R$ that measures the loss of $f(\theta)$ on $D$.

Consider the problem of finding a parameter setting $\theta \in \Theta$ where each component $\theta_i$, $i \in [d]$, is represented by exactly $p_i \in \N$ many (signed) bits given a bit-value map $(v_1, v_2, \dots)$ that assigns a scalar $v_j \in \mathbb R$ to each position $j \in \mathbb N$ in a bit string. In other words, we have 
\[\theta_i = \sum_{j=1}^{p_i} b_{i,j} \cdot v_j \,,\]
where $b_{i,j} \in \{\pm 1\}$ is the $j$'th bit used to represent $\theta_i$ and $(v_j)_j$ is the bit-value map.

Let 
\begin{equation}
    \mathbb V = \Big\{ \sum_{j=1}^{\infty} b_j \cdot v_j ~|~ b_j \in \{ \pm 1 \} \Big\}
\end{equation}
be the set of all possible values that each component of $\theta$ can take. We will assume w.l.o.g. that the bit-value map is (countably) infinite and that $\mathbb V^d \subseteq \Theta$.

Next, let $V$ be the function that maps bit strings $b \in \{\pm 1\}^{d \times \infty}$ and precision values $p \in \mathbb N^d$ to parameter values $\theta \in \mathbb V^d$ element-wise:
\begin{equation}
    \label{def-qt-vfunc}
    V(b, p)_i = V(b_i,p_i) = \sum_{j=1}^{p_i} b_{i,j} \cdot v_j \,.
\end{equation}

We can then formalize the mixed precision problem as finding a set of bit precisions $p \in \mathbb N^d$ and bit strings $b \in \{\pm 1\}^{d \times \infty}$ such that the total precision $\sum_{i=1}^d p_i$ is minimized and the parameter setting $\theta = V(b,p) \in \mathbb V^d$ achieves some target suboptimality $\delta$:
\begin{definition} (Mixed Precision Problem)
    \label{def-qt-formal}
    Let $f$ be a neural architecture with weight space $\Theta \subseteq \R^d$. For any $\theta \in \Theta$, $\mathcal L(f(\theta)) \in \R$ denotes the loss incurred by $f(\theta)$ on the fixed dataset $D$. Then, the \emph{mixed precision problem} is:
    \begin{equation}
        \min_{\substack{p \in \mathbb N^d \\ b \in \{\pm 1\}^{d \times \infty}}} \quad \sum_{i=1}^d p_i \quad\quad \text{s.t.} \quad \mathcal L(f(V(b, p))) \leq \mathcal L^* + \delta \,,
    \label{eq-originalmixedproblem}
    \end{equation}
    where $V: \{\pm 1\}^{d \times \infty} \times \mathbb N^d \to \mathbb V^d \subseteq \Theta$ is defined in Equation \eqref{def-qt-vfunc}, $\mathcal L^*$ is the smallest achievable loss by the architecture $f$, and $\delta \geq 0$ is the suboptimality tolerance.
\end{definition}

This problem poses a few obstacles: first, it involves a search over multiple binary values for each of the $d$ parameters of the network; second, and most importantly, the derivative w.r.t.~each integer precision $p_i$ is undefined, making gradient-based methods inapplicable.

An alternative approach is to directly optimize $\theta \in \mathbb V^d$ while adopting its finite-precision representation to evaluate the loss objective:
\begin{definition} (Mixed Quantization Problem)
    \label{def-qt-formal2}
    Let $f$ be a neural architecture with weight space $\Theta \subseteq \R^d$. For any $\theta \in \Theta$, $\mathcal L(f(\theta)) \in \R$ denotes the loss incurred by $f(\theta)$ on the fixed dataset $D$. Then, the \emph{mixed quantization problem} is:
    \begin{equation}
        \min_{\substack{p \in \mathbb N^d \\ \theta \in \mathbb V^d}} \quad \sum_{i=1}^d p_i \quad\quad \text{s.t.} \quad \mathcal L(f(Q(\theta, p))) \leq \mathcal L^* + \delta \,,
    \label{eq-quantizedmixedproblem}
    \end{equation}
    where $Q: \mathbb V^d \times \N^d \to \mathbb V^d$ maps $\theta, p$ to a finite-precision representation of $\theta$ that takes exactly $p$ many bits, $\mathcal L^*$ is the smallest achievable loss by the architecture $f$, and $\delta \geq 0$ is the suboptimality tolerance. 
\end{definition}

 To define $Q$ precisely we will first introduce a notion of inverse for $V$ (a `notion' since $V$ is not generally bijective). For any $\theta_i \in \mathbb V$, we let $V^{-1}(\theta_i)$ denote a setting $(b_i, p_i)$, with $p_i \in \mathbb N$ and $b_i \in \{\pm 1\}^\infty$ such that $V(b_i, p_i) = \theta_i$ and, moreover, for any $(b_i', p_i')$ such that $V(b_i', p_i') = \theta_i$, it follows that $p_i' \geq p_i$. That is, $V^{-1}$ maps each element $\theta_i \in \mathbb V$ to a precision-bitstring pair that perfectly represents $\theta_i$ using the least number of bits. We also define $V^{-1}_p(\theta_i) = p_i$ and $V^{-1}_b(\theta_i) = b_i$ where, as before, $V^{-1}(\theta_i) = (b_i, p_i)$.

Then, we can define $Q$ by 
\begin{equation}
    Q(\theta, p)_i = Q(\theta_i, p_i) = \sum_{j=1}^{p_i} V^{-1}_b(\theta_i)_j \cdot v_j \,,
\end{equation}
\ie the value induced by the \textit{shortest} bit string that represents $\theta$, but truncated to $p_i$ elements. Note that this yields 
\begin{equation}
    Q(\theta, p)_i = Q(\theta_i, p_i) = V(V^{-1}_b(\theta_i), p_i) \,,
\end{equation}
which will be a key property in the proofs below.

While it might be tempting to see $Q$ as a quantization mapping, it not always acts as one. For example, consider the bit-value map given by $v_j = 2^{1-j}$, where $\theta_i=1$ can be represented 
by $p^{(1)} = 1$, $b_i^{(1)} = (1,1,1,1, \dots)$,
by $p^{(2)} = 1$, $b_i^{(2)} = (1,-1,-1,-1, \dots)$,
among other possibilities.
In this case, both $V^{-1}(1) = (b^{(1)}, p^{(1)})$ and $V^{-1}(1) = (b^{(2)}, p^{(2)})$ are valid choices
since $p^{(1)} = p^{(2)} = 1$.
For the former we have $Q(1,2) = \frac32$, while for the latter we get $Q(1,2) = \frac12$, which are \textit{not} quantizations of $\theta_i = 1$ in the common sense since they are not representations of $1$ using \emph{fewer} bits than necessary.

With this definition of $Q$ we can show that the optimization problems in Definitions \ref{def-qt-formal} and \ref{def-qt-formal2} are equivalent (under $V$ and $V^{-1}$). 

\begin{lemma}
For any setting $(b,p)$ that satisfies $\mathcal L(f(V(b,p))) \leq \mathcal L^* + \delta$ in Definition \ref{def-qt-formal}, we have that $(\theta=V(b,p),p)$ satisfies $\mathcal L(f(Q(\theta,p))) \leq \mathcal L^* + \delta$ in Definition \ref{def-qt-formal2}.
\label{lemma-origtoquant}
\end{lemma}

\begin{proof}
By definition, we have:
\begin{equation}
    \begin{split}
        Q(V(b,p), p) &= V\left( V^{-1}_b(V(b,p)), p \right) = V(b,p) \,.
    \end{split}
\end{equation}

Therefore, the parameter setting induced by $V(p,b)$ in Definition \ref{def-qt-formal} will be equal to the one induced by $Q(\theta,p)$ in Definition \ref{def-qt-formal2}, and hence the losses will be evaluated at the same point and satisfiability follows.
\end{proof}

\begin{lemma}
For any setting $(\theta,p)$ that is a minimizer of the problem in Definition \ref{def-qt-formal2}, we have that $(b=V^{-1}_b(\theta), p)$ satisfies $\mathcal L(f(V(b,p))) \leq \mathcal L^* + \delta$ in Definition \ref{def-qt-formal}.
\label{lemma-quanttoorig}
\end{lemma}

\begin{proof}
Since $(\theta, p)$ minimizes $\sum_{i=1}^d p_i$ by assumption, we have that no $\theta_i$ can be represented with less than $p_i$ many bits and hence $p = V^{-1}_p(\theta)$. Therefore
\begin{equation}
    \begin{split}
        V(V^{-1}_b(\theta), p) &= V(V^{-1}_b(\theta), V^{-1}_p(\theta)) = V(V^{-1}(\theta)) = \theta \,.
    \end{split}
\end{equation}

It then follows that the parameter setting induced by $Q(\theta,p) = \theta$ in Definition \ref{def-qt-formal2} will be equal to the one induced by $V(b,p)$ in Definition \ref{def-qt-formal}, and hence the losses will be evaluated at the same point and satisfiability follows.
\end{proof}

With the two Lemmas we can prove that the two problems are indeed equivalent:

\begin{corollary}
The optimization problems in Definitions \ref{def-qt-formal} and \ref{def-qt-formal2} are equivalent under $V$ and $V^{-1}$: for any $(b,p)$ that minimizes \eqref{eq-originalmixedproblem}, we have that $(\theta=V(b,p), p)$ minimizes \eqref{eq-quantizedmixedproblem}, and for every $(\theta,p)$ that minimizes \eqref{eq-quantizedmixedproblem} we have that $(b=V^{-1}_b(\theta), p)$ minimizes \eqref{eq-originalmixedproblem}. Moreover, the induced model $f(\theta)$ will be the same in all cases.
\end{corollary}

\begin{proof}
~
    \begin{itemize}
        \item \eqref{eq-quantizedmixedproblem} $\to$ \eqref{eq-originalmixedproblem}: assume for the sake of contradiction that $(b=V^{-1}_b(\theta), p)$ is suboptimal in \eqref{eq-originalmixedproblem}, then there exists a satisfiable $(b',p')$ with $\|p'\| < \|p\|$, but from Lemma \ref{lemma-origtoquant} it follows that $(\theta=V(b',p'),p')$ is satisfiable in \eqref{eq-quantizedmixedproblem}, which contradicts the optimality of $(\theta,p)$.
        
        \item \eqref{eq-originalmixedproblem} $\to$ \eqref{eq-quantizedmixedproblem}: assume for the sake of contradiction that $(\theta=V(b,p),p)$ is suboptimal in \eqref{eq-quantizedmixedproblem}, then there exists a satisfiable $(\theta',p')$ with $\|p'\| < \|p\|$, but from Lemma \ref{lemma-quanttoorig} it follows that $(b=V^{-1}_b(\theta',p'),p')$ is satisfiable in \eqref{eq-originalmixedproblem}, which contradicts the optimality of $(b,p)$.
    \end{itemize}
\end{proof}

With this, we have shown that optimizing the quantized weights directly, as in Definition \ref{def-qt-formal2}, is equivalent to the original problem of optimizing bit-strings of arbitrary length in Definition \ref{def-qt-formal}. Although we have simplified the search space significantly, note that gradient-based methods remain unfeasible since the gradient w.r.t.~$p$ is still undefined and $Q$ is piece-wise constant,

\subsection{Re-framing the Quantization Problem}
\label{sec-qt-method-reframing}

We now present a sequence of steps to re-frame the original quantization problem in an equivalent way which is amenable to stochastic approximations. A key difference from the previous chapters is that we rely on an additional assumption to guarantee that the re-framed problem remains equivalent to the original one.

We start by defining
\begin{equation}
    R(\theta,p) \coloneqq Q(\theta,p) - \theta \,,
\end{equation}
which allows us to write
\begin{equation}
    Q(\theta,p) = \theta + (Q(\theta,p) - \theta) = \theta + R(\theta,p) \,.
\end{equation}

% By relying on the definition of $Q$  and $V^{-1}$ alone, we can show that
% \begin{equation}
% \begin{split}
%     R(\theta_i,p_i) &=
%     \sum_{j=1}^{p_i} V^{-1}_b(\theta_i)_j \cdot v_j - \sum_{j=1}^{V^{-1}_p(\theta_i)} V^{-1}_b(\theta_i)_j \cdot v_j\\
%     &= 
%     \left( \sum_{j=V^{-1}_p(\theta_i)+1}^{p_i} V^{-1}_b(\theta_i)_j \cdot v_j \right)
%     -
%     \left( \sum_{j=p_i+1}^{V^{-1}_p(\theta_i)} V^{-1}_b(\theta_i)_j \cdot v_j \right)  \,.
% \end{split}
% \end{equation}

Next, we consider the following assumption: if $(\theta_i, p_i)$ is satisfiable, then $(\theta_i, p_i')$ is satisfiable for any $p_i' \geq p_i$ (i.e. assigning more bits to a parameter preserves satisfiability). This can be formally written as
\begin{equation}
    \mathcal L(f(Q(\theta,p))) \leq \mathcal L^* + \delta \implies \forall i~p_i' \geq p_i, \quad \underbrace{\mathcal L(f(Q(\theta,p')))}_{= \mathcal L(f(\theta + R(\theta,p')))} \leq \mathcal L^* + \delta \,,
\end{equation}
and is the core assumption for our method.

Using the definition of $R$ and letting $w = Q(\theta,p)$, we get from the above that
\begin{equation}
    \label{as-qt-main}
    \underbrace{\mathcal L(f(Q(w,p)))}_{= \mathcal L(f(w))} \leq \mathcal L^* + \delta \implies \forall i~p_i' \geq V^{-1}_p (w_i), \quad \mathcal L(f(w + R(w,p'))) \leq \mathcal L^* + \delta \,,
\end{equation}
\ie using the trinary system, this means that if a bitstring $(1,1,0,0,0,\dots)$ is satisfiable, then any bitstring $(1,1,\pm1,\pm1,\pm1,\dots)$ will also be satisfiable.

Now, we can re-frame the original mixed-quantization problem as having infinitely many constraints:
\begin{definition}[Re-framed Mixed Quantization Problem]
    \label{def-qt-re1}
    Let $f$ be a neural architecture with weight space $\Theta \subseteq \R^d$. For any $\theta \in \Theta$, $\mathcal L(f(\theta)) \in \R$ denotes the loss incurred by $f(\theta)$ on the fixed dataset $D$. If Equation \eqref{as-qt-main} holds for any $w \in \mathbb V^d \subseteq \Theta$, then the mixed quantization problem can be rewritten as:
\begin{equation}
    \min_{\substack{p \in \mathbb N^d \\ w \in \mathbb V^d}} \quad \sum_{i=1}^d p_i \quad\quad \text{s.t.} \quad \mathcal L(f(w+R(w, p'))) \leq \mathcal L^* + \delta, \quad \forall i ~ p'_i \geq p_i \,,
\end{equation}
    which is equivalent to the problem in Definition \ref{def-qt-formal2}.
\end{definition}

\subsection{Approximating the Quantization Problem}
\label{sec-qt-method-approx}

Now that we have re-framed the problem as one with infinitely many constraints, we are ready to employ our approximation strategy.

Note that, for $p'_i > p_i = V^{-1}_b(w_i)$,
\begin{equation}
    R(w_i, p_i') = \sum_{j=p_i+1}^{p'_i} V^{-1}_b(w_i)_j \cdot v_j \,,
\end{equation}
where $V^{-1}_b(w_i)$ admits \emph{any} configuration for all bits after the $V^{-1}_b(w_i)$-pth one. Therefore, we assume that $\{R(w_i, p_i') ~|~ p_i' > p_i = V^{-1}_p(w_i) \}$ is dense in some interval $I(w_i, p_i) = [l,u] \subset \mathbb V$.

We first approximate the infinitely many constraints in Definition \ref{def-qt-re1} by a single constraint on the expected loss over possible values $\epsilon$ for $R(w, p)$:
\begin{equation}
    \mathbb E \left[ \mathcal L(f(w+\epsilon)) \right] \leq \mathcal L^* + \delta, \quad \forall i, ~ \epsilon_i \sim \mathbb P(I(w_i, p_i)) \,,
\end{equation}
where $\mathbb P$ is some probability distribution with support $I(w_i, p_i)$. 

Computing the expected loss above is still generally intractable, therefore we estimate by using $K$ samples:
\begin{equation}
    \label{eq-qt-est}
    \frac1K \sum_{k=1}^K \mathcal L(f(w+\epsilon^{(k)})), \quad \forall i,k, ~ \epsilon_i^{(k)} \sim \mathbb P(I(w_i, p_i)) \,.
\end{equation}

Our approximation is given by adopting the above estimate for the expected loss, relaxing the discrete domain of $p$, and using Lagrange multipliers to yield an unconstrained problem:
\begin{definition}[Approximate Mixed Quantization Problem]
    \label{def-qt-approx}
    Let $f$ be a neural architecture with weight space $\Theta \subseteq \R^d$. For any $\theta \in \Theta$, $\mathcal L(f(\theta)) \in \R$ denotes the loss incurred by $f(\theta)$ on the fixed dataset $D$. Then the \emph{approximate mixed quantization problem} is
    \begin{equation}
        \min_{\substack{p \in [1, \infty)^d \\ w \in \mathbb V^d}} \quad 
        \frac1K \sum_{k=1}^K \mathcal L(f(w+\epsilon^{(k)})) + \lambda \sum_{i=1}^d p_i, \quad \forall i,k, ~ \epsilon_i^{(k)} \sim \mathbb P(I(w_i, p_i)) \,,
    \end{equation}
    where $\mathbb P$ is some probability distribution with support $I(w_i, p_i)$.
\end{definition}

Note that the derivation so far is agnostic to the underlying bit-value map $(v_j)_j$. Moreover, the approximate problem given above still poses a key obstacle: it is unclear how to choose $\mathbb P$ such that the gradients w.r.t.~$p$ are well-defined and informative.

To circumvent this remaining limitation, we first assume that the bit-value map is given by $v_j = 2^{1-j}$ \ie fixed-point binary representation. In this case, for any $w_i$, the distribution $\mathcal U (\pm 1)$ for all bits after the $V^{-1}_p(w_i)$-th one induces $\mathbb P(I(w_i, p_i)) = \mathcal U(\pm 2^{1-p_i})$. Instead of having $R(w_i, p_i) \sim \mathcal U(\pm 2^{1-p_i})$, we can write $R(w_i, p_i) = 2^{1-p_i} \cdot \epsilon$ with $\epsilon \sim \mathcal U(\pm 1)$ to yield well-defined gradients w.r.t~$p$.

Next, we reparameterize $p$ to avoid having to perform projections. For $p_i \in [1, \infty)$, we have $2^{1-p_i} \in (0,1]$, which we reparameterize as $\sigma(s_i)$ for $s_i \in \R$, \ie $s = \sigma^{-1}(2^{1-p})$. Finally, note that, for the special case $v_j = 2^{1-j}$, we have
\begin{equation}
    \mathbb V = \left[ \pm \sum_{j=1}^\infty 2^{1-j} \right] = [\pm 2] \,,
\end{equation}
which characterizes the domain of $w$.

\begin{definition}[Approximate Mixed Quantization Problem (Fixed-point)]
    \label{def-qt-approx2}
    Let $f$ be a neural architecture with weight space $\Theta \subseteq \R^d$. For any $\theta \in \Theta$, $\mathcal L(f(\theta)) \in \R$ denotes the loss incurred by $f(\theta)$ on the fixed dataset $D$. Then, for the bit-value map $v_j = 2^{1-j}$, the \emph{approximate mixed quantization problem} is
    \begin{equation}
        \min_{\substack{s \in \R^d \\ w \in [\pm 2]^d}} \quad
        \frac1K \sum_{k=1}^K \mathcal L(f(w+\sigma(s) \odot \epsilon^{(k)})) + \lambda \sum_{i=1}^d p_i, \quad \forall i,k, ~ \epsilon_i \sim \mathcal U^d(\pm 1) \,.
    \end{equation}
\end{definition}

The problem above is fully differentiable w.r.t.~$w$ and $s$, where the latter is a reparameterization for the precisions $p$. Like our previous approximations, the objective can be optimized with gradient descent, where the network weights $w$ and the quantization scheme induced by $s$ are trained jointly in an end-to-end fashion. It is a stochastic approximation since we use $K < \infty$ noise samples to approximate the expected loss, but we use new samples at each gradient descent iteration in order to decrease the estimation error without any computational overhead.

\subsection{The Proposed Quantization Method}
\label{sec-qt-method-method}

\begin{figure}[t]
   \centering
      \centerline{\includegraphics[width=\columnwidth]{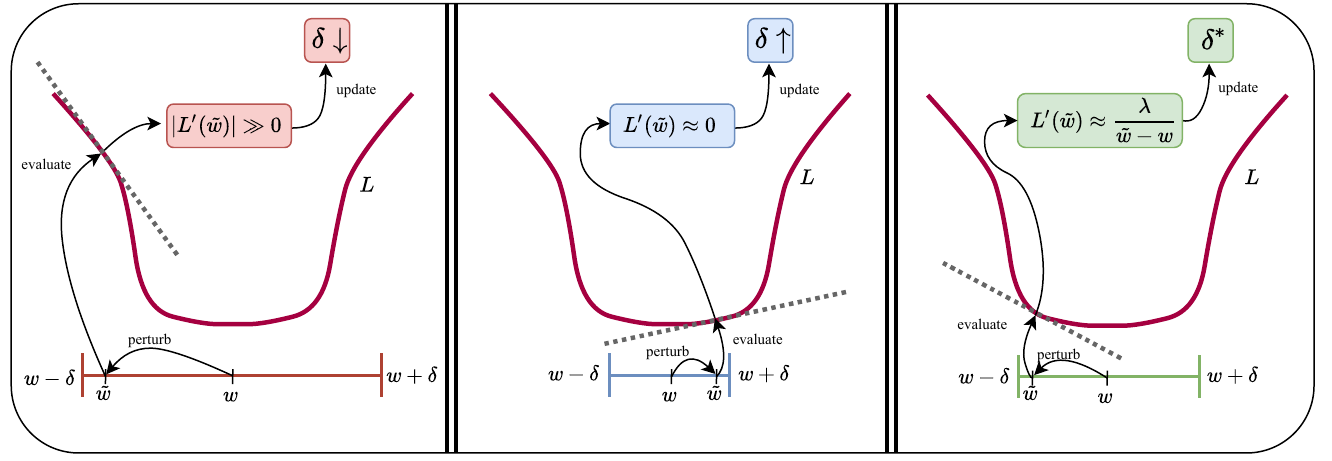}}
      \caption{%
         \textbf{Weight noise as a differentiable proxy for precision.}
         A learnable magnitude $\delta$ scales uniform random noise added to weight $w$ during training.  The width of
         the basin over which it is possible to perturb $w$ without increasing task loss $L$ drives learning of
         $\delta$.  After training, we reduce the bit precision of the numeric representation of $w$ as much as
         possible, with the constraint of remaining in the $(w-\delta,w+\delta)$ range.
         \emph{\textbf{Left:}}
            A random perturbation $\tilde{w}$ increases loss, driving a decrease in $\delta$.
         \emph{\textbf{Middle:}}
            Perturbation leaves loss unchanged, driving an increase in $\delta$.
         \emph{\textbf{Right:}}
            Noise level $\delta$, and the corresponding implied precision, stabilize when matched to the size of the
            basin.%
      }%
      \label{fig-diagram}
\end{figure}

\begin{wrapfigure}[26]{R}{0.5\textwidth}
  \centering
    \includegraphics[width=\linewidth]{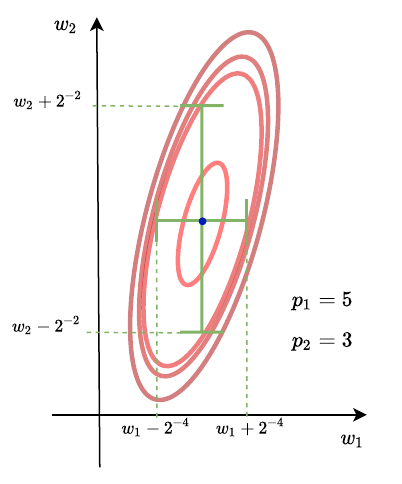}
  \caption{
  Multidimensional example where different precisions should be assigned to each parameter due to their distinct perturbation limits. Note the connection between each parameter's perturbation limit, the width of the level curve in the parameter's axis, and the allocated precision for each parameter.
  }
\end{wrapfigure}

\begin{algorithm}[t]
  \textbf{Input:} Architecture $f$, $p_{init} \in [1,\infty)$, $\lambda \geq 0$, $T \in \N$
  \caption{SMOL}
  \label{alg-method}
  \begin{algorithmic}[1]
        \State Initialize $w \sim \dist_w$, $s \gets - \ln \left( 2^{\vec p_{init}-1} - 1 \right)$
        \For{$t = 1$ to $T_1$}
            \State Sample $\epsilon \sim \mathcal U^d(\pm 1)$
            \State Compute $\mathcal L(w, s) = \mathcal L (f(w + \sigma(s) \odot \epsilon)) + \lambda \|\log_2(1 + e^{-s}) \|_1$ and $\nabla_{w,s} \mathcal L(w,s)$
            \State Update $w$ and $s$ using $\nabla_w \mathcal L(w,s)$ and $\nabla_s \mathcal L(w,s)$
            \State Clip $w$ to $\pm (2 - \sigma(s))$
        \EndFor
        \State Set $p \gets 1 + \text{round}(\log_2(1 + e^{-s}))$
        \State ZPA (Optional): If $|w_i|<|w_i-Q(w_i,p_i)|$ then set $p_i=0$ (element-wise)
        \For{$t = T_1$ to $T_2$}
            \State Compute $\mathcal L (f(w_q))$ and $\nabla_{w_q} \mathcal L(f(w_q))$, where $w_q = Q(w,p)$
            \State Update $w$ using $\nabla_{w_q} \mathcal L(f(w_q))$ instead of $\nabla_{w} \mathcal  L(f(w_q))$
        \EndFor
        \State Output $f(Q(w,p))$
  \end{algorithmic}
\end{algorithm}

Although the problem given in Definition \ref{def-qt-approx2} required a sequence of approximations to be derived, it admits an intuitive description and motivation that relies on a connection between quantization and perturbations. More specifically, if we assume that a weight $w$ can be perturbed in any direction by at least $\epsilon$ without degrading the performance of the network, then we can safely represent $w$ with $p$ bits as long as the quantization error does not exceed $\epsilon$ -- that is, $|w-q(w,p)| \leq \epsilon$.

In other words, if we know the largest perturbation that each weight admits (its \textit{perturbation limit}) without yielding performance degradation, then we can easily assign a precision to each weight: it suffices to choose the smallest number of bits such that the quantization error falls below the perturbation limit of the corresponding weight.

Our method works by estimating the perturbation limit of each parameter to be quantized, which is achieved by directly optimizing the magnitude of the weight perturbations through a novel loss function. We call our method SMOL.

Once the perturbation limit of each weight has been estimated through optimization by our method, we assign per-weight precisions by mapping each perturbation limit to a number of bits.

A key aspect of our proposed loss function lies in the fact that it is fully differentiable w.r.t.~$s$; hence, gradient-based methods like SGD and Adam can be applied off-the-shelf to optimize both the original weights $w$ and the auxiliary variables $s$. This enables the parameters $s$ to be seen as being part of the model itself, allowing for our method to be easily combined with higher-order optimizers, neural architecture search, and other algorithms that operate on networks. After training, we map the final values of $s$ to precisions via some function $\lceil \cdot \rfloor$ that outputs integers \eg rounding or truncation:
\begin{equation}
\label{eq-pmap}
    p = 1 + \lceil \log_2(1 + e^{-s}) \rfloor \,.
\end{equation}

The resulting tensor $p$ will have the same shape as $w$, and its components represent the number of bits assigned to each parameter in $w$. Having allocated a precision to each weight parameter, we can discard the auxiliary variables $s$ and use quantized weights $w_q = Q(w,p)$ to compute the model's predictions instead of applying perturbations. This will typically lead to non-negligible changes in the model's activations, which we circumvent by fine-tuning the model by further optimizing $w$ to minimize the training loss. Following prior work, we use straight-through estimators to optimize $w$ since $Q$ is non-differentiable, \ie we set $\nabla_w \mathcal L = \nabla_{w_q} \mathcal L$ and perform gradient descent on $w$.

Lastly, note that inverting the $s \to p$ mapping offers a way to initialize $s$ given a desired initial precision $p_{init}$, more specifically we set $s_{init} = - \ln \left( 2^{p_{init}-1} - 1 \right)$. For example, adopting an initial precision $p_{init}=8$ results in $s_{init} = -\ln(2^{8-1}-1) = -\ln(127)$, and the perturbation limit estimate will start as $\sigma(s_{init}) = 2^{1-p_{init}} = 2^{-7}$ for all weights $w$ of the network.

Our method is also applicable if groups of parameters must share the same precision value \eg the layer-wise setting where all parameters of each layer are represented with the same number of bits. In this case, we assign each parameter group $i$ to a single component $s_i$ of $s$, and when applying perturbations to the weights in the group we scale all the noise samples by the same scalar $\sigma(s_i)$. At the end of training, we map $s_i$ to a single integer $p_i$ via \eqref{eq-pmap} which is shared across the group.

\paragraph{Zero Precision Allocation.}

A fundamental limitation of having each $i$'th bit map to either $+2^{1-i}$ or $-2^{1-i}$ is that zero weights cannot be represented: the possible values that a 1-bit weight can assume are $\{-1, +1\}$, for a 2-bit weight $\{-1.5, -0.5, +0.5, +1.5\}$, and moreover a quantized weight cannot assume the value $0$ regardless of its precision. However, in our setting -- where we can assign a different precision to each weight -- we can directly re-define the quantization function $q$ to map any $w$ to $0$ whenever $p=0$, hence introducing the notion of \textit{zero-precision weights}. This is analogous to the procedure of assigning zero precision to a filter in BSQ, which leads to the corresponding convolution to be completely skipped when computing a model's outputs.

We propose to assign zero precision to a weight $w$ whenever $|w| \leq |w-q(w,p)|$. Since the quantization function $q$ maps $w$ to the closest value that is representable with $p$ bits, whenever a weight's precision is set to zero due to our proposed strategy, the induced quantization error is guaranteed not to increase. This follows since prior to changing $p$ the quantization error is $|w-q(w,p)|$; but once we set $p=0$, it becomes $|w-q(w,0)| = |w-0| = |w|$, which cannot lead to an increase since the condition is precisely $|w| \leq |w-q(w,p)|$.

For example, $q(w=0.2,p=2)=0.5$ since it is the value in $\{-1.5, -0.5, +0.5, +1.5\}$ that is closest to $w=0.2$, while $q(w=0.2,p=0)=0$ following our re-definition of $q$ yields a quantization error of $|0.2-0|=0.2$. On the other hand, for $p=2$ we have $|0.5-0.2|=0.3$. Hence, assigning zero precision in this case not only frees up $2$ bits but also decreases the quantization error by $0.1$.

One advantage of adopting this procedure to assign zero precisions to weights is that it only changes $q$ and hence the quantized network, therefore not affecting the procedure described in the previous section. In other words, one can allocate precisions to weights by optimizing perturbations as described previously, and then quantize the model separately with and without zero precision allocation.

This results in two quantized models with different precision assignments, hence we can obtain two networks by training the auxiliary variables $s$ only once. The notion of zero precisions also unifies quantization and pruning, since in practice assigning zero precision is equivalent to pruning a weight.

\section{Experimental Setup}
\label{sec-qt-expsetup}

\subsection{Datasets}
\label{sec-qt-expsetup-data}

\paragraph{CIFAR-10.}

The CIFAR-10 dataset~\citep{cifar} consists of $32 \times 32$ color images split into 50,000 and 10,000 training and test samples, each belonging to one out of 10 classes. We pre-process the data by applying channel-wise normalization to all images using statistics computed from the training set. We use the standard data augmentation pipeline from~\cite{resnet1}, which includes random crops and horizontal flips.

\paragraph{ImageNet.}

We employ the ILSVRC 2012 subset of ImageNet~\citep{imagenet}, which contains roughly 1.28 million training images and 50,000 validation images belonging to 1,000 different object classes. We use single $224 \times 224$ center-crop images for both training and validation.
We follow~\cite{gross} for pre-processing and data augmentation, which consists of scale augmentation (random crops of different sizes and aspect ratios are rescaled back to the original size with bicubic interpolation), photometric distortions (random changes to brightness, contrast, and saturation), lighting noise, and horizontal flips. Channel-wise normalization is employed using statistics from a random subset of the training data.

\paragraph{IWSLT’14.}

We use the IWSLT'14 German-to-English dataset, a popular benchmark for low-resource neural machine translation. It contains roughly 160,000 sentence pairs, with 7,283 sentences used for validation and 6,750 for testing. Following prior works, we apply byte-pair encoding with a joint vocabulary of 10,000 tokens.

\subsection{Models}
\label{sec-qt-expsetup-models}

\paragraph{ResNet-20.} For CIFAR-10 classification, we employ ResNet-20, a residual network composed of multiple residual blocks, each containing two $3 \times 3$ convolutions with batch normalization and ReLU activations. This network provides a canonical baseline to evaluate quantization on CIFAR-10.

\paragraph{MobileNetV2.} Also on CIFAR-10, we use MobileNetV2, a lightweight convolutional network commonly adopted in resource-constrained settings. It relies on depthwise separable convolutions and linear bottlenecks, significantly reducing the FLOPs required for inference compared to standard CNNs.

\paragraph{ShuffleNet.} We also evaluate quantization on ShuffleNets for CIFAR-10. These are highly efficient architectures that leverage channel shuffle operations, minimizing computational and memory costs while maintaining accuracy. ShuffleNets provide a complementary benchmark to MobileNetV2 for quantization of highly efficient architectures.

\paragraph{ResNet-18 and ResNet-50.} On ImageNet, we use ResNet-18 and ResNet-50, allowing for direct comparisons with existing quantization methods. These architectures adopt bottleneck residual blocks that consist of a $1 \times 1$ depth-reducing convolution, a $3 \times 3$ dept-preserving convolution, and a $1 \times 1$ depth-increasing convolution, each followed by batch normalization and ReLU activations.

\paragraph{DCGAN.} For image generation on CIFAR-10, we use DCGAN \citep{dcgan}, a widely adopted generative network whose generator consists of four transposed convolutions with batch norm and ReLU activations, while its discriminator adopts four strided convolutional layers with batch norm and LeakyReLU activations.

\paragraph{Transformer.} For IWSLT'14, we adopt a 6-layer encoder-decoder Transformer model \citep{transformers}. Each encoder an decoder block includes multi-head attention followed by token-wise residual feedforward layers and layer normalization.

\paragraph{Quantized Activations on ImageNet.}

We follow BSQ \citep{BSQ} and replace all ReLU modules by ReLU6 throughout the network, resulting in outputs constrained to the $[0,6]$ interval. Activations are not quantized prior to fine-tuning \ie precision training with SMOL uses full-precision activations, which are only quantized once precisions have been allocated and fine-tuning has started. All our experiments that quantize activations consider 4-bit precisions \ie the activation tensors are quantized using $2^4 = 16$ different values.

Also following BSQ, we use straight-through estimators to allow for gradient flow through the activation quantization procedure. The quantization values are uniform over half of the activation range -- typically $[0,3]$, which results in the full-precision activations being quantized to one of the 16 values in the set $\{0, \frac{1}{5}, \frac{2}{5}, \dots, \frac{14}{5}, \frac{15}{5}\}$.

\subsection{Training}
\label{sec-qt-expsetup-training}

We evaluate our method in the tasks of quantizing CNNs trained for image classification and generation. For all experiments, we adopt an initial precision $p_{init} = 8$, which corresponds to initializing each component of the new parameter $s$ as $-\ln(2^7-1) \approx -4.84$.

We use the floor operation to map real-valued precisions to integers: based on our preliminary experiments this more aggressive rounding operation yields more compact models and rarely results in performance degradation.

For all experiments we train the auxiliary parameters $s$ with Adam \citep{adam}, using the default learning rate of $10^{-3}$ and no weight decay -- all its other hyperparameters are set to their default values. 

\paragraph{CIFAR.}

We adopt the standard data augmentation procedure of applying random translations and horizontal flips to training images, and train each network for a total of 650 epochs: the precisions are trained with SMOL for the first 350 while the remaining 300 are used to fine-tune the weights while the precisions remain fixed. Note that this training budget assigned to our method is considerably smaller than BSQ's 1000 total epochs which are split between pre-training, precision allocation, and fine-tuning.

Following prior work, we keep the batch normalization parameters in full-precision as they represent a small fraction of the network's total parameters. Like in BSQ, weights of parameterized shortcut connections are also kept in full-precision -- namely, shortcuts that consist of $1 \times 1$ convolutions followed by normalization in ResNet-20 and MobileNetV2.

To train the weights we use SGD with a momentum of $0.9$ and an initial learning rate of $0.1$, which is decayed at epochs 250, 500, and 600. We use a batch size of 128 and a weight decay of $10^{-4}$ for ResNet-20, $4 \cdot 10^{-5}$ for MobileNetV2, and $5 \cdot 10^{-4}$ for ShuffleNet.

When training ResNet-20 networks on CIFAR-10, we used $\lambda \in \{10^{-6}, 7 \cdot 10^{-7}, 5 \cdot 10^{-7}\}$ when running SMOL, where larger values for $\lambda$ result in less bits per parameter: $\lambda = 10^{-6}$ resulted in a network with $2.1$ bpp while $\lambda = 5 \cdot 10^{-7}$ yielded a model with $2.8$ bpp. For MobileNetV2 we used $\lambda \in \{10^{-7}, 2 \cdot 10^{-7}\}$, while for ShuffleNet we adopted $\lambda \in \{5 \cdot 10^{-8}, 7 \cdot 10^{-10}\}$

\paragraph{ImageNet.}

For both ResNet-18 and ResNet-50 we train the weight parameters with SGD, a momentum of 0.9, a weight decay of $10^{-4}$ and a batch size of 256 which is distributed across 4 GPUs. We follow LQ-Nets in terms of data augmentation.

We train ResNet-18 for a total of 180 epochs: the first 120 are used for precision training and the last 60 for fine-tuning, and SGD has an initial learning rate of 0.1 which is decayed by 10 at epochs 45, 90, 150, and 165.

For ResNet-50, we start from a pre-trained, full-precision model and train for 100 epochs: allocating the first 60 for precision training and the remaining 40 for fine-tuning. An initial learning rate of $0.01$ is decayed by 10 at epoch $30$ for the precision training phase, while fine-tuning starts with the same learning rate of $0.01$ which is decayed by 10 at epochs 15 and 30. Note that, as in the CIFAR-10 experiments, our training budget is considerably smaller than BSQ's, which also starts from a pre-trained model but has a budget of 180 additional epochs.

For our ImageNet experiments, we adopted $\{10^{-6}, 10^{-8}\}$ and $\{10^{-7}, 10^{-8}\}$ as values for $\lambda$ when training ResNet-18 and ResNet-50 networks, respectively. Different values for $\lambda$ were required (compared to the ones we used for CIFAR-10) due to the ImageNet ResNets having considerably more parameters than the ResNet-20 model we adopted for CIFAR-10: note that our regularizer is a sum over all precisions instead of, for example, an average.

\paragraph{Image Generation.}

Models are trained with the binary cross-entropy loss using Adam with a learning rate = $2 \cdot 10^{-4}$ and ($\beta_1$, $\beta_2$) = (0.5, 0.999). We only quantize the weights of the generator since it is the network used for deployment.

For BSQ, we first train a full-precision DCGAN for 30 epochs, quantize its weights to 8-bits, and conduct 100 epochs of precision learning with a pruning interval of 10 epochs and $\lambda = 2 \cdot 10^{-3}$. The models are then fine-tuned for 30 epochs, with a $10 \times$ decayed learning rate. Similarly, for SMOL we train for a total of 160 epochs: 130 for precision training followed by 30 epochs for fine-tuning.

For image generation we used $\lambda \in \{2 \cdot 10^{-5}, 10^{-5}, 5 \cdot 10^{-6}, 2 \cdot 10^{-6}, 10^{-6}\}$, where the most compact model (0.2 bpp) was achieved with $\lambda=2 \cdot 10^{-5}$ along with zero precision allocation. In most cases we saw a decrease of between 1 and 2 bpp when enabling zero precision allocation: training a DCGAN with $\lambda=10^{-5}$ results in 0.5 bpp with zero precision allocation and 1.7 without it.

\paragraph{Neural Machine Translation.}

The Transformer models is trained for a total of 50 epochs, where the first 30 are used by SMOL to optimize precisions and the remaining 20 for fine-tuning. All weights are quantized except for the layer normalization parameters.

\subsection{Evaluation}
\label{sec-qt-expsetup-eval}

\paragraph{Average Number of Bits.}

To compare different methods we measure the performance and the size of the resulting network. We report the average number of bits assigned to the model's weights (the sum of all precisions divided by the number of weights), which we refer to as `average bpp', along with the compression ratio relative to a full-precision model, which equals to 32 divided by the average bpp.

\paragraph{Test Accuracy and Top-1/Top-5 Accuracy.}

For the CIFAR datasets, we measure and report the accuracy on the 10,000 test samples. On ImageNet, we use the top-1 and top-5 accuracy on the 50,000 validation images. Top-5 accuracy is the fraction of samples for which the true label appears among the top five predicted classes.
\section{Results}
\label{sec-qt-results}

% We run \method with and without Zero Precision Allocation: allowing for zero precisions results in more compact models, and we use \methodns* to denote results with zero precision allocation. When allocating per-layer precisions, we refer to our method as \methodns-L.

\begin{figure}[t]
\centering
  \includegraphics[width=.8\linewidth]{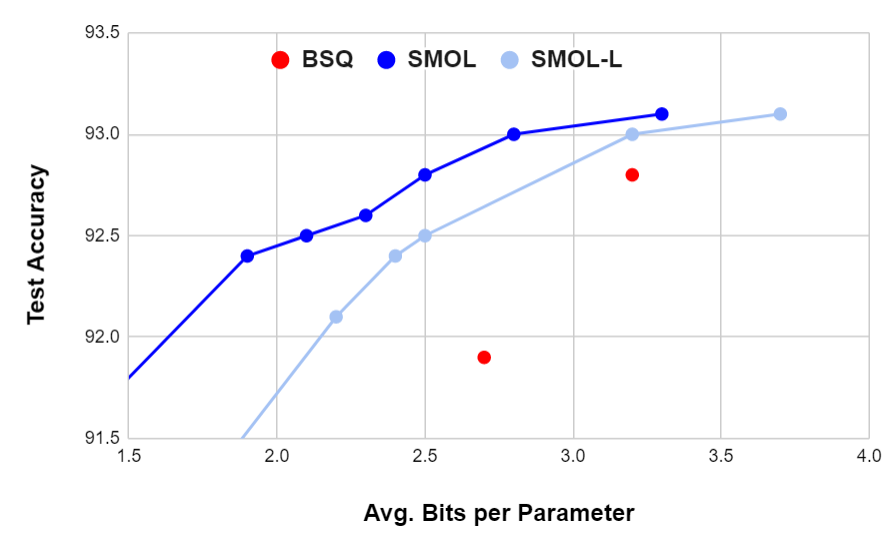}
  \caption{Performance of SMOL and BSQ when quantizing a ResNet-20 trained on CIFAR-10. SMOL-L denotes SMOL with layer-wise precisions.}
  \label{fig-c10}
\end{figure}

\begin{figure}[t]
\centering
  \includegraphics[width=.8\linewidth]{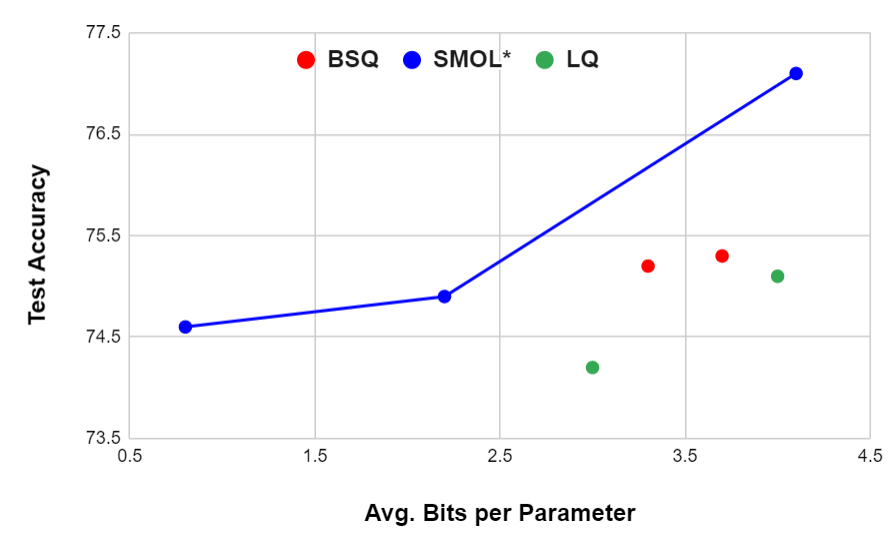}
  \caption{Performance of SMOL, BSQ, and LQ-Nets when quantizing a ResNet-50 trained on ImageNet. SMOL* denotes SMOL with zero-precision allocation.}
  \label{fig-quant-imagenet}
\end{figure}

\begin{table}[t]
\caption{Performance of different quantization methods on ResNet-20 when trained on CIFAR-10. * denotes results with Zero Precision Allocation.}
\label{tab-cifar}
\centering
\begin{tabular}{@{}lccc@{}}
\toprule
                        &             & \multicolumn{2}{c}{ResNet-20}   \\ [-2pt]  \cmidrule{3-4}
\multirow{2}{*}{Method} & Precision   & Avg. Bpp $\downarrow$   & Test         \\ [-1pt]
                        & Granularity & (Ratio $\uparrow$) & Acc. $(\%)$  \\ \midrule
LQ-FP                   &             & 32.0 (1.0)       & 92.1         \\
BSQ-FP                  &             & 32.0 (1.0)       & 92.6         \\ \\[-9pt]\hdashline\\[-8pt]
LQ-Nets                 &  Network    & 1.0  (32.0)      & 90.1         \\
DSQ                     &  Network    & 1.0  (32.0)      & 90.2         \\
SLB                     &  Network    & 1.0  (32.0)      & 90.6         \\ \\[-9pt]\hdashline\\[-8pt]
LQ-Nets                 &  Network    & 2.0  (16.0)      & 91.8         \\
SLB                     &  Network    & 2.0  (16.0)      & 92.0         \\
SMOL*                &  Parameter  & 1.3  (24.6)      & 91.5         \\
SMOL*                &  Parameter  & 1.7  (18.8)      & 92.6         \\ \\[-9pt]\hdashline\\[-8pt]
LQ-Nets                 &  Network    & 3.0  (10.7)      & 92.0         \\
BSQ                     &  Layer      & 2.7  (11.9)      & 91.9         \\
SMOL-L               &  Layer      & 2.4  (13.3)      & 92.4         \\
SMOL                 &  Parameter  & 2.1  (15.2)      & 92.5         \\
SMOL                 &  Parameter  & 2.5  (12.8)      & 92.8         \\ \\[-9pt]\hdashline\\[-8pt]
%SLB                     &  Network    & 4.0  (8.0)       & 92.1         \\
BSQ                     &  Layer      & 3.2  (10.0)      & 92.8         \\
SMOL-L               &  Layer      & 3.2  (10.0)      & 93.0         \\
SMOL                 &  Parameter  & 2.8  (11.4)      & 93.0         \\ \bottomrule
\end{tabular}
\end{table}

\begin{table}[t]
\caption{Performance of SMOL and BSQ on MobileNetV2 and ShuffleNet when trained on CIFAR-10.}
\label{tab-cifar2}
\centering
\begin{tabular}{@{}llcccccccc@{}}
\toprule
                        &  \multicolumn{2}{c}{MobileNetV2}    & & \multicolumn{2}{c}{ShuffleNet}     \\[-2pt]\cmidrule{2-3}\cmidrule{5-6}
\multirow{2}{*}{Method} &  Avg.             & Test            & & Avg.             & Test            \\[-1pt]
                        &  Bpp $\downarrow$ & Acc.     & & Bpp $\downarrow$ & Acc. \\\midrule
FP                      &  32.0             & 94.4            & & 32.0             & 90.7            \\\\[-9pt]\hdashline\\[-8pt]
BSQ                     &  2.8              & 94.1            & & 3.4              & 91.8            \\
SMOL                 &  1.5              & 94.5            & & 1.8              & 91.6            \\
SMOL                 &  1.7              & 94.8            & & 2.9              & 92.0            \\
\bottomrule
\end{tabular}
\end{table}

\subsection{Quantizing Weights}
\label{sec-qt-results-weights}

\paragraph{CIFAR.}

We first compare SMOL against different quantization methods on the small-scale CIFAR-10 dataset, adopting different networks to evaluate how aggressively each method can quantize weights without degrading the model's generalization performance. Since prior works typically rely on other methods to quantize activations (\eg using PACT \citep{PACT} for activations while applying a new method to the weights), we only consider results with full-precision activations for CIFAR-10 as this isolates each method's performance to its ability to quantize weights, leading to a more direct comparison.

ResNet-20 results are given in Table \ref{tab-cifar}: SMOL comfortably outperforms BSQ and other competing methods by offering higher performance at lower precision. Comparing against BSQ's $91.9\%$ accuracy at $2.7$ bpp, SMOL provides $0.6\%$ higher accuracy at $0.6$ lower bpp ($92.5\%$ at $2.1$ bpp). With zero-precision allocation (SMOL*), our method outperforms BSQ by $0.7\%$ at $1.0$ lower bpp ($92.6\%$ at $1.7$ bpp), and matches the performance of the full-precision model with a $18.8 \times$ compression ratio.

When allocating layer-wise precisions (SMOL-L), we observe a $0.5\%$ higher accuracy at $0.3$ lower bpp compared to BSQ, showing that although per-parameter precisions improve efficiency, our method outperforms the state-of-the-art even when constrained to the less flexible, layer-wise setting (more details in Section \ref{sec-qt-analysis-layer}). Figure \ref{fig-c10} shows efficiency curves for BSQ and SMOL-L. 

Table \ref{tab-cifar2} presents results for MobileNetV2 and ShuffleNet: SMOL also comfortably outperforms BSQ, offering $0.7\%$ higher accuracy at $1.1$ lower bpp on MobileNetV2 ($94.8\%$ at $1.7$ bpp, compared to $94.1\%$ at $2.8$ bpp), and $0.2\%$ higher performance at $0.5$ lower bpp on ShuffleNet ($92.0\%$ at $2.9$ bpp, compared to $91.8\%$ at $3.4$ bpp), while outperforming the full-precision model in both cases.

\begin{table}[t]
\caption{Performance of different quantization methods on ResNet-18 and ResNet-50 models trained on ImageNet. * denotes results with Zero Precision Allocation.}
\label{tab-imagenet}
\centering
\begin{tabular}{@{}llccccccc@{}}
\toprule
                        & & \multicolumn{3}{c}{ResNet-18}                         & & \multicolumn{3}{c}{ResNet-50}                          \\[-2pt]\cmidrule{3-5}\cmidrule{7-9}
\multirow{2}{*}{Method} & & Average          & Compression      & Test            & & Average           & Compression      & Test            \\[-1pt]
                        & & Bpp $\downarrow$ & Ratio $\uparrow$ & Accuracy $(\%)$ & & Bpp $\downarrow$  & Ratio $\uparrow$ & Accuracy $(\%)$ \\\midrule
FP          &     & 32.0  &  1.0  & 69.6             && 32.0 & 1.0 & 76.1               \\\\[-9pt]\hdashline\\[-8pt]
SLB         &     & 2.0/4 &  16.0 & 67.5             && \multicolumn{3}{c}{\Vhrulefill}  \\\\[-9pt]\hdashline\\[-8pt]
LQ-Nets     &     & 3.0/3 &  10.7 & 68.2             && 3.0/3 & 10.7 & 74.2             \\
SMOL*    &     & \multicolumn{3}{c}{\Vhrulefill}  && 0.8/4 & 40.0 & 74.6         \\
SMOL*    &     & 2.3/4 &   13.9 & 69.9             && 2.2/4 & 14.5 & 74.9             \\\\[-9pt]\hdashline\\[-8pt]
LQ-Nets     &     & 4.0/4 &   8.0 & 69.3             && 4.0/4 &  8.0 & 75.1             \\
DSQ         &     & 4.0/4 &   8.0 & 69.6             && \multicolumn{3}{c}{\Vhrulefill}  \\
BSQ         &     & \multicolumn{3}{c}{\Vhrulefill}  && 3.3/4 & 9.7  & 75.2  \\
BSQ         &     & \multicolumn{3}{c}{\Vhrulefill}  && 3.7/4 & 8.6  & 75.3  \\
SMOL*    &     & \multicolumn{3}{c}{\Vhrulefill}  && 4.1/4 & 7.0  & 77.1 \\
SMOL     &     & 4.5/4 &   7.1& 70.6              && \multicolumn{3}{c}{\Vhrulefill} \\
SMOL     &     & 4.2/4 &   7.6& 70.4              && 5.3/4 & 5.9  & 76.9 \\
\bottomrule
\end{tabular}
\end{table}

\paragraph{ImageNet.}

For the large-scale ImageNet classification task, we quantize the ResNet-18 and ResNet-50 models which allow for comparisons against LQ-Nets, DSQ, SLB, and BSQ. 

Results in Table \ref{tab-imagenet} show that SMOL outperforms competing methods while achieving higher performance than the full-precision baselines. On ResNet-18, our method offers $0.6\%$ higher accuracy than LQ-Nets at $1.7$ lower bpp ($69.9\%$ at $2.3$ bpp compared to $69.3\%$ at $4.0$ bpp), while also outperforming the full-precision model by $0.3\%$.

On ResNet-50, SMOL outperforms LQ-Nets at $2.2$ lower bpp, with an average number of bits of $0.8$ -- lower than a binary network. With $4.1$ bpp, our method provides a $1.0\%$ improvement over the full-precision baseline, suggesting that the noise injection adopted by our method has additional regularizing effects that can further improve generalization. Performance by average precision plots for BSQ, SMOL, and LQ are given in Figure \ref{fig-quant-imagenet}.

\begin{table}[t]
\caption{Performance of BSQ and SMOL on image generation on CIFAR-10 with DCGANs.}
\label{tab:image_generation}
\centering
\begin{tabular}{@{}lccc@{}}
    \toprule
                            &  \multicolumn{3}{c}{DCGAN}                                                 \\[-2pt]\cmidrule{2-4} 
    \multirow{2}{*}{Method} &  Avg.              & Inception         & \multirow{2}{*}{FID $\downarrow$} \\[-1pt]
                            &  Bpp $\downarrow$  & Score $\uparrow$  &                                   \\ \midrule
    FP (30 epochs)    &  32.0              & 4.70              & 38.6   \\
    FP (160 epochs)   &  32.0              & 5.67              & 25.8   \\\\[-9pt]\hdashline\\[-8pt]
    BSQ                     &  2.8               & 4.85              & 38.3    \\
    SMOL*                &  0.2               & 4.75              & 35.6    \\
    SMOL*                &  0.5               & 5.03              & 34.1    \\
    SMOL*                &  1.0               & 5.01              & 31.1    \\
    SMOL                 &  1.7               & 5.02              & 34.1    \\
    %\method                 &  2.5  & 4.93              & 35.3    \\
    \\[-9pt]\hdashline\\[-8pt]
    BSQ                     &  3.9               & 4.95              & 37.1    \\
    SMOLs*                &  3.5               & 5.29              & 29.9    \\
    %\method*                &  3.9              & 5.07              & 31.1    \\
    \bottomrule
    \end{tabular}
\end{table}

\begin{figure}[t]
    \centering
    \begin{minipage}[l]{0.49\linewidth}
      \centering
      \tiny{\textsf{BSQ-generated DCGAN with 3.9 bpp}}
      \vspace{1pt}
      \includegraphics[width=\linewidth]{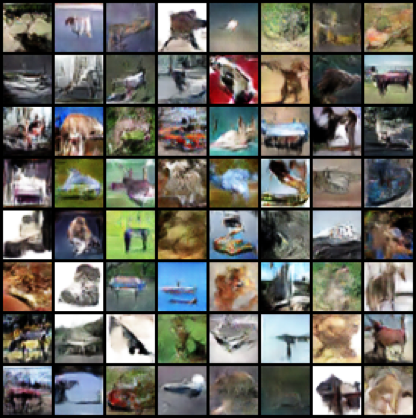}
    \end{minipage}
    \hfill
    \begin{minipage}[l]{0.49\linewidth}
      \centering
      \tiny{\textsf{SMOL-generated DCGAN with 3.9 bpp}}
      \vspace{1pt}
      \includegraphics[width=\linewidth]{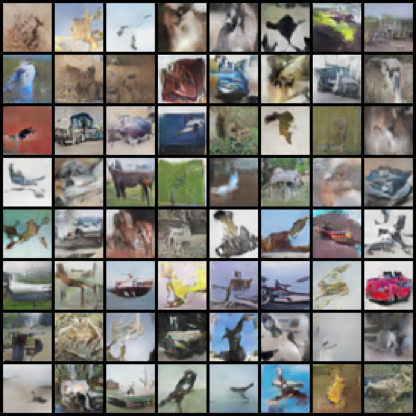}
    \end{minipage}
    \vspace{-5pt}
    \begin{minipage}[l]{0.49\linewidth}
      \centering
      \tiny{\textsf{SMOL-generated DCGAN with 1.2 bpp}}
      \vspace{1pt}
      \includegraphics[width=\linewidth]{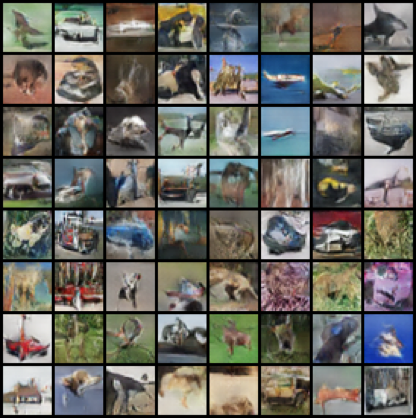}
    \end{minipage}
    \hfill
    \begin{minipage}[l]{0.49\linewidth}
      \centering
      \tiny{\textsf{SMOL-generated DCGAN with 0.2 bpp}}
      \vspace{1pt}
      \includegraphics[width=\linewidth]{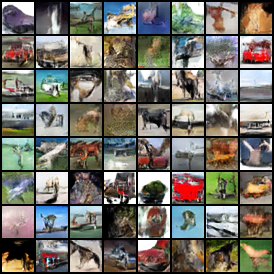}
    \end{minipage}
    \caption{Image generations with a DCGAN trained on CIFAR-10, quantized with BSQ and SMOL.}
    \label{fig:gan}
\end{figure}

\subsection{Quantizing GANs}
\label{sec-qt-results-gan}

We train a DCGAN~\citep{dcgan} model on the CIFAR-10 dataset to perform unconditional image generation on CIFAR-10.

We randomly generate 10,000 samples for all methods, each assessed with Inception Score(IS)~\citep{DBLP:conf/nips/SalimansGZCRCC16} and Fr{\'{e}}chet Inception Distance (FID)~\citep{DBLP:conf/nips/HeuselRUNH17}. As shown in Table~\ref{tab:image_generation}, our method consistently outperforms BSQ, with higher generation quality at lower bpp, 
demonstrating generalization capability to the challenging task of quantizing GANs.

We drastically improve BSQ's FID from $37.1$ at $3.9$ bpp to $29.9$ at only $3.5$ bpp.
For BSQ at $2.8$ bpp, SMOL* achieves $7.2$ better FID with only $1.0$ bpp.
Even when evaluated at an extremely low bpp of $0.2$, our method still generates images with good quality.

\begin{table}[t]
\caption{Performance of PACT and SMOL when quantizing activations of a ResNet-20 trained on CIFAR-10.}
\label{tab-act}
\centering
    \begin{tabular}{@{}lcc@{}}
    \toprule
                            &  \multicolumn{2}{c}{ResNet-20}   \\ [-2pt]  \cmidrule{2-3}
    \multirow{2}{*}{Method} &  Avg.             & Test         \\ [-1pt]
                            &  Bpa $\downarrow$ & Acc. $(\%)$  \\ \midrule
    FP                      &  32.0             & 91.6         \\ \\[-9pt]\hdashline\\[-8pt]
    PACT                    &  2.0              & 89.2         \\
    SMOL-A               &  1.8              & 90.3         \\ \\[-9pt]\hdashline\\[-8pt]
    PACT                    &  3.0              & 91.4         \\
    SMOL-A               &  2.9              & 91.8         \\ 
    SMOL-A               &  2.5              & 91.4         \\ \\[-9pt]\hdashline\\[-8pt]
    PACT                    &  5.0              & 91.6         \\ \bottomrule
    \end{tabular}
\end{table}

\subsection{Quantizing Activations}
\label{sec-qt-results-activations}

 In order to extend SMOL to train precisions for hidden activations, we introduce new trainable parameters to estimate the perturbation limit of activation outputs instead of weights. For an activation tensor $u$, we instantiate a new variable $s$ of the same shape which will be trained jointly with the network's original parameters.

 Similarly to the weight quantization case described in Section \ref{sec-qt-method-method}, at each training iteration $t$ we sample a tensor $\epsilon^{(t)}$ with the same shape as $u$, where each component is drawn uniformly from $[-1, +1]$, and generate perturbed activations $u + \frac{M}{2} \cdot \sigma(s) \odot \epsilon^{(t)}$, which are used in place of $u$ throughout the next layers of the network. The scalar $M$ denotes the range of activation function used to compute $u$, \ie $M=1$ for sigmoid and $M=2$ for tanh activations, and is used to match the magnitudes of the activations and its perturbations. Since ReLU activations are unbounded, we use PACT \citep{PACT} which clips values to $[0, \alpha]$ where $\alpha$ is a new trainable parameter -- this yields $M = \alpha$. Once $s$ has been trained, we map its components to integer precisions which are used to quantize the activations $u$.
 
 Table \ref{tab-act} shows the performance of a ResNet-20 trained on CIFAR-10 whose activations are quantized by PACT and SMOL (referred as SMOL-A): our method provides off-the-shelf improvements over PACT, offering higher accuracy at lower average bits per activation (bpa).

\subsection{Computational Efficiency}
A key question is whether per-parameter precisions can result in inference energy cost reductions. The power required to multiply a 2-bit and 3-bit weight with a 4-bit activation is $2.41 \times$ and $3.83 \times$ higher than what is required for a 1-bit weight; and the latency is $1.91 \times$ and $2.10 \times$ higher, respectively. These numbers are estimated using ripple-carry adder based multiplier designs (which are suitable for low-precision operations) for the corresponding precision settings \citep{hardwarebook}. For the accumulation operations, we assume that the power and latency values are the same for different precision settings.

We take as an example the last convolutional layer of ResNet-20, for which BSQ assigns 2 bits for all parameters while our method assigns 1, 2, and 3 bits to 74\%, 24.5\%, and 1.5\% of the parameters, respectively. In this case, the layer with precisions assigned by SMOL requires only 57.5\% of computation power while improving latency by 36\% compared to BSQ, which amounts to an energy cost reduction (power x latency) of 62.7\%. Note that in real hardware designs, additional control overheads, \eg around 25\% \citep{hardware1}, are required to perform fine-grained mixed-precision operations. Although estimates, these suggest that a fine-grained precision allocation scheme can result in significant energy savings on hardware designed to support the corresponding arithmetic operations.

\section{Experimental Analysis}
\label{sec-qt-analysis}

\begin{figure}[t]
   \centering
      \centerline{\includegraphics[width=0.8\columnwidth]{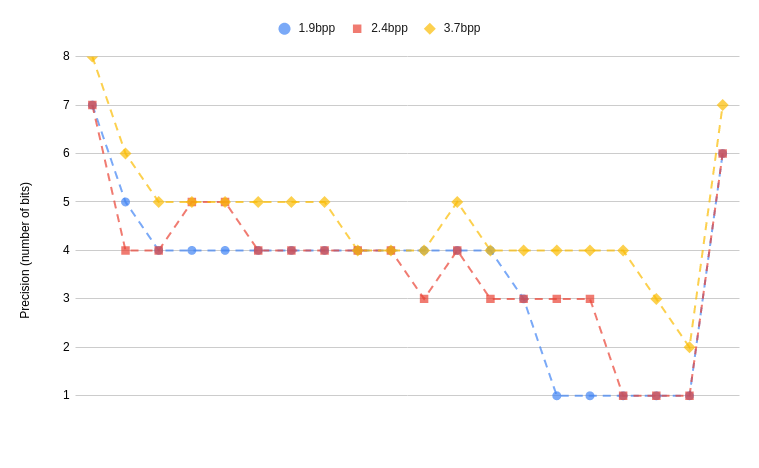}}
      \caption{Layer-wise precisions allocated by SMOL-L on ResNet-20 models trained on CIFAR-10. The x-axis denotes the layer indices: leftmost points denote earlier layers while rightmost denote layers in the later stages of ResNet, which are closer to the final fully-connected layer.}
      \label{fig-layerwise}
\end{figure}

\subsection{Layer-wise Precision Pattern}
\label{sec-qt-analysis-layer}

In Section \ref{sec-qt-results-weights} we present results for SMOL when training layer-wise precisions on a ResNet-20 trained on CIFAR-10 (SMOL-L in Table \ref{tab-cifar}), where higher performance under lower bpp is achieved compared to BSQ. 

Figure \ref{fig-layerwise} shows the precisions allocated to each layer of the ResNet-20 under three different values for $\lambda$: in all cases, later convolutional layers are assigned low precisions compared to earlier stages of the network. The first convolution, along with the fully-connected layer at the end of the network, are assigned significantly higher precision than other layers.

\begin{figure}[t]
   \centering
      \centerline{\includegraphics[width=0.8\columnwidth]{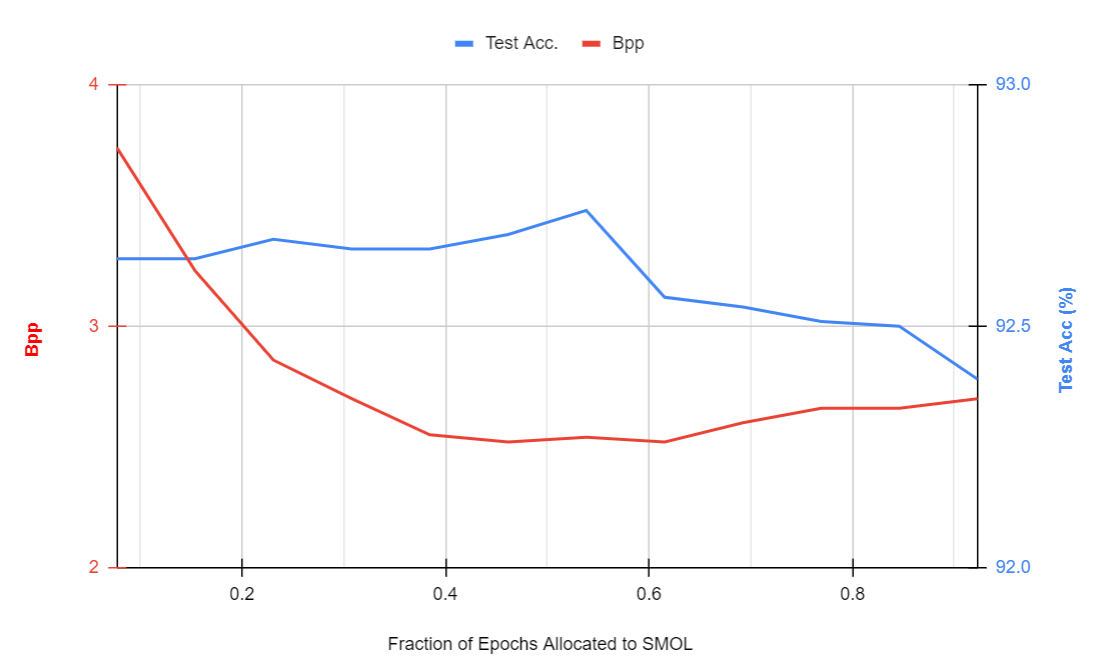}}
      \caption{Performance and average precision of ResNet-20 trained on CIFAR-10 with SMOL when allocating different number of epochs to precision training versus fine-tuning. Curve shows increases of 50 epochs over a total of 650 epochs.}
      \label{fig-ratio}
\end{figure}

\subsection{Balancing Precision Optimization and Fine-Tuning}

In our main experiments we allocate roughly half of the training budget to precision training and the other half to fine-tuning \ie further optimization of the weight parameters while applying quantization and adopting straight-through estimators. Here we show how the performance and size (measured in average number of bits per parameter) of the final model behaves under different ratios for the number of epochs allocated to precision training over the total training budget.

Here, we still train the models with $\lambda=7 \cdot 10^{-7}$, but we change what how many of the total 650 epochs are used for precision training -- the remaining ones are always allocated to fine-tuning the model.

Results are shown in Figure \ref{fig-ratio}: there is significant change in the model's bpp when the fraction of epochs allocated to precision training is lower than $\frac{250}{650}$, indicating that SMOL requires between 200 and 250 epochs to fully optimize the auxiliary variables $s$. For ratios larger than $\frac{250}{650}$ our method yields networks with similar bits-per-parameter (around $2.5$ bpp), but the test accuracy starts to decrease as the ratio increases due to fewer epochs being used to fine-tune the model. The best performance is achieved by allocating 350 epochs for precision training, which results in the remaining 300 epochs being assigned for fine-tuning.

\subsection{Underlying Structure in Precision Patterns}

To evaluate whether the per-parameter precision assignments learned by SMOL have an underlying structure or not, we set up an additional experiment which aims to re-train and evaluate a network after destroying any possible structure in its precision map. In particular, we start from a ResNet-20 model that has been fully trained and fine-tuned by SMOL, and proceed to re-train it under two different settings.

For the first, we re-initialize the weight parameters (using a different random seed) but fully preserve the precision assignments generated by SMOL \ie each $i$-th weight $w_i$ with precision $p_i$ is randomly re-initialized ($w_i \sim \mathcal D_{init}$), but $p_i$ is kept the same. As for the second setting, we also apply a random permutation on the precision tensor of each layer \ie weights $w_i$ and $w_j$ of the same layer have their precisions $p_i$ and $p_j$ swapped -- with the goal of destroying any possible structure that the tensor $p$ might have.

We then re-train the two networks, but skipping precision training and hence allocating a budget of 300 epochs to weight training under quantization and straight-through estimators, akin to fine-tuning.

Table \ref{tab-shuffle} and Figure \ref{fig-shuffle} present results: randomly permuting (shuffling) the precisions yields significantly lower performance ($89.3\%$ compared to $91.5\%$), suggesting that the precision assignments generated by SMOL do have an underlying structure, which can be some form of per-layer organization \ie allocating high or low precision to all elements in the same convolutional filter, or cross-layer structure \ie feature maps (layer outputs) have low or high precision connections to both the previous and the next layer. We leave further the characterization of such structure to future work.

\begin{table}[t]
\caption{Comparison between a ResNet-20 trained on CIFAR-10 with SMOL and its performance when re-trained, either when the per-weight precision assignments are maintained or randomly permuted (shuffled).}
\label{tab-shuffle}
\centering
\begin{tabular}{@{}llcccccccc@{}}
\toprule
\multicolumn{2}{c}{}                        && \multicolumn{3}{c}{ResNet-20} &&  \\    [-2pt]  \cmidrule{4-6}
\multicolumn{2}{c}{\multirow{2}{*}{Method}} && Average &  Compression      & \multicolumn{1}{c}{Test}             && \\[-1pt]
\multicolumn{2}{c}{}                        && Bpp $\downarrow$    &  Ratio $(\times)$ $\uparrow$ & \multicolumn{1}{c}{Accuracy $(\%)$}  &&  \\
\midrule
Original (SMOL)                &     && 2.5 &  12.8& 92.8          &&  \\
Re-trained with same precisions     &     && 2.5 &  12.8& 91.5          &&  \\
Re-trained with shuffled precisions &     && 2.5 &  12.8& 89.3          &&  \\
\bottomrule
\end{tabular}
\end{table}

\begin{figure}[t]
   \centering
      \centerline{\includegraphics[width=0.8\columnwidth]{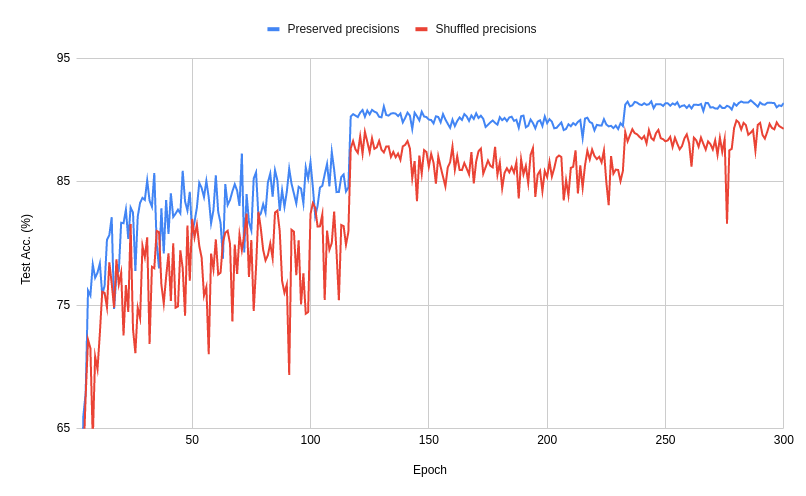}}
      \caption{Generalization performance when re-training networks, with and without randomly permuted (shuffled) precisions. The model whose precisions have been shuffled has visibly lower performance across all stages of training and achieves significantly lower final performance.}
      \label{fig-shuffle}
\end{figure}

\begin{table}[t]
\caption{Performance of SMOL when quantizing a Transformer on IWSLT'14.}
\label{tab-nmt}
\centering
\begin{tabular}{@{}lccc@{}}
\toprule
                        &  \multicolumn{2}{c}{Transformer}         \\[-2pt]\cmidrule{2-3}
\multirow{2}{*}{Method} &  Avg.             & BLEU                 \\[-1pt]
                        &  Bpp $\downarrow$ & Score $\uparrow$     \\\midrule
FP                      &  32.0             & 34.9                 \\
SMOL                 &  3.9              & 34.7                 \\
SMOL                 &  5.8              & 34.9                 \\
\bottomrule
\end{tabular}
\end{table}

\subsection{Quantizing Transformers}

An immediate question is whether SMOL is able to efficiently quantize architectures of different families, for example Transformers which do not use convolutions and whose main component is the attention mechanism.

To answer this question, we run preliminary experiments where we train a standard encoder-decoder Transformer with 6 attention blocks on the neural machine translation IWSLT'14 German to English task, where we quantize all weights except for the layer normalization parameters. A total of 50 epochs were used for training: for SMOL, the first 30 were used to train the precisions and the remaining 20 for fine-tuning.

Results in Table~\ref{tab-nmt} show that SMOL can match the performance of the full-precision model at $5.8$ bpp, which amounts to a $5.5 \times$ compression ratio. At a lower bpp of $3.9$ ($8.2 \times$ ratio), the BLEU score is lower by only $0.2$. Note that we use the same training budget that is commonly employed when training a full-precision model for this task, which is likely suboptimal for quantization. Therefore, although promising, these results are preliminary and can likely be improved by simply training for more epochs.

\section{Discussion}
\label{sec-qt-discussion}

In this chapter, we proposed a novel strategy to optimize precision assignments and quantized weights of a neural network. The core of our method, SMOL, lies in a stochastic approximation for the intractable mixed precision problem. Unlike prior works that require a global precision to be given a-priori, we allow for precisions to be trained jointly with a network's weights in an end-to-end fashion.

\subsection{Contributions}
\label{sec-qt-discussion-contribs}

This chapter makes several key contributions:
\begin{itemize}
    \item We develop a novel approximation for the intractable mixed precision problem, where parameter-wise precisions are free variables to be optimized. Our strategy provides proxy gradients w.r.t.~precisions, allowing them to be trained jointly with the network's parameters with gradient descent.
    \item We show that SMOL is capable of achieving state-of-the-art compression across different domains -- from image classification with ResNets to machine translation with Transformers -- while matching, and sometimes even surpassing, the performance of full-precision models.
\end{itemize}

\subsection{Research Directions}
\label{sec-qt-discussion-directions}

Future research directions could focus on making our proposed method more hardware-friendly. Rather than assigning a different precision for each weight, which is difficult to leverage on accelerators and GPUs, one could restrict the precisions to hardware supported bitwidths (\eg 1, 2, 4, 8) and apply them to continuous parameter blocks. Another direction is to expand the pool of adopted architectures and activations, such as applying SMOL to recurrent networks or vision transformers, as well as using it to quantize softmax activations in attention blocks. Finally, developing a fully deterministic approximation for the original quantization problem may yield an alternative method that is more stable, especially at very low precisions.

\subsection{Impact}
\label{sec-qt-discussion-impact}

Following its publication, SMOL has inspired further work in heterogeneous and mixed-precision quantization for deep learning. Notably, \cite{sysmol} incorporate SMOL into a hardware co-design framework, constraining the precision assignment values and structure so it can be leveraged by existing hardware. By aligning the training procedure with hardware constraints, their SySMOL algorithm achieves practical gains in terms of compression and inference efficiency on CPUs and GPUs. In parallel, \cite{nipq} propose NIPQ, which further improves the idea of using additive noise as a proxy for quantization robustness by optimizing truncation thresholds for weights and activations. Together, these works demonstrate how our framework can be adopted to further advance algorithmic and system-level approaches to mixed-precision quantization.

\newpage
\chapter{Leveraging Parameter Structure}

\section{Introduction}
\label{sec-ava-intro}

\subsection{Motivation \& Strategy}
\label{sec-ava-intro-motivation}

In the previous chapters, we addressed network efficiency through compression methods, where we optimized the network's architecture, sparsity schemes, and precision assignments for quantization. In this chapter, we focus on a different but equally important dimension of efficiency in deep learning: the efficiency of the training process itself. Optimizing neural networks can be often prohibitive due to its heavy computational requirements, and the choice of optimizer plays a key role not only in training speed but also in final generalization performance.

Stochastic Gradient Descent (SGD) remains the dominant optimizer when training simple architectures for computer vision such as ResNets~\citep{resnet1, resnet2}, where components such as batch normalization~\citep{bn} suffice to alleviate obstacles of training networks with many layers. On the other hand, adaptive methods such as Adam~\citep{adam} are typically adopted in natural language processing, where the state-of-the-art consists of sequence-to-sequence models \citep{lstm, transformers}. This division persists despite the fact that adaptive methods should improve training efficiency across domains. A key reason is the belief that adaptive methods converge faster but yield models with worse generalization compared to SGD, particularly on vision tasks.

We revisit this conventional wisdom by analyzing adaptive methods both theoretically and empirically, and by proposing modifications based on rigorous analysis. Section \ref{sec-ava-intro-related} discusses previous works that have analyzed or proposed new adaptive optimizers, typically with the goal of closing the generalization gap. In Section \ref{sec-ava-method-nonconv}, we discuss the stochastic non-convex optimization framework and formally define adaptive methods. Next, Section \ref{sec-ava-method-role} revisits the role of adaptivity and shows that properly controlling it changes Adam's behavior from non-convergent to offering a SGD-like convergence rate. Finally, Section \ref{sec-ava-method-ava} introduces AvaGrad, a novel adaptive optimizer inspired by our theoretical analysis that guarantees better convergence rates and improved hyperparameter separability. Subsequent sections present the adopted experimental setup, discuss empirical results comparing different adaptive methods, and provide further analysis, along with a final discussion on our contributions and impacts.

\subsection{Related Work}
\label{sec-ava-intro-related}

Although many prior works have focused on analyzing adaptive methods under the online or stochastic convex frameworks, we focus on stochastic non-convex optimization which best captures neural network training. \cite{yogi} show that Adam and newly-proposed Yogi converge as $\bO(1 / T)$ given a batch size of $\Theta(T)$, a setting that neither captures the small batch sizes used in practice nor fits in the fully stochastic non-convex optimization framework -- their analysis does not yield convergence for a batch size of $1$. \cite{adamlike} prove a rate of $\bO(\log T / \sqrt T)$ for AdaGrad and AMSGrad given a decaying learning rate. \cite{adabound} proposes AdaBound, whose adaptability is decreased during training, but its convergence is only shown for convex problems \citep{adabound2}.

The correlation between the optimizer's velocity and its parameter-wise learning rates is studied in \cite{adashift}, which proposes making both independent of the current sample: the proposed method, AdaShift, is guaranteed to converge in the convex case but at an unknown rate. \cite{adagradconv} provide convergence rates for a form of AdaGrad without parameter-wise adaptation, also showing that AdaGrad converges but at an unknown rate. \cite{padam} propose PAdam, which matches or outperforms SGD given proper tuning of a newly-introduced hyperparameter -- in contrast to their work, we show that even Adam can match SGD given proper tuning and without introducing new hyperparameters.

\section{Method}
\label{sec-ava-method}

\subsection{Non-convex Optimization}
\label{sec-ava-method-nonconv}

 We consider problems of the form
\begin{equation}
   \min_{\w \in \R^d} f(w) \coloneqq \expec{s \sim \dist}{\fs(\w)}\,,
\end{equation}
where $\dist$ is a probability distribution over a set $\supp$ of
``data points'', $f_s: \R^d \to \R$ are not necessarily convex and indicate the instant loss for each data point $s \in \supp$. As is typically done in non-convex optimization, we assume throughout the paper that $f$ is $\smooth$-smooth, \ie there exists $\smooth$ such that
\begin{equation}
	\normed{\nabla f(w) - \nabla f(w')} \leq \smooth \|w - w'\|
\end{equation}
for all $w,w' \in \R^d$.

We also assume that the instant losses have bounded gradients, \ie $\normed{\nabla \fs (\w)}_\infty \leq \gradb$ for some $\gradb$ and all $s \in \supp$, $\w \in \R^d$.

Following the literature on stochastic non-convex optimization \cite{nonconvex}, we evaluate optimization methods in terms of number
of gradient evaluations required to achieve small loss gradients. We assume that the
algorithm takes a sequence of data points $S = (s_1, \ldots, s_T)$ from which
it sequentially and deterministically computes iterates
$w_1, \ldots, w_T$, using a single gradient evaluation per iterate.

The algorithm then constructs a distribution ${\cal P}(t|S)$ over
$t \in \{1,\ldots,T\}$, samples $t' \sim {\cal P}$ and outputs $w_{t'}$.  We say an algorithm has a convergence rate of $\bO(g(T))$ if
\begin{equation}
	\expec{}{\normed{\nabla f(w_t)}^2} \leq \bO(g(T)) \,,
\end{equation}
where the expectation is over the draw of the $T$ data points $S \sim \mathcal D^T$ and the chosen iterate $w_t$, $t \sim {\cal P}(t|S)$.

\paragraph{Adaptive Methods.}
\label{sec-ava-method-adaptive}

We consider methods which, at each
iteration $t$, receive or compute a gradient estimate $\currg \coloneqq \nabla \fst(\currw)$
and perform an update
\begin{equation}
    \nextw = \currw - \curra \cdot \curre \odot \currm \,,
\label{eq-update}
\end{equation}
where $\curra \in \R$ is the \emph{global learning rate},
$\curre \in \R^d$ are the \emph{parameter-wise learning rates}, and
$\currm \in \R^d$ is the update direction, typically defined in terms of momentum
\begin{equation}
   \currm = \bonet \prevm + (1 - \bonet) \currg
   \quad \textrm{and} \quad m_0 = 0 \,.
\end{equation}
Note that this definition includes non-momentum methods such as AdaGrad and RMSProp, since setting $\bonet = 0$ yields $\currm = \currg$. While in \eqref{eq-update} $\curra$ can always be
absorbed into $\curre$, our representation will be convenient throughout the paper. SGD is a special case of \eqref{eq-update} when $\curre = \vec 1$, and although it offers no adaptation, it enjoys a convergence rate of $\bO(1 / \sqrt T)$ with either constant, increasing, or decreasing learning rates \citep{nonconvex}. It is widely used when training relatively simple networks such as feedforward CNNs \citep{resnet1, densenet}.

Adaptive methods, \eg RMSProp \citep{rmsprop}, AdaGrad
\citep{adagrad}, Adam \citep{adam} use
\begin{equation}
    \curre = \frac{1}{\sqrt{\currv} + \epsilon} \,,
\end{equation}
with $\currv \in \R^d$ as an exponential moving average of second-order gradient statistics:
\begin{equation}
   \currv = \btwot \prevv + (1-\btwot) \currg^2
   \quad \textrm{and} \quad v_0 = 0 \,.
\label{eq-vupdate}
\end{equation}
Here, $\currm$ and $\curre$ are functions of $\currg$ and can be non-trivially correlated, causing the update direction $\curre \odot \currm$ \textbf{not} to be an unbiased estimate of the expected update. Precisely this ``bias'' causes RMSProp and Adam to present nonconvergent behavior even in the stochastic convex setting \citep{amsgrad}.

\subsection{The Role of Adaptivity}
\label{sec-ava-method-role}

We start with a key observation to motivate our studies on how adaptivity affects the behavior of adaptive methods like Adam in both theory and practice: if we let $\curra = \gamma \epsilon$ for some positive scalar $\gamma$, then as $\epsilon$ goes to $\infty$ we have
\begin{equation}
    \frac{\curra}{\sqrt{\currv} + \epsilon} \to \vec \gamma\,,
\end{equation}
where $\vec \gamma$ is the $d$-dimensional vector with all components equal to $\gamma$, and $d$ is the dimensionality of $\currv$ (\ie the total number of parameters in the system). This holds as long as $\currv$ does not explode as $\epsilon \to \infty$, which is guaranteed under the assumption of bounded gradients.

In other words, we have that adaptive methods such as AdaGrad and Adam lose their adaptivity as $\epsilon$ increases, and \emph{behave like SGD} in the limit where $\epsilon \to \infty$ \ie all components of the parameter-wise learning rate vector $\curre$ converge to the same value. This observation raises two questions which are central in our work:
\begin{enumerate}
    \item \textbf{How does $\boldsymbol{\epsilon}$ affect the convergence behavior of Adam?} It has been shown that Adam does not generally converge even in the linear case \citep{amsgrad}. However, as $\epsilon$ increases it behaves like SGD, which in turn has well-known convergence guarantees, suggesting that $\epsilon$ plays a key, although overlooked, role in the convergence properties of adaptive methods.
    \item \textbf{Is the preference towards SGD for computer vision tasks purely due to insufficient tuning of $\boldsymbol{\epsilon}$?} SGD is de-facto the most adopted method when training convolutional networks \citep{vgg, googlenet, resnet1, resnet2, wide, resnext}, and it is believed that it offers better generalization than adaptive methods \citep{marginal}. Moreover, recently proposed adaptive methods such as AdaBelief \citep{adabelief} and RAdam \citep{radam} claim success while underperforming SGD on ImageNet.  However, it is not justified to view SGD as naturally better suited for computer vision, because SGD itself can be seen as a special case of Adam.
\end{enumerate}

\paragraph{Adam's Non-convergence: from Optimality to Stationarity.}

We focus on the first question regarding how the convergence behavior of Adam changes with $\epsilon$. As mentioned previously, ~\cite{amsgrad} has shown that Adam can fail to converge in the stochastic convex setting. The next Theorem, stated informally, shows that Adam's nonconvergence also holds in the stochastic non-convex case, when convergence is measured in terms of stationarity instead of suboptimality:

\begin{theorem}
  For any $\epsilon\geq0$ and constant $\btwot = \btwo \in [0,1)$, there is a
  stochastic optimization problem for which Adam does not converge to
  a stationary point.
  \label{thm:adamdiv}
\end{theorem}

Note that the problem is constructed \emph{adversarially} in terms of $\epsilon$. The problem considered in Theorem 3 of ~\cite{amsgrad}, used to show Adam's nonconvergence, has no dependence on $\epsilon$ because the proof assumes that $\epsilon=0$.

\paragraph{Adaptive Methods with Controlled Adaptivity.}

The next result shows that for stochastic non-convex problems \emph{that do not depend on $\epsilon$}, Adam actually converges like SGD as long as $\epsilon$ is large enough (or, alternatively, increases during training):

\begin{theorem}
	Assume that $f$ is smooth and $f_s$ has bounded gradients. If $\epsilon_t \geq \epsilon_{t-1} > 0$ for all $t \in [T]$, then for the iterates $\{w_1, \dots, w_T\}$ produced by Adam we have
  \begin{equation}
      \expec{}{\normed{\nabla f(\currw)}^2} \leq \bO \left(
\frac{1 + \sum_{t=1}^T \frac{\curra}{\epsilon_{t-1}^2}\left(1 + \curra + \epsilon_t - \epsilon_{t-1} \right) }{\sum_{t=1}^T \frac{\curra}{1+\epsilon_{t-1}}} \right) \,,
  \label{eq-adamrate}
  \end{equation}
  where $\currw$ is sampled from $p(t) \propto \frac{\curra}{\gradb + \epsilon_{t-1}}$.
  \label{thm:adamconv}
\end{theorem}

\begin{corollary}
    \label{cor:adamconv}
	Setting $\epsilon_t = \Theta(T^{p_1} t^{p_2})$ for any $p_1, p_2 > 0$ such that $p_1 + p_2 \geq \frac12$ (e.g. $\epsilon_t = \Theta(\sqrt T), \epsilon_t = \Theta(\sqrt[\leftroot{-2}\uproot{2}4] {Tt}), \epsilon_t = \Theta(\sqrt t)$) and $\curra = \Theta\left(\frac{\epsilon_t}{\sqrt T}\right)$ in Theorem \ref{thm:adamconv} yields a bound of $\bO(1 / \sqrt T)$ for Adam.
\end{corollary}

Together, the two theorems above give a precise characterization of how $\epsilon$ affects the theoretical behavior of Adam and other adaptive methods: not only is convergence ensured but a SGD-like rate of $\bO(1 / \sqrt T)$ is guaranteed as long as $\epsilon$ is large enough. While Adam behaves like SGD in the limit $\epsilon \to \infty$, we show that it suffices for $\epsilon$ to be $\bO(\sqrt T)$ to guarantee a SGD-like convergence rate. We believe Theorem~\ref{thm:adamconv} is more informative than Theorem~\ref{thm:adamdiv} for characterizing Adam's behavior, as convergence analyses in the optimization literature typically consider non-adversarial examples.

At a first glance, the above statement seems to contradict Theorem 1 of \cite{amsgrad}, but note the difference in the order of quantifies. In Theorem 1 the optimization problem is designed according to a given fixed $\epsilon$, while in Theorem \ref{thm:adamconv} $\epsilon$ is either increasing or fixed as a function of $T$, causing an implicit dependence on the underlying problem -- in particular, achieving gradient norm less than some positive constant requires setting $T$ large enough in terms of $M$, $\gradb$ and $\delta$.

Indeed, there is no contradiction since Theorem \ref{thm:adamdiv} sets $\gradb$ to be large w.r.t.~a given, fixed $\epsilon$, while Theorem \ref{thm:adamconv} only guarantees small gradient norms with a fixed $\epsilon$ if its value is large w.r.t.~$T$ (and consequently, also $\gradb$). Note that $\epsilon_t \propto \sqrt t$ avoids implicit dependencies between $\epsilon$ and problem-specific constants $\smooth$ and $\gradb$.

For natural, non-adversarial problems, Theorem \ref{thm:adamconv} better captures the behavior of Adam and suggests that adopting a large enough $\epsilon$  not only avoids nonconvergence, but also yields a $\bO(1 / \sqrt T)$ rate similar to SGD. Note that $\alpha / (\sqrt \currv + \epsilon) \to \vec 1$ for $\alpha=\epsilon \to \infty$, meaning that Adam degenerates to SGD (excluding the bias correction) as $\epsilon$ becomes sufficiently large, and thus \emph{must} yield a similar empirical performance on non-adversarial tasks.

\paragraph{Adaptive Methods with Delayed Adaptivity.}

\begin{algorithm}[t]
    \caption{\textsc{Delayed Adam}}
    \label{alg:dadam}
   \textbf{Input:}
      $\w_1 \in \R^d$, $\curra, \epsilon>0$, $\bonet, \btwot \in [0,1)$
   \begin{algorithmic}[1]
      \State Set $m_0 = 0, v_0 = 0$
      \For{$t = 1$ \textbf{to} $T$}
         \State Draw $s_t \sim \dist$
         \State Compute $\currg = \nabla \fst(\currw)$
         \State $\currm = \bonet \prevm + (1-\bonet) \currg$
         \State $\curre = \frac{1}{\sqrt{\prevv} + \epsilon}$
         \State $\nextw = \currw - \curra \cdot \curre \odot \currm$
         \State $\currv = \btwot \prevv + (1-\btwot) \currg^2$
      \EndFor
    \end{algorithmic}
\end{algorithm}

We present a theoretical analysis on the convergence of adaptive methods, but instead of analyzing how $\epsilon$ affects the convergence of Adam, here we focus on better understanding \emph{why} Adam can fail to converge for problems that depend on its hyperparameter settings.

We first note that the `adversarial' problems designed to show Adam's nonconvergence in Theorem 3 of \cite{amsgrad} and our Theorem~\ref{thm:adamdiv} exploit the correlation between $\currm$ and $\curre$ to guarantee that Adam takes overly conservative steps when presented with rare samples that contribute significant to the objective.

While \cite{amsgrad} already propose a modification for Adam that guarantees its convergence, it relies on explicitly constraining the parameter-wise learning rates $\curre$ to be point-wise decreasing which can harm the method's adaptiveness. We present a simple way to directly circumvent the fact that $\currm$ and $\curre$ are correlated without constraining $\curre$, guaranteeing a $\bO(1 / \sqrt T)$ rate for stochastic non-convex problems while being applicable to virtually any adaptive method.

Our modification consists of employing a 1-step delay in the update of $\curre$, or equivalently replacing $\curre$ by $\preve$ in the method's update rule for $\w_{t+1}$. Although there is still statistical dependence between $\currm$ and $\currv$, this ensures that $\curre$ is independent of the current sample $s_t$, which the following result shows to suffice for SGD-like convergence:

\begin{theorem}
  \label{thm:dadamconv}
	Assume that $f$ is smooth and $f_s$ has bounded gradients for all $s \in \supp$. For any optimization method that performs updates following \eqref{eq-update} such that $\curre$ is independent of $s_t$ and $\lowinf \leq \currei \leq \highinf$ for positive constants $\lowinf$ and $\highinf$, setting $\curra = \alpha' / \sqrt T$ yields
\begin{equation}
\begin{split}
     \expec{}{\normed{\nabla f(\currw)}^2} & \leq \bO \left( \frac{1}{\lowinf \sqrt T} \left(\frac{1}{\alpha'} + \alpha' \highinf^2 \right) \right) \,,
\end{split}
\label{eq-simplerate}
\end{equation}
where $\currw$ is uniformly sampled from $\{w_1, \dots, w_T\}$. 

Moreover, if $s_t$ is independent of $Z \coloneqq \sum_{t=1}^T \alpha_t \min_i \currei$, then setting $\curra = \alpha'_t / \sqrt T$ yields
  \begin{equation}
      \expec{}{\normed{\nabla f(\currw)}^2} \leq \bO \left(  \frac{1}{\sqrt T}
            \expec{}{\frac{\sum_{t=1}^T 1 +
              {\alpha'_t}^2 \norm{\curre}^2}{\sum_{t=1}^T \alpha'_t \min_i \currei}} \right) \,,
	\label{eq-refrate}
  \end{equation}
  where $\currw$ is sampled from $p(t) \propto \curra \cdot \min_i \currei$.
\end{theorem}

\subsection{AvaGrad}
\label{sec-ava-method-ava}

\begin{algorithm}[t]
    \caption{\textsc{AvaGrad}}
    \label{alg:avagrad}
   \textbf{Input:}
      $\w_1 \in \R^d$, $\curra, \epsilon>0$, $\bonet, \btwot \in [0,1)$
   \begin{algorithmic}[1]
      \State Set $m_0 = 0, v_0 = 0$
      \For{$t = 1$ \textbf{to} $T$}
         \State Draw $s_t \sim \dist$
         \State Compute $\currg = \nabla \fst(\currw)$
         \State $\currm = \bonet \prevm + (1-\bonet) \currg$
         \State $\curre = \frac{1}{\sqrt{\prevv} + \epsilon}$
         \State $\nextw = \currw - \curra \cdot \frac{\curre}{ \norm{\curre / \sqrt d}_2} \odot \currm$
         \State $\currv = \btwot \prevv + (1-\btwot) \currg^2$
      \EndFor
    \end{algorithmic}
\end{algorithm}

The bound in Equation \eqref{eq-refrate} depends on the learning rate $\curra$ and on both the squared norm and smallest value of the parameter-wise learning rate $\curre$, namely $\norm{\curre}^2$ and $\min_i \currei$, enabling us to analyze how the relation between $\curra$ and $\curre$ affects the convergence rate, including how the rate can be improved by adopting a learning rate $\curra$ that depends on $\curre$.

Setting $\curra' = \norm{\curre}^{-1}$ yields a bound on the convergence rate of
\begin{equation}
	\bO \left( \frac{\sqrt T}{\sum_{t=1}^T \frac{\min_i \currei}{\norm{\curre}}} \right)
\end{equation}
Note that such bound is stronger than the one in Equation \eqref{eq-simplerate}: given constants $\lowinf$ and $\highinf$ as in Theorem \ref{thm:dadamconv}, we have $\lowinf \leq \min_i \currei$ and $\norm{\curre} \leq \sqrt d \highinf$, yielding an upper bound of $\bO(H / L \sqrt T)$ that matches $\eqref{eq-simplerate}$ when $\alpha' = H^{-1}$.

Note that, for $\curre = 1/(\sqrt \currv + \epsilon)$, having $\curra' = \norm{\curre}^{-1}$ is equivalent to normalizing $\curre$ prior to each update step, which is precisely how we arrived at AvaGrad's update rule (with the exception of accounting for $d$, the dimensionality of $\curre$, when performing normalization). 

Additionally, Theorem~\ref{thm:dadamconv} predicts the existence of two distinct regimes in the behavior of Adam-like methods. Taking $\alpha' = \Theta(H^{-1})$ minimizes the bound in Equation \eqref{eq-simplerate} and yields a rate of $\bO(\frac1{\sqrt T} \frac{H}{L}) = \bO \left( \frac1{\sqrt T}\max \left(1, \frac{\gradb}{\epsilon} \right) \right)$ once we take $\lowinf = (\gradb+\epsilon)^{-1}$ and $\highinf = \epsilon^{-1}$, which can be shown to satisfy $\lowinf \leq \currei \leq \highinf$ for Adam-like methods.

In this case, the convergence rate depends on $\frac{\gradb}{\epsilon}$, or, informally, how $\currv$ compares to $\epsilon$ ($\gradb$ is an upper bound on the magnitude of the gradients, hence directly connected to $\currv$).

We now introduce AvaGrad, a novel adaptive method presented as pseudo-code in Algorithm \ref{alg:avagrad}. The key difference between AvaGrad and Adam lies in how the parameter-wise learning rates $\curre$ are computed and their influence on the optimization dynamics. In particular, AvaGrad adopts a \emph{normalized} vector of parameter-wise learning rates, which we later show to be advantageous in multiple aspects: it yields better performance and easier hyperparameter tuning in practice, while in theory it results in better convergence rate guarantees.

For convenience, we also account for the dimensionality $d$ of $\curre$ (\ie the total number of parameters in the system) when performing normalization: more specifically, we divide $\curre$ by $\lVert \curre / \sqrt d \rVert_2$ in the update rule, which is motivated by fact that the norm of random vectors increases as $\sqrt d$, and also observed to be experimentally robust to changes in $d$ (\eg networks with different sizes). Alternatively, this normalization can be seen as acting on the global learning rate $\curra$ instead, in which case AvaGrad can be seen as adding an internal, dynamic learning rate schedule for Adam.

Lastly, AvaGrad also differs from Adam in the sense that it updates $\currv$, the exponential moving average of gradients' second moments, \emph{after} the update step. This implies that parameters are updated according to the second-order estimate of the previous step \ie there is a 1-step \emph{delay} between second-order estimates and parameter updates. Such delay is fully motivated by theoretical analyses, and our preliminary experiments suggest that it does not impact AvaGrad's performance on natural tasks.

\section{Experimental Setup}
\label{sec-ava-expsetup}

~

\subsection{Datasets}
\label{sec-ava-expsetup-data}

\paragraph{CIFAR-10 and CIFAR-100.}

Each of the CIFAR datasets~\citep{cifar} consist of $32 \times 32$ color images. CIFAR-10 is split into 50,000 and 10,000 training and test samples, each belonging to one out of 10 classes. CIFAR-100 has 100 classes with the same training/test set split. We pre-process the data by applying channel-wise normalization to all images using statistics computed from the training set. We use the standard data augmentation pipeline from~\cite{resnet1}, which includes random crops and horizontal flips.

\paragraph{ImageNet.}

We employ the ILSVRC 2012 subset of ImageNet~\citep{imagenet}, which contains roughly 1.28 million training images and 50,000 validation images belonging to 1,000 different object classes. We use single $224 \times 224$ center-crop images for both training and validation.
We follow~\cite{gross} for pre-processing and data augmentation, which consists of scale augmentation (random crops of different sizes and aspect ratios are rescaled back to the original size with bicubic interpolation), photometric distortions (random changes to brightness, contrast, and saturation), lighting noise, and horizontal flips. Channel-wise normalization is employed using statistics from a random subset of the training data.

\paragraph{Penn Treebank.}

For language modeling, we use the Penn Treebank (PTB) corpus~\citep{ptb}, a widely adopted benchmark consisting of roughly one million words drawn from the Wall Street Journal articles. Rare words are replaced with a special unknown token, and the vocabulary is limited to 10,000 words. The dataset is split into 930k, 74k, and 82k tokens for training, validation, and testing, respectively.

\subsection{Models}
\label{sec-ava-expsetup-models}

\paragraph{Wide ResNet 28.}

We train Wide ResNets~\citep{wide} with 28 layers on CIFAR-10 and CIFAR-100. Wide ResNets (WRNs) build upon pre-activation Residual Networks by increasing the number of channels by a factor. We adopt both WRNs 28-10 and WRNs 28-4, where the number of channels per convolutional layer are scaled by factors of 4 and 10, respectively.

\paragraph{ResNet-50.} For experiments on ImageNet, we use ResNet-50, a widely adopted benchmark architecture that achieves a strong balance between depth, computational efficiency, and performance. ResNet-50 is built from bottleneck residual blocks -- each block consists of a $1 \times 1$ depth-reducing convolution, a $3 \times 3$ dept-preserving convolution, and a $1 \times 1$ depth-increasing convolution, each followed by batch normalization and ReLU activations. Our ResNet-50 experiments are designed to test the scalability of sparsification methods to large-scale networks and complex datasets.

\paragraph{GGAN.}

We use a Geometric GAN (GGAN, \cite{ggan}) for image generation on CIFAR-10. It follows the same architecture as a DCGAN, where the generator has four transposed convolutions with batch norm and ReLU activations and the discriminator uses four strided convolutional layers with batch norm and LeakyReLU activations. Unlike standard DCGANs, GGANs are trained with the hinge loss.

\paragraph{LSTM.}

For language modeling on Penn Treebank, we adopt the 3-layer LSTM~\citep{lstm} model from \cite{awd} with 300 hidden units per LSTM layer.

\subsection{Training}
\label{sec-ava-expsetup-training}

\paragraph{CIFAR.}

We adopt the same learning rate schedule as \cite{wide}, decaying $\alpha$ by a factor of 5 at epochs 60, 120 and 160 -- each network is trained for a total of 200 epochs on a single GPU. We use a weight decay of $0.0005$, a batch size of 128, a momentum of $0.9$ for SGD, and $\bone = 0.9, \btwo = 0.999$ for each adaptive method.

We select a random subset of 5,000 samples from CIFAR-10 to use as the validation set when tuning $\alpha$ and $\epsilon$ of each adaptive method. We perform grid search over a total of 441 hyperparameter settings, given by all combinations of $\epsilon \in \{10^{-8}, 2 \cdot 10^{-8}, 10^{-7}, \dots, 100\}$ and $\alpha \in \{5 \cdot 10^{-7}, 10^{-6}, 5 \cdot 10^{-6}, \dots, 5000\}$.

Next, we fix the best $(\alpha, \epsilon)$ values found for each method and train a Wide ResNet-28-10 on both CIFAR-10 and CIFAR-100, this time evaluating the test performance. We do not re-tune $\alpha$ and $\epsilon$ for adaptive methods due to the practical infeasibility of training a Wide ResNet-28-10 roughly 8000 times. We tune the learning rate $\alpha$ of SGD using the same search space as before, and confirm that the learning rate $\alpha=0.1$ commonly adopted when training ResNets \cite{resnet1, resnet2, wide} performs best in this setting.

\paragraph{ImageNet.}

We transfer the hyperparameters from our CIFAR experiments for all methods. The network is trained for 100 epochs with a batch size of $256$, split between 4 GPUs, where the learning rate $\alpha$ is decayed by a factor of 10 at epochs 30, 60 and 90, and we also adopt a weight decay of $10^{-4}$. Note that the learning rate of $0.1$ adopted for SGD agrees with prior work that established new state-of-the-art results on ImageNet with residual networks \citep{resnet1, resnet2, resnext}.

\paragraph{Language Modeling.}

We first perform hyperparameter tuning over all combinations of $\epsilon \in \{10^{-8}, 5 \cdot 10^{-7}, \dots, 100\}$ and $\alpha \in \{2 \cdot 10^{-4}, 10^{-3}, \dots, 20\}$, training each model for 500 epochs and decaying $\alpha$ by $10$ at epochs 300 and 400. Since $\epsilon$ affects AdaBelief differently and its official codebase recommends values as low as $10^{-16}$ for some tasks \footnote{\scriptsize{\url{github.com/juntang-zhuang/Adabelief-Optimizer}, ver.~9b8bb0a}}, we adopt a search space where candidate values for $\epsilon$ are smaller by a factor of $10^{-8}$ \ie starting from $10^{-16}$ instead of $10^{-8}$.

We use a batch size of 128, BPTT length of 150, and weight decay of $1.2 \times 10^{-6}$. We also employ dropout with the recommended settings for this model \citep{awd}. Not surprisingly, our tuning procedure returned small values for $\epsilon$ as being superior for adaptive methods, with Adam, AMSGrad, and AvaGrad performing optimally with $\epsilon=10^{-8}$.

Next, we train the same 3-layer LSTM but with 1000 hidden units, transferring the $(\alpha, \epsilon)$ configuration found by our tuning procedure. For SGD, we again confirmed that the transferred learning rate performed best on the validation set when training the wider model.

\paragraph{GANs.}

We adopt a batch size of 64 and train the networks for a total of 60,000 steps, where the discriminator is updated twice for each generator update. We train the GAN model with the same optimization methods considered previously, performing hyperparameter tuning over $\epsilon \in \{10^{-8}, 10^{-6}, 10^{-4}\}$ and $\alpha \in \{10^{-5}, 2 \cdot 10^{-5}, 10^{-4}, \dots, 0.1 \}$ for each adaptive method, and $\alpha \in \{10^{-6}, 2 \cdot 10^{-6}, 10^{-5}, \dots, 1.0\}$ for SGD.

\subsection{Evaluation}
\label{sec-ava-expsetup-eval}

\paragraph{Test Accuracy and Top-1/Top-5 Accuracy.}

For the CIFAR datasets, we measure and report the accuracy on the 10,000 test samples. On ImageNet, we use the top-1 and top-5 accuracy on the 50,000 validation images. Top-5 accuracy is the fraction of samples for which the true label appears among the top five predicted classes.

\paragraph{FID.}

For our GAN experiments, the performance of each model is measured in terms of the Fréchet Inception Distance (FID) \citep{fid} computed from a total of 10,000 generated images.
\section{Results}
\label{sec-ava-results}

Unlike other works in the literature, we perform extensive hyperparameter tuning on $\epsilon$ (while also tuning the learning rate $\alpha$): following our observation that Adam behaves like SGD when $\epsilon$ is large, we should expect adaptive methods to perform comparably to SGD if hyperparameter tuning explores large values for $\epsilon$.

For all experiments we consider the following popular adaptive methods: Adam \citep{adam}, AMSGrad \citep{amsgrad}, AdaBound \citep{adabound}, AdaShift \citep{adashift}, RAdam \citep{radam}, AdaBelief \citep{adabelief}, and AdamW \citep{adamw}. We also report results of AvaGrad, our newly-proposed adaptive method, along with its variant with decoupled weight decay \citep{adamw}, which we refer to as AvaGradW.

\begin{table}
\caption{
   Test performance of standard models on benchmark tasks, when trained with
   different optimizers.  Gray background indicates the optimization method
   (baseline) adopted by the paper that proposed the corresponding network
   model.  The best task-wise results are in bold, while other improvements
   over the baselines are underlined.  Numbers in parentheses indicate standard
   deviation over three runs.  Across tasks, AvaGrad closely matches or exceeds
   the results delivered by existing optimizers, and offers notable improvement
   in FID when training GANs.
}%
\label{tab:results}
\begin{center}
\begin{tabular}{@{}lcccc@{}}
%%%%%%%%%%%%%%%%%%%%%%%%%%%%%%%%%%%%%%%%%%%%%%%%%%%%%%%%%%%%%%%%%%%%%%%%%%%%%%%%%%%%%%%%%%%%%%%%
\toprule
& \begin{tabular}{@{}c@{}c@{}}CIFAR-10\\Test Err\%\end{tabular}
& \begin{tabular}{@{}c@{}}CIFAR-100\\Test Err \%\end{tabular}
& \begin{tabular}{@{}c@{}}ImageNet\\Val Err \%\end{tabular}
& \begin{tabular}{@{}c@{}}CIFAR-10\\FID $\downarrow$\end{tabular}
\\ 
Model       & WRN 28-10          & WRN 28-10          & ResNet-50    &  GGAN      \\ \hline
SGD         & \pb{3.86 (0.08)}   & \pb{19.05 (0.24)}  & \pb{24.01}   & 133.0     \\
Adam        & \textbf{3.64 (0.06)}   & \underline{18.96 (0.21)}  & \textbf{23.45}  & \pb{43.0} \\
AMSGrad     & 3.90 (0.17)        & \underline{18.97 (0.09)}  & \underline{23.46}   & \underline{41.3} \\
AdaBound    & 5.40 (0.24)        & 22.76 (0.17)       & 27.99        & 247.3     \\
AdaShift    & 4.08 (0.11)        & \underline{18.88 (0.06)}  & OOM          & 43.7      \\
RAdam       & 3.89 (0.09)        & 19.15 (0.13)       & \underline{23.60}   & \underline{42.5} \\
AdaBelief   & 3.98 (0.07)        & 19.08 (0.09)       & 24.11        & 44.8      \\
AdamW       & 4.11 (0.17)        & 20.13 (0.22)       & 26.70        & ---       \\ \\[-9pt]\hdashline\\[-8pt]
AvaGrad     & \underline{3.80 (0.02)}   & \textbf{18.76 (0.20)}  & \underline{23.58}   & \textbf{35.3} \\
AvaGradW    & 3.97 (0.02)        & \underline{19.04 (0.37)}  & \underline{23.49}     & ---       \\
\bottomrule
%%%%%%%%%%%%%%%%%%%%%%%%%%%%%%%%%%%%%%%%%%%%%%%%%%%%%%%%%%%%%%%%%%%%%%%%%%%%%%%%%%%%%%%%%%%%%%%%
\end{tabular}
\end{center}
\end{table}

\subsection{Improving Performance}
\label{sec-ava-results-perf}

\paragraph{CIFAR.}

We train a Wide ResNet-28-4 \citep{wide} on the CIFAR dataset \citep{cifar}, which consists of 60,000 RGB images with $32 \times 32$ pixels, and comes with a standard train/test split of 50,000 and 10,000 images. Following \citet{wide}, we normalize images prior to training. We augment the training data with horizontal flips and by sampling $32 \times 32$ random crops after applying a $4$-pixel padding to the images.

The leftmost columns of Table~\ref{tab:results} present results: on CIFAR-10, SGD ($3.86\%$) is outperformed by Adam ($3.64\%$) and AvaGrad ($3.80\%$), while on CIFAR-100 Adam ($18.96\%$), AMSGrad ($18.97\%$), AdaShift ($18.88\%$), AvaGrad ($18.76\%$), and AvaGradW ($19.04\%$) all outperform SGD ($19.05\%$). 
These results disprove the conventional wisdom that adaptive methods are not suited for computer vision tasks such as image classification. While tuning $\epsilon$, a step typically overlooked or skipped altogether in practice, suffices for adaptive methods to outperform SGD (and hence can be a confounding factor in comparative studies), our results also suggest that adaptive methods might require large compute budgets for tuning to perform optimally on some tasks.

\paragraph{ImageNet.}

To further validate that adaptive methods can indeed outperform SGD in settings where they have not been historically successful, we consider the challenging task of training a ResNet-50 \cite{resnet2} on the ImageNet dataset \citep{imagenet}.

SGD yields $24.01\%$ top-1 validation error, underperforming Adam ($23.45\%$), AMSGrad ($23.46\%$), RAdam ($23.60\%$), AvaGrad ($23.58\%$) and AvaGradW ($23.49\%$), \ie 5 out of the 8 adaptive methods evaluated on this task. We were unable to train with AdaShift due to memory constraints: since it keeps a history of past gradients, our GPUs ran out of memory even with a reduced batch size of $128$, meaning that circumventing the issue with gradient accumulation would result in considerably longer training time.

The third column of Table~\ref{tab:results} summarizes the results. In contrast to numerous papers that surpassed the state-of-the-art on ImageNet by training networks with SGD \citep{vgg, googlenet, resnet1, resnet2, wide, resnext}, our results show that adaptive methods can yield superior results in terms of generalization performance as long as $\epsilon$ is appropriately chosen. Most strikingly, Adam outperforms AdaBound, RAdam, and AdaBelief: sophisticated methods whose motivation lies in improving the performance of adaptive methods.

\begin{table}
\caption{
   Test performance of standard models on benchmark tasks, when trained with
   different optimizers.  Gray background indicates the optimization method
   (baseline) adopted by the paper that proposed the corresponding network
   model.  The best task-wise results are in bold, while other improvements
   over the baselines are underlined.  Numbers in parentheses indicate standard
   deviation over three runs.  Across tasks, AvaGrad closely matches or exceeds
   the results delivered by existing optimizers, and offers notable improvement
   in FID when training GANs.
}%
\label{tab:results2}
\begin{center}
\begin{tabular}{@{}lcc@{}}
%%%%%%%%%%%%%%%%%%%%%%%%%%%%%%%%%%%%%%%%%%%%%%%%%%%%%%%%%%%%%%%%%%%%%%%%%%%%%%%%%%%%%%%%%%%%%%%%
\toprule
& \begin{tabular}{@{}c@{}}Penn Treebank\\Test BPC $\downarrow$\end{tabular}
& \begin{tabular}{@{}c@{}}Penn Treebank\\Test BPC $\downarrow$\end{tabular}
\\ 
Model       & 3xLSTM(300)  & 3xLSTM(1000)     \\ \hline
SGD         & 1.403 (0.000)& 1.237 (0.000)     \\
Adam        & \pb{1.378 (0.001)}& \pb{1.182 (0.000)} \\
AMSGrad     & 1.384 (0.001)& 1.187 (0.001) \\
AdaBound    & 4.346 (0.000)& 2.891 (0.041)     \\
AdaShift    & 1.399 (0.006)& 1.199 (0.001)      \\
RAdam       & 1.401 (0.002)& 1.349 (0.003) \\
AdaBelief   & 1.379 (0.001)& 1.198 (0.000)      \\
AdamW       & 1.398 (0.002)& 1.227 (0.003)       \\ \\[-9pt]\hdashline\\[-8pt]
AvaGrad     & \textbf{1.375 (0.000)}& \underline{1.179 (0.000)} \\
AvaGradW    & \textbf{1.375 (0.001)}& \textbf{1.175 (0.000)}       \\
\bottomrule
%%%%%%%%%%%%%%%%%%%%%%%%%%%%%%%%%%%%%%%%%%%%%%%%%%%%%%%%%%%%%%%%%%%%%%%%%%%%%%%%%%%%%%%%%%%%%%%%
\end{tabular}
\end{center}
\end{table}

\paragraph{Language Modeling.}

We now consider a task where state-of-the-art results are achieved by adaptive methods with small values for $\epsilon$ and where SGD has little success: character-level language modeling with LSTMs \citep{lstm} on the Penn Treebank dataset \citep{ptb,ptbc}. We adopt the 3-layer LSTM \citep{lstm} model from \citet{awd} with 300 hidden units per LSTM layer. 

Results in Table~\ref{tab:results} show that only AvaGrad and AvaGradW outperform Adam, achieving test
BPCs of $1.179$ and $1.175$ compared to $1.182$. Combined with the previous results, we validate that, depending on the underlying task, adaptive methods might require vastly different values for $\epsilon$ to perform optimally, but, given enough tuning, are indeed capable of offering best overall results across domains.

We also observe that AdaBound, RAdam, and AdaBelief all visibly underperform Adam in this setting where adaptivity (small $\epsilon$) is advantageous, even given extensive hyperparameter tuning. RAdam, and more noticeably AdaBound, perform poorly in this task. We hypothesize that this is result of RAdam incorporating learning rate warmup, which is not typically employed when training LSTMs, and AdaBound's adoption of SGD-like dynamics early in training \citep{adabound2}.

\paragraph{Generative Adversarial Networks.}

Finally, we consider a task where adaptivity is not only advantageous, but often seen as necessary for successful training: generative modeling with GANs. We train a Geometric GAN \citep{ggan}, \ie a DCGAN model \citep{dcgan} with the hinge loss, on the CIFAR-10 dataset to perform image generation. We do not apply gradient penalties.

Results are summarized in Table~\ref{tab:results}, showing that AvaGrad offers a significant improvement in terms of FID over all other methods, achieving an improvement of 7.7 FID over Adam (35.3 against 43.0). Note that the performance achieved by Adam matches other sources\footnote{\scriptsize{\url{github.com/POSTECH-CVLab/PyTorch-StudioGAN}}} \citep{contragan}, and Adam performed best with $\alpha=0.0002, \epsilon = 10^{-6}$ in our experiments, closely matching the commonly-adopted values in the literature.

\begin{figure}[t]
   \centering
      \includegraphics[width=0.4\linewidth]{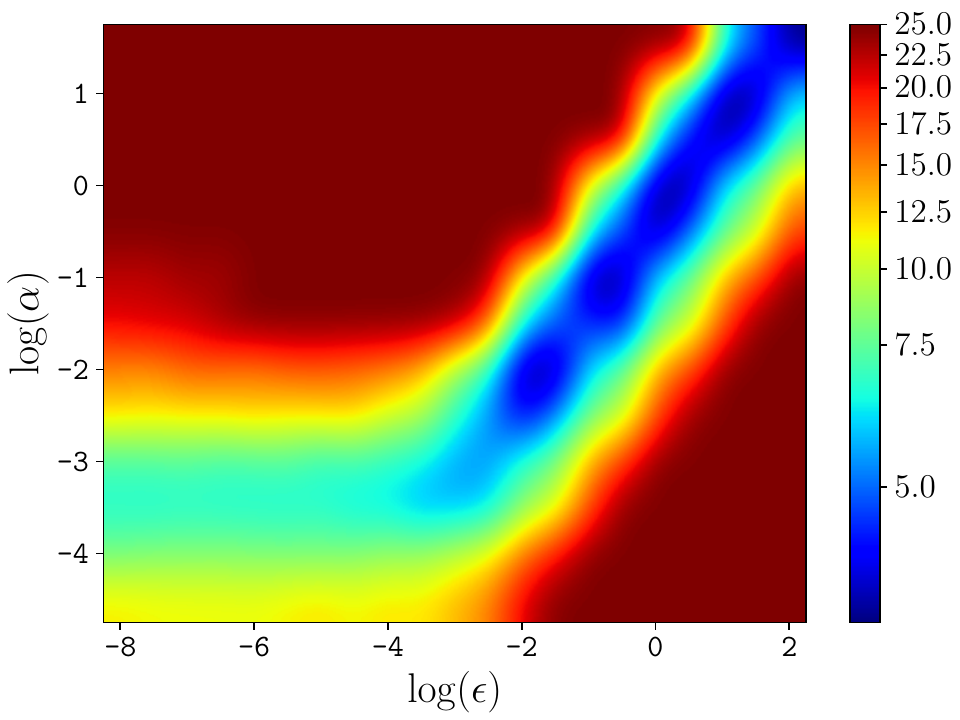}
   \caption{
	  Performance of Adam with different learning rate $\alpha$ and adaptability parameter $\epsilon$, measured in terms of validation error on CIFAR-10 of Wide ResNet 28-4. Best performance is achieved with low adaptability/large $\epsilon$.
   }
\label{fig-heat1}
\end{figure}

\begin{figure}[t]
   \centering
      \includegraphics[width=0.4\linewidth]{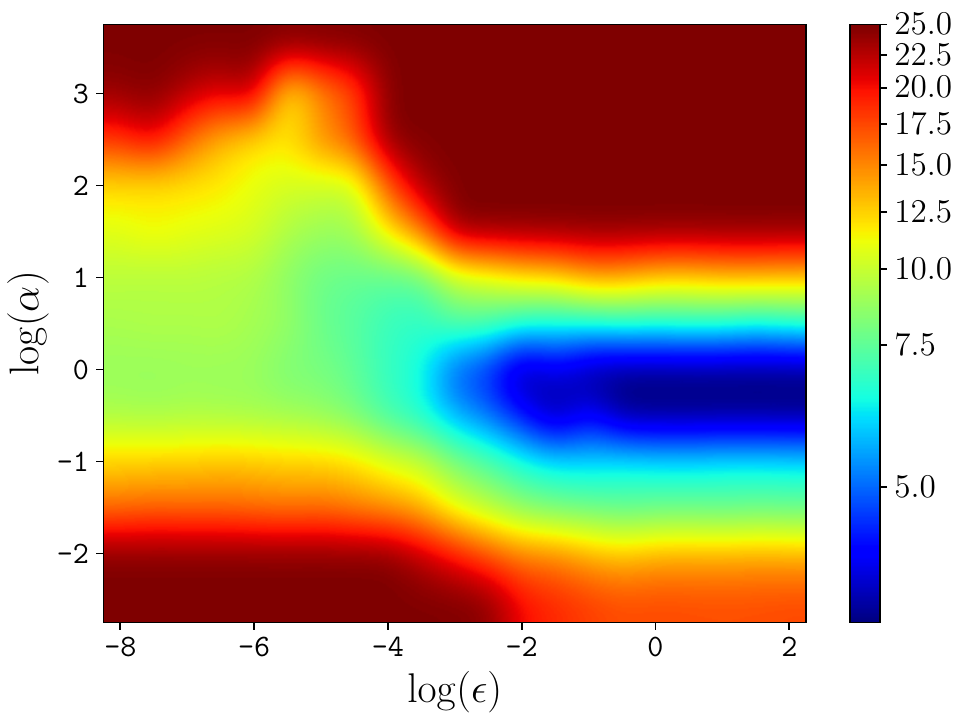}
   \caption{
	  Performance of AvaGrad with different learning rate $\alpha$ and adaptability parameter $\epsilon$, measured in terms of validation error on CIFAR-10 of Wide ResNet 28-4. Best performance is achieved with low adaptability/large $\epsilon$.
   }
\label{fig-heat12}
\end{figure}

\begin{figure}[t]
   \centering
      \includegraphics[width=0.4\linewidth]{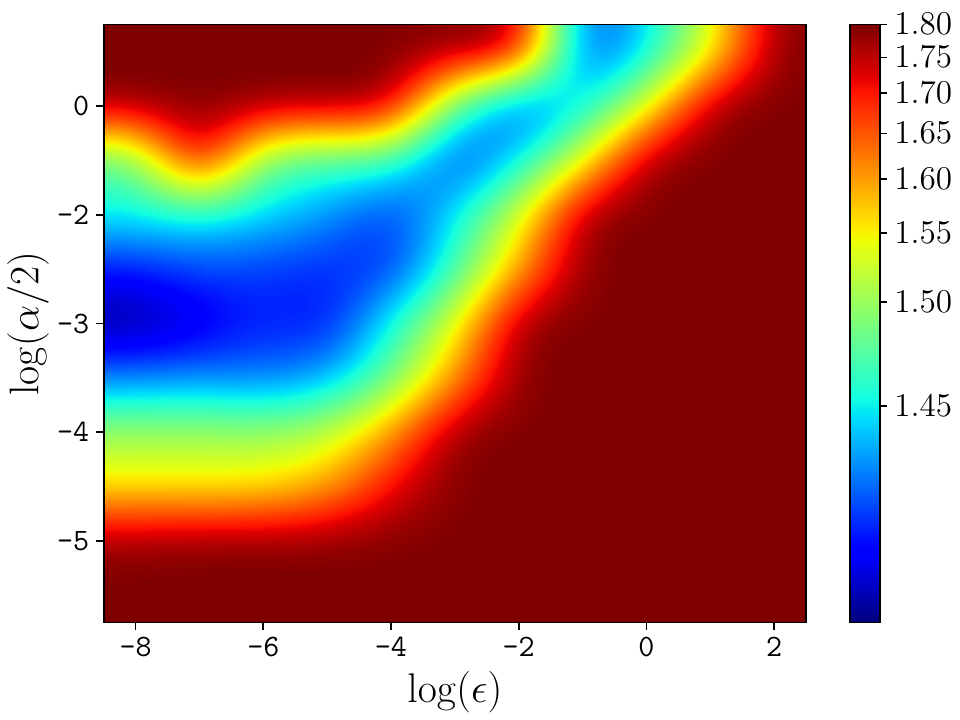}
   \caption{
	  Performance of Adam with different learning rate $\alpha$ and adaptability parameter $\epsilon$, measured in terms of validation BPC (\emph{lower is better}) on PTB of a 3-layer LSTM. Best performance is achieved with high adaptability/small $\epsilon$.
   }
\label{fig-heat2}
\end{figure}

\begin{figure}[t]
   \centering
      \includegraphics[width=0.4\linewidth]{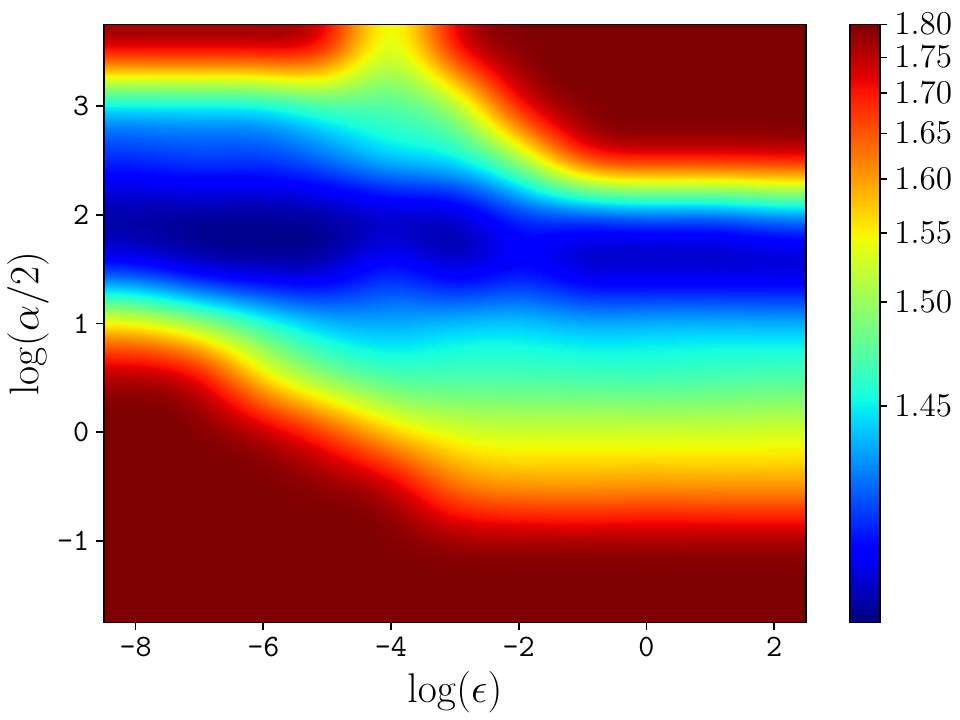}
   \caption{
	  Performance of AvaGrad with different learning rate $\alpha$ and adaptability parameter $\epsilon$, measured in terms of validation BPC (\emph{lower is better}) on PTB of a 3-layer LSTM. Best performance is achieved with high adaptability/small $\epsilon$.
   }
\label{fig-heat22}
\end{figure}

\subsection{Hyperparameter Separability}

The results in the previous section establish the importance of optimizing $\epsilon$ when using adaptive methods, and how not tuning $\epsilon$ can be a confounding factor when comparing different adaptive optimizers.

A key obstacle to proper tuning of $\epsilon$ is its interaction with the learning rate $\alpha$: as discussed in Section~\ref{sec-ava-method-role}, `emulating' SGD with a learning rate $\gamma$ can be done by setting $\alpha = \gamma \epsilon$ in Adam and then increasing $\epsilon$: once its value is large enough (compared to $\currv$), scaling up $\epsilon$ any further will not affect Adam's behavior as long as $\alpha$ is scaled up by the same multiplicative factor. Conversely, when $\epsilon$ is small (compared to components of $\currv$), we have that $\sqrt{\currv} + \epsilon \approx \sqrt{\currv}$, hence decreasing $\epsilon$ even further will not affect the optimization dynamics as long as $\alpha$ remains fixed.

This suggests the existence of two distinct regimes for Adam (and other adaptive methods): an adaptive regime, when $\epsilon$ is small and there is no interaction between $\alpha$ and $\epsilon$, and a non-adaptive regime, when $\epsilon$ is large and the learning rate $\alpha$ must scale linearly with $\epsilon$ to preserve the optimization dynamics. The exact phase transition is governed by $\currv$ \ie the second moments of the gradients, which depends not only on the task but also on the model.

By normalizing the parameter-wise learning rates $\curre$ at each iteration, AvaGrad guarantees that the magnitude of the effective learning rates is independent of $\epsilon$, essentially decoupling it from $\alpha$. With AvaGrad, $\alpha$ governs optimization dynamics in both regimes: when $\epsilon$ is small, changing its value has negligible impact on $\curre$ and $\|\curre\|$, hence the updates will be the same, while in the non-adaptive regime we have that $\curre \approx [\frac1{\epsilon}, \frac1{\epsilon}, \dots]$ and $\|\curre / \sqrt {d}\|_2 \approx \frac{1}{\epsilon}$, hence normalizing $\curre$ yields an all-ones vector regardless of $\epsilon$ (as long as it remains large enough compared to all components of $\currv$).

Figures~\ref{fig-heat1} and \ref{fig-heat12} show the performance of a Wide ResNet 28-4 on CIFAR-10 when trained with Adam and AvaGrad, for different $(\alpha, \epsilon)$ configurations \ie the grid search employed for CIFAR, described in Section~\ref{sec-ava-expsetup-training}. For Adam, the non-adaptive regime is indeed characterized by a linear relation between $\alpha$ and $\epsilon$, while its performance in the adaptive regime depends mostly on $\alpha$ alone. AvaGrad offers decoupling between the two parameters, with $\alpha$ precisely characterizing the non-adaptive regime (\ie the performance is independent of $\epsilon$) while almost fully describing the adaptive regime as well, except for regions close to the phase transition. For each of the $21$ different values of $\epsilon$, AvaGrad performed best with $\alpha=1.0$.

\begin{figure}[t]
   \centering
      \includegraphics[width=0.7\linewidth]{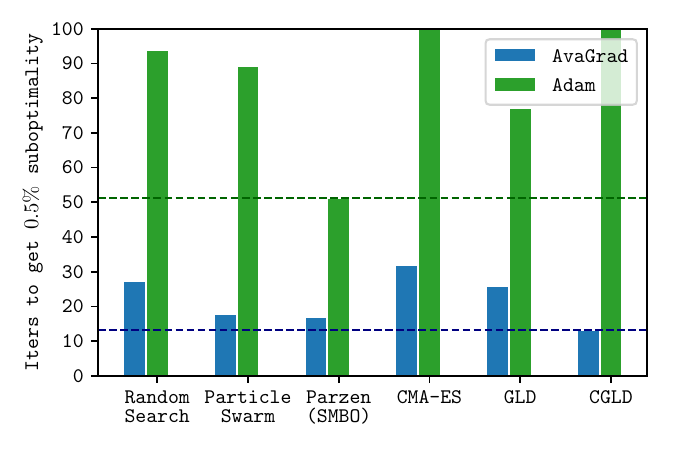}
   \caption{
	Iterations to achieve $0.5\%$ suboptimality, measured in terms of validation accuracy on CIFAR-10, for Adam and AvaGrad when tuning $\alpha$ and $\epsilon$ with various standard hyperparameter optimizers.
   }
\label{fig-hpo}
\end{figure}

\begin{figure}[t]
   \centering
      \includegraphics[width=0.7\linewidth]{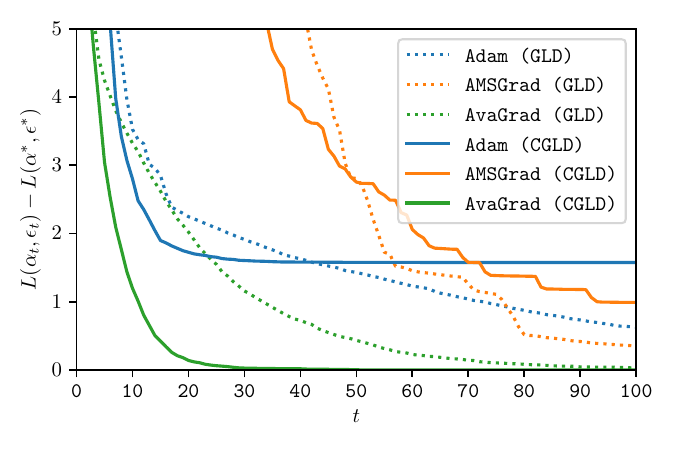}
   \caption{
	  Suboptimality (gap in validation accuracy) when optimizing $\alpha$ and $\epsilon$ with GLD/CGLD, as a function of trials (\ie validation accuracy evaluations for a value of $(\alpha, \epsilon)$): AvaGrad is significantly cheaper to tune than Adam, being especially efficient when adopting Coordinate GLD due to its hyperparameter separability.
   }
\label{fig-hpo2}
\end{figure}

\subsection{Hyperparameter Optimization}
\label{sec-ava-results-hpo}

To assess our hypothesis that AvaGrad offers hyperparameter decoupling, which enables $\alpha$ and $\epsilon$ to be tuned \emph{independently} via two line-search procedures instead of a grid search, we compare tuning costs of Adam and AvaGrad with prominent hyperparameter optimization methods such as Parzen Trees and CMA-ES. We also consider Gradientless Descent (GLD) \citep{gld}, a powerful zeroth-order method.

We frame the task of tuning $\alpha$ and $\epsilon$ as a 2D optimization problem with a $21 \times 21$ discrete domain representing all $(\alpha,\epsilon)$ configurations for CIFAR discussed in Section~\ref{sec-ava-expsetup-training}, with the minimization objective being the error of a Wide ResNet-28-4 on CIFAR-10 when trained with the corresponding $(\alpha,\epsilon)$ values.

Figure \ref{fig-hpo} shows the number of iterations required by different hyperparameter optimizers to achieve $0.5\%$ suboptimality \ie an error at most $0.5\%$ higher than the lowest achieved in the grid. AvaGrad is significantly cheaper to tune than Adam, regardless of the adopted tuning algorithm, including random search -- showing that AvaGrad is able to perform well with a wider range of hyperparameter values.

We also consider a variant of GLD where descent steps on $\alpha$ and $\epsilon$ are performed separately in an alternating manner, akin to coordinate descent \citep{cd1,cd2}. This variant, which we denote by CGLD, is in principle well-suited for problems where variables have independent contributions to the objective, as is approximately the case for AvaGrad. Results are given in Figure \ref{fig-hpo2}: AvaGrad achieves less than $1\%$ suboptimality in 13 iterations when tuned with CGLD, while Adam requires $74$ with GLD. As expected, coordinate-wise updates result in considerably faster tuning for AvaGrad.

\section{Analysis}
\label{sec-ava-analysis}

\begin{figure}[t]
   \centering
      \includegraphics[width=0.7\linewidth]{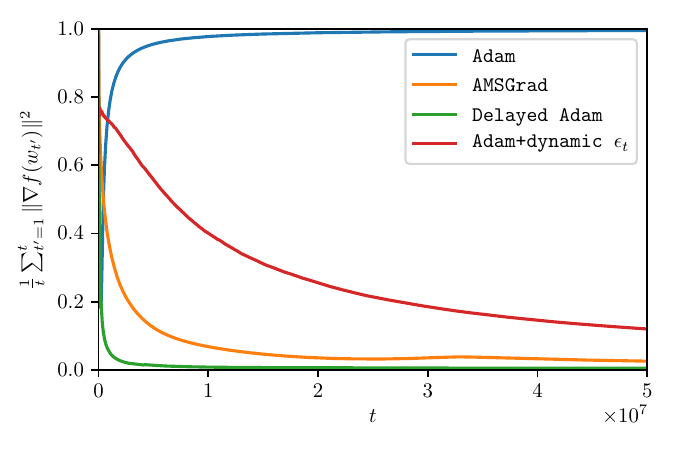}
   \caption{
	  The mean gradient norm as function of the iteration $t$ when optimizing Equation~\eqref{eq:synth}. Matching our theoretical results, Delayed Adam and Adam with dynamic $\epsilon_t$ both converge, while Adam fails to converge.
   }
\label{fig-synth}
\end{figure}

\subsection{Convergence}

Lastly, we demonstrate our convergence results in Theorems~\ref{thm:adamconv} and \ref{thm:dadamconv} experimentally by employing Adam, AMSGrad, Adam with a 1-step delay (Delayed Adam), and Adam with $\epsilon_t=\sqrt{t^3}$, on a synthetic problem with the same form as the one used in the proof of Theorem \ref{thm:adamdiv}:
\begin{equation}
\begin{split}
   \min_{\w \in [0, 1]} f(\w), \quad\quad \fs(\w) =
   \begin{cases}
      999 \frac{\w^2}2, \text{ w.p. } \frac{1}{500} \\
      -\w, \text{ otherwise}
   \end{cases}
   \label{eq:synth}
\end{split}
\end{equation}
where $f(\w) \coloneqq \expec{}{\fs(\w)}$.

This problem admits a stationary point $\w^\star \approx 0.5$, and satisfies
Theorem \ref{thm:adamconv} for $\bone = 0, \btwo = 0.99, \epsilon=10^{-8}$. Figure~\ref{fig-synth} presents $\frac1t \sum_{t'=1}^t \normed{\nabla f(\w_{t'})}^2$ during training, and shows that Adam fails to converge (Theorem \ref{thm:adamdiv}), while both Delayed Adam and dynamic Adam converge successfully (Theorems \ref{thm:dadamconv} and \ref{thm:adamconv}). We attribute the faster convergence of Delayed Adam to the lack of constraints on the parameter-wise learning rates.
\section{Discussion}
\label{sec-ava-discussion}

\subsection{Contributions}
\label{sec-ava-discussion-contribs}

This chapter makes several key contributions:
\begin{itemize}
   \item{We show that Adam can \textit{provably converge} for non-convex problems given a proper tuning of its adaptability parameter $\epsilon$. We address the apparent contradiction with \cite{amsgrad}, providing new insights on the role of $\epsilon$ in terms of convergence and performance of adaptive methods.}
   \item{
      Extensive experiments show that Adam can outperform SGD in tasks where adaptive methods have found little success. As suggested by our theoretical results, tuning $\epsilon$ is key to achieving optimal results with adaptive methods.%
   }%
   \item{
		We propose \textbf{AvaGrad}, a theoretically-motivated adaptive method that decouples the learning rate $\alpha$ and the adaptability parameter $\epsilon$. Quantifying the hyperparameter tuning cost using a zeroth-order method, we observe that AvaGrad is significantly cheaper to tune than Adam.
	 Matching the generalization accuracy of SGD and other adaptive methods across tasks and domains, AvaGrad offers performance dominance given low tuning budgets.
   }
\end{itemize}

\subsection{Research Directions}
\label{sec-ava-discussion-directions}

Our analysis and empirical results suggest several promising directions for future research. Further investigating optimizer design with a focus on reducing or simplifying hyperparameter tuning costs, for example via adaptive schedules or meta-learned strategies, could further decrease the computational costs of finding optimal configurations for adaptive methods. Another direction is to analyze the generalization properties of adaptive methods and SGD. This could lead to a better understanding of how adaptivity affects generalization gaps, and how to best control it during training in order to balance training speed and final generalization performance.

\subsection{Impact}
\label{sec-ava-discussion-impact}

\cite{autoeps} propose a method to directly estimate the optimal value for $\epsilon$ and hence control an optimizer's adaptivity by inspecting gradient statistics during training. Their approach circumvents the need to design a hand-crafted schedule for $\epsilon_t$, and balances training speed and final generalization performance dynamically and in a task-dependent fashion.
\section{Proofs}
\label{sec-ava-proofs}

\section{Full Statement and Proof of Theorem \ref{thm:adamdiv}}
\label{app-adamdiv}

\begin{theorem*}
  For any $\epsilon\geq0$ and constant $\btwot = \btwo \in [0,1)$, there is a
  stochastic optimization problem for which Adam does not converge to
  a stationary point.
\end{theorem*}

\begin{proof}
Consider the following stochastic optimization problem:
\begin{equation}
   \min_{\w \in [0,1]} f(\w) \coloneqq \expec{s \sim \dist}{\fs(\w)} \qquad
   \fs(\w) =
   \begin{cases}
      C \frac{\w^2}2,
         \quad \text{with probability }
         \quad p \coloneqq \frac{1+\delta}{C+1} \\
      -\w, \quad \text{otherwise}
   \end{cases} \,,
\end{equation}
where $\delta$ is a positive constant to be specified later, and $C > \frac{1-p}p > 1 + \frac{\epsilon}{\w_1 \sqrt{1 - \btwo}}$ is another constant that can depend on $\delta, \btwo$ and $\epsilon$, and will also be determined later. Note
that $\nabla f(\w) = p C \w - (1-p)$, and $f$ is minimized at
$\w^\star = \frac{1-p}{Cp} = \frac{C-\delta}{C(1+\delta)}$.

The proof follows closely from \citet{amsgrad}.  We assume w.l.o.g. that $\bone = 0$. We first consider the difference between two consecutive iterates computed by Adam with a constant learning rate $\alpha$:
\begin{equation}
   \Delta_t
      = \nextw - \currw
      = - \alpha \frac{\currg}{\sqrt{\currv} + \epsilon}
      = - \alpha \frac{\currg}{\sqrt{\btwo \prevv + (1-\btwo) \currg^2} +
                             \epsilon} \,,
\end{equation}
and then we proceed to analyze the expected change in iterates divided by the learning rate. First, note that with probability $p$ we have $\currg = \nabla (C \frac{\currw^2}{2}) = C \currw$, and while $\currg = \nabla (-w) = -1$ with probability $1-p$. Therefore, we have
\begin{equation}
\begin{split}
   \frac{\expec{}{\Delta_t}}{\alpha}
     &= \frac{\expec{}{\nextw - \currw} }{\alpha}
      = - \expec{}{\frac{g_t}{\sqrt{\btwo \prevv + (1-\btwo) g_t^2} + \epsilon}} \\
     &=   p  \expec{}{ \underbrace{\frac{- C \currw}{\sqrt{\btwo \prevv + (1-\btwo) C^2 \currw^2} + \epsilon}}_{T_1}} +
       (1-p) \expec{}{ \underbrace{\frac{1}{\sqrt{\btwo \prevv + (1-\btwo)} + \epsilon}}_{T_2}} \,,
\end{split}
\label{eq:onlinebound}
\end{equation}
where the expectation is over all the randomness in the algorithm up to time
$t$, as all expectations to follow in the proof. We will proceed by computing lower bounds for the terms $T_1$ and $T_2$ above. Note that $T_1 = 0$ for
$\currw = 0$, while for $\currw > 0$ we can bound $T_1$ by
\begin{equation}
\begin{split}
   T_1  = \frac{- C \currw}{\sqrt{\btwo \prevv + (1-\btwo) C^2 \currw^2} + \epsilon}\geq \frac{- C \currw}{\sqrt{(1-\btwo) C^2 \currw^2}}
          = \frac{- 1 }{\sqrt{1-\btwo}} \,.
\end{split}
\end{equation}
Combining the cases $\currw = 0$ and $\currw > 0$ (note that the feasible region is $w \in [0,1]$), we have that, generally,
$T_1 \geq \min(0, \frac{- 1 }{\sqrt{1-\btwo}}) = \frac{- 1 }{\sqrt{1-\btwo}}$.

Next, we bound the expected value of $T_2$ using Jensen's inequality coupled with the convexity of $x^{-1/2}$ as
\begin{equation}
\begin{split}
   \expec{}{T_2} = \expec{}{\frac{1}{\sqrt{\btwo \prevv + 1-\btwo} + \epsilon}} \geq
      \frac{1}{\sqrt{\btwo \expec{}{\prevv} + 1-\btwo} + \epsilon}
\end{split} \,.
\label{eq:thm1_1}
\end{equation}

Let us consider $\expec{}{\prevv}$ now. Note that
\begin{equation}
\begin{split}
  \prevv &= \btwo v_{t-2} + (1-\btwo) \prevg^2 \\ 
  &= \btwo \left( \btwo v_{t-3} + (1-\btwo) g_{t-2}^2 \right) + (1-\btwo) \prevg^2  \\
  & \quad = \btwo^2 v_{t-3} + \btwo (1-\btwo) g_{t-2}^2 + (1-\btwo) \prevg^2\\ 
  & = \btwo^2 \left( \btwo v_{t-4} + (1-\btwo) g_{t-3}^2\right) + \btwo (1-\btwo) g_{t-2}^2 + (1-\btwo) \prevg^2 \\
  &\quad = \btwo^3 v_{t-4} + \btwo^2 (1-\btwo) g_{t-3}^2 + \btwo (1-\btwo) g_{t-2}^2 + (1-\btwo) \prevg^2 \\
  & \quad\quad\quad \vdots \\
  & = \btwo^{t-1} v_0 + \btwo^{t-2} (1-\btwo) g_1^2 + \btwo^{t-3} (1-\btwo) g_2^2 + \dots + (1-\btwo) \prevg^2 \\
  & = (1-\btwo) \sum_{i=1}^{t-1} \btwo^{t-i-1} g_i^2 \,,
\end{split}
\end{equation}
where we used the fact that $v_0 = 0$ (i.e. the second-moment estimate is initialized as zero).

Taking the expectation of the above expression for $\prevv$, we get
\begin{equation}
\begin{split}
   \expec{}{\prevv}
   &= (1-\btwo) \sum_{i=1}^{t-1} \btwo^{t-i-1} \expec{}{g_i^2} \\
   &= (1-\btwo) \sum_{i=1}^{t-1} \btwo^{t-i-1} \left(1 -p + p C^2 \expec{}{\currw^2} \right) \,,
\end{split}
\end{equation}
where we can use the fact that $\currw \in [0,1]$, so $\currw^2 \leq 1$ to get
\begin{equation}
\begin{split}
   \expec{}{\prevv} &\leq (1-\btwo) \sum_{i=1}^{t-1} \btwo^{t-i-1} \left(1 -p + p C^2 \right) \\
   &= (1-\btwo) \left(1 -p + p C^2 \right) \sum_{i=1}^{t-1} \btwo^{t-i-1}  \\
   &= (1-\btwo) \left(1 -p + p C^2 \right) \sum_{i=0}^{t-2} \btwo^{i}  \\
   &= \left(1 -p + p C^2 \right) \sum_{i=0}^{t-2} \left(\btwo^{i} - \btwo^{i+1}\right)  \\
   &= \left(1 -p + p C^2 \right) \left(1 - \btwo^{t-1}\right)  \\
   &\leq (1+\delta) C^2 \,,
\end{split}
\end{equation}
where $\sum_{i=0}^{t-2} \left(\btwo^{i} - \btwo^{i+1}\right) = 1 - \btwo^{t-1}$ follows from the fact that the sum telescopes.

Plugging the above bound in \eqref{eq:thm1_1} yields
\begin{equation}
\begin{split}
   \expec{}{T_2} \geq \frac{1}{\sqrt{\btwo (1+\delta) C + 1-\btwo} + \epsilon}
\end{split}
\end{equation}

Combining the bounds for $T_1$ and $T_2$ in \eqref{eq:onlinebound} gets us that
\begin{equation}
\begin{split}
   \frac{\expec{}{\Delta_t}}{\alpha}
   &\geq
      \frac{1+\delta}{C+1} \frac{- 1 }{\sqrt{1-\btwo}} +
      \left(1-\frac{1+\delta}{C+1}\right)
      \frac{1}{\sqrt{\btwo (1+\delta) C + 1-\btwo} + \epsilon}
\end{split}
\end{equation}

Now, recall that $\w^\star = \frac{C-\delta}{C(1+\delta)}$, so for $C$ sufficiently large in comparison to $\delta$ we get $\w^\star \approx \frac{1}{1+\delta}$. On the other hand, the above quantity can be made non-negative for large enough $C$, so $\expec{}{\currw} \geq \expec{}{w_{t-1}} \geq \dots \geq \w_1$.  In other words, Adam
will, in expectation, update the iterates towards $\w = 1$ even though the stationary point is $w^* \approx \frac{1}{1+\delta}$ and we have $\normed{\nabla f(1)}^2 = \delta$ at $w=1$. Setting $\delta=1$, for example,
implies that
$\lim_{T \to \infty}
   \frac1T \sum_{t=1}^T \expec{}{\normed{\nabla f(\currw)}^2} = 1$, and hence Adam presents nonconvergence in terms of stationarity. To see that $w=1$ is not a stationary point due to the feasibility constraints, check that $\nabla f(1) = 1 > 0$: that is, the negative gradient points \textit{towards} the feasible region.
\end{proof}

%%%%%%%%%%%%%%%%%%%%%%%%%%%%%%%%%%%%%%%%%

\section{Technical Lemmas}

This section presents intermediate results that are used in the proofs given in the next sections.

For simplicity we adopt the following notation for all following results: 
\begin{equation}
	\highinft = \max_i \currei \quad\quad \lowinft = \min_i \currei \,,
\end{equation}
where $\curre \in \R^d$ denotes the parameter-wise learning rates computed at iteration $t$ (the method being considered and consequently the exact expression for $\curre$ will be specified in each result).

For the following Lemmas we rely extensively on the assumption that $\norm{\nabla f_s(\w)}_\infty \leq \gradb$ for some constant $\gradb$, and also that this assumption implies that there exists $\gradbtwo$ such that $\norm{\nabla f_s(\w)} \leq \gradbtwo$ for all $s \in \supp$ and $w \in \R^d$, which can be seen by noting that
\begin{equation}
	\norm{\nabla f_s(\w)} = \left( \sum_{i=1}^d (\nabla f_s(\w))_i^2 \right)^{\frac12} \leq \left( d \norm{\nabla f_s (\w)}_\infty^2  \right)^{\frac12} = \sqrt d \cdot \norm{\nabla f_s(\w)}_\infty \leq \sqrt d \cdot \gradb \,,
\end{equation}
hence such constant $\gradbtwo$ must exist as any $\gradbtwo \geq \sqrt d \cdot \gradb$ satisfies $\norm{\nabla f_s(\w)} \leq \gradbtwo$.

% m bound

\begin{lemma}
	Assume that there exists a constant $\gradb$ such that $\norm{\nabla f_s(\w)}_\infty \leq \gradb$ for all $s \in \supp$ and $w \in \R^d$, and let $\gradbtwo$ be a constant such that $\norm{\nabla f_s(\w)} \leq \gradbtwo$ for all $s \in \supp$ and $w \in \R^d$. Moreover, assume that $\bonet \in [0,1)$ for all $t \in \N$.
	
	 Let $\currm \in \R^d$ be given by 
\[\currm = \bonet \prevm + (1 - \bonet) \currg
   \quad \textrm{and} \quad m_0 = 0 \,,\]
   where $\bonet \in [0,1)$ for all $t \in \N$.
	
	Then, we have \[ \norm{\currm}_\infty \leq \gradb \quad \text{ and } \quad \norm{\currm} \leq \gradbtwo \] for all $t \in \N$ and all possible sample sequences $(s_1, \dots, s_t) \in \supp^t$.
\label{lem:mbound}
\end{lemma}

\begin{proof}

Assume for the sake of contradiction that $\norm{\currm}_\infty > \gradb$ for some $t \in \N$ and some sequence of samples $(s_1, \dots, s_t)$. Moreover, assume w.l.o.g. that $\norm{m_{t'}}_\infty \leq \gradb$ for all $t' \in \{1, \dots, t-1\}$ and note that there is no loss of generality since $(m_{t'})_{t'=0}^t$ must indeed have a first element that satisfies $\norm{m_{t'}}_\infty > \gradb$, which cannot be $m_0$ since we have $m_0 = 0$ by definition.

Then, we have that $m_{t,i} > \gradb$ for some $i \in [d]$, but
\begin{equation}
\begin{split}
  m_{t,i} &= \bonet m_{t-1,i} + (1-\bonet) \currgi \\ 
  & \leq \bonet \norm{m_{t-1}}_\infty + (1-\bonet) \norm{\currg}_\infty \\
  & \leq \bonet \gradb + (1-\bonet) \gradb \\
  & = \gradb \,,
\end{split}
\end{equation}
where we used $\bonet \in [0,1)$ and the assumptions $\norm{\prevm}_\infty \leq \gradb$ and $\norm{\currg}_\infty \leq \gradb$.

To show that $\norm{\currm} \leq \gradbtwo$, note that if we assume $\norm{\currm} > \gradbtwo$ and $\norm{m_{t'}} \leq \gradbtwo$ for all $t' \in \{1, \dots, t-1\}$, we again get a contradiction since, by the triangle inequality,
\begin{equation}
\begin{split}
  \norm{\currm} &= \norm{\bonet \prevm + (1 - \bonet) \currg}  \\ 
  & \leq \bonet \norm{\prevm} + (1-\bonet) \norm{\currg} \\
  & \leq \bonet \gradbtwo + (1-\bonet) \gradbtwo \\
  & = \gradbtwo \,,
\end{split}
\end{equation}
therefore it must indeed follow that $\norm{\currm} \leq \gradbtwo$.

\end{proof}

% v bound

\begin{lemma}
	Assume that there exists a constant $\gradb$ such that $\norm{\nabla f_s(\w)}_\infty \leq \gradb$ for all $s \in \supp$ and $w \in \R^d$, and let $\gradbtwo$ be a constant such that $\norm{\nabla f_s(\w)} \leq \gradbtwo$ for all $s \in \supp$ and $w \in \R^d$. Moreover, assume that $\btwot \in [0,1)$ for all $t \in \N$.
	
	 Let $\currv \in \R^d$ be given by 
\[\currv = \btwot \prevv + (1-\btwot) \currg^2
   \quad \textrm{and} \quad v_0 = 0 \,,\]
   where $\btwot \in [0,1)$ for all $t \in \N$.
	
	Then, we have \[ \norm{\currv}_\infty \leq \gradb^2 \quad \text{ and } \quad \norm{\currv} \leq \gradbtwo^2 \] for all $t \in \N$ and all possible sample sequences $(s_1, \dots, s_t) \in \supp^t$.
\label{lem:vbound}
\end{lemma}

\begin{proof}
The proof follows closely from the one of Lemma \ref{lem:mbound}. Assume for the sake of contradiction that there exists $t \in \N$ and some sequence of samples $(s_1, \dots, s_t)$ such that $\norm{\currv}_\infty > \gradb^2$ and $\norm{v_{t'}}_\infty \leq \gradb^2$ for all $t' \in \{1, \dots, t-1\}$.

Then $v_{t,i} > \gradb^2$ for some $i \in [d]$ but
\begin{equation}
\begin{split}
  v_{t,i} &= \btwot v_{t-1,i} + (1-\btwot) \currgi^2 \\ 
  & \leq \btwot \norm{\prevv}_\infty + (1-\btwot) \norm{\currg}_\infty^2 \\
  & \leq \btwot \gradb^2 + (1-\btwot) \gradb^2 \\
  & = \gradb^2 \,,
\end{split}
\end{equation}
where we used $\btwot \in [0,1)$ and the assumptions $\norm{\prevv}_\infty \leq \gradb^2$ and $\norm{\currg}_\infty \leq \gradb$, which raises a contradiction and shows that indeed $\norm{\currv}_\infty \leq \gradb^2$.

For the $\ell_2$ case, assume that $\norm{\currv} > \gradbtwo^2$ and $\norm{v_{t'}} \leq \gradbtwo^2$ for all $t' \in \{1, \dots, t-1\}$, which yields
\begin{equation}
\begin{split}
  \norm{\currv} &= \norm{\btwot \prevv + (1 - \btwot) \currg^2}  \\ 
  & \leq \btwot \norm{\prevv} + (1-\btwot) \norm{\currg^2} \\
  & \leq \btwot \gradbtwo^2 + (1-\btwot) \gradbtwo^2 \\
  & = \gradbtwo^2 \,,
\label{eq-ph1}
\end{split}
\end{equation}
where we used the assumption $\norm{\currg} \leq \gradbtwo$ which also implies that
\begin{equation}
\begin{split}
	\norm{\currg^2} &= \left[\sum_{i=1}^d \currgi^4 \right]^{\frac12}\\
	& \leq \left[\sum_{i=1}^d \currgi^4 + \sum_{i=1}^d \sum_{j=1}^d \currgi^2 g_{t,j}^2 \right]^{\frac12} \\
	&= \left[ \left( \sum_{i=1}^d \currgi^2 \right)^2 \right]^{\frac12} \\
	&= \left[ \left( \sum_{i=1}^d \currgi^2 \right)^{\frac12} \right]^2 \\
	&\leq \gradbtwo^2 \,.
\end{split}
\end{equation}

Checking that \eqref{eq-ph1} yields a contradiction completes the argument.

\end{proof}

% me bound

\begin{lemma}
	Under the same assumptions of Lemma \ref{lem:mbound}, we have
	\begin{equation}
		\norm{m_{t'} \odot \curre} \leq \min \left(\gradb \norm{\curre}, \gradbtwo \highinft  \right) \,,
	\end{equation}	
for all $t, t' \in \N$ and all possible sample sequences $(s_1, \dots, s_{\max(t,t')})$.
\label{lem:mebound}
\end{lemma}

\begin{proof}
By definition,
\begin{equation}
\begin{split}
  \norm{m_{t'} \odot \curre}^2 &= \sum_{i=1}^d m_{t,i}^2 \cdot \currei^2 \\
  & \leq \sum_{i=1}^d (\max_{j \in [d]} m_{t',j}^2) \cdot \currei^2 \\
  & \leq \norm{m_{t'}}_\infty^2 \sum_{i=1}^d \currei^2 \\
  & = \norm{m_{t'}}_\infty^2 \cdot \norm{\curre}^2 \,,
\end{split}
\end{equation}
hence invoking Lemma \ref{lem:mbound} and taking the square root yields $\norm{m_{t'} \odot \curre} \leq \gradb \norm{\curre}$.

Additionally, we have
\begin{equation}
\begin{split}
  \norm{m_{t'} \odot \curre}^2 & \leq \sum_{i=1}^d m_{t',i}^2 (\max_{j \in [d]} \eta_{t,j}^2) \\
  & \leq \norm{\curre}_\infty^2 \sum_{i=1}^d m_{t',i}^2 \\
  & = \norm{\curre}_\infty^2 \cdot \norm{m_{t'}}^2 \,,
\end{split}
\end{equation}
hence recalling that $\norm{\curre}_\infty = \highinft$ and by Lemma \ref{lem:mbound} we get $\norm{m_{t'} \odot \curre} \leq \gradbtwo \highinft$.

Combining the two bounds completes the proof.
\end{proof}

% dot bound

\begin{lemma}
	Under the same assumptions of Lemma \ref{lem:mbound}, we have
	\begin{equation}
		\left\langle \nabla f(\currw), \currm \odot \curre \right\rangle \geq (1 - \bonet) \left\langle \nabla f(\currw), \currg \odot \curre \right\rangle - \bonet \gradbtwo \normed{\prevm \odot \curre} \,,
	\end{equation}
	for all $t \in \N$ and all possible sample sequences $(s_1, \dots, s_t) \in \supp^t$.
\label{lem:dotbound}
\end{lemma}

\begin{proof}
Using the definition of $\currm$, we have
\begin{equation}
\begin{split}
   \left\langle \nabla f(\currw), \currm \odot \curre \right\rangle 
   & = \Big\langle \nabla f(\currw), \big(\bonet \prevm + (1-\bonet) \currg\big) \odot \curre \Big\rangle  \\ 
   &= (1 - \bonet) \left\langle \nabla f(\currw), \currg \odot \curre \right\rangle + \bonet \left\langle \nabla f(\currw), \prevm \odot \curre \right\rangle \\
   &\geq (1 - \bonet) \left\langle \nabla f(\currw), \currg \odot \curre \right\rangle - \bonet \normed{\nabla f(\currw)} \cdot \normed{\prevm \odot \curre}\,,
\end{split}
\label{eq-ph2}
\end{equation}
where we used Cauchy-Schwarz in the last step.

Next, by Jensen's inequality and the fact that $\norm{\cdot}$ is convex we have, for all $w \in \R^d$,
\begin{equation}
	\norm{\nabla f(\w)} = \norm{\expec{s}{\nabla f_s(\w)}} \leq \expec{s}{\normed{\nabla f_s(\w)}} \leq \expec{s}{\gradbtwo} = \gradbtwo \,.
\end{equation}

Applying this bound in \eqref{eq-ph2} yields the desired inequality.

\end{proof}

\begin{lemma}
	Assume that there exists a constant $\gradb$ such that $\norm{\nabla f_s(\w)}_\infty \leq \gradb$ for all $s \in \supp$ and $w \in \R^d$, and let $\gradbtwo$ be a constant such that $\norm{\nabla f_s(\w)} \leq \gradbtwo$ for all $s \in \supp$ and $w \in \R^d$. Moreover, assume that $\bonet \in [0,1)$ and $\bonet \leq \beta_{1,t-1}$ for all $t \in \N$.
	
	If $\curre$ is independent of $s_t$ for all $t \in \N$, i.e. $P(\curre = \eta, s_t = s) = P(\curre = \eta) P(s_t = s)$ for all $\eta \in \R^d, s \in \supp$, then
	\begin{equation}
		\expec{s_t}{\left\langle \nabla f(\currw), \currm \odot \curre \right\rangle} \geq (1 - \bone) \lowinft \norm{\nabla f (\currw)}^2 - \bonet \gradbtwo \normed{\prevm \odot \curre} \,,
	\end{equation}
	for all $t \in \N$ and all possible sample sequences $(s_1, \dots, s_t) \in \supp^t$.
\label{lem:edotbound}
\end{lemma}

\begin{proof}
From Lemma \ref{lem:dotbound} we have that
	\begin{equation}
		\left\langle \nabla f(\currw), \currm \odot \curre \right\rangle \geq (1 - \bonet) \left\langle \nabla f(\currw), \currg \odot \curre \right\rangle - \bonet \gradbtwo \normed{\prevm \odot \curre} \,.
	\end{equation}
	
Then, taking the expectation over the draw of $s_t \in \supp$ and recalling that $\currw$, and hence also $\nabla f (\currw)$, is computed from $(s_1, \dots, s_{t-1})$,
	\begin{equation}
		\expec{s_t}{\left\langle \nabla f(\currw), \currm \odot \curre \right\rangle} \geq (1 - \bonet) \left\langle \nabla f(\currw), \expec{s_t}{\currg \odot \curre} \right\rangle - \bonet \gradbtwo \expec{s_t}{\normed{\prevm \odot \curre}} \,.
	\label{eq-ph5}
	\end{equation}

Now, note that since we assume that $\curre$ is independent of $s_t$, we get
\begin{equation}
	\expec{s_t}{\currg \odot \curre} = \curre \odot \expec{s_t}{\currg} = \curre \odot \expec{s_t}{\nabla f_{s_t} (\currw)} = \curre \odot \nabla f (\currw) \,,
\label{eq-ph4}
\end{equation}
and also
\begin{equation}
	\expec{s_t}{\normed{\prevm \odot \curre}} = \normed{\prevm \odot \curre} \,.
\label{eq-ph3}
\end{equation}

Combining \eqref{eq-ph3} and \eqref{eq-ph4} into \eqref{eq-ph5} yields
	\begin{equation}
		\expec{s_t}{\left\langle \nabla f(\currw), \currm \odot \curre \right\rangle} \geq (1 - \bonet) \left\langle \nabla f(\currw), \nabla f(\currw) \odot \curre \right\rangle - \bonet \gradbtwo \normed{\prevm \odot \curre} \,.
	\label{eq-ph6}
	\end{equation}
	
	Moreover, we have
	\begin{equation}
	\begin{split}
		\left\langle \nabla f(\currw), \nabla f(\currw) \odot \curre \right\rangle 
		&= \sum_{i=1}^d (\nabla f(\currw))_i (\nabla f(\currw))_i \currei \\
		&= \sum_{i=1}^d (\nabla f(\currw))^2_i \currei \\
		&\geq \sum_{i=1}^d (\nabla f(\currw))^2_i (\min_j \eta_{t,j}) \\
		&= \lowinft \sum_{i=1}^d (\nabla f(\currw))^2_i \\
		&= \lowinft \norm{\nabla f(\currw)}^2 \,,
	\end{split}
	\end{equation}
	which, when applied to \eqref{eq-ph6} yields
	\begin{equation}
		\expec{s_t}{\left\langle \nabla f(\currw), \currm \odot \curre \right\rangle} \geq (1 - \bonet)  \lowinft \norm{\nabla f(\currw)}^2 - \bonet \gradbtwo \normed{\prevm \odot \curre} \,,
	\end{equation}
	where we also used that $\bonet \in [0,1)$ and $\lowinft \geq 0$. Using the fact that $\bonet \leq \beta_{1,t-1} \leq \bone$ for all $t \in \N$ and hence $1 - \bonet \geq 1 - \bone$ yields the desired inequality.

\end{proof}

\section{Proof of \thmref{thm:dadamconv}}
\label{app:dadamconv}

We organize the proof as follows: we first prove an intermediate result (Lemma \ref{lem:intproof2}) and split the proof of the bounds in \eqref{eq-simplerate} and \eqref{eq-refrate} in two, where the latter can be seen as a refinement of \eqref{eq-simplerate} given the additional assumption that $Z \coloneqq \sum_{t=1}^T \alpha_t \min_i \currei$ is independent of each $s_t$.

Throughout the proof we use the following notation for clarity: 
\begin{equation}
	\highinft = \max_i \currei \quad\quad \lowinft = \min_i \currei \,.
\end{equation}

\begin{lemma}
	Assume that $f$ is $\smooth$-smooth, lower-bounded by some $f^*$ (i.e. $f^* \leq f(\w)$ for all $\w \in \R^d$), and that there exists a constant $\gradb$ such that $\norm{\nabla f_s(\w)}_\infty \leq \gradb$ for all $s \in \supp$ and $w \in \R^d$, and let $\gradbtwo$ be a constant such that $\norm{\nabla f_s(\w)} \leq \gradbtwo$ for all $s \in \supp$ and $w \in \R^d$.
	
	Consider any optimization method that performs updates following \begin{equation}
    \nextw = \currw - \curra \cdot \curre \odot \currm \,,
\end{equation}
where we further assume assume that for all $t \in \N$ we have $\curra \geq 0$, $\bonet = \frac{\bone}{\sqrt t}$ for some $\bone \in [0,1)$, and $\currei \geq 0$ for all $i \in [d]$.

If $\curre$ is independent of $s_t$ for all $t \in \N$, i.e. $P(\curre = \eta, s_t = s) = P(\curre = \eta) P(s_t = s)$ for all $\eta \in \R^d, s \in \supp$, then
	\begin{equation}
\begin{split}
   \sum_{t=1}^T \curra \lowinft \normed{\nabla f(\currw)}^2 \leq \frac{1}{1-\bone} \Bigg( \sum_{t=1}^T \left( f(\currw) -  \expec{s_t}{f(\nextw)} \right) &+ \sum_{t=1}^T \curra \bonet \gradbtwo \normed{\prevm \odot \curre} \\
   & + \frac{\smooth}2 \sum_{t=1}^T \curra^2 \expec{s_t}{\normed{ \currm \odot \curre}^2} \Bigg) \,,
\end{split}
	\end{equation}
	for all $T \in \N$ and all possible sample sequences $(s_1, \dots, s_T) \in \supp^T$.
\label{lem:intproof2}
\end{lemma}

\begin{proof}

We start from the assumption that $f$ is $\smooth$-smooth, which yields
\begin{equation}
   f(\nextw) \leq
      f(\currw) +
      \langle \nabla f(\currw), \nextw - \currw \rangle +
      \frac{\smooth}2 \normed{\nextw - \currw}^2 \,.
\end{equation}

Plugging the update expression $\nextw = \currw - \curra \cdot \curre \odot \currm$,
\begin{equation}
\begin{split}
   f(\nextw) & \leq f(\currw) - \curra \left\langle \nabla f(\currw), \currm \odot \curre \right\rangle + \frac{\curra^2 \smooth}2 \normed{ \currm \odot \curre}^2 \,.
\end{split}
\end{equation}

Now, taking the expectation over the random sample $s_t \in \supp$, we get
\begin{equation}
\begin{split}
   \expec{s_t}{f(\nextw)} & \leq f(\currw) - \curra \expec{s_t}{\left\langle \nabla f(\currw), \currm \odot \curre \right\rangle} + \frac{\curra^2 \smooth}2 \expec{s_t}{\normed{ \currm \odot \curre}^2} \,,
\end{split}
\end{equation}
where we used the fact that $\currw$ and $\curra$ are not functions of of $s_t$ -- in particular, recall that $\currw$ is deterministically computed from $(s_1, \dots, s_{t-1})$.

Using the assumption that $\curre$ is independent of $s_t$ and applying Lemma \ref{lem:edotbound}, we get
\begin{equation}
\begin{split}
   \expec{s_t}{f(\nextw)} \leq f(\currw) - \curra (1 - \bone) \lowinft \norm{\nabla f (\currw)}^2 &+ \curra \bonet \gradbtwo \normed{\prevm \odot \curre} \\
   &+ \frac{\curra^2 \smooth}2 \expec{s_t}{\normed{ \currm \odot \curre}^2} \,,
\end{split}
\end{equation}
which can be re-arranged into
\begin{equation}
\begin{split}
   \curra \lowinft \normed{\nabla f(\currw)}^2 \leq \frac{1}{1-\bone} \Bigg( f(\currw) -  \expec{s_t}{f(\nextw)} &+ \curra \bonet \gradbtwo \normed{\prevm \odot \curre} \\
   & + \frac{\curra^2 \smooth}2 \expec{s_t}{\normed{ \currm \odot \curre}^2} \Bigg)\,,
\end{split}
\end{equation}
where we used the assumption that $\bone \in [0,1)$, hence $1-\bone > 0$ which was used to divide both sides of the inequality.

Now, summing over $t=1$ to $T$,
\begin{equation}
\begin{split}
   \sum_{t=1}^T \curra \lowinft \normed{\nabla f(\currw)}^2 \leq \frac{1}{1-\bone} \Bigg( \sum_{t=1}^T \left( f(\currw) -  \expec{s_t}{f(\nextw)} \right) &+ \sum_{t=1}^T \curra \bonet \gradbtwo \normed{\prevm \odot \curre} \\
   & + \sum_{t=1}^T \frac{\curra^2 \smooth}2 \expec{s_t}{\normed{ \currm \odot \curre}^2} \Bigg) \,,
\end{split}
\end{equation}
which yields the desired result.

\end{proof}

\subsection{Proof of the first guarantee (Equation \ref{eq-simplerate})}

\begin{proof}

We start from the bound given in Lemma \ref{lem:intproof2}:
\begin{equation}
\begin{split}
   \sum_{t=1}^T \curra \lowinft \normed{\nabla f(\currw)}^2 \leq \frac{1}{1-\bone} \Bigg( \sum_{t=1}^T \left( f(\currw) -  \expec{s_t}{f(\nextw)} \right) & + \sum_{t=1}^T \curra \bonet \gradbtwo \normed{\prevm \odot \curre} \\
   & + \frac{\smooth}2 \sum_{t=1}^T \curra^2 \expec{s_t}{\normed{ \currm \odot \curre}^2} \Bigg) \,.
\end{split}
\end{equation}

Now, using Lemma \ref{lem:mebound} to upper bound both $\normed{ \prevm \odot \curre}$ and $\normed{ \currm \odot \curre}$ by $\gradbtwo \highinft$,
\begin{equation}
\begin{split}
   \sum_{t=1}^T \curra \lowinft \normed{\nabla f(\currw)}^2 \leq \frac{1}{1-\bone} \Bigg( \sum_{t=1}^T \left( f(\currw) -  \expec{s_t}{f(\nextw)} \right) & + \sum_{t=1}^T \curra \bonet \gradbtwo^2 \highinft \\
   & + \frac{\smooth}2 \sum_{t=1}^T \curra^2 \gradbtwo^2 \highinft^2 \Bigg) \,,
\end{split}
\end{equation}
where we used that $\expec{s_t}{\highinft^2} = \highinft^2$ since $\highinft$ is deterministically computed from $\curre$, which in turn is independent of $s_t$.

Next, from the assumption in \thmref{thm:dadamconv} that there are positive constants $\lowinf$ and $\highinf$ such that $\lowinf \leq \currei \leq \highinf$ for all $t \in \N, i \in [d]$ and sample sequences $(s_1, \dots, s_t)$, it follows that \[\lowinf \leq \lowinft = \min_{i \in [d]} \currei \quad \text{ and } \quad \highinf \geq \highinft = \max_{i \in [d]} \currei\] for all $t \in \N$, therefore
\begin{equation}
\begin{split}
   \lowinf \sum_{t=1}^T \curra \normed{\nabla f(\currw)}^2 \leq \frac{1}{1-\bone} \Bigg( \sum_{t=1}^T \left( f(\currw) -  \expec{s_t}{f(\nextw)} \right) & + \gradbtwo^2 \highinf \sum_{t=1}^T \curra \bonet  \\
   & + \frac{\smooth \gradbtwo^2 \highinf^2 }2 \sum_{t=1}^T \curra^2  \Bigg) \,.
\end{split}
\end{equation}

Dividing both sides by $\lowinf \geq 0$ and letting $\curra = \alpha' / \sqrt T$ yields
\begin{equation}
\begin{split}
   \sum_{t=1}^T \frac{\alpha'}{\sqrt T} \normed{\nabla f(\currw)}^2 \leq \frac{1}{\lowinf(1-\bone)} \Bigg( \sum_{t=1}^T \left( f(\currw) -  \expec{s_t}{f(\nextw)} \right) & + \gradbtwo^2 \highinf \sum_{t=1}^T \frac{\alpha'}{\sqrt T} \bonet \\
   & + \frac{\smooth \gradbtwo^2 \highinf^2 }2 \sum_{t=1}^T \frac{\alpha'^2}{T}\Bigg) \,,
\end{split}
\end{equation}
and, rearranging and using the fact that \[\sum_{t=1}^T \bonet = \bone \sum_{t=1}^T \frac{1}{\sqrt t} \leq \bone \int_0^T \frac{1}{\sqrt t} dt \leq  2 \bone \sqrt T \,, \] which implies that $\sum_{t=1}^T \frac{\alpha'}{\sqrt T} \bonet \leq 2 \alpha' \bone$, we get
\begin{equation}
\begin{split}
   \frac{\alpha'}{\sqrt T} \sum_{t=1}^T \normed{\nabla f(\currw)}^2 \leq \frac{1}{\lowinf(1-\bone)} \Bigg( \sum_{t=1}^T \left( f(\currw) -  \expec{s_t}{f(\nextw)} \right) & + 2 \alpha' \bone \gradbtwo^2 \highinf \\
   & + \frac{\alpha'^2 \smooth \gradbtwo^2 \highinf^2 }2\Bigg) \,.
\end{split}
\end{equation}

Now, taking the expectation over the full sample sequence $(s_1, \dots, s_T)$ yields
\begin{equation}
\begin{split}
   \frac{\alpha'}{\sqrt T} \sum_{t=1}^T \expec{}{\normed{\nabla f(\currw)}^2} \leq \frac{1}{\lowinf(1-\bone)} \Bigg( \sum_{t=1}^T \left( \expec{}{f(\currw)} -  \expec{}{f(\nextw)} \right) & + 2 \alpha' \bone \gradbtwo^2 \highinf \\
   & + \frac{\alpha'^2 \smooth \gradbtwo^2 \highinf^2 }2\Bigg) \,.
\end{split}
\label{eq-ph8}
\end{equation}

Note that, by telescoping sum,
\begin{equation}
	\sum_{t=1}^T \expec{}{f(\currw)} -  \expec{}{f(\nextw)} = \expec{}{f(\w_1)} - \expec{}{f(\w_{T+1})} \leq f(\w_1) - f^* \,,
\end{equation}
where the last step follows since $\w_1$ (the parameters at initialization) is independent of the drawn samples, and also from the assumption that $f^*$ lower bounds $f$.

Combining the above with \eqref{eq-ph8} and dividing both sides by $\alpha' \cdot \sqrt T$,
\begin{equation}
\begin{split}
   \frac{1}{T} \sum_{t=1}^T \expec{}{\normed{\nabla f(\currw)}^2} 
   & \leq \frac{1}{\lowinf \sqrt T (1-\bone)} \Bigg( \frac{f(\w_1) - f^*}{\alpha'} + 2 \bone \gradbtwo^2 \highinf + \frac{\alpha' \smooth \gradbtwo^2 \highinf^2 }2\Bigg),.
\end{split}
\label{eq-ph10}
\end{equation}

Finally, we will use Young's inequality with $p=2$ and the conjugate exponent $q=2$, which states that $xy \leq \frac{x^2}{2} + \frac{y^2}{2}$ for any nonnegative numbers $x,y$.
	
	In that context, let
	\begin{equation}
		x = \frac{1}{\sqrt{ \alpha'}} \quad \text{ and } \quad y = \sqrt{\alpha'} \highinf \,,
	\end{equation}
	which yields
	\begin{equation}
		\highinf = xy \leq \frac{x^2}{2} + \frac{y^2}{2} = \frac{1}{\alpha'} + \alpha' \highinf^2 \,,
	\end{equation}
and hence
	\begin{equation}
		2 \bone \gradbtwo^2 \cdot \highinf \leq \frac{2 \bone \gradbtwo^2}{\alpha'} + 2 \bone \gradbtwo^2 \cdot \alpha'\highinf^2 \,.
	\end{equation}

Plugging the above in \eqref{eq-ph10} yields
\begin{equation}
\begin{split}
   \frac{1}{T} \sum_{t=1}^T \expec{}{\normed{\nabla f(\currw)}^2} 
   & \leq \frac{1}{\lowinf \sqrt T (1-\bone)} \Bigg( \frac{2 \bone \gradbtwo^2 + f(\w_1) - f^*}{\alpha'} + \alpha' \highinf^2 \frac{\gradbtwo^2 (\smooth + 2 \bone)}2\Bigg),.
\end{split}
\end{equation}

Verifying that the above is $\bO \left(\frac{1}{L \sqrt T} \left(\frac{1}{\alpha'} + \alpha' \highinf^2 \right) \right)$ in terms of $T, \alpha', \lowinf$ and $\highinf$ finishes the proof.

\end{proof}

\subsection{Proof of the second guarantee (Equation \ref{eq-refrate})}

\begin{proof}

As before, we start from Lemma \ref{lem:intproof2}, which states that
\begin{equation}
\begin{split}
   \sum_{t=1}^T \curra \lowinft \normed{\nabla f(\currw)}^2 \leq \frac{1}{1-\bone} \Bigg( \sum_{t=1}^T \left( f(\currw) -  \expec{s_t}{f(\nextw)} \right) & + \sum_{t=1}^T \curra \bonet \gradbtwo \normed{\prevm \odot \curre} \\
   & + \frac{\smooth}2 \sum_{t=1}^T \curra^2 \expec{s_t}{\normed{ \currm \odot \curre}^2} \Bigg) \,.
\end{split}
\end{equation}

We then invoke Lemma \ref{lem:mebound} to upper bound $\normed{ \prevm \odot \curre}$ by $\gradbtwo \highinft$ and $\normed{ \currm \odot \curre}$ by $\gradb \norm{\curre}^2$:

\begin{equation}
\begin{split}
   \sum_{t=1}^T \curra \lowinft \normed{\nabla f(\currw)}^2 \leq \frac{1}{1-\bone} \Bigg( \sum_{t=1}^T \left( f(\currw) -  \expec{s_t}{f(\nextw)} \right) & + \sum_{t=1}^T \curra \bonet \gradbtwo^2 \highinft \\
   & + \frac{\smooth}2 \sum_{t=1}^T \curra^2 \gradb^2 \expec{s_t}{\normed{ \curre}^2} \Bigg) \,.
\end{split}
\end{equation}

% \begin{equation}
% \begin{split}
%   \sum_{t=1}^T \curra \lowinft \normed{\nabla f(\currw)}^2 \leq \frac{1}{1-\bone} \Bigg( \sum_{t=1}^T \left( f(\currw) -  \expec{s_t}{f(\nextw)} \right) &+ \frac{\gradbtwo^2 \sqrt T}{2} \sum_{t=1}^T \frac{\curra^2 \highinft^2}{\sqrt t} + \gradbtwo^2 \bone^2 \\
%   & + \frac{\smooth}2 \sum_{t=1}^T \curra^2 \expec{s_t}{\normed{ \currm \odot \curre}^2} \Bigg) \,.
% \end{split}
% \end{equation}

% Then, using Lemma \ref{lem:mebound} to upper bound $\normed{ \currm \odot \curre}$ by $\gradb \norm{\curre}^2$,
% \begin{equation}
% \begin{split}
%   \sum_{t=1}^T \curra \lowinft \normed{\nabla f(\currw)}^2 \leq \frac{1}{1-\bone} \Bigg( \sum_{t=1}^T \left( f(\currw) -  \expec{s_t}{f(\nextw)} \right) &+ \frac{\gradbtwo^2 \sqrt T}{2} \sum_{t=1}^T \frac{\curra^2 \highinft^2}{\sqrt t} + \gradbtwo^2 \bone^2 \\
%   & + \frac{\smooth}2 \sum_{t=1}^T \curra^2 \gradb^2 \expec{s_t}{ \norm{\curre}^2} \Bigg) \,.
% \end{split}
% \end{equation}

Next, define the unormalized probability distribution $\tilde p(t) = \curra \lowinft$, so that $p(t) = \tilde p(t) / Z$ with $Z = \sum_{t=1}^T \tilde p(t) = \sum_{t=1}^T \curra \lowinft$ is a valid distribution over $t \in \{1, \dots T\}$. Dividing both sides by $Z$ yields
\begin{equation}
\begin{split}
   \sum_{t=1}^T p(t) \normed{\nabla f(\currw)}^2 & \leq \frac1{Z(1-\bone)} \sum_{t=1}^T  \left( f(\currw) - \expec{s_t}{f(\nextw)} + \curra \bonet \gradbtwo^2 \highinft + \frac{\curra^2 \smooth \gradb^2 \expec{s_t}{\normed{ \curre}^2}}2 \right)
\end{split}
\end{equation}

Now, taking the conditional expectation over all samples $S$ given $Z$:
\begin{equation}
\begin{split}
   \expec{}{\sum_{t=1}^T p(t) \normed{\nabla f(\currw)}^2 \Big| Z}
   &\leq \frac1{Z(1-\bone)} \Big( \sum_{t=1}^T \big( \expec{}{f(\currw) | Z} - \expec{}{ \expec{s_t}{f(\nextw)} | Z} \big) \\
   &\qquad + \sum_{t=1}^T \expec{}{\curra \bonet \gradbtwo^2 \highinft + \frac{\curra^2 \smooth \gradb^2 \expec{s_t}{\normed{ \curre}^2}}2 \Big| Z} \Big) \\
   &\leq \frac1{Z(1-\bone)} \Big( \sum_{t=1}^T \big( \expec{}{f(\currw) | Z} - \expec{}{f(\nextw) | Z} \big) \\
   &\qquad + \sum_{t=1}^T \expec{}{\curra \bonet \gradbtwo^2 \highinft + \frac{\curra^2 \smooth \gradb^2 \norm{\curre}^2}2 \Big| Z} \Big) \\
   &= \frac1{Z(1-\bone)} \Big(f(w_1) - f^* \\
   &\qquad + \sum_{t=1}^T \expec{}{\curra \bonet \gradbtwo^2 \highinft + \frac{\curra^2 \smooth \gradb^2 \norm{\curre}^2}2 \Big| Z} \Big) \,.
\end{split}
\end{equation}
where in the second step we used $\expec{}{ \expec{s_t}{\cdot} | Z} = \expec{}{\cdot | Z}$ which follows from the assumption that $p(Z|s_t) = p(Z)$, and the third step follows from telescoping sum and the assumption that $f^*$ lower bounds $f$.

Then, taking the expectation over $Z$ and re-arranging:
\begin{equation}
\begin{split}
   \expec{}{\sum_{t=1}^T p(t) \normed{\nabla f(\currw)}^2}
   &\leq \expec{}{\frac1{Z(1-\bone)} \sum_{t=1}^T \left( \frac{f(w_1) - f^*}{T} + \curra \bonet \gradbtwo^2 \highinft + \frac{\curra^2 \smooth \gradb^2 \norm{\curre}^2}2 \right) } \,.
\end{split}
\end{equation}

Setting $\bone = 0$ for simplicity yields
\begin{equation}
\begin{split}
   \expec{}{\sum_{t=1}^T p(t) \normed{\nabla f(\currw)}^2}
   &\leq \expec{}{\frac1{Z} \sum_{t=1}^T \left( \frac{f(w_1) - f^*}{T} + \frac{\curra^2 \smooth \gradb^2 \norm{\curre}^2}2 \right) } \,.
\end{split}
\end{equation}

Now, let $\curra = \curra' / \sqrt T$
\begin{equation}
\begin{split}
   \expec{}{\sum_{t=1}^T p(t) \normed{\nabla f(\currw)}^2}
   &\leq \expec{}{\frac1{Z} \sum_{t=1}^T \left( \frac{f(w_1) - f^*}{T} + \frac{\curra'^2 \smooth \gradb^2 \norm{\curre}^2}{2T} \right) } \\
   & = \frac1T \cdot \expec{}{\frac1{Z} \sum_{t=1}^T \left( f(w_1) - f^* + \frac12 \curra'^2 \smooth \gradb^2 \norm{\curre}^2 \right) }  \,.
\end{split}
\end{equation}

Now, recall that $Z = \sum_{t=1}^T \curra \lowinft = \frac{1}{\sqrt T} \sum_{t=1}^T \curra' \lowinft$, hence
\begin{equation}
\begin{split}
   \expec{}{\sum_{t=1}^T p(t) \normed{\nabla f(\currw)}^2}
   &\leq \frac1{\sqrt T} \cdot \expec{}{\frac{\sum_{t=1}^T  f(w_1) - f^* + \frac12 \curra'^2 \smooth \gradb^2 \norm{\curre}^2}{\sum_{t=1}^T \curra' \lowinft} }  \\
   &\leq \bO \left( \frac1{\sqrt T} \cdot  \expec{}{\frac{\sum_{t=1}^T  1 + \curra'^2 \norm{\curre}^2}{\sum_{t=1}^T \curra' \lowinft} }  \right) \,.
\end{split}
\end{equation}

Finally, checking that $\sum_{t=1}^T p(t) \normed{\nabla f(\currw)}^2 = \expec{t \sim \mathcal P (t|S)}{\normed{\nabla f(\currw)}^2}$:
\begin{equation}
\begin{split}
     \expec{\algrand}{\normed{\nabla f(\currw)}^2} & \leq \bO \left( \frac1{\sqrt T} \cdot  \expec{}{\frac{\sum_{t=1}^T  1 + \curra'^2 \norm{\curre}^2}{\sum_{t=1}^T \curra' \lowinft} }  \right) \,.
\end{split}
\label{eq-proof2final}
\end{equation}

Recalling that $\lowinft = \min_i \currei$ completes the proof.
\end{proof}

\section{Full Statement and Proof of Theorem \ref{thm:adamconv}}
\label{app-adamconv}

We organize the formal statement and proof of Theorem \ref{thm:adamconv} as follows: we first state a general convergence result for Adam which depends on the step-wise adaptivity parameter $\epsilon_t$ and the learning rates $\curra$ in Theorem \ref{thm:adamconv}, and then present a Corollary that shows how a $\bO(1/ \sqrt T)$ rate follows from such result (Corollary~\ref{cor:adamconv}). This section proceeds the proof of Theorem \ref{thm:dadamconv} (Appendix~\ref{app:dadamconv}) as the proof presented here is more easily seen as a small variant (although overall simpler) of the analysis given in the previous section. Steps which also appear in the proof of Theorem \ref{thm:dadamconv} are not necessarily described in full detail, hence the following arguments can be better understood with the previous section in context. 

Throughout the proof we use the following notation for clarity: 
\begin{equation}
	\highinft = \max_i \currei \quad\quad \lowinft = \min_i \currei \,.
\end{equation}

\begin{theorem*}
	Assume that $f$ is smooth and $f_s$ has bounded gradients. If $\epsilon_t \geq \epsilon_{t-1} > 0$ for all $t \in [T]$, then for the iterates $\{w_1, \dots, w_T\}$ produced by Adam we have
  \begin{equation}
  \begin{split}
      \expec{}{\normed{\nabla f(\currw)}^2} \leq \bO \left(
\frac{1 + \sum_{t=1}^T \frac{\curra}{\epsilon_{t-1}^2}\left(1 + \curra + \epsilon_t - \epsilon_{t-1} \right) }{\sum_{t=1}^T \frac{\curra}{1+\epsilon_{t-1}}} \right) \,,
  \end{split}
	\label{eq-adamrate}
  \end{equation}
  where $\currw$ is sampled from $p(t) \propto \frac{\curra}{\gradb + \epsilon_{t-1}}$.
\end{theorem*}

\begin{corollary*}
    \label{cor:adamconv}
	Setting $\epsilon_t = \Theta(T^{p_1} t^{p_2})$ for any $p_1, p_2 > 0$ such that $p_1 + p_2 \geq \frac12$ (e.g. $\epsilon_t = \Theta(\sqrt T), \epsilon_t = \Theta(\sqrt[\leftroot{-2}\uproot{2}4] {Tt}), \epsilon_t = \Theta(\sqrt t)$) and $\curra = \Theta\left(\frac{\epsilon_t}{\sqrt T}\right)$ on Theorem \ref{thm:adamconv} yields a bound of $\bO(1 / \sqrt T)$ for Adam.
\end{corollary*}

\begin{proof}

Similarly to the proof of \thmref{thm:dadamconv}, we plug the update rule $\nextw = \currw - \curra \cdot \curre \odot \currm$ in
\begin{equation}
   f(\nextw) \leq
      f(\currw) +
      \langle \nabla f(\currw), \nextw - \currw \rangle +
      \frac{\smooth}2 \normed{\nextw - \currw}^2 \,.
\end{equation}
yielding
\begin{equation}
\begin{split}
   f(\nextw) & \leq f(\currw) - \curra \left\langle \nabla f(\currw), \currm \odot \curre \right\rangle + \frac{\curra^2 \smooth}2 \normed{ \currm \odot \curre}^2 \,.
\end{split}
\end{equation}

By Lemmas \ref{lem:mebound} and \ref{lem:dotbound}, we have
\begin{equation}
\begin{split}
   f(\nextw) & \leq f(\currw) - \curra(1 - \bonet) \left\langle \nabla f(\currw), \currg \odot \curre \right\rangle + \curra \bonet \gradbtwo^2 \highinft + \frac{\curra^2 \smooth \gradb^2 \normed{\curre}^2}2\,.
\end{split}
\label{eq:intermed3}
\end{equation}

Now, note that we can write \[ \langle \nabla f(\currw), \currg \odot \curre \rangle = \langle \nabla f(\currw), \currg \odot \preve \rangle + \langle \nabla f(\currw), \currg \odot (\curre-\preve) \rangle \,,\] therefore we have that
\begin{equation}
\begin{split}
   f(\nextw) & \leq f(\currw) - \curra(1 - \bonet) \left\langle \nabla f(\currw), \currg \odot \preve \right\rangle + \curra \bonet \gradbtwo^2 \highinft + \frac{\curra^2 \smooth \gradb^2 \normed{\curre}^2}2 \\
   & \quad\quad - \curra(1 - \bonet) \left\langle \nabla f(\currw), \currg \odot (\curre - \preve) \right\rangle \\
   & \leq f(\currw) - \curra(1 - \bonet) \left\langle \nabla f(\currw), \currg \odot \preve \right\rangle + \curra \bonet \gradbtwo^2 \highinft + \frac{\curra^2 \smooth \gradb^2 \normed{\curre}^2}2 \\
   & \quad\quad + \curra(1 - \bonet) \left| \left\langle \nabla f(\currw), \currg \odot (\curre - \preve) \right\rangle \right|
\end{split}
\end{equation}

We will proceed to bound $\left|\left\langle \nabla f(\currw), \currg \odot (\curre - \preve) \right\rangle \right|$. By Cauchy-Schwarz we have
\begin{equation}
	\left|\left\langle \nabla f(\currw), \currg \odot (\curre - \preve) \right\rangle \right| \leq \normed{\nabla f(\currw)} \cdot \normed{\currg \odot (\curre - \preve)} \leq \gradbtwo \normed{\currg \odot (\curre - \preve)} \,,
\end{equation}
and moreover 
\begin{equation}
\begin{split}
	\normed{\currg \odot (\curre - \preve)} &= \left(\sum_{i=1}^d \currgi^2 |\currei - \prevei|^2\right)^{1/2} \\
	&\leq \left(\sum_{i=1}^d \gradb^2 |\currei - \prevei|^2\right)^{1/2} \\
	&= \gradb \normed{\curre - \preve} \,,
\end{split}
\end{equation}
therefore we get
\begin{equation}
\begin{split}
   f(\nextw) & \leq f(\currw) - \curra(1 - \bonet) \left\langle \nabla f(\currw), \currg \odot \preve \right\rangle + \curra \bonet \gradbtwo^2 \highinft + \frac{\curra^2 \smooth \gradb^2 \normed{\curre}^2}2 \\
   & \quad\quad + \curra(1 - \bonet) \gradb \gradbtwo \normed{\curre - \preve}
\end{split}
\end{equation}

Using the fact that $\preve$ is independent of $s_t$ and that $\expec{s_t}{\currg} = \nabla f(\currw)$, taking expectation over $s_t$ yields
\begin{equation}
\begin{split}
   \expec{s_t}{f(\nextw)}
   &\leq f(\currw) - \curra (1 - \bonet) \left\langle \nabla f(\currw), \nabla f(\currw) \odot \preve \right\rangle + \curra \bonet \gradbtwo^2 \expec{s_t}{\highinft}  \\
   & \quad\quad + \frac{\curra^2 \smooth \gradb^2 \expec{s_t}{\normed{\curre}^2}}2 + \curra(1 - \bonet) \gradb \gradbtwo \expec{s_t}{\normed{\curre - \preve}} \\
   &\leq f(\currw) - \curra (1 - \bone) \normed{\nabla f(\w)}^2 L_{t-1}  + \curra \bonet \gradbtwo^2 \expec{s_t}{\highinft}  \\
   & \quad\quad + \frac{\curra^2 \smooth \gradb^2 \expec{s_t}{\normed{\curre}^2}}2 + \curra(1 - \bone) \gradb \gradbtwo \expec{s_t}{\normed{\curre - \preve}} \,,
\end{split}
\end{equation}
where in the second step we used
$\bonet \leq \bone$ and
\[\left\langle \nabla f(\currw), \nabla f(\currw) \odot \preve \right\rangle = \sum_{i=1}^d \nabla f(\w)_i^2 \prevei \geq \min_j \eta_{t-1,j} \sum_{i=1}^d \nabla f(\w)_i^2 = L_{t-1} \normed{\nabla f(\w)}^2 \,. \]

Re-arranging,
\begin{equation}
\begin{split}
   \curra L_{t-1} (1 - \bone) \normed{\nabla f(\currw)}^2 & \leq f(\currw) -  \expec{s_t}{f(\nextw)} + \curra \bonet \gradbtwo^2 \expec{s_t}{\highinft}  \\
   & \quad\quad + \frac{\curra^2 \smooth \gradb^2 \expec{s_t}{\normed{\curre}^2}}2 + \curra(1 - \bone) \gradb \gradbtwo \expec{s_t}{\normed{\curre - \preve}} \,,
   \label{eq-prebound}
\end{split}
\end{equation}

Next, we will bound $L_{t-1}, H_t, \normed{\curre}$, and $\normed{\curre - \preve}$. Recall that, for Adam, we have \[\curre = \frac{1}{\sqrt{\currv} + \epsilon_t}\,, \] and since $\currvi \leq \gradb^2$, we also have that \[\frac{1}{\gradb + \epsilon_t} \leq \currei \leq \frac{1}{\epsilon_t} \,.\]

From the above it follows that \[\frac{1}{\gradb + \epsilon_{t-1}} \leq L_{t-1}\] and \[\highinft \geq \frac{1}{\epsilon_t}\,,\] which also implies that $\normed{\curre} \leq \frac{\sqrt d}{\epsilon_t}$.

As for $\normed{\curre - \preve}$, check that
\begin{equation}
	\Big| \currei - \prevei \Big| \leq \frac{1}{\epsilon_{t-1}}- \frac{1}{\gradb + \epsilon_t} = \frac{\gradb + \epsilon_t - \epsilon_{t-1}}{\gradb \epsilon_{t-1} + \epsilon_t \epsilon_{t-1}} \leq \frac{\gradb + \epsilon_t - \epsilon_{t-1}}{\epsilon_{t-1}^2} \,,
\end{equation}
where we used the assumption that $\epsilon_t \geq \epsilon_{t-1}$. The above implies that $\normed{\curre - \preve} \leq \sqrt d \cdot \frac{\gradb + \epsilon_t - \epsilon_{t-1}}{\epsilon_{t-1}^2}$.

Applying the bounds given above to \eqref{eq-prebound} yields
\begin{equation}
\begin{split}
   \frac{\curra}{\gradb + \epsilon_{t-1}} (1 - \bone) \normed{\nabla f(\currw)}^2 & \leq f(\currw) -  \expec{s_t}{f(\nextw)} + \bonet \gradbtwo^2 \frac{\curra}{\epsilon_t} + \frac{\curra^2 \smooth d \gradb^2}{\epsilon_t^2} \\
   & \quad\quad + \curra(1 - \bone) \sqrt d \gradb \gradbtwo  \cdot \frac{\gradb + \epsilon_t - \epsilon_{t-1}}{\epsilon_{t-1}^2} \,,
\end{split}
\end{equation}

Next, define the unormalized probability distribution $\tilde p(t) = \frac{\curra}{\gradb + \epsilon_{t-1}}$, so that $p(t) = \tilde p(t) / Z$ with $Z = \sum_{t=1}^T \tilde p(t) = \sum_{t=1}^T \frac{\curra}{\gradb + \epsilon_{t-1}}$ is a valid distribution over $t \in \{1, \dots T\}$. Adopting this notation and dividing both sides by $Z (1-\bone)$:
\begin{equation}
\begin{split}
   p(t) \normed{\nabla f(\currw)}^2 \leq \frac{1}{Z (1-\bone)} & \Big( f(\currw) -  \expec{s_t}{f(\nextw)} + \bonet \gradbtwo^2 \frac{\curra}{\epsilon_t} + \frac{\curra^2 \smooth d \gradb^2}{\epsilon_t^2} \\
   & \quad\quad + \curra(1 - \bone) \sqrt d \gradb \gradbtwo  \cdot \frac{\gradb + \epsilon_t - \epsilon_{t-1}}{\epsilon_{t-1}^2} \Big)\,.
\end{split}
\end{equation}

Taking the expectation over all samples and summing over $t$ yields
\begin{equation}
\begin{split}
   \sum_{t=1}^T p(t) \expec{}{\normed{\nabla f(\currw)}^2} 
   & \leq \frac{1}{Z (1-\bone)} \sum_{t=1}^T \Big(  \expec{}{f(\currw)} - \expec{}{f(\nextw)} + \bonet \gradbtwo^2 \frac{\curra}{\epsilon_t} + \frac{\curra^2 \smooth d \gradb^2}{\epsilon_t^2} \\
   & \quad\quad\quad\quad\quad\quad + \curra(1 - \bone) \sqrt d \gradb \gradbtwo  \cdot \frac{\gradb + \epsilon_t - \epsilon_{t-1}}{\epsilon_{t-1}^2} \Big)\\
   & \leq \frac{1}{Z (1-\bone)} \Big[ f(\w_1) - f^* + \sum_{t=1}^T \Big(\bonet \gradbtwo^2 \frac{\curra}{\epsilon_t} + \frac{\curra^2 \smooth d \gradb^2}{\epsilon_t^2} \\
   & \quad\quad\quad\quad\quad\quad + \curra(1 - \bone) \sqrt d \gradb \gradbtwo  \cdot \frac{\gradb + \epsilon_t - \epsilon_{t-1}}{\epsilon_{t-1}^2} \Big) \Big] \,.
\end{split}
\end{equation}
where we used the fact that $\sum_{t=1}^T \expec{}{f(\currw)} - \expec{}{f(\nextw)} = f(\w_1) - \expec{}{f(\w_{T+1})} \leq f(\w_1) - f^*$ by telescoping sum and where $f^*$ lower bounds $f$.

For simplicity, assume that $\bonet = 0$ (or, alternatively, let $\bonet = \frac{\bone}{\sqrt T}$ and apply Young's inequality as in the proof of \thmref{thm:dadamconv}). In this case, we get
\begin{equation}
\begin{split}
   \sum_{t=1}^T p(t) \expec{}{\normed{\nabla f(\currw)}^2} 
   & \leq \frac{1}{Z (1-\bone)} \Big[ f(\w_1) - f^* + \sum_{t=1}^T \frac{\curra}{\epsilon_{t-1}^2} \Big(\curra \smooth d \gradb^2 + (1 - \bone) \sqrt d \gradb \gradbtwo \cdot \left(\gradb + \epsilon_t - \epsilon_{t-1}\right) \Big) \Big] \,,
\end{split}
\end{equation}
and recalling that $Z = \sum_{t=1}^T \frac{\curra}{\gradb + \epsilon_{t-1}}$ yields
\begin{equation}
\begin{split}
   \expec{t \sim P(t)}{\expec{}{\normed{\nabla f(\currw)}^2}}
   & \leq \frac{f(\w_1) - f^* + \sum_{t=1}^T \frac{\curra}{\epsilon_{t-1}^2} \Big(\curra \smooth d \gradb^2 + (1 - \bone) \sqrt d \gradb \gradbtwo \cdot \left(\gradb + \epsilon_t - \epsilon_{t-1}\right) \Big)}{(1-\bone) \sum_{t=1}^T \frac{\curra}{\gradb + \epsilon_{t-1}}} \\
   & \leq \bO \left( \frac{1 + \sum_{t=1}^T \frac{\curra}{\epsilon_{t-1}^2} \left(1 + \curra + \epsilon_t - \epsilon_{t-1} \right)}{\sum_{t=1}^T \frac{\curra}{1 + \epsilon_{t-1}}}  \right)\,,
\end{split}
\end{equation}
which completes the argument.
\end{proof}

%\newpage
%\input{ch7_gaw/ch7_gaw}

%\newpage
%\input{sec8_conclusion}

%\makebibliography

% \newpage
% \input{appendix/appendix}

\end{document}